\documentclass[10pt]{article}
\usepackage[accepted]{tmlr-style-file-main/tmlr}

\usepackage[]{amsfonts}
\usepackage{amsmath}
\usepackage{nicefrac}       %

\usepackage{subfigure}

\usepackage{todonotes}
\usepackage{caption}
\usepackage{graphicx}

\usepackage{coseoul}
\usepackage{titletoc}
\usepackage[page,header]{appendix}

\author{\name Mihaela Rosca \email mihaelacr@deepmind.com\\
      \addr DeepMind, University College London\\
      \AND
      \name Yan Wu \email yanwu@deepmind.com \\
      \addr DeepMind
      \AND
      \name Chongli Qin \email chongliqin@deepmind.com \\
      \addr DeepMind
      \AND
      \name Benoit Dherin \email dherin@google.com \\
      \addr Google}

\def\month{01}  %
\def\year{2023} %
\def\openreview{\url{https://openreview.net/forum?id=EYrRzKPinA}}

\usepackage{tikz}
\usetikzlibrary{calc,patterns,decorations.pathreplacing,shapes.multipart}

\title{On a continuous time model of gradient descent dynamics and instability in deep learning}
\usepackage{amsmath, latexsym}
\usepackage{amsfonts}
\usepackage{mathtools}
\usepackage{ntheorem}
\usepackage{dsfont}
\usepackage{booktabs}
\usepackage{xfrac}
\usepackage{bm}
\usepackage{hyperref}

\usepackage{bbm}

\usepackage[most]{tcolorbox}
\usepackage{xparse}
\usepackage{lipsum}
\usepackage{changepage}
\usepackage{enumitem}
\usepackage{pifont}

\newtheorem{theorem}{Theorem}[section]
\newtheorem{lemma}[theorem]{Lemma}
\newtheorem{remark}{Remark}[section]
\newenvironment{proof}{\paragraph{Proof:}}{\hfill$\square$}

\newtheorem{definition}{Definition}[section]

\newcommand{\qedwhite}{\hfill \ensuremath{\Box}}

\newcommand\cut[1]{}

\usepackage[percent]{overpic}
\newcommand{\norm}[1]{\left\lVert#1\right\rVert}

\newcommand{\squishlist}{
   \begin{list}{$\bullet$}
    { \setlength{\itemsep}{0pt}      \setlength{\parsep}{3pt}
      \setlength{\topsep}{3pt}       \setlength{\partopsep}{0pt}
      \setlength{\leftmargin}{1.5em} \setlength{\labelwidth}{1em}
      \setlength{\labelsep}{0.5em} } }

\newcommand{\squishlisttwo}{
   \begin{list}{$\bullet$}
    { \setlength{\itemsep}{0pt}    \setlength{\parsep}{0pt}
      \setlength{\topsep}{0pt}     \setlength{\partopsep}{0pt}
      \setlength{\leftmargin}{2em} \setlength{\labelwidth}{1.5em}
      \setlength{\labelsep}{0.5em} } }

\newcommand{\squishend}{
    \end{list}  }

\newcommand{\sign}{\mbox{sign}}

\newcommand{\myvec}[1]{\mathbf{#1}}

\newcommand{\myvecsym}[1]{\bm{#1}}

\newcommand{\vtheta}{\myvecsym{\theta}}

\newcommand{\vb}{\myvec{b}}

\newcommand{\vg}{\myvec{g}}

\newcommand{\vu}{\myvec{u}}
\newcommand{\vv}{\myvec{v}}

\newcommand{\vx}{\myvec{x}}

\newcommand{\vA}{\myvec{A}}

\newcommand{\vI}{\myvec{I}}
\newcommand{\vJ}{\myvec{J}}

\newcommand{\vP}{\myvec{P}}

\newcommand{\hatdotvtheta}{\bm{{\dot{\tilde{\theta}}}}}

\newcommand{\be}{\begin{equation}}
\newcommand{\ee}{\end{equation}}
\newcommand{\bea}{\begin{eqnarray}}
\newcommand{\eea}{\end{eqnarray}}
\newcommand{\beaa}{\begin{eqnarray*}}
\newcommand{\eeaa}{\end{eqnarray*}}

\DeclareMathAlphabet{\mathpzc}{OT1}{pzc}{m}{n}

\newtheorem{corollary}{Corollary}[section]

\newcommand{\owntag}[1]{\stepcounter{equation}   \tag{#1, \theequation} }

\newcommand{\rebuttalrthree}[1]{\textcolor{black}{#1}}
\newcommand{\rebuttalrtwo}[1]{\textcolor{black}{#1}}
\newcommand{\rebuttalrone}[1]{\textcolor{black}{#1}}

\graphicspath{{figures/}}

\begin{document}

\maketitle

\begin{abstract}
    The recipe behind the success of deep learning has been the combination of neural networks and gradient-based optimization.
    Understanding the behavior of gradient descent however, and particularly its instability, has lagged behind its empirical success. 
    To add to the theoretical tools available to study gradient descent we propose \textit{the principal flow} (PF), a continuous time flow that approximates gradient descent dynamics. To our knowledge, the PF is the only continuous flow that captures the divergent and oscillatory behaviors of gradient descent, including escaping local minima and saddle points. Through its dependence on the eigendecomposition of the Hessian the PF sheds light on the recently observed edge of stability phenomena in deep learning. Using our new understanding of instability we propose a learning rate adaptation method which enables us to control the trade-off between training stability and test set evaluation performance.
\end{abstract}

\section{Introduction}

Our goal is to use continuous time models to understand the behavior of gradient descent.  Using continuous dynamics to understand discrete time systems opens up tools from dynamical systems such as stability analysis, and has a long history in optimization and machine learning~\citep{glendinning1994stability,saxe2013exact,nagarajan2017gradient,lampinen2018analytic,arora2018optimization,advani2020high,elkabetz2021continuous,vardi2021implicit,franca2020,igr,igr_sgd}. Most theoretical analysis of gradient descent using continuous time systems uses the negative gradient flow, but this has well known limitations such as not being able to explain any behavior contingent on the learning rate.
To mitigate these limitations 
we find a new continuous time flow which
reveals important new roles of the Hessian in gradient descent training. 
To do so, we use backward error analysis (BEA), a method with a long history in the numerical integration community ~\citep{hairer2006geometric} that has only recently been used in the deep learning context~\citep{igr,igr_sgd}.

We find that the proposed flow sheds new light on gradient descent stability, including but not limited to divergent and oscillatory behavior around a fixed point.
Instability --- areas  of training where the loss consistently increases --- and edge of stability behaviors ~\citep{cohen2021gradient} ---areas  of training where the loss does not behave monotonically but decreases over long time periods --- are pervasive in deep learning and occur for all learning rates and architectures \cite{cohen2021gradient,gur2018gradient,gilmer2021loss,lewkowycz2020large}. We use our novel insights to understand and mitigate these instabilities.

\rebuttalrthree{The structure of the presented work is as follows:}
\begin{itemize}
    \item \rebuttalrthree{We discuss the advantages of a continuous time approach in Section \ref{sec:motivation}, where we also highlight the limitations of existing continuous time flows.}
    \item We introduce \textbf{the principal flow} (the PF), a flow \rebuttalrthree{in complex space} defined by the eigendecomposition of the Hessian (Section~\ref{sec:principal_flow}). To our knowledge the PF is the first continuous time flow that captures that gradient descent can diverge around local minima and saddle points. \rebuttalrthree{We show that using a complex flow is crucial in understanding instabilities in gradient descent.}
    \item We show the PF \rebuttalrthree{is better than existing flows at modelling neural network training dynamics} in Section~\ref{sec:the_pf_and_nns}. In Section~\ref{sec:instability_deep_learning} we use the PF to shed new light on edge of stability behaviors in deep learning. \rebuttalrthree{We do so by connecting changes in the loss and Hessian eigenvalues with core quantities exposed by the PF and neural network landscapes explored through the behavior of gradient flows.}
    \item \rebuttalrthree{Through a continuous time perspective we demonstrate empirically how to control the trade-off between stability and performance in deep learning in Section \ref{sec:stabilising_training}.  We do so using DAL (Drift Adjusted Learning rate), an approach to setting the learning rate dynamically based on insights on instability derived from the PF.}
    \item \rebuttalrthree{We end by showcasing the potential of integrating our continuous time approach with other optimization schemes and highlighting how the PF can be used as a tool for existing continuous time analyses in Section \ref{sec:future_work}}.
\end{itemize}

\textbf{Notation}: We denote as $E$ the loss function, $\vtheta$ the parameter vector of dimension $D$, $\nabla_{\vtheta}^2 E$ the loss Hessian and $\lambda_i$ the Hessian's $i$'th largest eigenvalue with $\vu_i$ the corresponding eigenvector. Since if $\vu_i$ is an eigenvector of $\nabla_{\vtheta}^2 E$ so is $-\vu_i$, we always use $\vu_i$ such that $Re[\nabla_{\vtheta} E^T \vu_i] \ge 0$; this has no effect on our results and is only used for convenience. For a continuous time flow $\vtheta(h)$ refers to the solution of the flow at time $h$.

\textbf{Experiments}: A list of figures with details on how to reproduce each of them is provided in the Appendix. Code available at \url{https://github.com/deepmind/discretisation_drift}.

\section{Continuous time models of gradient descent}
\label{sec:motivation}

The aim of this work is to understand the dynamics of gradient descent updates with learning rate $h$
\begin{align}
    \vtheta_t = \vtheta_{t-1} - h \nabla_{\vtheta} E(\vtheta_{t-1})
    \label{eq:gd-basic}
\end{align}

from the perspective of continuous dynamics. 
When using continuous time dynamics to understand gradient descent it is most common to use \textit{the negative gradient flow} (NGF)

\begin{align}
\dot{\vtheta} = - \nabla_{\vtheta} E
\label{eq:ngf}
\end{align}

Gradient descent can be obtained from the NGF through Euler numerical integration, with an error of $\mathcal{O}(h^2)$ after one gradient descent step. 
\rebuttalrthree{
Studying gradient descent and its behavior around equilibria and beyond has thus taken two main approaches: directly studying the discrete updates of Eq~\ref{eq:gd-basic} \citep{bartlett2018gradient,bartlett2018representing,mescheder2017numerics,gunasekar2018implicit,du2019gradient,allen2019convergence,du2019width,ziyin2021sgd,liu2021noise}, or the continuous time NGF of Eq~\ref{eq:ngf} \citep{glendinning1994stability,saxe2013exact,nagarajan2017gradient,lampinen2018analytic,arora2018optimization,advani2020high,elkabetz2021continuous,vardi2021implicit,franca2020,balduzzi2018mechanics}.
The appeal of continuous time systems lies in their connection with dynamical systems and the plethora of tools that thus become available, such as stability analysis; the simplicity by which conserved quantities can be obtained \citep{du2018algorithmic,franca2020}; and analogies that can be constructed through similarities with physical systems \citep{franca2020}. Because of the availability of tools for the analysis of continuous time systems, it has been previously noted that discrete time approaches are often more challenging and discrete time proofs are often inspired from continuous time ones \citep{may1976simple,elkabetz2021continuous}. We use an example to showcase the ease of continuous time analyses: when following the NGF the loss $E$ decreases since $\frac{dE}{dt} = \frac{d E}{d\vtheta}^T\frac{d\vtheta}{dt} = - ||\nabla_{\vtheta}{E}||^2$.  Showing that and \textit{when} following the discrete time gradient descent update in Eq~\ref{eq:gd-basic} is more challenging and requires adapting the analysis on the form of the loss function $E$. Classical convergence guarantees associated with other optimization approaches such as natural gradient are also derived in continuous time \citep{amari1998natural,ollivier2015riemannian,ollivier2015riemannian2}.
By analyzing the properties of continuous time systems one can also determine whether optimizers should more closely follow the underlying continuous time flow \citep{song2018accelerating,odegan}, what regularizers should be constructed to ensure convergence or stability \citep{nagarajan2017gradient,balduzzi2018mechanics,rosca2021discretization}, construct converge guarantees in functional space for infinitely wide networks \citep{jacot2018neural,lee2019wide}.
}

\subsection{\rebuttalrthree{Limitations of existing continuous time flows}}

The well-known discrepancy between Euler integration and the NGF, often called \textit{discretization error} or \textit{discretization drift} (Figure~\ref{fig:dd_def}) leads to certain limitations when using the NGF to describe gradient descent, namely: the NGF cannot explain divergence around a local minima for high learning rates or convergence to flat minima as often seen in the training of neural networks.  
Critically, since the NGF does not depend on the learning rate, it cannot explain any learning rate dependent behavior.

\begin{figure}[t]
\centering
\begin{subfigure}[Discretisation drift.]{
\begin{tikzpicture}[every text node part/.style={align=center,inner sep=0,outer sep=0}][overlay]
\coordinate (theta_t_minus_1) at (0,1);
\coordinate (theta_t) at (2.2,0.6);

\coordinate (cont_theta_t) at (2.2,3.3);

\node(draw) at ($(theta_t_minus_1) + (+0.3,-0.32)$) {$\vtheta_{t-1}$};
\node(draw) at ($(theta_t) + (0.5,-0.2)$) {$\vtheta_{t}$};

\coordinate (first_time_transition) at ($(theta_t_minus_1) + (+0.55,1.6)$);

\draw [thick,dashed] (theta_t_minus_1) -- (theta_t);

\draw [thick]  (theta_t_minus_1) to[out=50,in=180]  node[near end,above,yshift=0.1cm, xshift=-0.2cm] {$\dot{\vtheta} = -\nabla_{\vtheta} E$} (cont_theta_t);

\draw [
    thick,
    decoration={
        brace,
        mirror,
        raise=0.1cm
    },
    decorate
] (theta_t) -- (cont_theta_t)
node [pos=0.5,anchor=west,xshift=0.15cm,yshift=-0.cm,text width=1cm,align=left] { \scriptsize {\color{black}discretization \protect\newline drift}};
\end{tikzpicture}
\label{fig:dd_def}
}\end{subfigure}
\hspace{1em}
\begin{subfigure}[2D convex case.]{
 \includegraphics[width=0.3\columnwidth]{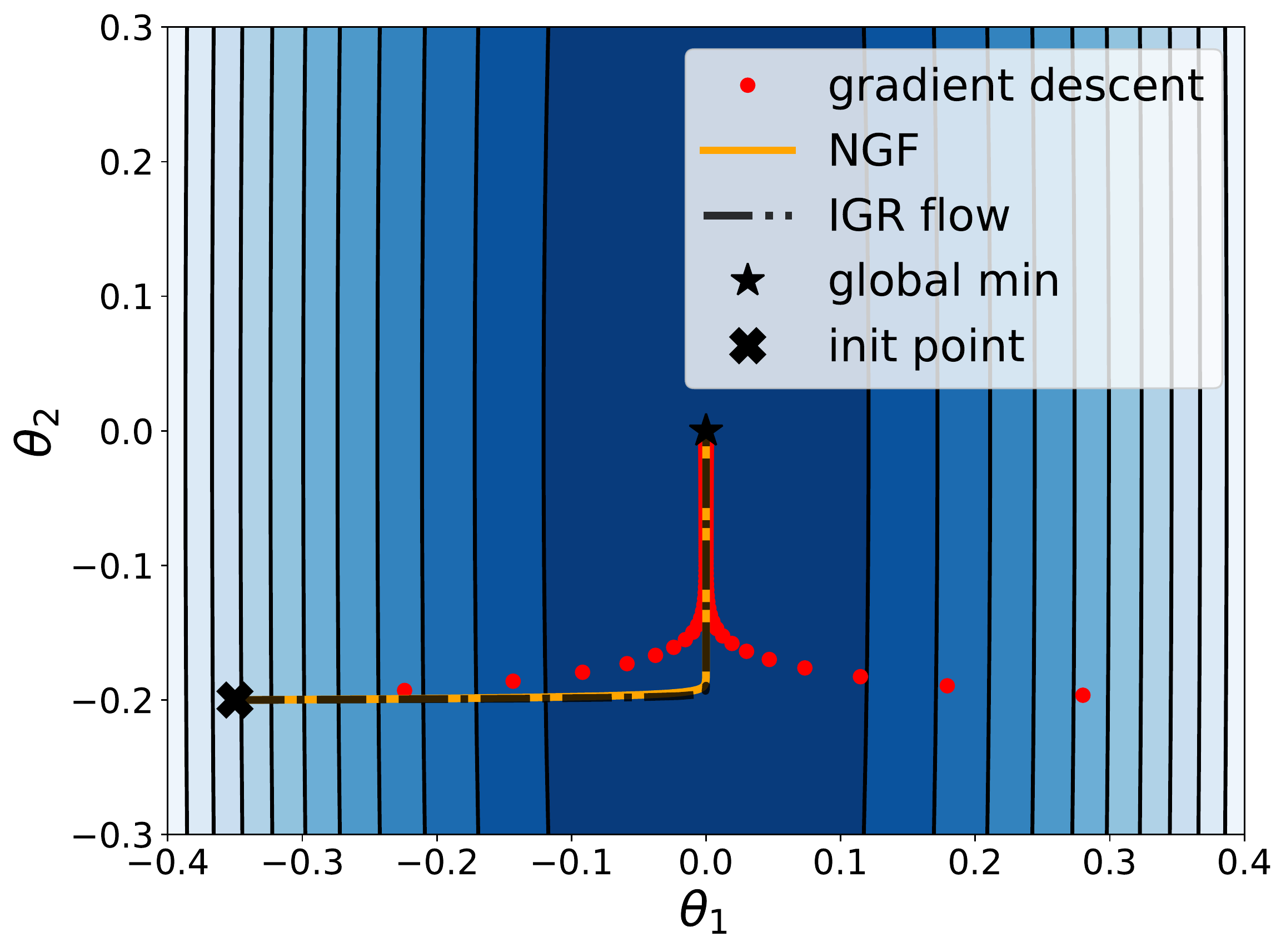}
 }\end{subfigure}
 \begin{subfigure}[Banana function.]{
  \includegraphics[width=0.3\columnwidth]{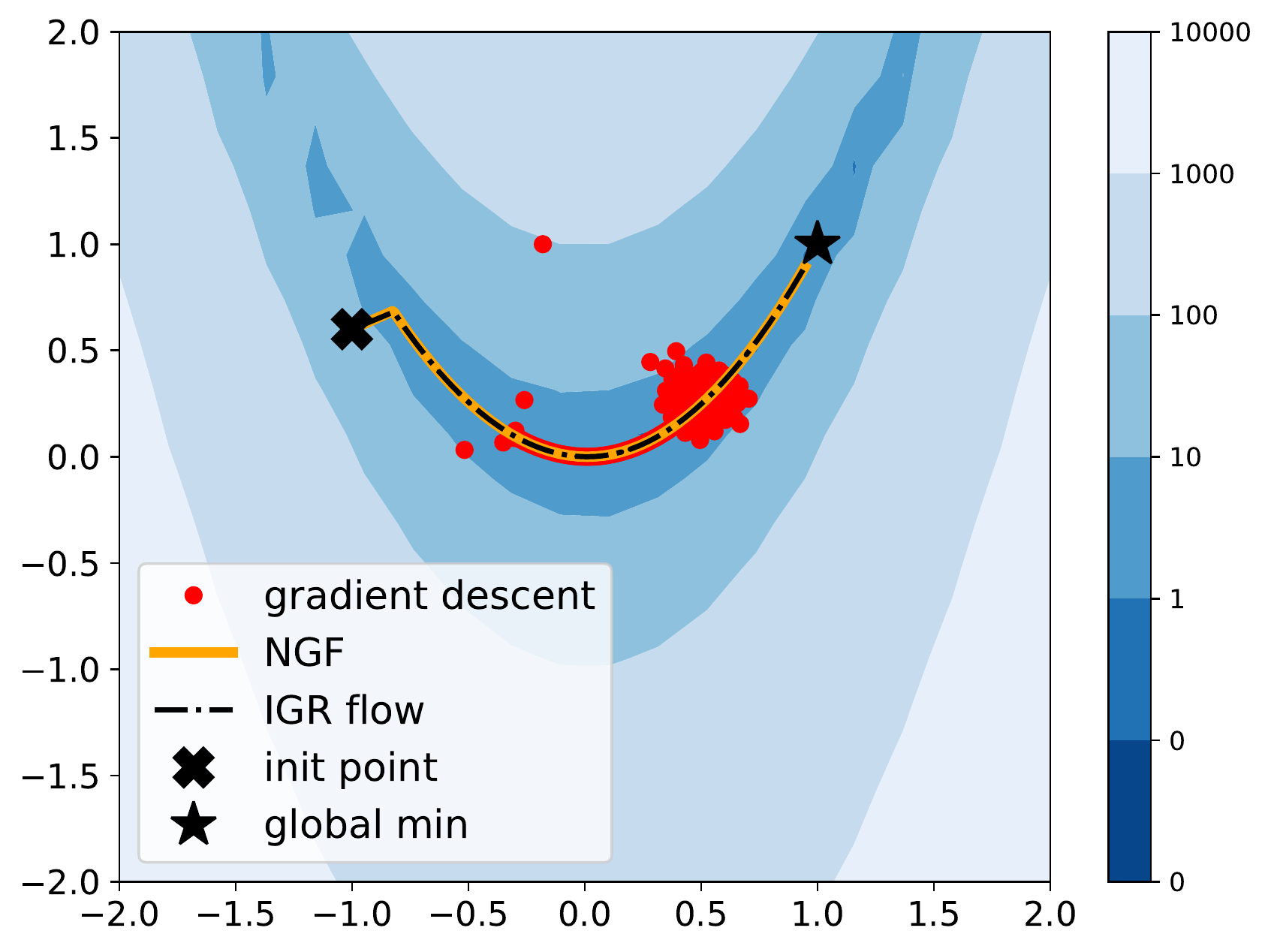}
 }\end{subfigure}
\caption[2D convex plot. The loss function is $E = \vtheta^T A \vtheta$ with  $A = ((1, 0.0), (0.0, 0.01))$. The learning rate is 0.9.]{\textbf{Motivation}. Using continuous time flows to understand gradient descent is limited by the gap between the discrete and continuous dynamics. In the case of the negative gradient flow, we call this gap \textit{discretization drift}. Other flows have been introduced to capture part of the drift, but they also fail to capture the oscillatory or unstable behavior of gradient descent.}
\label{fig:motivation_small_examples}
\end{figure}
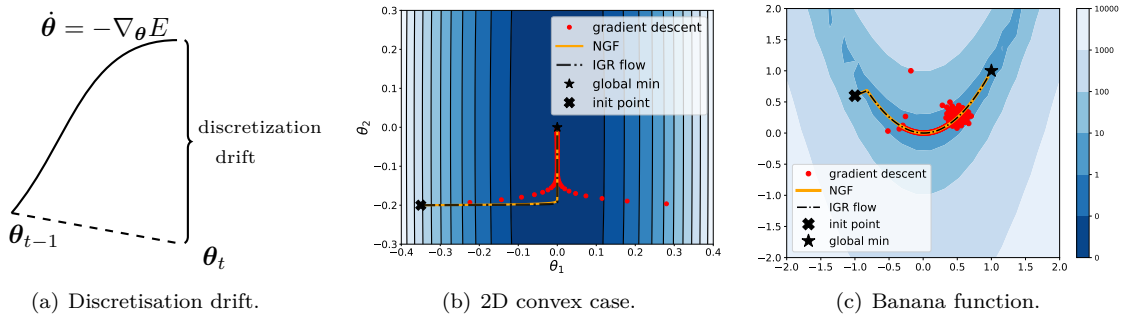

The appeal of continuous time methods together with the limitations of the NGF have inspired the machine learning community to look for other continuous time systems which may better approximate the gradient descent trajectory.  One approach to constructing continuous time flows approximating gradient descent that takes into account the learning rate is backward error analysis (BEA). Using this approach, \citet{igr} introduce the Implicit Gradient Regularization flow (IGR flow):
\begin{align} 
\dot{\vtheta} = -\nabla_{\vtheta}E   -\frac{h}{2} \nabla_{\vtheta}^2 E \nabla_{\vtheta} E
\label{eq:modified_flow_igr}
\end{align}

which tracks the dynamics of the gradient descent step $\vtheta_t = \vtheta_{t-1} - h \nabla_{\vtheta} E(\vtheta_{t-1})$ with an error of $\mathcal{O}(h^3)$, thus reducing the order of the error compared to the NGF.
Unlike the NGF flow, the IGR flow depends on the learning rate $h$. This dependence explains certain properties of gradient descent, such as avoiding trajectories with high gradient norm; the authors connect this behavior to convergence to flat minima.

Like the NGF flow however, the IGR flow does not explain the instabilities of gradient descent, as we illustrate in Figure~\ref{fig:motivation_small_examples}. Indeed,~\citet{igr} (their Remark 3.4) show that performing stability analysis around local minima using the IGR flow does not lead to qualitatively different conclusions from those using the NGF: both NGF and the IGR flow predict gradient descent to be always locally attractive around a local minimum (proofs in Section~\ref{sec:jacobian_igr_ngf}), contradicting the empirically observed behavior of gradient descent.
To understand why both the NFG and the IGR flow cannot capture oscillations and divergence around a local minimum, we note that stationary points $\nabla_{\vtheta}E = \mathbf{0}$ are fixed points for both flows.
We visualize an example in Figure~\ref{fig:intuition_real}: since to go from the initial point to the gradient descent iterates requires passing through the local minimum,  both flows would stop at the local minimum and never reach the following gradient descent iterates. In the case of neural networks we show in Figure~\ref{fig:validating_bea} in the Appendix that while the IGR flow is better than the NGF at describing gradient descent, a substantial gap remains. 

\rebuttalrthree{The lack of ability of existing continuous time flows to model instabilities empirically observed in gradient descent such as those shown in Figure~\ref{fig:motivation_small_examples} has been used as a motivation to use discrete-time methods instead \citep{yaida2018fluctuation,liu2021noise}. The goal of our work is to overcome this issue by introducing a novel continuous time flow which captures instabilities observed in gradient descent. To do so, we follow the footsteps of \citet{igr} and use Backward Error Analysis. By using a continuous time flow we can leverage the tools and advantages of continuous time methods discussed earlier in this section; by incorporating discretization drift into our model of gradient descent we can increase their applicability to explain unstable training behavior.}
\rebuttalrthree{Indeed, we show in Figure~\ref{fig:intuition_complex} that the flow we propose captures the training instabilities; a key reason why is that, unlike existing flows, it operates in complex space. In Section ~\ref{sec:principal_flow} we show the importance of operating in complex space in order to understand oscillatory and instability behaviors of gradient descent. }

\begin{figure}[tb]
\begin{subfigure}[Real flows.]{
  \includegraphics[width=0.45\columnwidth]{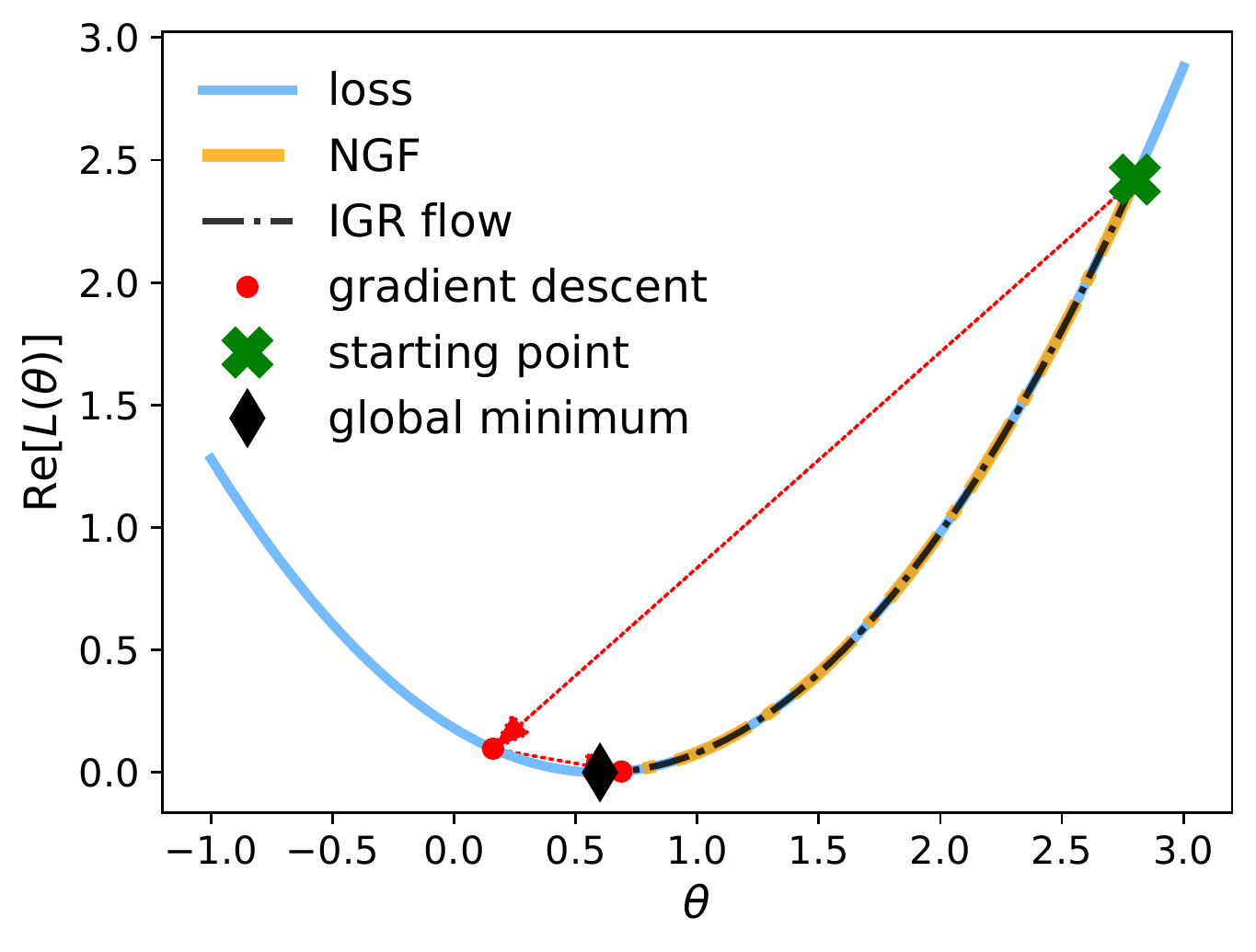}
  \label{fig:intuition_real}
 }\end{subfigure}
\begin{subfigure}[Complex flow.]{
 \includegraphics[width=0.45\columnwidth]{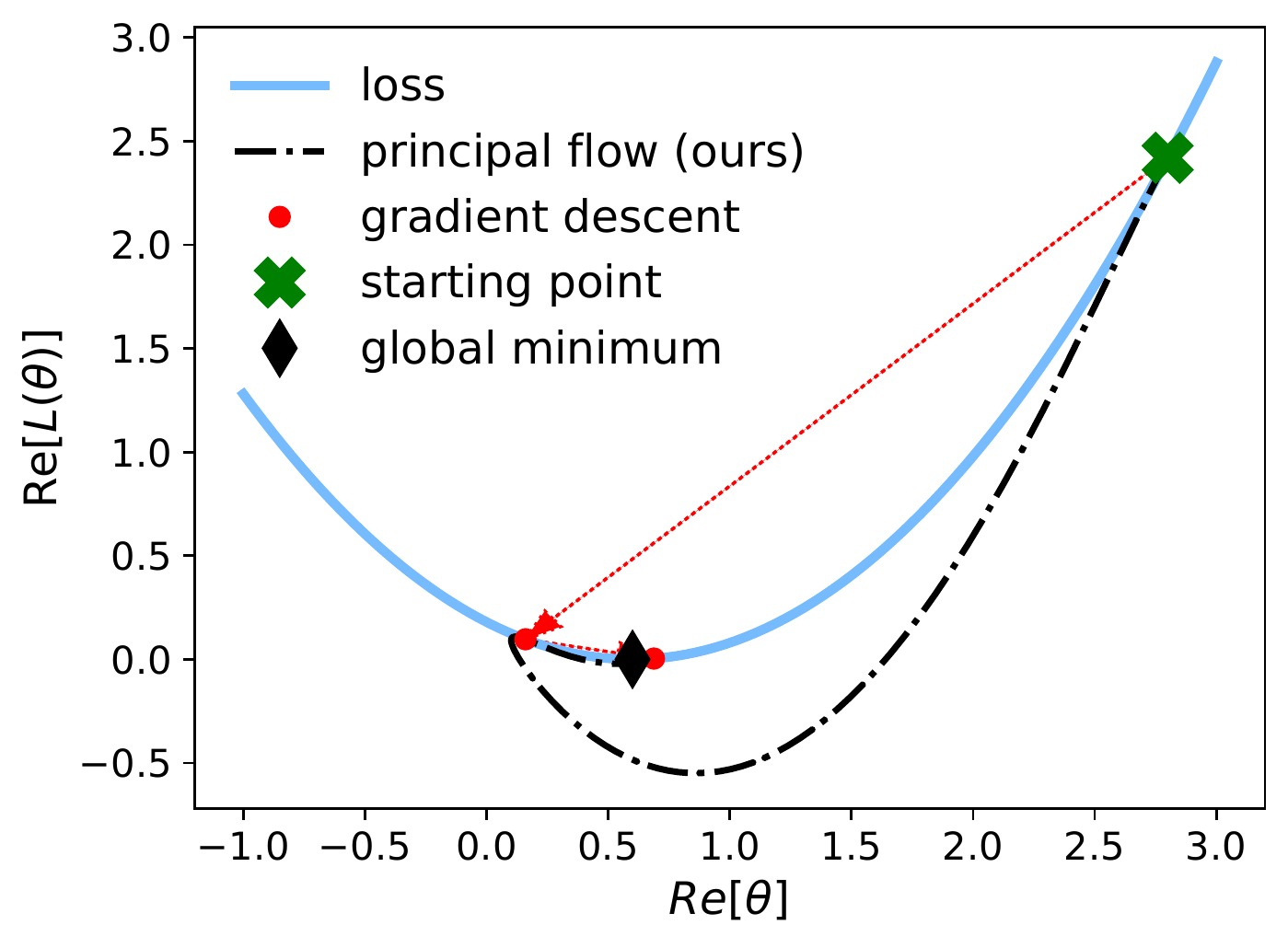}
  \label{fig:intuition_complex}
 }\end{subfigure}
\caption[1D example. $E = \frac 1 2 (\vtheta - 0.6)^2$ and $h=1.2$]{\textbf{Complex flows capture oscillations and divergence around local minima}. In the real space, the trajectory going from the starting point to the second gradient descent iterate goes through the global minima, and real flows stop there. In complex space however, that need not be the case.}
\label{fig:intuition_convex1d_complex_part_needed}
\end{figure}

\subsection{Backward error analysis}

\label{sec:bea}

Backward error analysis (BEA) is a tool in numerical analysis developed to understand the discretization error of numerical integrators. We now present an overview of how to use it in the context of gradient descent; for a general overview see~\citet{hairer2006geometric}.
BEA provides a modified vector field:
\begin{equation}
\tilde f_n(\vtheta) =  - \nabla_{\vtheta} E + h f_1(\vtheta) + \cdots + h^n f_n(\vtheta),
\end{equation}
 by finding functions $f_1$, ... $f_n$ such that the solution of the modified ODE at order $n$, that is,
\begin{align}
    \bm{{\dot{\tilde{\theta}}}} =  - \nabla_{\vtheta} E + h f_1(\vtheta) + \cdots + h^n f_n(\vtheta)
\label{eq:general_modified_vector_field}
\end{align} follows the discrete dynamics of the gradient descent update with an error $\| \vtheta_t - \tilde{\vtheta}(h)\|$ of order $\mathcal O(h^{n+2})$, where $\tilde{\vtheta}(h)$ is the solution of the modified equation truncated at order $n$ at time $h$, with $\tilde{\vtheta}(0) = \vtheta_{t-1}$. 
The full modified vector field with all orders ($n \rightarrow \infty$)
\begin{equation}
\tilde f(\vtheta) = - \nabla_{\vtheta} E + h f_1(\vtheta) + \cdots + h^n f_n(\vtheta) + \cdots,
\label{eq:bea_series}
\end{equation}
is usually divergent and only forms an asymptotic expansion. What BEA provides is the Taylor expansion in $h$ of an unknown $h$-dependent vector field $f_h(\vtheta)$ developed at $h=0$:
\begin{equation}
    \tilde f(\vtheta) = \textrm{Taylor}_{h=0} f_h(\vtheta).
\end{equation}

Thus a strategy for finding $f_h$ is to find a series of the form in Eq~\ref{eq:bea_series} via BEA and then find the function  $f_h$ such that its Taylor expansion in $h$ at 0 results in the found series.
Using this approach we can find the flow $ \hatdotvtheta = f_h(\tilde{\vtheta})$ which exactly describes the gradient descent step $\vtheta_t = \vtheta_{t-1} - h \nabla_{\vtheta} E(\vtheta_{t-1})$. 

While flows obtained using BEA are constructed to approximate one gradient descent step, the same flows can be used over multiple gradient descent steps as shown in Section~\ref{sec:multiple_steps_proof} in the Appendix.

\textbf{BEA proofs}.
The general structure of BEA proofs is as follows: start with a Taylor expansion in $h$ of the modified flow in Eq~\ref{eq:general_modified_vector_field}; write each term in the Taylor expansion as a function of $\nabla_{\vtheta} E$ and the desired $f_i$ (this often requires applying the chain rule repeatedly); group together terms of the same order in $h$ in the expansion; and identify $f_i$ such that all terms of $\mathcal{O}(h^p)$ are 0 for $p \ge 2$, as is the case in the gradient descent update. A formal overview of BEA proofs can be found in Section \ref{sec:bea_proof_structure} in the Appendix.

We now exemplify how to use BEA to find the IGR flow (Eq~\ref{eq:modified_flow_igr})~\citep{igr}.  Since we are only looking for the first correction term, we only need to find $f_1$.
We perform a Taylor expansion to find the value of $\tilde{\vtheta}(h)$ up to order $\mathcal O(h^{3})$ and then identify $f_1$ from that expression such that the error $\| \vtheta_t - \tilde{\vtheta}(h)\|$ is of order $\mathcal O(h^{3})$.
We have: ${\tilde{\vtheta}(h) = \vtheta_{t-1} + h \tilde{\vtheta}^{(1)}(\vtheta_{t-1}) + \frac{h^2}{2} \tilde{\vtheta}^{(2)}(\vtheta_{t-1}) +  \mathcal{O}(h^3)}$. 
We know by the definition of the modified vector field (Eq~\ref{eq:general_modified_vector_field}) that $\tilde{\vtheta}^{(1)} = - \nabla_{\vtheta} E + h f_1({\tilde{\vtheta}})$. We can then use the chain rule to obtain ${\tilde{\vtheta}^{(2)} = \frac{- \nabla_{\vtheta} E + h f_1({\vtheta})}{dt}  = \frac{- \nabla_{\vtheta} E}{dt} + \mathcal{O}(h) = \frac{- \nabla_{\vtheta} E}{d\vtheta}\frac{d \vtheta}{d t} + \mathcal{O}(h) = \nabla_{\vtheta}^2 E \nabla_{\vtheta} E + \mathcal{O}(h) }$. 
Thus
$ {\tilde{\vtheta}(h) = \vtheta_{t-1} - h \nabla_{\vtheta} E(\vtheta_{t-1}) + h^2 f_1(\vtheta_{t-1}) + \frac{h^2}{2}  \nabla_{\vtheta}^2 E (\vtheta_{t-1})\nabla_{\vtheta} E (\vtheta_{t-1})+  \mathcal{O}(h^3)}$. 
We can then write ${\vtheta_t - \tilde{\vtheta}(h) = \vtheta_{t-1} - h \nabla_{\vtheta} E(\vtheta_{t-1}) -  \left(\vtheta_{t-1} - h \nabla_{\vtheta} E(\vtheta_{t-1}) + h f_1(\vtheta_{t-1}) + \frac{h^2}{2}  \nabla_{\vtheta}^2 E (\vtheta_{t-1})\nabla_{\vtheta} E (\vtheta_{t-1})+  \mathcal{O}(h^3)\right)}$. 
After simplifying we obtain 
${\vtheta_t - \tilde{\vtheta}(h) = h^2 f_1(\vtheta_{t-1}) + \frac{h^2}{2}  \nabla_{\vtheta}^2 E (\vtheta_{t-1})\nabla_{\vtheta} E (\vtheta_{t-1})+  \mathcal{O}(h^3)}$.
For the error to be of order $\mathcal{O}(h^3)$ the terms of order $\mathcal{O}(h^2)$ have to be $\mathbf{0}$.
This entails $f_1 =  -\frac{1}{2} \nabla_{\vtheta}^2 E \nabla_{\vtheta} E $ leading to Eq~\ref{eq:modified_flow_igr}.

\section{The principal flow}
\label{sec:principal_flow}

In the previous section we have seen how BEA can be used to define continuous time flows which capture the dynamics of gradient descent up to a certain order in learning rate.
We have also explored the limitations of these flows, including the lack of ability to explain oscillations observed empirically when using gradient descent.
To further expand our understanding of gradient descent via continuous time methods, we would like to get an intuition for the structure of higher order modified vector fields provided by BEA. We start with the following modified vector field, which we will call \textit{the third order flow} (proof in Section~\ref{sec:third_order_flow_proof}):
\begin{align}
    \dot{\vtheta} = -\nabla_{\vtheta}E   -\frac{h}{2} \nabla_{\vtheta}^2E  \nabla_{\vtheta}E - h^2 \left( \frac{1}{3}  (\nabla_{\vtheta}^2E)^2 \nabla_{\vtheta} E + \frac {1} {12} \nabla_{\vtheta}E^T (\nabla_{\vtheta}^3E) \nabla_{\vtheta}E\right)
\label{eq:third_order_modified_vector_field}
\end{align}

The third order flow tracks the dynamics of the gradient descent step $\vtheta_t = \vtheta_{t-1} - h \nabla_{\vtheta} E(\vtheta_{t-1})$ with an error of $\mathcal{O}(h^4)$, thus further reducing the order of the error compared to the IGR flow.
Like the IGR flow and the NGF, the third order flow has the property that $ \dot{\vtheta} = \mathbf{0}$ if $\nabla_{\vtheta} E = \mathbf{0}$
and thus will exhibit the same limitations observed in Figure~\ref{fig:intuition_convex1d_complex_part_needed}.
The third order flow allows us to spot a pattern:
the correction term of order $\mathcal{O}(h^n)$ in 
the BEA modified flow describing gradient descent contains the term $(\nabla_{\vtheta}^2E)^{n} \nabla_{\vtheta} E$ and terms which contain higher order derivatives with respect to parameters, terms which we will denote as $\mathcal{C}(\nabla_{\vtheta}^3 E)$. 

\rebuttalrtwo{\textbf{Our approach}}.
We will use the terms of the form $(\nabla_{\vtheta}^2E)^{n} \nabla_{\vtheta} E$ to construct a new continuous time flow.
We will take a three-step approach. First, for an arbitrary order $\mathcal{O}(h^n)$ we will find the terms containing only first and second order derivatives in the modified vector field given by BEA and show they are of the form $(\nabla_{\vtheta}^2 E)^n \nabla_{\vtheta} E$ (Theorem \ref{thm:order_n_flow}). Second, we will use all orders to create a series (Corollary \ref{col:principal_series}). Third, we will use the series to find the modified flow given by BEA (Theorem \ref{thm:principal_ode}). All proofs are provided in Section~\ref{sec:all_proofs} of the Appendix.

\begin{theorem}
The modified vector field with an error of order $\mathcal{O}(h^{n+2})$ to the gradient descent update  $\vtheta_t = \vtheta_{t-1} -h \nabla_{\vtheta}E(\vtheta_{t-1})$
has the form:
\begin{align}
    \dot{\vtheta} = \sum_{p=0}^{n} \frac{-1}{p+1} h^p (\nabla_{\vtheta}^2 E)^p \nabla_{\vtheta} E + \mathcal{C}(\nabla_{\vtheta}^3 E)
\end{align}
\label{thm:order_n_flow}

where $\mathcal{C}(\nabla_{\vtheta}^3 E)$ denotes the family of functions which can be written as a sum of terms, each term containing a derivative of higher order than 3 with respect to parameters.
\end{theorem}

The result is proven by induction. The base cases for $n= 1, 2$ and $3$ follow from the NGF, IGR and third order flows. For higher order terms, the proof uses induction to find the term in $f_i$ depending on $\nabla_{\vtheta}^2 E $ and $\nabla_{\vtheta} E$ only and follows the BEA proof structure highlighted in Section \ref{sec:bea}, but Step 3 is modified to not account for terms in $\mathcal{C}(\nabla_{\vtheta}^3 E)$.  From the above, we can obtain the following corollary by using all orders $n$ and the eigen decomposition of $\nabla_{\vtheta}^2 E$:
\begin{corollary}
\rebuttalrtwo{The full order modified flow obtained by performing BEA on gradient descent updates is of the form:}
\begin{align}
  \dot{\vtheta} 
   &=  \sum_{p=0}^{\infty} \frac{-1}{p+1} h^p (\nabla_{\vtheta}^2 E)^p \nabla_{\vtheta} E + \mathcal{C}(\nabla_{\vtheta}^3 E) 
   =  \sum_{p=0}^{\infty} \frac{-1}{p+1} h^p \left( \sum_{i=0}^{D-1} \lambda_i^p \vu_i \vu_i^T\right) \nabla_{\vtheta} E + \mathcal{C}(\nabla_{\vtheta}^3 E)\\
   &=  \sum_{i=0}^{D-1} \left(\sum_{p=0}^{\infty} \frac{-1}{p+1} h^p  \lambda_i^p  \right) (\nabla_{\vtheta} E^T \vu_i) \vu_i  + \mathcal{C}(\nabla_{\vtheta}^3 E)
\label{eq:principal_series}
\end{align}
where $\lambda_i$ and $\vu_i$ are the respective eigenvalues and eigenvectors of the Hessian $\nabla_{\vtheta}^2 E$.
\label{col:principal_series}
\end{corollary}

\rebuttalrtwo{If $\lambda_0 > 1/h$ the BEA series above diverges.
Generally BEA series are not convergent and approximate the discrete scheme only by truncation \citep{hairer2006geometric}.
When the series in Eq~\ref{eq:principal_series} diverges, truncating it up to any order $n$ however will result in a flow which will not be able to capture instabilities, even in the quadratic case. 
}
\rebuttalrtwo{Such flows (including the IGR flow) will always predict the loss function will decrease for a quadratic loss where a minimum exists, since:
$ {\frac{dE}{dt} = \nabla_{\vtheta}E ^T \left(\sum_{p=0}^{n} \frac{-1}{p+1} h^p (\nabla_{\vtheta}^2 E)^p \nabla_{\vtheta} E\right) = - \sum_{p=0}^{n} \frac{1}{p+1} h^p \sum_{i=0}^{D-1} (\lambda_i^p) (\nabla_{\vtheta} E ^T \vu_i)^2}$ which is never positive for any quadratic loss where a minimum exists (i.e. when $\lambda_i \ge 0, \forall i$). The above also entails that the flows always predict convergence around a local minimum, which is not the case for gradient descent which can diverge for large learning rates.}

\rebuttalrtwo{To further track instabilities we can use the BEA series to formulate the following flow:}

\begin{definition} We define the \textbf{principal flow} (PF) as
\begin{align}
  \dot{\vtheta} =  \sum_{i=0}^{D-1} \frac{\log(1 - h \lambda_i)}{h\lambda_i}(\nabla_{\vtheta} E^T \vu_i)  \vu_i
\end{align}
\end{definition}

We note that $\lim_{\lambda \to 0} \frac{\log(1 - h \lambda)}{h\lambda} = -1$ and thus the PF is well defined when the Hessian $\nabla_{\vtheta}^2 E$ is not invertible. Unlike the NGF and the IGR flow, the modified vector field of the PF cannot be always written as the gradient of a loss function in $\mathbb{R}$, and can be complex valued.

\begin{theorem}
The Taylor expansion in $h$ at $h=0$ of the PF vector field coincides with the series coming from the BEA of gradient descent (Eq~\ref{eq:principal_series}).
\label{thm:principal_ode}
\end{theorem}

\begin{proof}
Using the Taylor expansion $\textrm{Taylor}_{z=0} \frac{\log(1-z)}{z} = \sum_{p=0}^{\infty} \frac{-1}{p+1} z^p$ we obtain:
\begin{align}
\textrm{Taylor}_{h=0}  \sum_{i=0}^{D-1}  \frac{\log(1 - h \lambda_i)}{h\lambda_i} (\nabla_{\vtheta} E^T \vu_i) \vu_i = \sum_{i=0}^{D-1} \left(\sum_{p=0}^{\infty} \frac{-1}{p+1} h^p \lambda_i^p \right)(\nabla_{\vtheta} E^T \vu_i) \vu_i
\end{align}
\end{proof}

We have used BEA to find the flow that when Taylor expanded at $h=0$ leads to the series in Eq~\ref{eq:principal_series}.  
\rebuttalrtwo{When the BEA series in Eq~\ref{eq:principal_series} converges, namely $\lambda_0 <1/h$, the PF and the flow given by the BEA series are the same. When $\lambda_0 > 1/h$ however, the PF is complex and the BEA series diverges. 
While in this case any BEA truncated flow will not be able to track gradient descent closely, we show that for quadratic losses the PF will track gradient descent exactly, and that it is a good model of gradient descent around fixed points. We show examples of the PF tracking gradient descent exactly in the quadratic case in Figures~\ref{fig:intuition_complex} and~\ref{fig:intuition_quadratic_2d}.}

\begin{remark} For quadratic losses of the form $E = \frac{1}{2}\vtheta^T \vA \vtheta + \vb^T \vtheta$, the PF captures gradient descent exactly. This case has been proven in \citet{hairer2006geometric}. The solution of the PF can also be computed exactly in terms of the eigenvalues of $\nabla_{\vtheta}^2 E$: $\vtheta(t) = \sum_{i=0}^{D-1} e^{\frac{\log(1 - h\lambda_i)}{h} t} \vtheta_0^T \vu_i \vu_i + t \sum_{i=0}^{D-1} \frac{\log (1 - h\lambda_i)}{h \lambda_i} b^T \vu_i$.
\label{remark:quadractic_remark}
\end{remark}

\begin{remark} In a small enough neighborhood around a critical point (where higher order derivatives can be ignored) the PF can be used to describe gradient descent dynamics closely. We show this also using a linearization argument in Section~\ref{sec:linearization} in the Appendix.
\end{remark}

\begin{definition} The terms $\mathcal{C}(\nabla_{\vtheta}^3 E)$ are called \textbf{non-principal terms}. The term $\frac {1} {12} \nabla_{\vtheta}E^T (\nabla_{\vtheta}^3E) \nabla_{\vtheta}E$ in Eq~\ref{eq:third_order_modified_vector_field} is a non-principal term (we will call this term non-principal third order term).
\end{definition}

\begin{definition} We define the \textbf{principal flow with third order non principal term} as
\begin{align}
  \dot{\vtheta} =  \sum_{i=0}^{D-1} \frac{\log(1 - h \lambda_i)}{h\lambda_i}(\nabla_{\vtheta} E^T \vu_i)  \vu_i
    - \underbrace{\frac{h^2}{12} \nabla_{\vtheta} E^T (\nabla^3_{\vtheta} E) \nabla_{\vtheta} E}_{\text{third order non principal term}}
    \label{eq:pf_with_non_principal}
\end{align}
\end{definition}

\rebuttalrtwo{General theoretical bounds on the error between continuous time flows and gradient descent are challenging to construct in the case of a general parametrised $E(\vtheta)$ as the error will be determined by the shape of $E$. We know the conditions which determine when certain flows follow gradient descent exactly. The NGF and gradient descent will follow the same trajectory in areas where $\nabla^2_{\vtheta} E \nabla_{\vtheta} E = 0$ (see Theorem~\ref{thm:total_drift}) and thus $E$ has a constant gradient in time, since $\frac{d \nabla_{\vtheta} E}{d t} = \nabla^2_{\vtheta} E \nabla_{\vtheta} E$. The PF generalises the NGF, in that it follows the same trajectory as gradient descent not only for trajectories where  $\nabla^2_{\vtheta} E \nabla_{\vtheta} E = 0$, but also when $E$ is quadratic. Informally, we can state that the closer we are to these exact conditions, the more likely the flows are to capture the dynamics of gradient descent. Formally, bounds on the error between GD and NGF can be provided by the Fundamental Theorem (Theorem 10.6 in~\citet{wanner1996solving}) which has recently been adapted to a neural network parametrisation by~\citet{elkabetz2021continuous}; this bound depends on the magnitude of the smallest Hessian eigenvalue along the NGF trajectory. We hope that future work can expand the Fundamental Theorem such that error bounds between the PF and gradient descent can be constructed for deep neural networks. Here we take an empirical approach and show that although not exact outside the quadratic case the PF captures key features of the gradient descent dynamics in stable or unstable regions of training, around and outside critical points, for small examples or large neural networks.}

\subsection{The principal flow and the eigen decomposition of the Hessian}
\label{sec:principal_flow_eigendecom}

\begin{table}[tb]
\centering
\begin{tabular}{ c | c  | c}
 \textbf{Negative Gradient Flow} & \textbf{IGR Flow} & \textbf{Principal Flow}  \\ 
 \hline
 \hline
 $\dot{\vtheta} =  \sum_{i=0}^{D-1} - (\nabla_{\vtheta} E^T \vu_i) \vu_i$ & $\dot{\vtheta} =  \sum_{i=0}^{D-1} - (1 +\frac{h}{2} \lambda_i) (\nabla_{\vtheta} E^T \vu_i) \vu_i$ &  $\dot{\vtheta} =  \sum_{i=0}^{D-1} \frac{\log(1 - h \lambda_i)}{h\lambda_i}  (\nabla_{\vtheta} E^T \vu_i) \vu_i$  \\  
 $\alpha_{NGF}(h \lambda_i) = -1$ &  $\alpha_{IGR}(h \lambda_i) = - (1 +\frac{h}{2} \lambda_i)$ & $\alpha_{PF}(h \lambda_i) = \frac{\log(1 - h \lambda_i)}{h\lambda_i}$
\end{tabular}
\caption{Understanding the differences between the flows discussed in terms of the eigendecomposition of the Hessian. All flows have the form $\dot{\vtheta} = \sum_{i=0}^{D-1} \alpha(h \lambda_i) (\nabla_{\vtheta} E^T \vu_i) \vu_i$ with different $\alpha$ summarized here.}
\label{tab:principal_flow_vs_negative_grad_flow} 
\end{table}

\begin{figure}[tb]
\begin{subfigure}[Real part.]{
 \includegraphics[width=0.45\columnwidth]{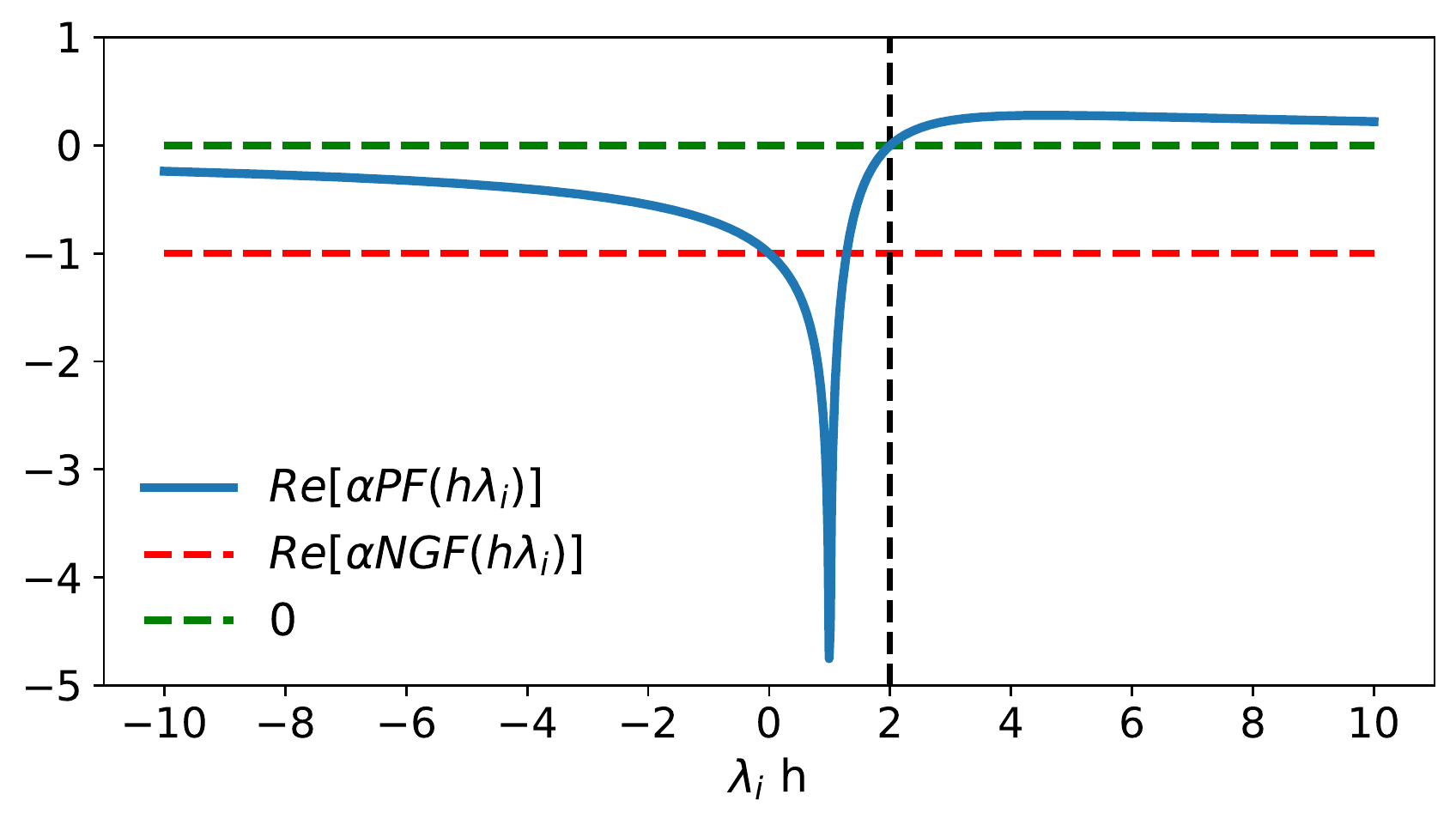}
 }\end{subfigure}
 \begin{subfigure}[Imaginary part.]{
 \includegraphics[width=0.45\columnwidth]{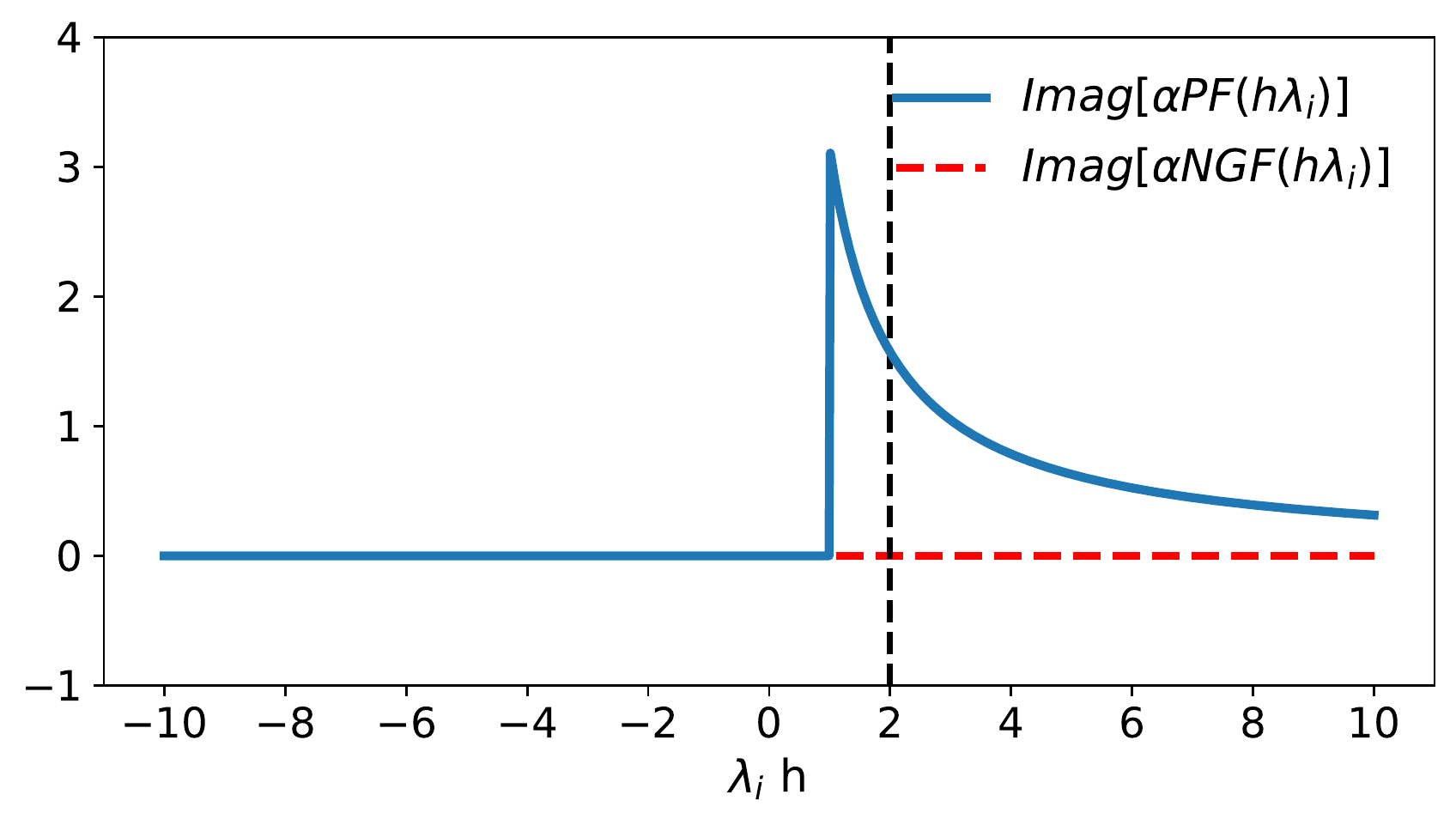}
 }\end{subfigure}
\caption{Comparing the coefficients $\alpha_{NGF}$ and $\alpha_{PF}$ across the training landscape.}
\label{fig:r_plot}
\end{figure}

All flows considered here have the form form $\dot{\vtheta} = \sum_{i=0}^{D-1} \alpha(h \lambda_i) (\nabla_{\vtheta} E^T \vu_i) \vu_i$, where $\alpha$ is a function computing the corresponding coefficient; we will denote the one associated with each flow as $\alpha_{NGF}$, $\alpha_{IGR}$ and $\alpha_{PF}$ respectively.
For a side-by-side comparison between the NGF, IGR flow and the PF as functions of the Hessian eigendecomposition see Table~\ref{tab:principal_flow_vs_negative_grad_flow}. 
 Since $\nabla_{\vtheta}E^T \vu_i \ge 0$, the $\alpha$ function determines the sign of a modified vector field in the direction $\vu_i$. For brevity it will be useful to define the coefficient of $\vu_i$ in the vector field of the PF:

\begin{definition} We call $sc_i = \frac{\log(1 - h \lambda_i)}{h\lambda_i}(\nabla_{\vtheta} E^T \vu_i) = \alpha_{PF}(h \lambda_i ) \nabla_{\vtheta} E^T \vu_i$ the \textbf{stability coefficient} for eigendirection $i$. $\sign(sc_i) = \sign (\alpha_{PF}(h \lambda_i))$.
\label{def:stab_coeff}
\end{definition}

In order to understand the PF and how it is different from the NGF we explore the change in each eigendirection $\vu_i$ and we perform case analysis on the relative value of the eigenvalues $\lambda_i$ and the learning rate $h$. To do so, we will compare $\alpha_{NGF}(h\lambda_i)$ and $\alpha_{PF}(h\lambda_i)$ since the sign of $\alpha_{NGF}(h\lambda_i)$ determines the direction which minimises $E$ given by $\vu_i$.
Since our goal is to understand the behavior of gradient descent, we perform the case by case analysis of what happens at the start of a gradient descent iteration and thus use real values for $\lambda_i$ and $\vu_i$ even when the PF is complex valued.
We visualize $\alpha_{NGF}$ and $\alpha_{PF}$ in Figure~\ref{fig:r_plot} and we use Figure~\ref{fig:z_square} to show examples of each case using a simple function.

\textbf{Real stable case}: $\lambda_i< 1/h$. $\sign(\alpha_{NGF}(h\lambda_i)) = \sign(\alpha_{PF}(h\lambda_i)) = -1$.

$\alpha_{NGF}(h\lambda_i) = -1$  and $\alpha_{PF}(h \lambda_i) = \frac{\log(1 - h \lambda_i)}{h \lambda_i} < 0$.  
The coefficients of both the NGF and PF in eigendirection $\vu_i$ are both negative and real.  The case is exemplified in Figure~\ref{fig:z_square_stable}.

\textbf{Complex stable case}: $1/h < \lambda_i < 2/h$. $\sign(\alpha_{NGF}(h\lambda_i)) = \sign(Re[\alpha_{PF}(h\lambda_i)]) = -1$. $\alpha_{PF}(h\lambda_i) \in \mathbb{C}$.

$\alpha_{NGF}(h\lambda_i) = -1$ and $\alpha_{PF}(h \lambda_i) = \frac{\log(1 - h \lambda_i)}{h \lambda_i} = \frac{\log(-1 + h \lambda_i) + i \pi}{h\lambda_i} \in \mathbb{C}$ and  $Re[\alpha_{PF}(h\lambda_i)] = \frac{\log(-1 + h \lambda_i)}{h\lambda_i} < 0$. The real part of the coefficient of the NGF and PF in eigendirection $\vu_i$ are both negative. The imaginary part of $\alpha_{PF}$ can still introduce instability and oscillations, as we show in Figure~\ref{fig:z_square_oscillation}.

\textbf{Unstable complex case}: $ 2/h < \lambda_i$. $\sign(\alpha_{NGF}(h\lambda_i)) \neq \sign(Re[\alpha_{PF}(h\lambda_i)])$. $\alpha_{PF}(h\lambda_i) \in \mathbb{C}$.

$\alpha_{NGF}(h\lambda_i) = -1$ and $\alpha_{PF}(h \lambda_i) = \frac{\log(1 - h \lambda_i)}{h \lambda_i} = \frac{\log(-1 + h \lambda_i) + i \pi}{h\lambda_i} \in \mathbb{C}$ and  $Re[\alpha_{PF}(h\lambda_i)] = \frac{\log(-1 + h \lambda_i)}{h\lambda_i} > 0$.  The real part of the coefficient of the NGF in eigendirection $\vu_i$ is negative, while the real part of the coefficient of the PF is positive.
The PF goes in the opposite direction of the NGF which minimises E; this change in sign can cause instabilities.
 The imaginary component can still introduce oscillations, however the larger $\lambda_i h$, the smaller the imaginary part of $\alpha_{PF}$. We visualize this case in Figure~\ref{fig:z_square_divergence}.

\begin{figure}[tb]
\begin{subfigure}[$\lambda_0< 1/h$, \hspace{1em} $\alpha_{PF}(\lambda_0 h) < 0$]{
 \includegraphics[width=0.32\columnwidth]{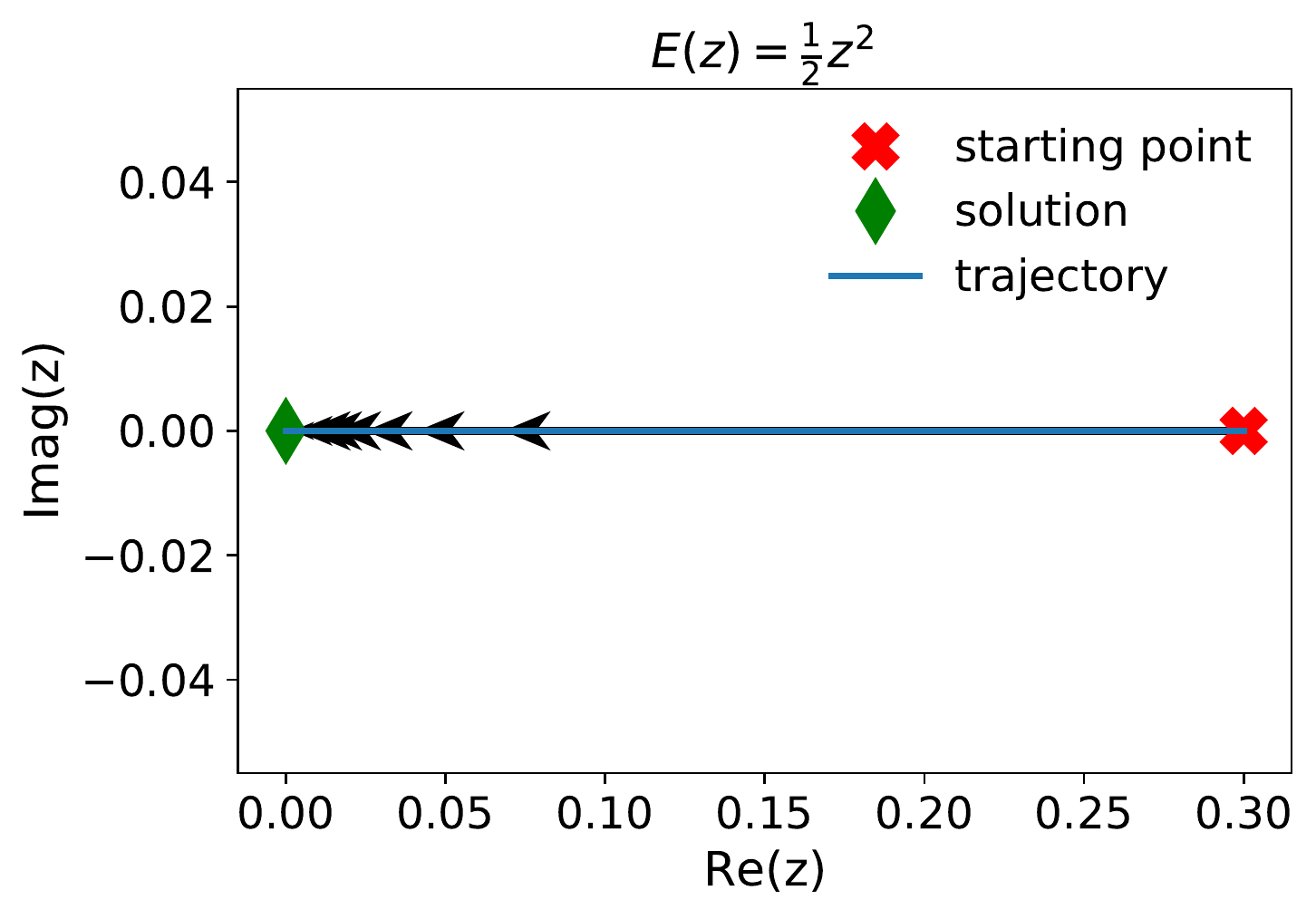}
 \label{fig:z_square_stable}
 }
\end{subfigure}
\begin{subfigure}[$1/h < \lambda_0 < 2/h, \hspace{0.5em} Re(\alpha_{PF}(\lambda_0 h)) < 0$]{
 \includegraphics[width=0.32\columnwidth]{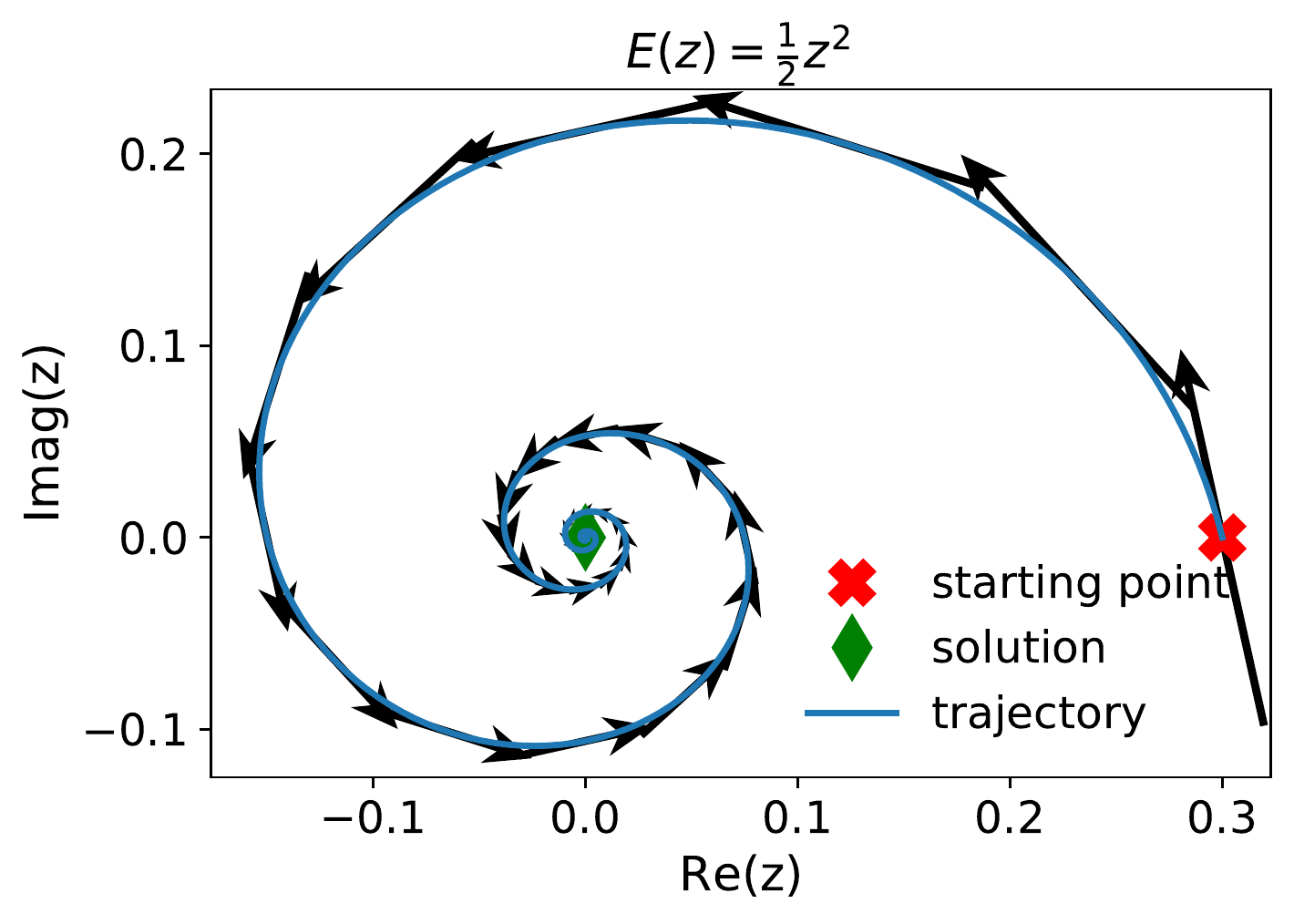}
 \label{fig:z_square_oscillation}
 } \end{subfigure}
\begin{subfigure}[$2/h < \lambda_0, \hspace{0.5em} Re(\alpha_{PF}(\lambda_0 h)) > 0$]{
 \includegraphics[width=0.32\columnwidth]{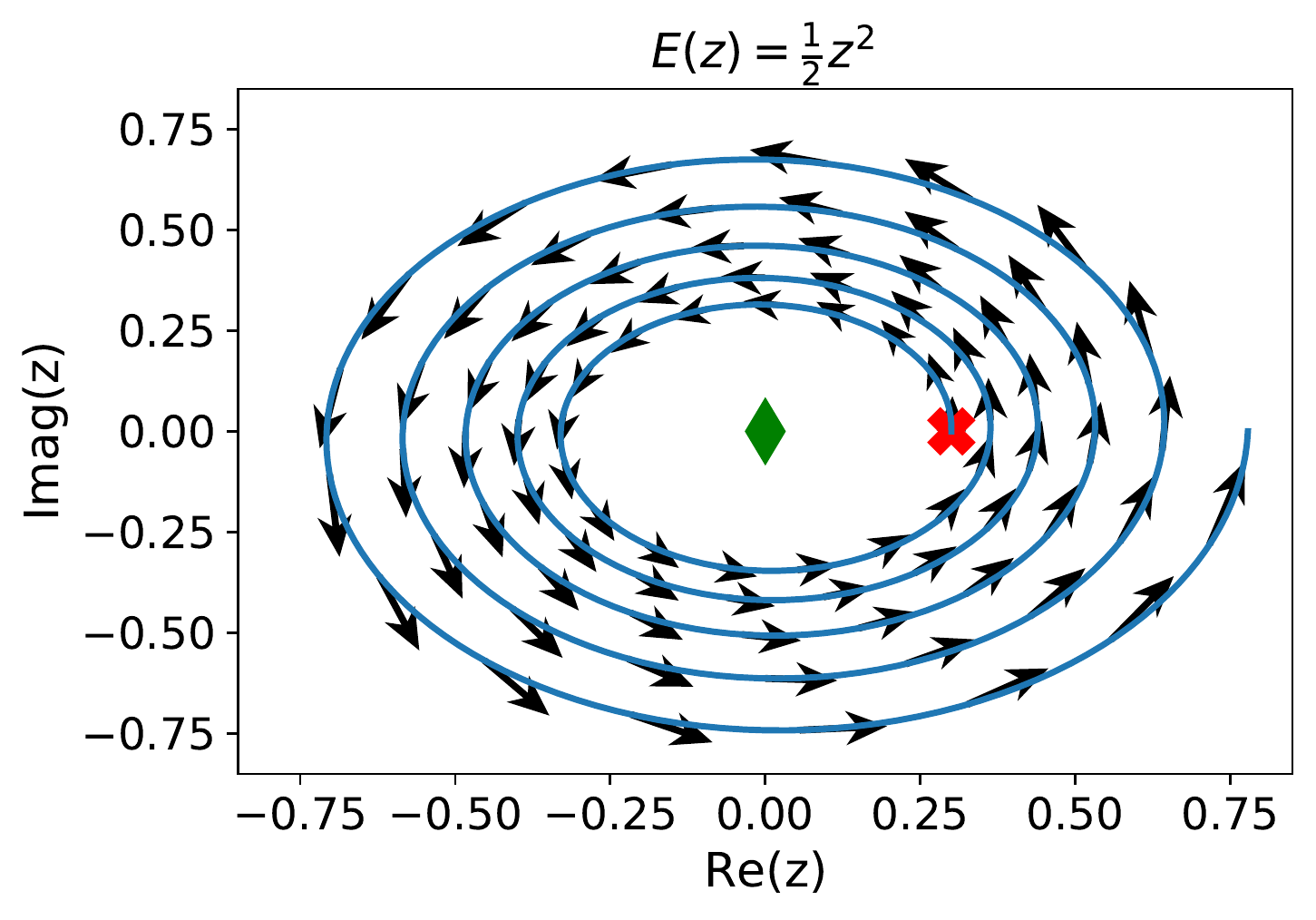}
 \label{fig:z_square_divergence}
 }\end{subfigure}
\caption[$E(z) = \frac 1 2 z^2$. Learning rates are $0.8$, $1.5$ and $2.1$.]{The behavior of PF on $E(z) = \frac 1 2 z^2$ with solution $z(t) = e^{\log(1 - h)/h}z(0)$. When $\lambda_0 < 1/h$,  $z(t) = (1-h)^{t/h} z(0)$ which is in real space and converges to the equilibrium. When $\lambda_0 > 1/h$, $z(t) = (h-1)^{t/h} \left(\cos(\pi t/h) + i \sin(\pi t/h) \right)z(0)$. This exhibits oscillatory behavior, and when $\lambda_0 > 2/h$, diverges.}
\label{fig:z_square}
\end{figure}

\textbf{The importance of the largest eigenvalue $\lambda_0$}.
The largest eigenvalue $\lambda_0$ plays an important part in the PF. Since $h \lambda_0 \ge h \lambda_i \hspace{1em} \forall i$, $\lambda_0$ determines  where in the above cases the PF is situated and thus whether there are oscillations and unstable behavior in training. 
For all flows of the form we consider we can write:
\begin{align}
\frac{d E(\vtheta)}{d t} = \frac{d E(\vtheta)}{d \vtheta}^T \frac{d \vtheta}{dt} =  \nabla_{\vtheta} E^T \sum_{i=0}^{D-1} \alpha(h\lambda_i)  \nabla_{\vtheta} E^T \vu_i \vu_i =  \sum_{i=0}^{D-1} \alpha(h\lambda_i)  (\nabla_{\vtheta}E^T \vu_i)^2
\label{eq:changes_in_e}
\end{align}

and thus if $\alpha(h \lambda_i) \in \mathbb{R}$ and $\alpha(h \lambda_i) < 0$ $\forall i$ then $\frac{d E(\vtheta)}{d t} \le 0$ and following the corresponding flow minimises $E$.
In the case of the PF this gets determined by $\lambda_0$.
 If $\lambda_0 < \frac{1}{h}$ then $\alpha_{PF}(h\lambda_i) < 0 \; \forall i$  (real stable case above)  and the PF minimises E.
If  $1/h < \lambda_0 < \frac{2}{h}$ then $Re[\alpha_{PF}(h\lambda_i)] < 0 \;\forall i$  (complex stable case above) close to a gradient descent iteration $\lambda_i, \vu_i \in \mathbb{R}$ we can write that $\frac{d Re[E(\vtheta)]}{d t} = \sum_{i=0}^{D-1} Re[\alpha_{PF}(h\lambda_i)]  (\nabla_{\vtheta}E^T \vu_i)^2$ and thus the real part of the loss function decreases. If $\lambda_0 > \frac{2}{h}$  then $ Re[\alpha_{PF}(h\lambda_0)] > 0 $ (unstable complex case above) and if $(\nabla_{\vtheta}E^T \vu_0)^2$ is sufficiently large we can no longer ascertain the behavior of $E$. We present a discrete time argument for this observation in Section~\ref{sec:changes_in_loss_discrete}.

\textbf{Building intuition}.
For quadratic objective $E(\vtheta) = \frac{1}{2}\vtheta^T A \vtheta$ the  PF describes gradient descent exactly. We show examples Figures~\ref{fig:intuition_convex1d_complex_part_needed} and~\ref{fig:intuition_quadratic_2d}. 
Unlike the NGF or the IGR flow, the PF captures the oscillatory  and divergent behavior of gradient decent. Importantly, to capture the unstable behavior which occurs when $\lambda_0 > 1/h$ the imaginary part of the PF is needed.
To expand intuition outside the quadratic case, we show the PF for the banana function~\citep{rosenbrock1960automatic} in Figure~\ref{fig:intuition_banana} and an additional example in 1D with a non-quadratic function (Figure~\ref{fig:1d_cosine} in the Appendix). In this case, the PF no longer follows the gradient descent trajectory exactly, but we still observe the importance of the PF in capturing instabilities of gradient descent; we also observe that adding non-principal terms can restabilize the trajectory. 

\begin{remark} For the banana function, the principal terms have a destabilizing effect when $h > 2/\lambda_0$ while the non principal terms can have a stabilizing effect.
\end{remark}

\begin{figure}[thb]
\begin{subfigure}[$\lambda_0 < 1/h$ (stability)]{
 \includegraphics[width=0.3\columnwidth]{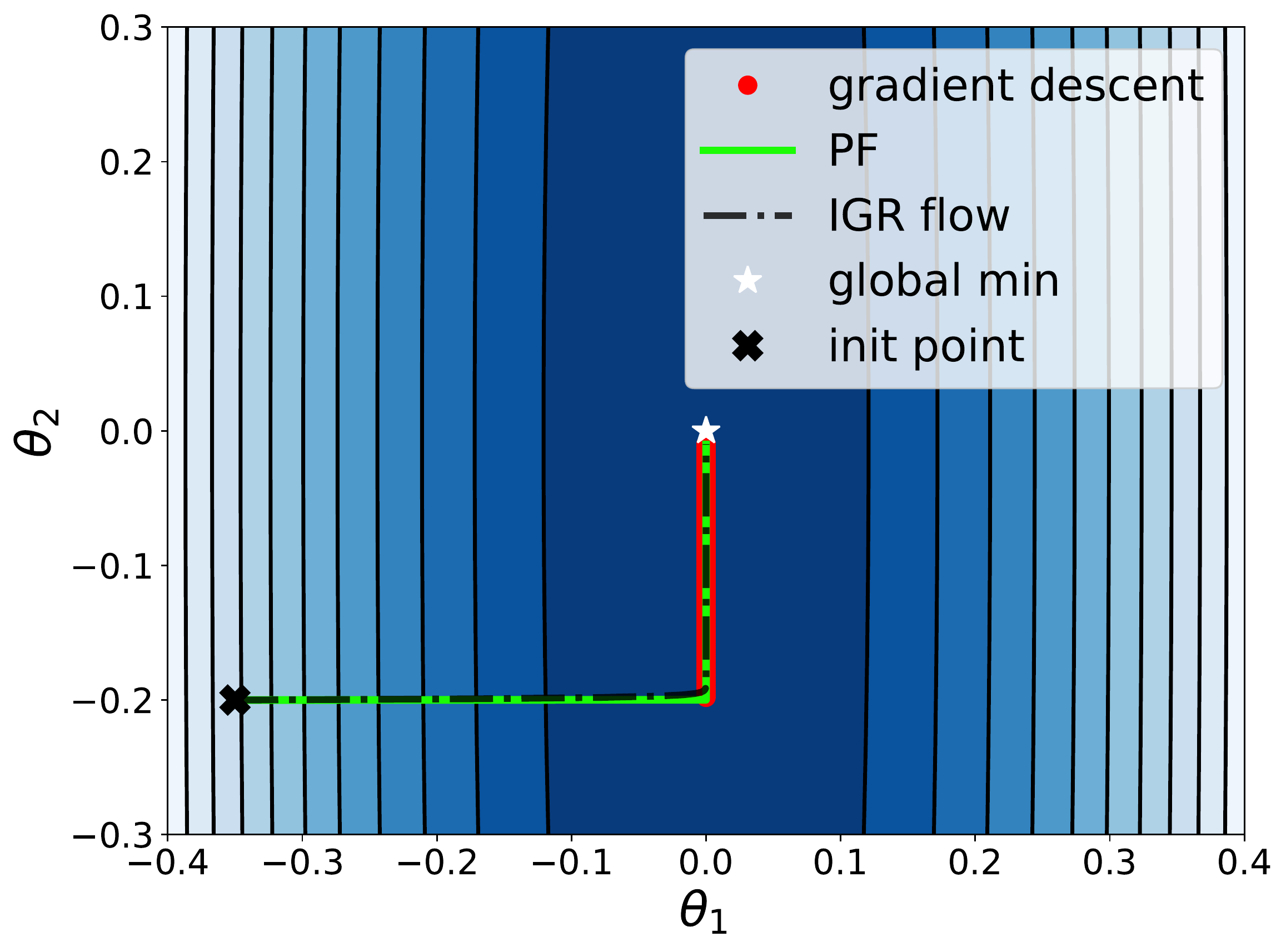}
 }\end{subfigure}
\begin{subfigure}[$1/h <\lambda_0 < 2/h$ (oscillations)]{
 \includegraphics[width=0.3\columnwidth]{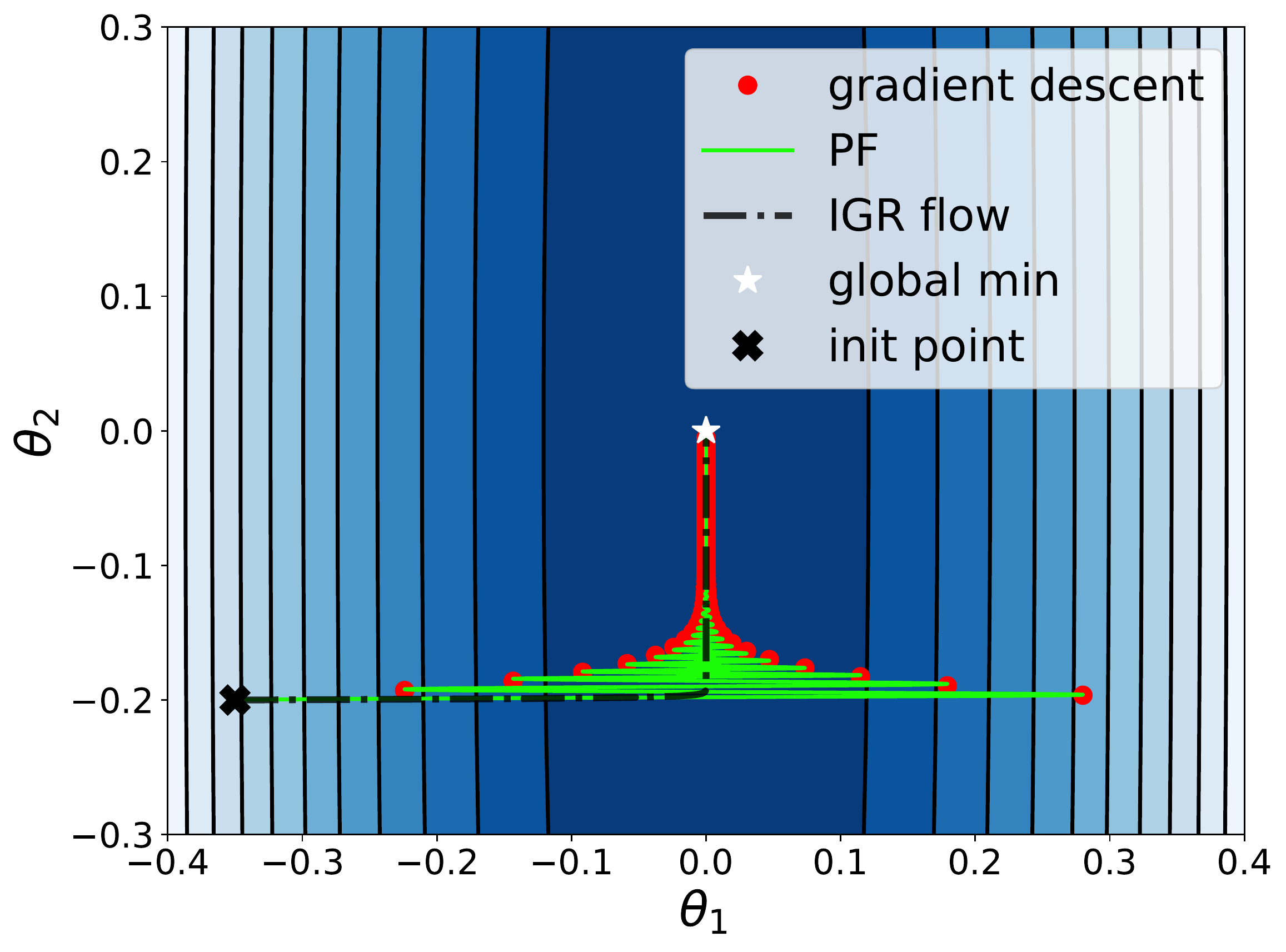}
 }\end{subfigure}
 \begin{subfigure}[$\lambda_0 > 2/h$ (divergence)]{
 \includegraphics[width=0.3\columnwidth]{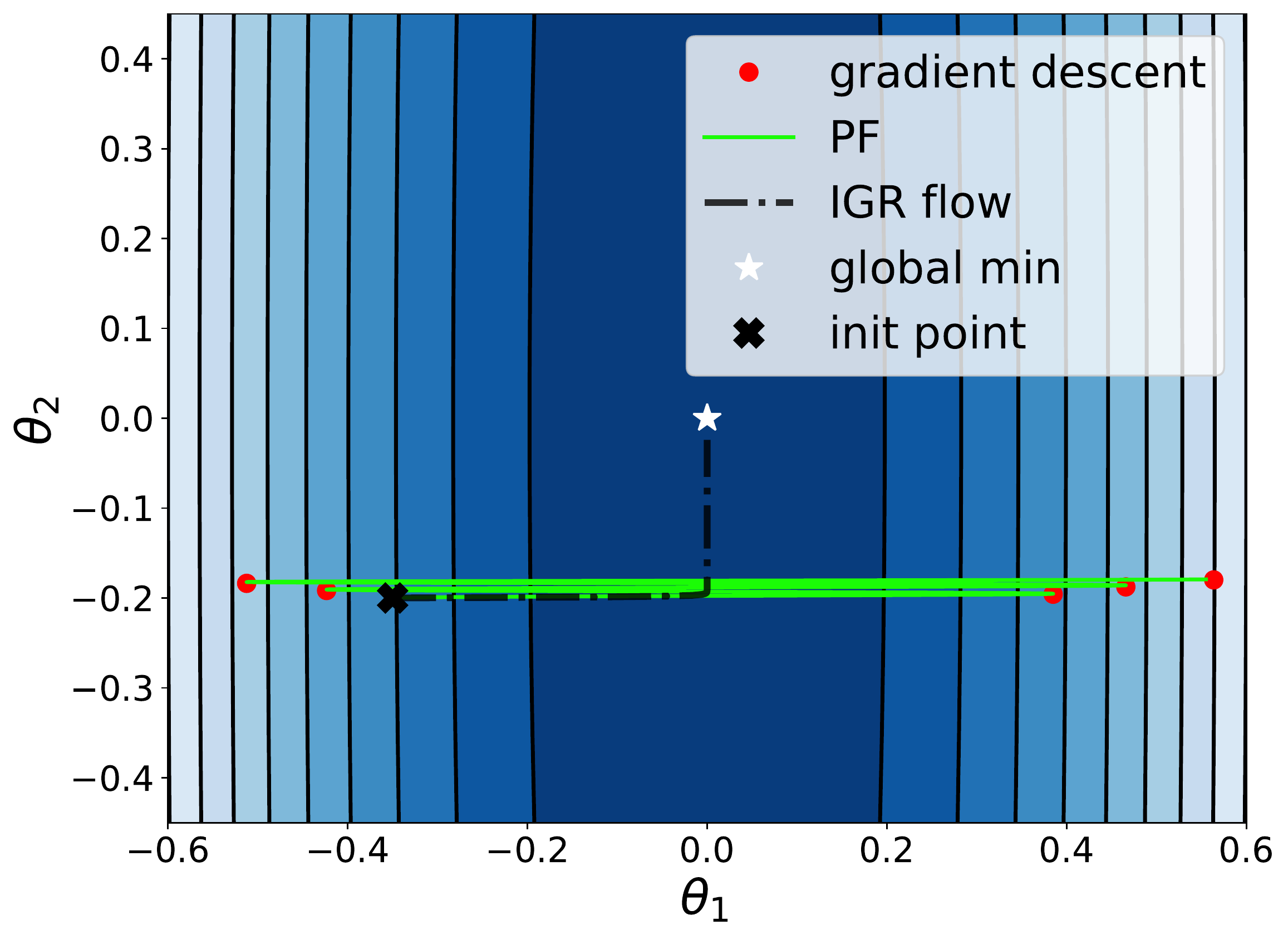}
 }\end{subfigure}
\caption[Quadratic losses in 2 dimensions. $E= \vtheta^T A \vtheta$ with  $A = ((1, 0.0), (0.0, 0.01))$, with learning rates $0.5$, $0.9$ and $1.05$.]{\textbf{Quadratic losses in 2 dimensions}. The PF captures the behavior of gradient descent exactly for quadratic losses, including oscillatory behavior and divergence. }
\label{fig:intuition_quadratic_2d}
\end{figure}

\begin{figure}[thb]
\centering
\begin{subfigure}[$\lambda_0 < 1/h$]{
 \includegraphics[width=0.31\columnwidth]{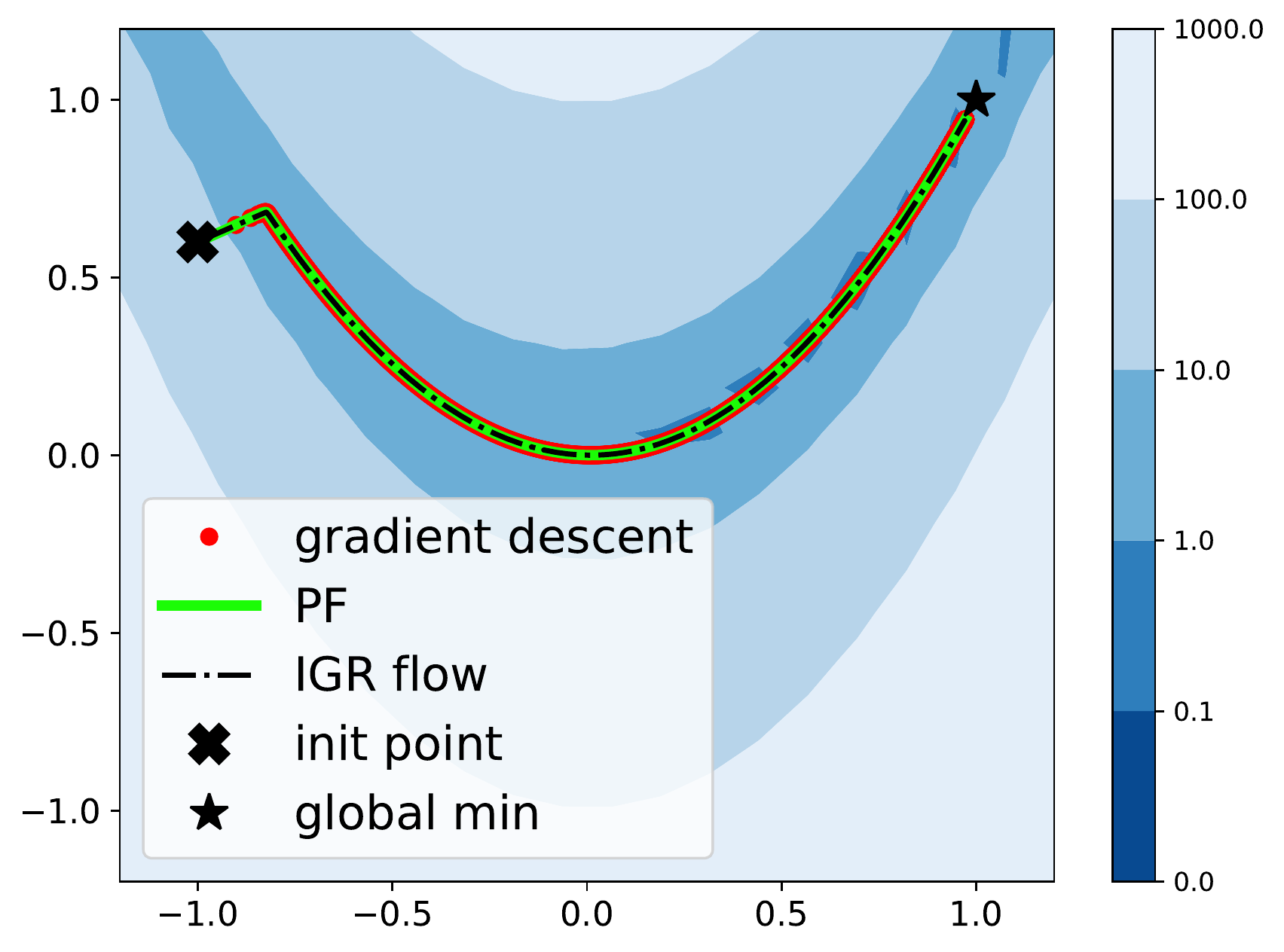}
 }\end{subfigure}
 \begin{subfigure}[$\lambda_0 < 2/h (\lambda_0 \approx 1.9/h)$]{
 \includegraphics[width=0.31\columnwidth]{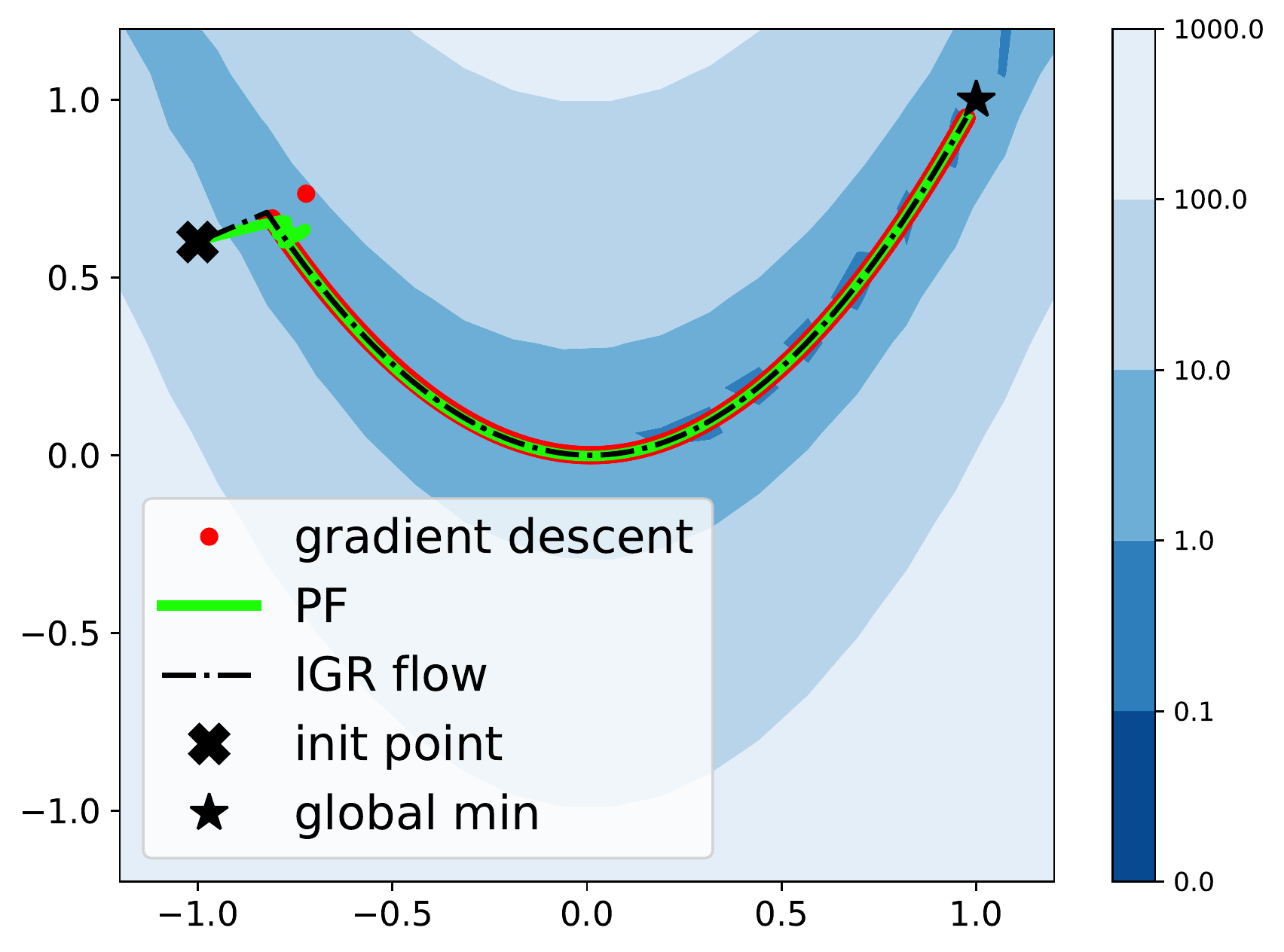}
 }\end{subfigure}
\begin{subfigure}[$\lambda_0 >> 2/h (\lambda_0 \approx 5/h)$]{
 \includegraphics[width=0.31\columnwidth]{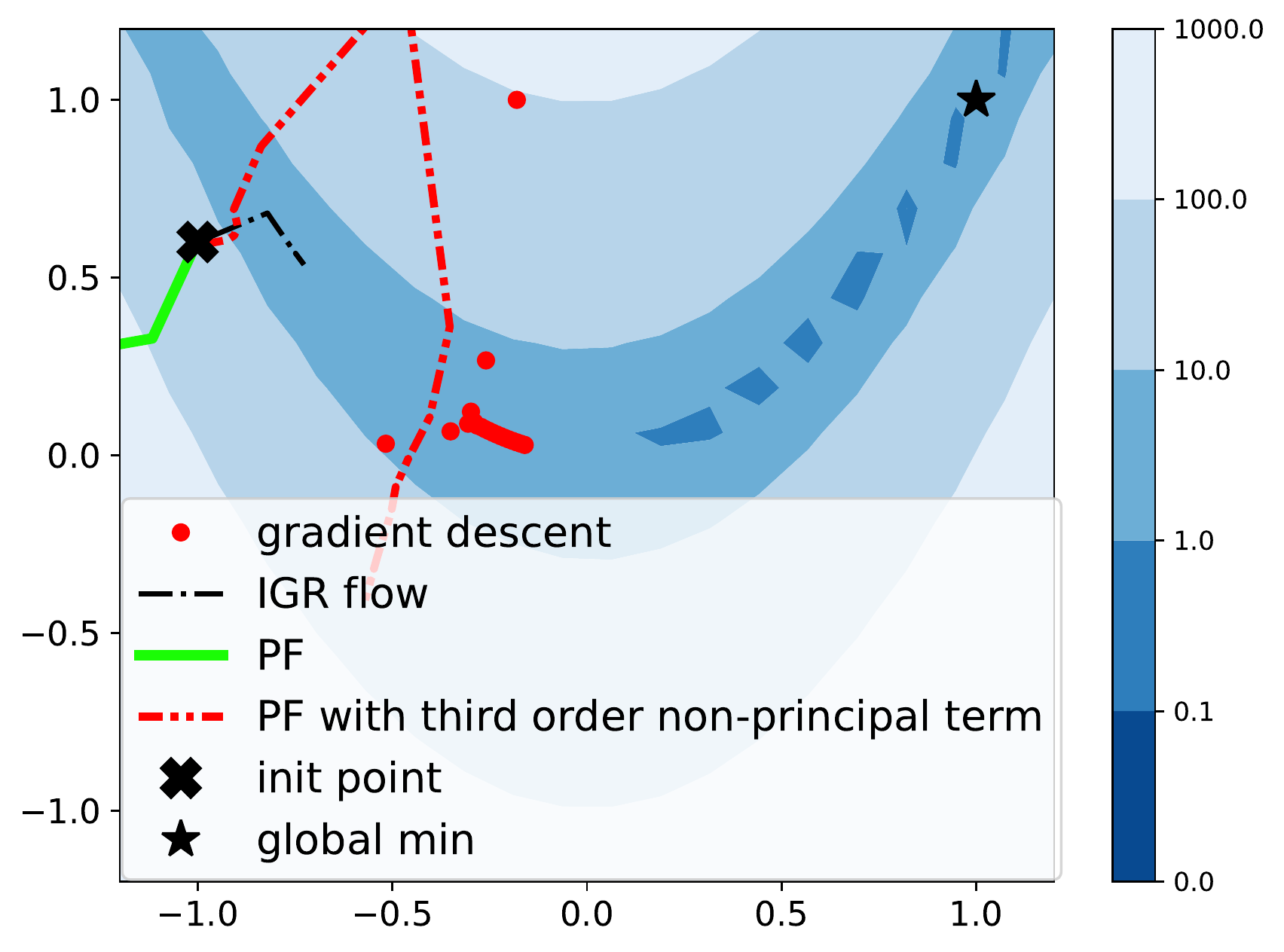}
 }\end{subfigure}
\caption[Banana function. Learning rates $0.0006$, $0.0017$ $0.005$.]{\textbf{Banana function}. The PF can capture instability and the gradient descent trajectory over many iterations  when $\lambda_0$ is close to $2/h$. When  $\lambda_0 >> 2/h$ (right) the PF does not track the GD trajectory over many gradient descent steps, but when including a non-principal term the flow is able to capture the general trajectory of gradient descent and unstable behavior of gradient descent.}
\label{fig:intuition_banana}
\end{figure}

\subsection{The stability analysis of the principal flow}
\label{sec:pf_stability_analysis}

We now perform stability analysis on the PF, to understand how it can be used to predict certain behaviors of gradient descent around critical points of the loss function $E$. Consider $\vtheta^*$ such a critical point, i.e $\nabla_{\vtheta} E(\vtheta^*) = \mathbf{0}$. 
For a critical point $\vtheta^*$ to be exponentially asymptotically attractive, all eigenvalues of the Jacobian evaluated at $\vtheta^*$ need to have strictly negative real part.

The PF has the following Jacobian at critical points (proof in Section~\ref{sec:jacobian} in the Appendix):

\begin{align}
J_{PF}(\vtheta^*) = \sum_{i=0}^{D-1} \frac{\log(1 - h \lambda_i^*)}{h} \vu_i^* {\vu_i^*}^T
\end{align}

where $\lambda_i^*$, $\vu_i^*$ are the eigenvalues and eigenvectors of the Hessian $\nabla_{\vtheta}^2E(\vtheta^*)$. 
We thus have that the eigenvalues of the Jacobian $J_{PF}(\vtheta^*)$ at the critical point $\vtheta^*$ are $\frac 1h \log(1 - h\lambda_i^*)$ for $i=1,\dots, D$.

\textbf{Local minima}.
Suppose that $\vtheta^*$ is a local minimum. Then all Hessian eigenvalues are non-negative $\lambda_i^* \ge 0$.
We perform the stability analysis in cases given by the value of $\lambda_i^*$, corresponding to the cases in Section~\ref{sec:principal_flow_eigendecom}:

$h < 1/\lambda_i^*$. The corresponding eigenvalue of the Jacobian $\frac 1h \log(1 - h\lambda_i^*)$ is negative, since ${0 < 1 - h\lambda_i^* < 1}$. The principal vector field is attractive in the corresponding eigenvector direction.

$h \in [1/\lambda_i^*,\, 2/\lambda_i^*)$. The corresponding eigenvalue of the Jacobian $\frac 1h \log(1 - h\lambda_i^*) = \frac 1 h \log(h\lambda_i^* - 1) + i\frac \pi h$ is complex, with negative real part since since $h\lambda_i^* -1 < 1$. The principal vector field is attractive in the corresponding eigenvector direction.

$h \ge 2/\lambda_i^*$.  The corresponding eigenvalue of the Jacobian $\frac 1h \log(1 - h\lambda_i^*) = \frac 1 h \log(h\lambda_i^* - 1) + i\frac \pi h$ is complex, with non-negative real part, since since $h\lambda_i^* - 1 \ge 1$. The principal vector field is not attractive in the corresponding eigenvector direction, and if $h > 2/\lambda_i^*$ it is repelled in the corresponding eigenvector direction.

The last case tells us that the PF is not always attracted to local minima, as it is not attractive in eigendrections where $h \ge 2/\lambda_i^*$. Thus \textbf{like gradient descent, the PF can be repelled around local minima for large learning rates}. This is in contrast to the NGF and the IGR flow, which always predict convergence around a local minimum: the eigenvalues of the NGF Jacobian are $-\lambda_i^*$, and for the IGR flow the eigenvalues are $-\lambda_i^* - \frac{h^2}{2} {\lambda_i^*}^2$, both are negative when $\lambda_i^*$ is positive. For derivations see Section~\ref{sec:jacobian_igr_ngf} in the Appendix.

\begin{remark} For quadratic losses, where the PF is exact, the results above recover the classical gradient descent result for quadratic losses namely that gradient descent convergences if $\lambda_0 < 2/h$, otherwise diverges.
\end{remark}

\textbf{Saddle points}.
Suppose that $\vtheta^*$ is a strict saddle point. In this case there exists $\lambda^*_s$ such that $\lambda^*_s < 0$. We want to analyse the behavior of the PF in the direction of the corresponding eigenvector $\vu_s^*$. In that case, $\log( 1 - h\lambda_s^*) > 0$ which entails that the PF is repelled in the eigendirections of strict saddle points. Note that this is also the case for the NGF since the corresponding eigenvalues of the Jacobian of the NGF would be $-\lambda_s^*$,  also positive. Unlike the NGF however, the subspace of eigendirections that the PF is repelled by can be larger since it includes also eigendirections where $\lambda_i^* > 2/h > 0$. 

\section{Predicting neural network gradient descent dynamics with the principal flow}
\label{sec:the_pf_and_nns}

Computing the PF on large neural networks during training is computationally prohibitive, as it requires finding all eigenvalues of the Hessian matrix once for each step of the flow simulation, corresponding to many eigen-decompositions per gradient descent step. 
To build intuition about the PF for neural networks, we start with a small MLP for a 2 dimensional input regression problem, with random inputs and labels. Here we can understand the behavior of the PF since we can compute its modified vector field exactly and compare it with the behavior of gradient descent. We show results in Figure~\ref{fig:intuition_nn_function_delta}, where we visualize the norm of the difference between gradient descent parameters at each iteration and the parameters produced by the continuous time flows we compare with. We observe that \textit{short term the principal flow is better than all other flows at tracking the behavior of gradient descent}. As the number of iterations increases however, the PF accumulates error in the case of $\lambda_0 > 2/h$; this is likely due to the fact that while gradient descent 
parameters are real, this is not the case for the PF, as discussed in Remark~\ref{rem:pf_multiple_ints}. Since we are primarily concerned with using the PF to understand gradient descent for a small number of iterations this will be less of a concern in our experimental settings. Additional results which confirm the PF is better than the other flows at tracking gradient descent on a bigger network trained the UCI breast cancer dataset~\citep{asuncion2007uci} are shown in Figure~\ref{fig:breast_cancer_principal_flow} in the Appendix.

\begin{remark}
\textbf{On the multiple iteration behavior of the PF}. We note that while gradient descent parameters are real for any iteration $\vtheta_t$, $\vtheta_{t+1}$, ... $\vtheta_{t+n}$ when we approximate the behavior of gradient descent by initializing $\vtheta(0) = \vtheta_t$ and running the PF for time $nh$, there is nothing enforcing that $\vtheta(h)$, ... $\vtheta(nh)$ will be real when the PF is complex valued ($\lambda_0 > 1/h$). We also note that in that case the symmetric Hessian is not Hermitian and the eigenvalues and eigenvector of the Hessian will not be real; furthermore, the eigenvectors need not form a basis\footnote{To avoid the concern around the eigenvectors of the Hessian no longer forming a basis, one can use the Jordan normal form instead, as we show in Section~\ref{sec:bea_pf_jordan}. We don't take this approach here as most of our following analysis is not affected, and is concerned with the behaviour of the PF around one gradient descent iteration. Furthermore, support of the Jordan normal form in code libraries is limited (especially for complex matrices), and we did not find this to be a significant issue in the experiments where we simulate the PF outside the quadratic case for a few iterations. We note, however, that mathematical analysis of long-term PF trajectories  for general functions should use the Jordan normal form.}. For long term trajectories (larger $n$), this can have an effect on long term error between gradient descent and PF trajectories, through an accumulating effect of the imaginary part in the PF. This can be mitigated by using the PF to understand the short term behavior of gradient descent (small $n$).
\label{rem:pf_multiple_ints}
\end{remark}

\begin{figure}[thb]
 \begin{subfigure}[$\lambda_0 < 1/h$]{
 \includegraphics[width=0.33\columnwidth]{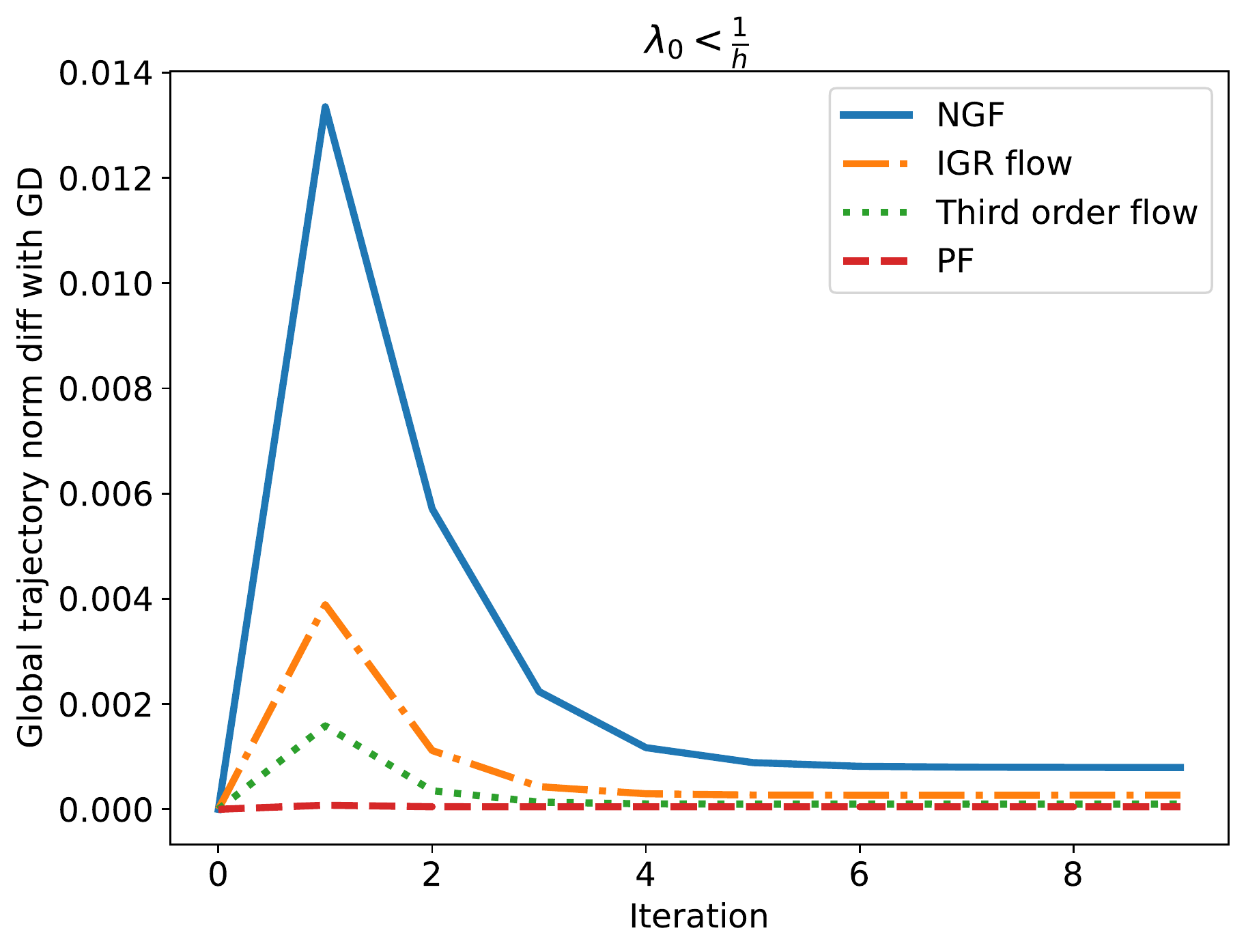}
 }\end{subfigure}
\begin{subfigure}[$1/h < \lambda_0 < 2/h$]{
 \includegraphics[width=0.33\columnwidth]{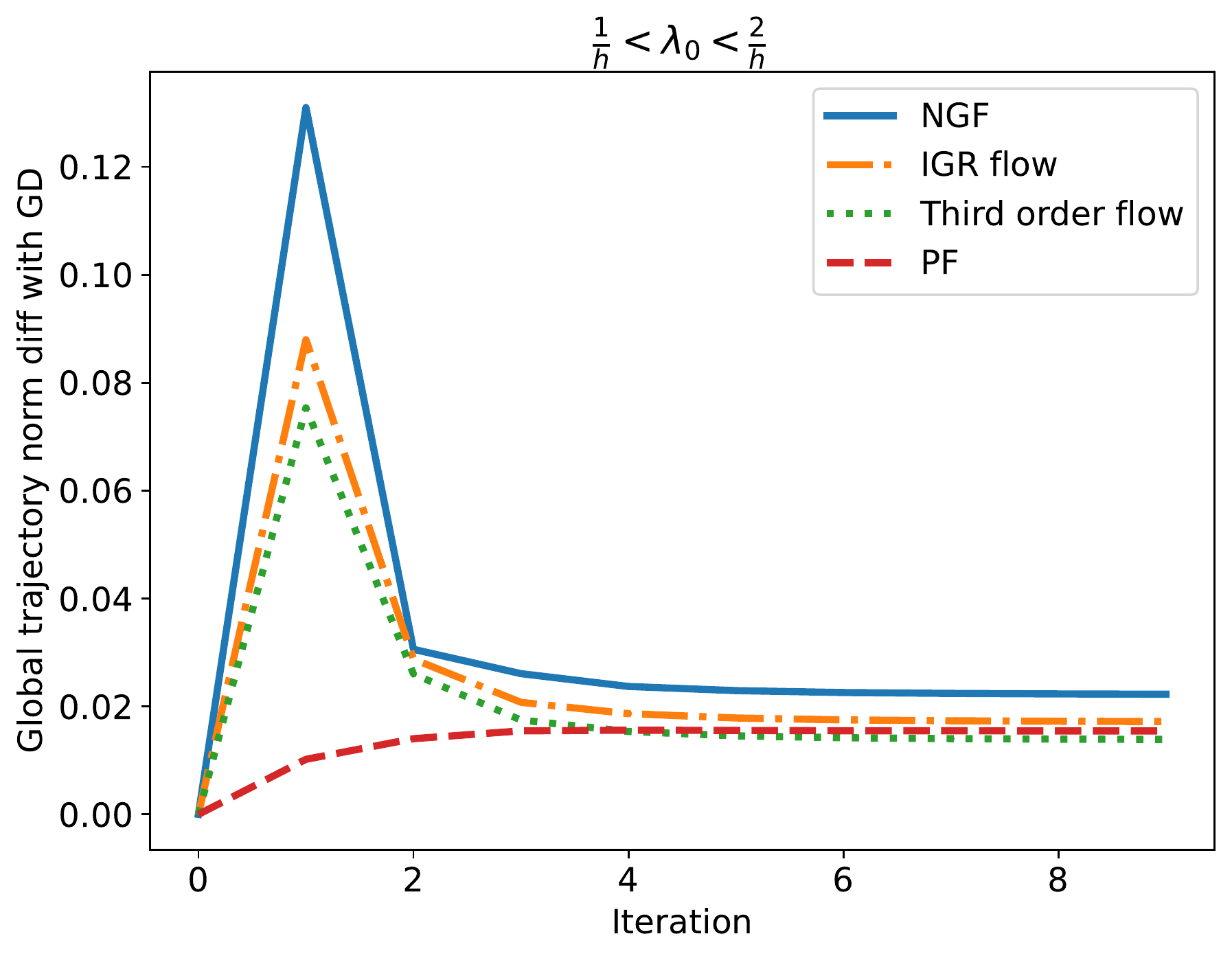}
 }\end{subfigure}
\begin{subfigure}[$\lambda_0 > 2/h$]{
 \includegraphics[width=0.33\columnwidth]{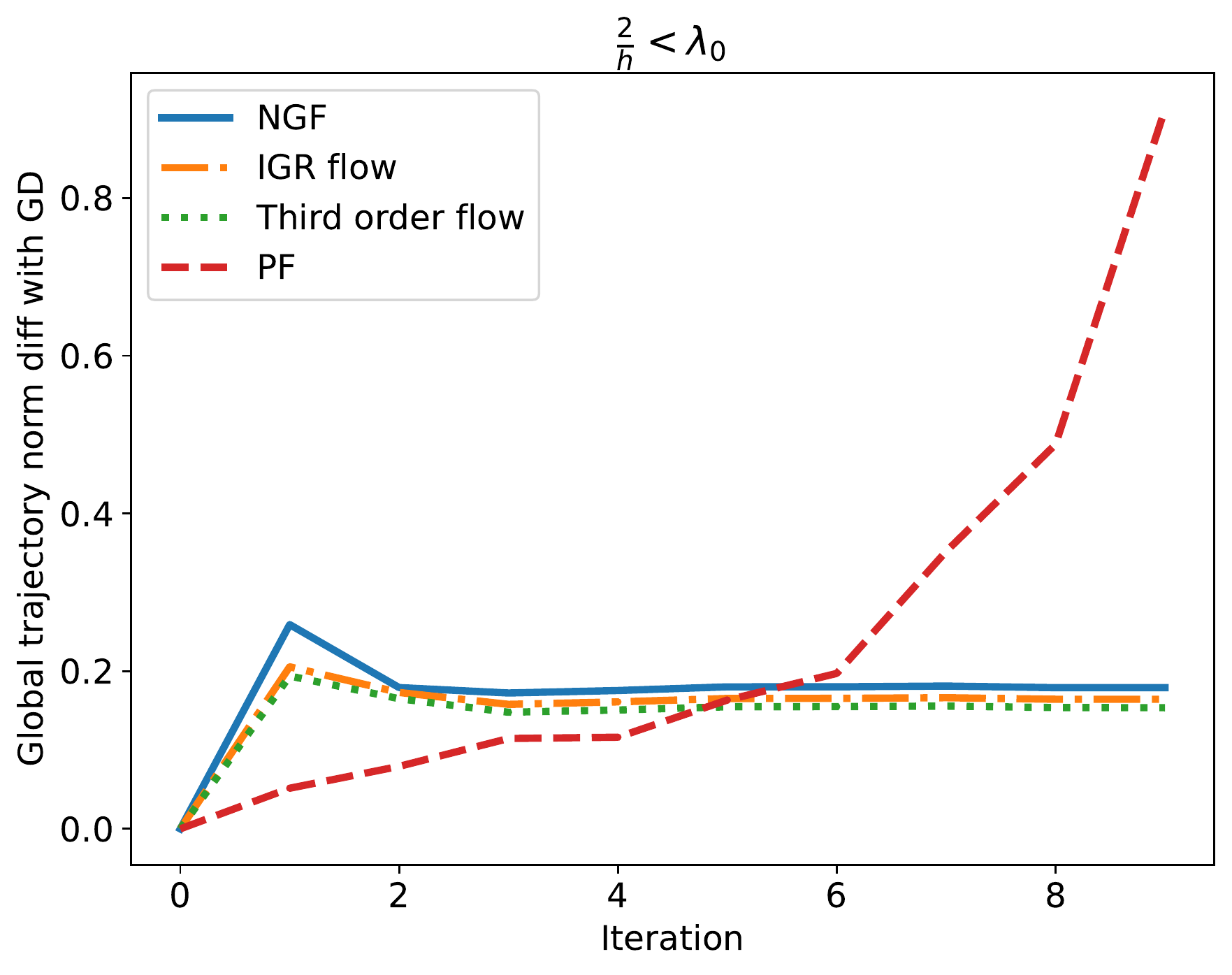}
 }\end{subfigure}
\caption[Error between gradient descent parameters and parameters obtained following continuous time flows for multiple iterations on a small MLP: $\norm{\vtheta_{t+n} - \vtheta(nh)}$. Input size 2, MLP with output units $10, 1$. Learning rates $0.1$, $0.2$ and $0.25$. We use a dataset with 5 examples where the input is generated from a Gaussian distribution with 2 dimensions and the targets are samples from a Gaussian distribution with 1 dimension. A mean square error loss is used.]{Error between gradient descent parameters and parameters obtained following continuous time flows for multiple iterations: $\norm{\vtheta_{n} - \vtheta(nh)}$ with $\vtheta(0) = \vtheta_0$. For small $n$, the PF is better at capturing the behavior of gradient descent across all cases.}
\label{fig:intuition_nn_function_delta}
\end{figure}

\subsection{Predicting $\nabla_{\vtheta} E ^T \vu_0$ using the principal flow}
\label{sec:nns_principal_flow}

For large neural networks,  instead of simulating the PF describing how the entire parameter vector changes in time we can use the PF to approximate changes in a scalar quantity only. This will allow us to compare the predictions of the PF against the predictions of the NGF and IGR flow on realistic settings. To do so, we first have to compute how the gradient changes in time:

\begin{corollary} If $\vtheta$ follows the PF, then:
$\dot{\left({\nabla_{\vtheta}E}\right)} =  \sum_{i=0}^{D-1} \frac{\log(1 - h \lambda_i)}{h} (\nabla_{\vtheta} E^T \vu_i)  \vu_i$.
\end{corollary}

This follows from applying the chain rule and using the definition of the PF. We contrast this with how the gradient evolves if the parameters follow the NGF:

\begin{corollary} If $\vtheta$ follows the NGF, then:
$\dot{\left({\nabla_{\vtheta}E}\right)} =  \sum_{i=0}^{D-1} -  \lambda_i (\nabla_{\vtheta} E^T \vu_i)  \vu_i$
\label{col:grad_ngf_ode}
\end{corollary}

\begin{corollary} If $\vtheta$ follows the IGR flow, then:
$\dot{\left({\nabla_{\vtheta}E}\right)} =  \sum_{i=0}^{D-1} - \left(\lambda_i + \frac{h}{2} \lambda_i^2 \right)  (\nabla_{\vtheta} E^T \vu_i)  \vu_i$
\end{corollary}

We would like to use the above to assess how $\nabla_{\vtheta} E^T \vu_i$ changes in time under the above flows and check their predictions empirically against results obtained when training neural networks with gradient descent. Since $\vu_i$ is an eigenvector of the Hessian it also changes in time according to the changes given by the corresponding flow, making $\dot{\left(\nabla_{\vtheta}E^T \vu_i\right)}$ difficult to calculate. Even when if we wrote an exact flow for $\dot{\left(\nabla_{\vtheta}E^T \vu_i\right)}$,  it would be computationally challenging to simulate it since finding the new values of $\vu_i$ would depend on the full Hessian and would lead to the same computational issues we are trying to avoid in the case of large neural networks. In order to mitigate these concerns, we will make the additional approximation that $\lambda_i$ and $\vu_i$ do not change inside an iteration which will allow us to  approximate changes to $\nabla_{\vtheta} E^T \vu_i$ and compare them against empirical observations. We note that we will not use this approximation for any other results.

\begin{remark}If we assume that $\lambda_i$, $\vu_i$ do not change between iterations, if $\vtheta$ follows the PF then $\dot{\left(\nabla_{\vtheta}E^T \vu_i\right)} = \frac{\log(1 - h \lambda_i)}{h} \nabla_{\vtheta} E^T \vu_i$.
\end{remark}

\begin{remark} If we assume that $\lambda_i$, $\vu_i$ do not change between iterations, if $\vtheta$ follows the NGF we can write $\dot{\left(\nabla_{\vtheta}E^T \vu_i\right)} = - \lambda_i \nabla_{\vtheta} E^T \vu_i$.
\end{remark} 

\begin{remark} If we assume that $\lambda_i$, $\vu_i$ do not change between iterations, if $\vtheta$ follows the IGR flow we can write $\dot{\left(\nabla_{\vtheta}E^T \vu_i\right)} = - \left(\lambda_i + \frac{h}{2} \lambda_i^2 \right)  \nabla_{\vtheta} E^T \vu_i$.
\end{remark} 

The above flows have the form $\dot{x} = c x$, with solution $x(t) = x(0) e^{ct}$.
We can thus test these solutions empirically by training neural networks with gradient descent with learning rate $h$ and at each step compute $\nabla_{\vtheta}E(\vtheta_t)^T (\vu_i)_{t-1}$ and compare it with the prediction $x(h)$ obtained from the solution from each flow initialized at the previous iteration, i.e. $x(0) = \nabla_{\vtheta}E(\vtheta_{t-1})^T (\vu_i)_{t-1}$.  We show results with a VGG model trained on CIFAR-10 in Figure~\ref{fig:prediction_grad_u}. The results show that the PF is substantially better than the NGF and IGR flow at predicting the behavior of $\nabla_{\vtheta} E^T \vu_0$. Since the NGF and the IGR flow solutions scale the initial value by the inverse of an exponential of magnitude given by $\lambda_0$ for large $\lambda_0$ this leads to a small prediction, which is not aligned with what is observed empirically. We also note that the higher the value of $\nabla_{\vtheta} E^T \vu_0$, the worse the prediction of the PF; these are the areas where the approximations made in the above remarks are likely not to hold due to large gradient norms.

\begin{figure}
 \includegraphics[width=0.32\columnwidth]{  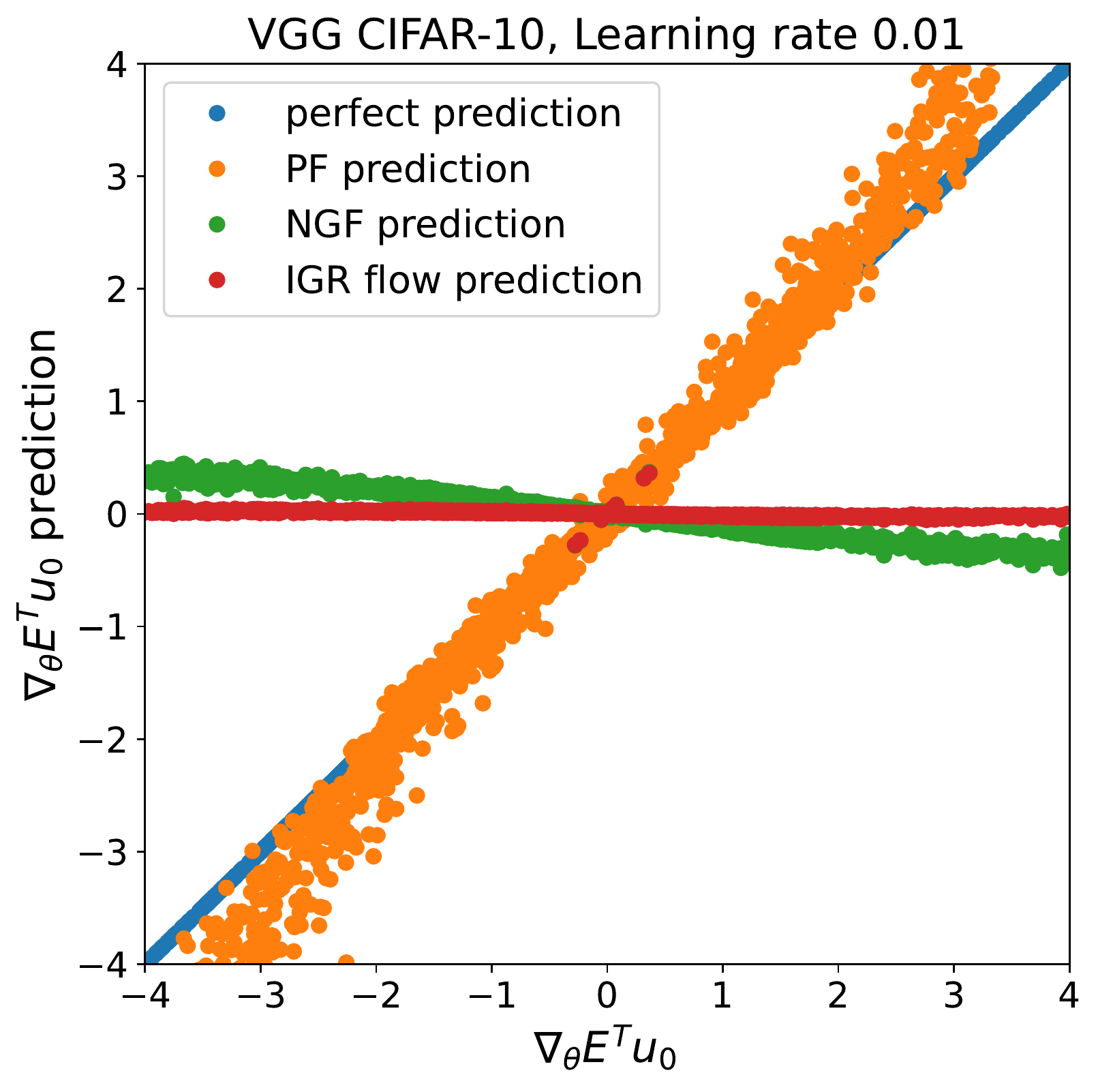}
 \includegraphics[width=0.32\columnwidth]{  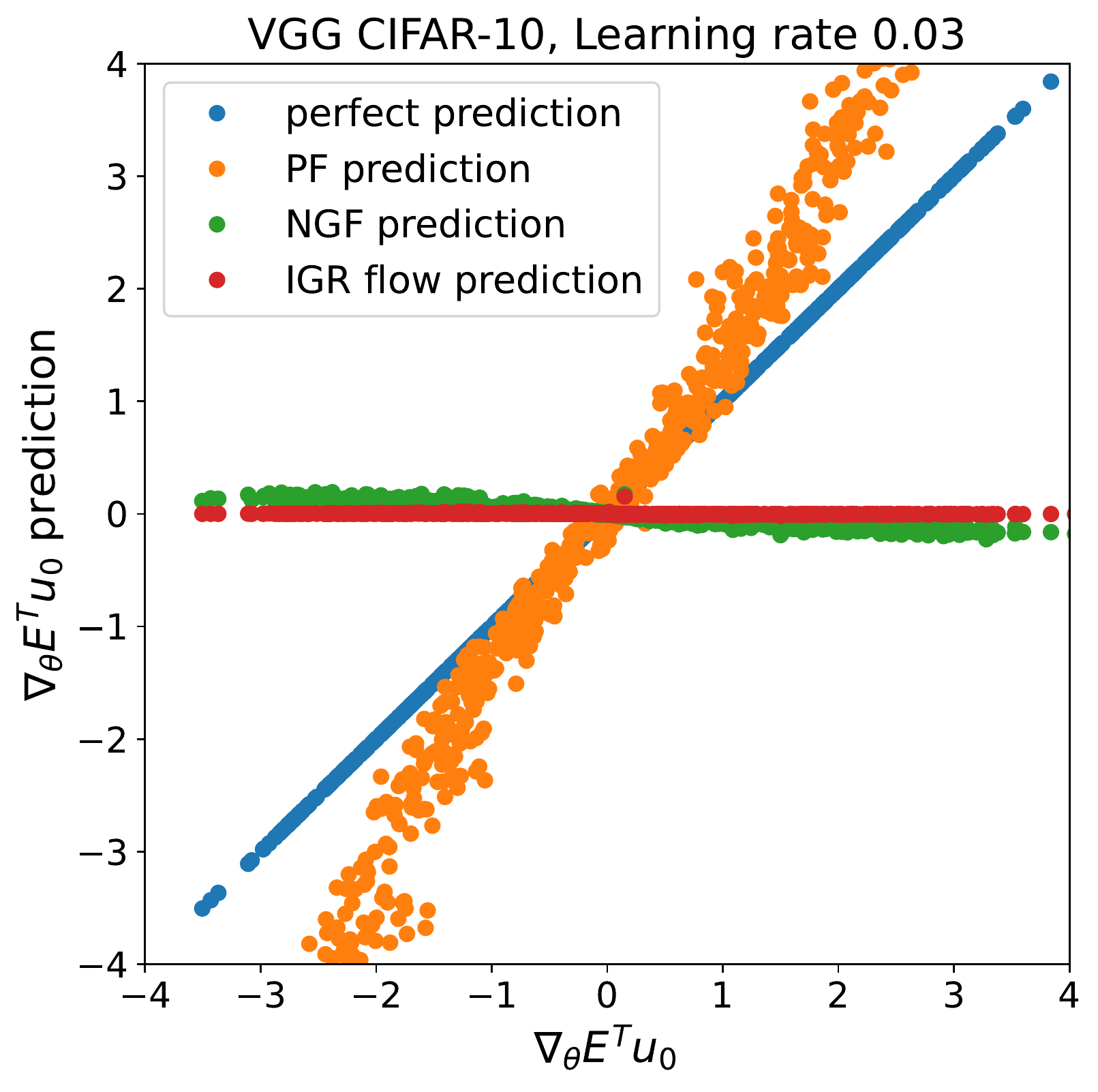}
 \includegraphics[width=0.32\columnwidth]{  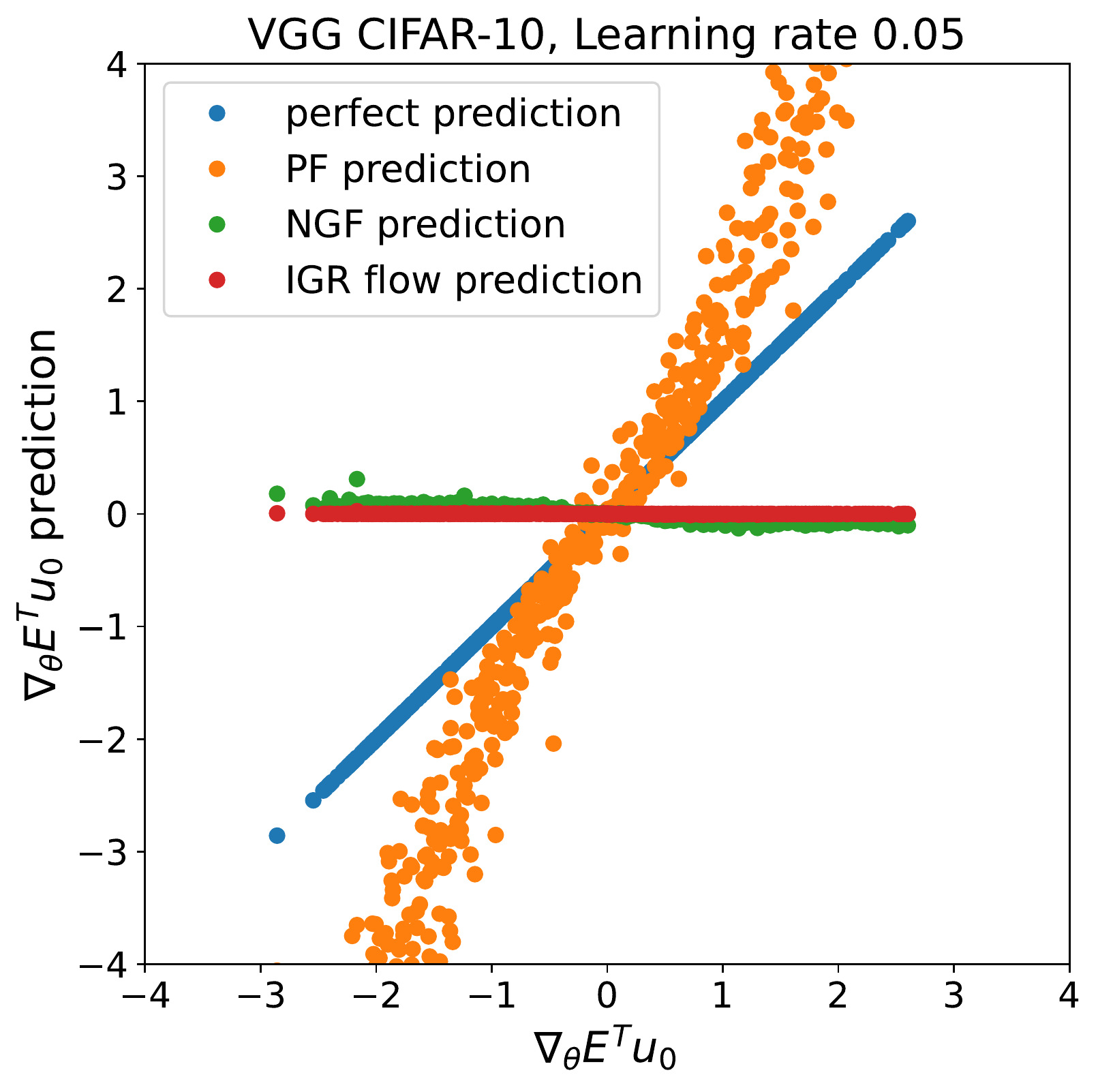}
\caption[Predictions of $\nabla_{\vtheta}E^T\vu_0$ using the principal flow on neural networks. Full batch training with a VGG network on CIFAR-10. Learning rates $0.01$, $0.03$  and $0.05$.]{Predictions of $\nabla_{\vtheta}E^T\vu_0$ according to the NGF, IGR flow and the PF. On the $x$ axis we plot the value of $\nabla_{\vtheta}E^T \vu_0$ as measured empirically in training, and on the $y$ axis we plot the corresponding prediction according to the flows from the value of the dot product at the previous iteration. The `exact match' line indicates a perfect prediction, the upper bound of performance. The PF performs best from all the compared flows, however for higher learning rates its performance degrades when $\nabla_{\vtheta}E^T \vu_0$ is large; this is due to the fact that the higher the learning rate and the higher the gradient norm, the more likely it is that the additional assumption we used that $\lambda_i, \vu_i$ do not change does not hold.}
\label{fig:prediction_grad_u}
\end{figure}

\subsection{\rebuttalrone{Around critical points: escaping sharp local minima and saddles}}

\rebuttalrone{The stability analysis we performed in Section~\ref{sec:pf_stability_analysis} showed the PF is repelled by local minima where $\lambda_0^* >2/h$: that is, even if the model is close to a sharp local minima (with $\lambda_0^* > 2/h$), that local minima will not be attractive and training will continue until a shallow minima is reached. We provide experimental evidence to support that hypothesis in the context of neural networks in Figure~\ref{fig:stability_analysis_nn} in the Appendix; these results are consistent with observations in the deep learning literature~\citep{jastrzkebski2018relation,cohen2021gradient}. Furthermore, while saddle points have long been considered a challenge with high dimensional optimisation~\citep{dauphin2014identifying} in practice gradient descent has not been observed to converge to saddles~\citep{lee2016gradient}. Our analysis suggests that saddles will be repelled not only in the direction of strictly negative eigenvalues, but also in the eigendirections with large positive eigenvalues when large learning rates are used; this can explain why neural networks do not converge to non-strict saddles which exist in deep neural landscapes~\citep{kawaguchi2016deep} but need not be repelling for the NGF and IGR flow (existing analyses of escaping saddle points by gradient descent apply only to strict saddles~\citep{du2017gradient,lee2016gradient}).}

\section{The principal flow, stability coefficients and edge of stability results}
\label{sec:instability_deep_learning}

\textbf{Edge of stability results.} 
 \citet{cohen2021gradient} did a thorough empirical study to show that when training deep neural networks with full batch gradient descent the largest eigenvalue of the Hessian, $\lambda_0$, keeps growing until reaching approximately $2/h$ (a phase of training they call \textit{progressive sharpening}), after which it remains in that area; for mean squared losses this continues indefinitely while for cross entropy losses they show it further decreases later in training. They also show that instabilities in training occur when $\lambda_0 >2/h$. Their empirical study spans neural architectures, data modalities and loss functions. We visualize the edge of stability behavior they observe in Figure~\ref{fig:reproduce_edge_of_stability}; since we use a cross entropy loss $\lambda_0$ decreases later in training.
 We also visualize that iterations where the loss increases compared to the previous iteration overwhelmingly occur when $\lambda_0 >2/h$.
 \citet{cohen2021gradient} also empirically observe that $\vtheta^T \vu_0$ has oscillatory behavior in the edge of stability area but is 0 or small outside it.

\begin{figure}[tb]
 \includegraphics[width=0.33\columnwidth]{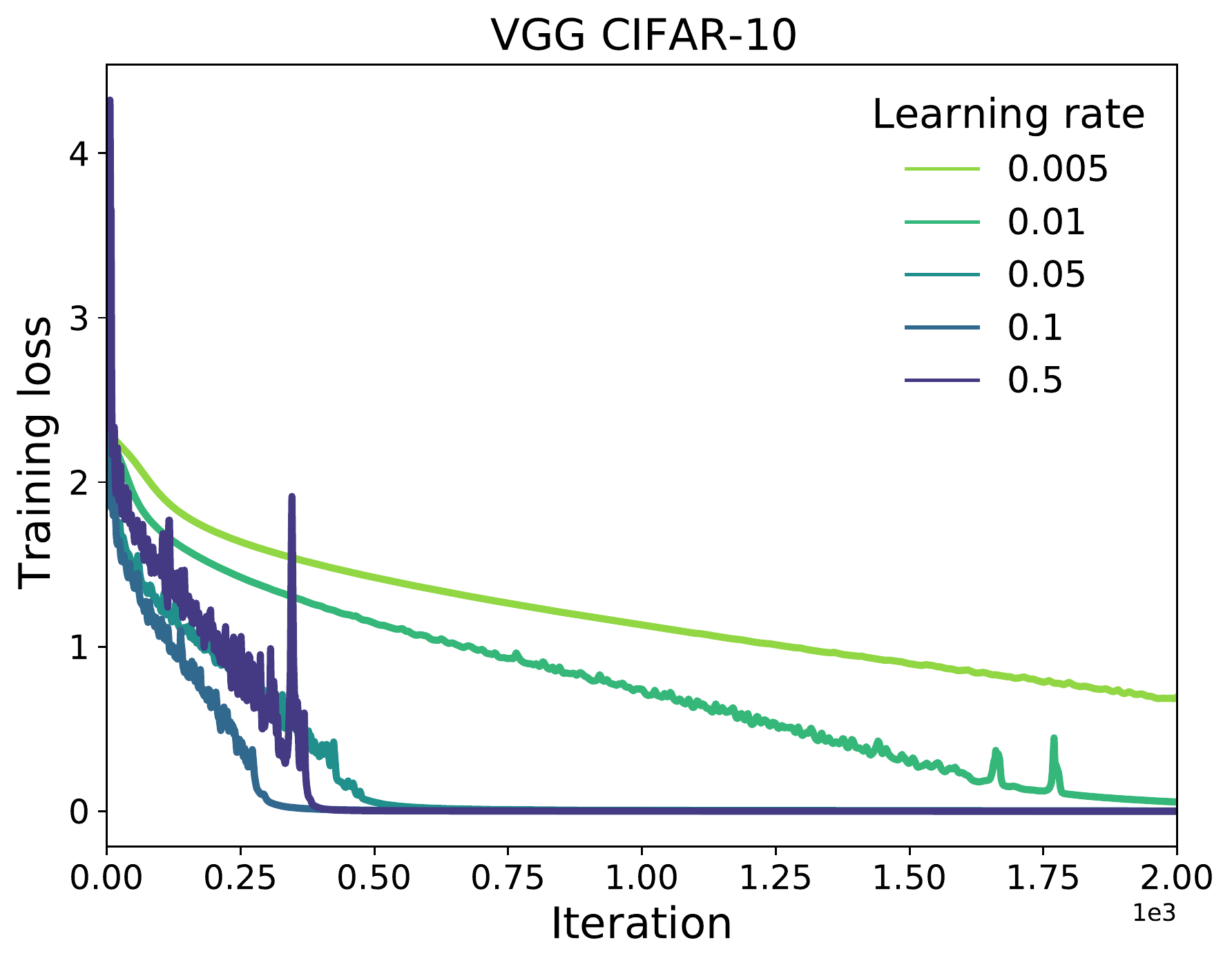}
 \includegraphics[width=0.34\columnwidth]{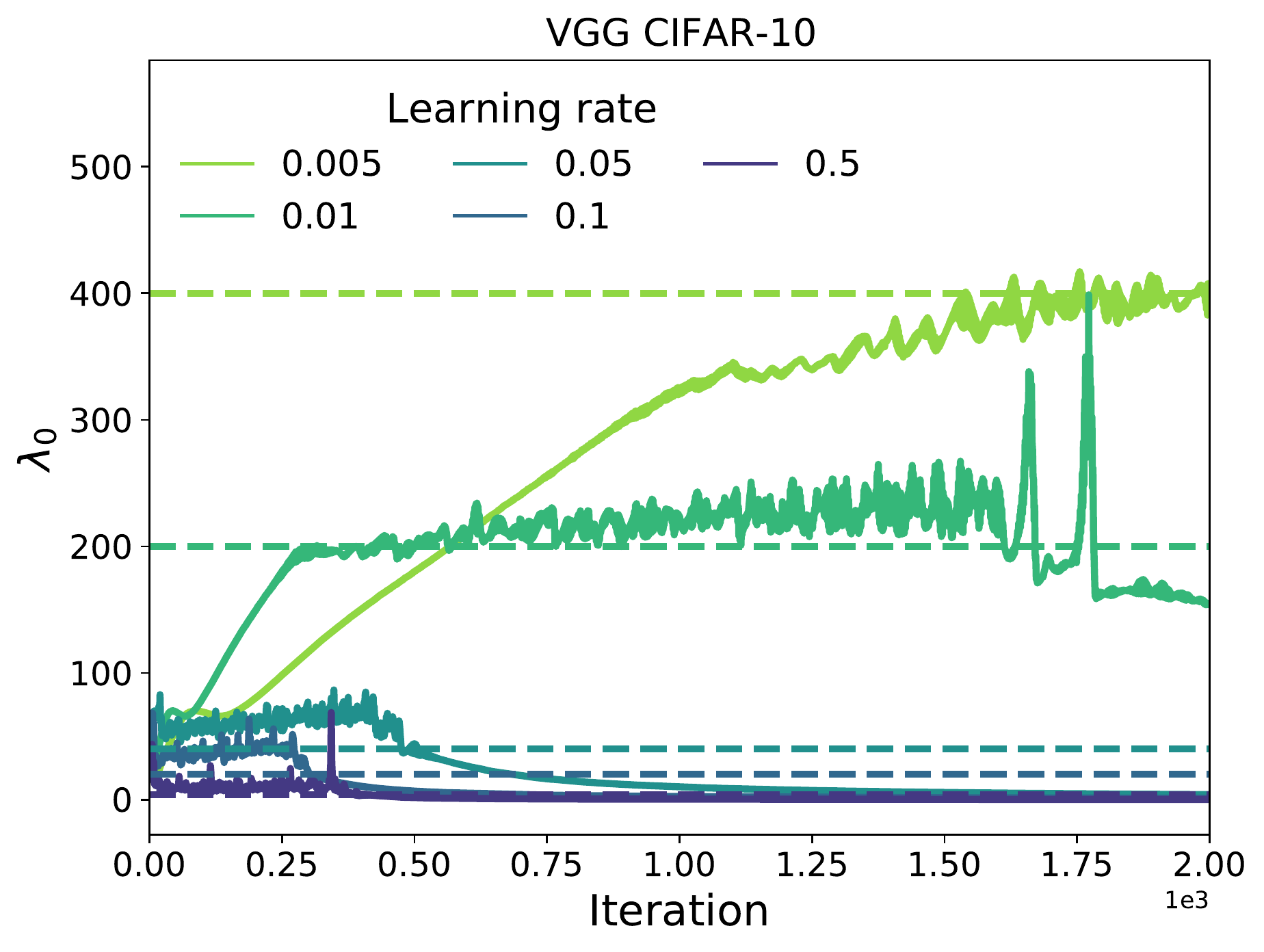}
 \includegraphics[width=0.32\columnwidth]{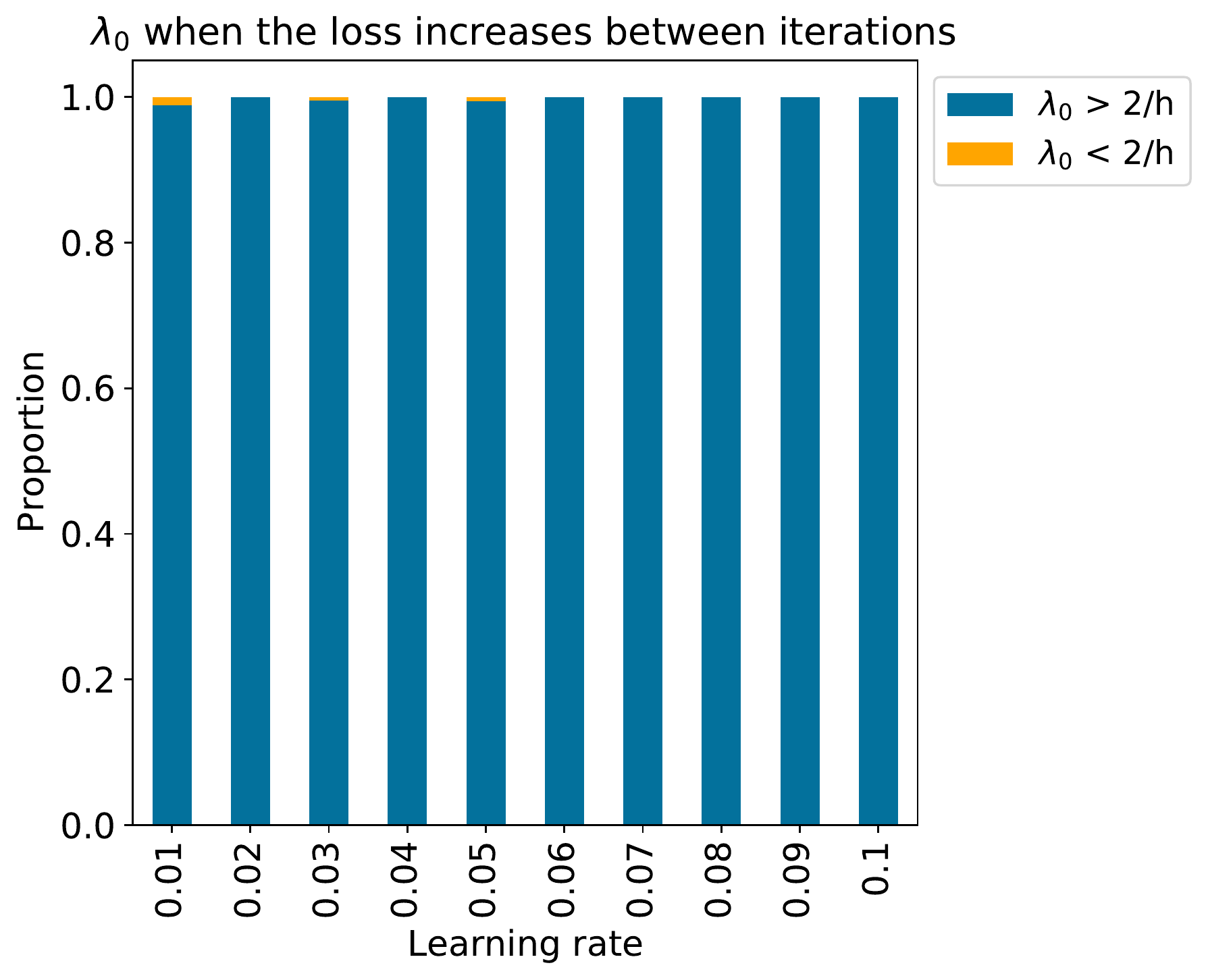}
\caption[Edge of stability on VGG networks on CIFAR-10. Full batch training.]{Edge of stability in neural networks~\citep{cohen2021gradient}: instability occurs when $\lambda_0 > 2/h$. }
\label{fig:reproduce_edge_of_stability}
\end{figure}

\rebuttalrone{\textbf{Continuous-time models of gradient descent at edge of stability.}} \rebuttalrone{To investigate if existing continuous time flows and the PF capture gradient descent behavior at the edge of stability we train a 5 layer MLP on the toy UCI Iris dataset~\citep{asuncion2007uci}; this simple setting allows for the computation of the full eigenspectrum of the Hessian. We show results in Figure~\ref{fig:model_of_gd_at_edge_of_stability}: 
the NGF and IGR flow have a larger error compared to the PF when predicting the parameters at the next gradient descent iteration in the edge of stability regime; the NGF and IGR flow  predict the loss will decrease, while the PF captures the loss increase observed when following gradient descent. 
As we remarked in Section~\ref{sec:motivation}, the NGF and the IGR flow do not capture instabilities when the eigenvalues of the Hessian are positive, which has been remarked to be largely the case for neural network training through empirical studies \citep{sagun2017empirical,ghorbani2019investigation,papyan2018full} and we observe here (Figure~\ref{fig:eigspectrum_small_mlp} in the Appendix).
We spend the rest of the section using the PF to understand and model edge of stability phenomena using a continuous time approach.}
\begin{figure}[tb]
 \includegraphics[width=0.33\columnwidth]{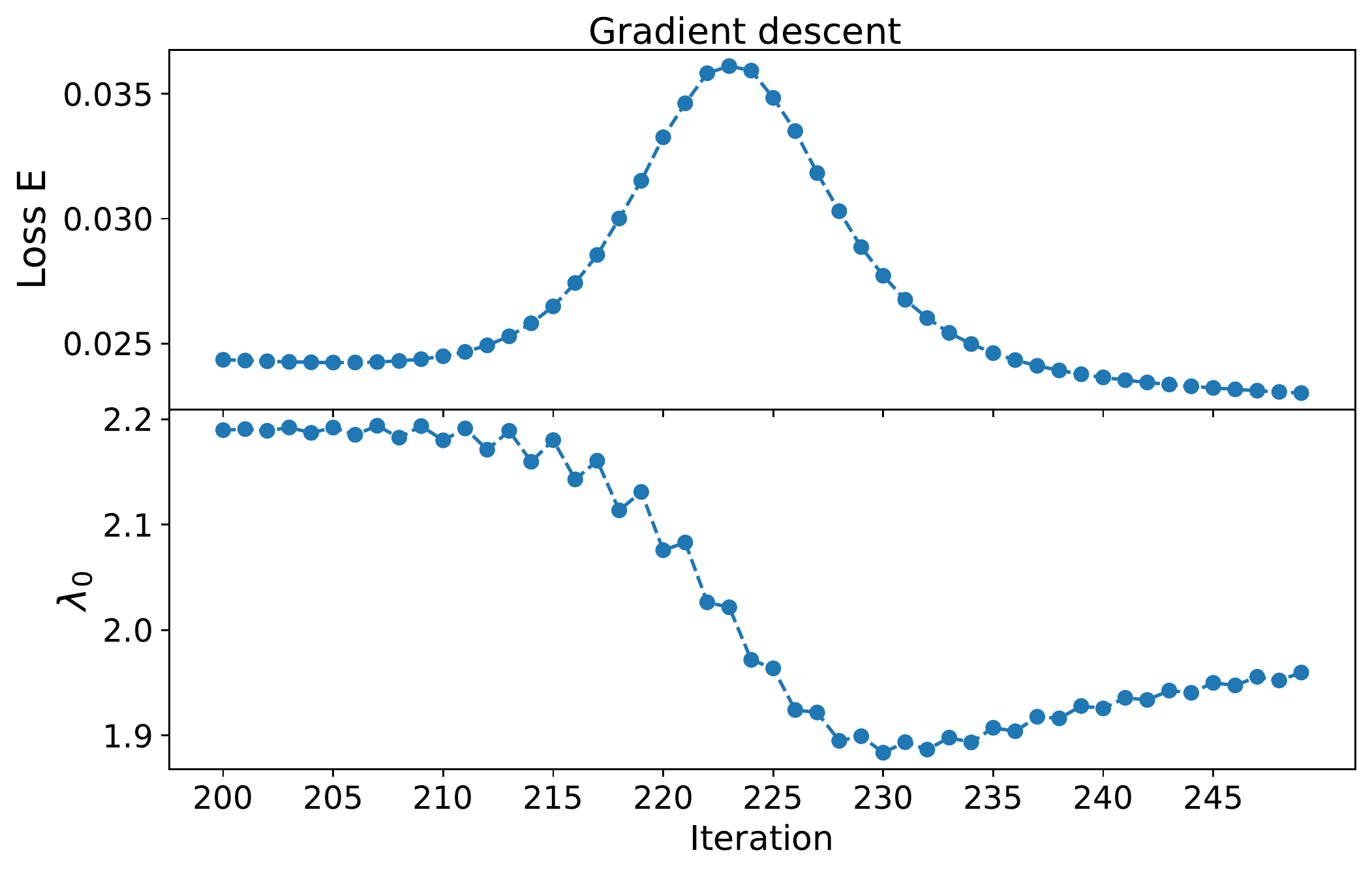}
 \includegraphics[width=0.32\columnwidth]{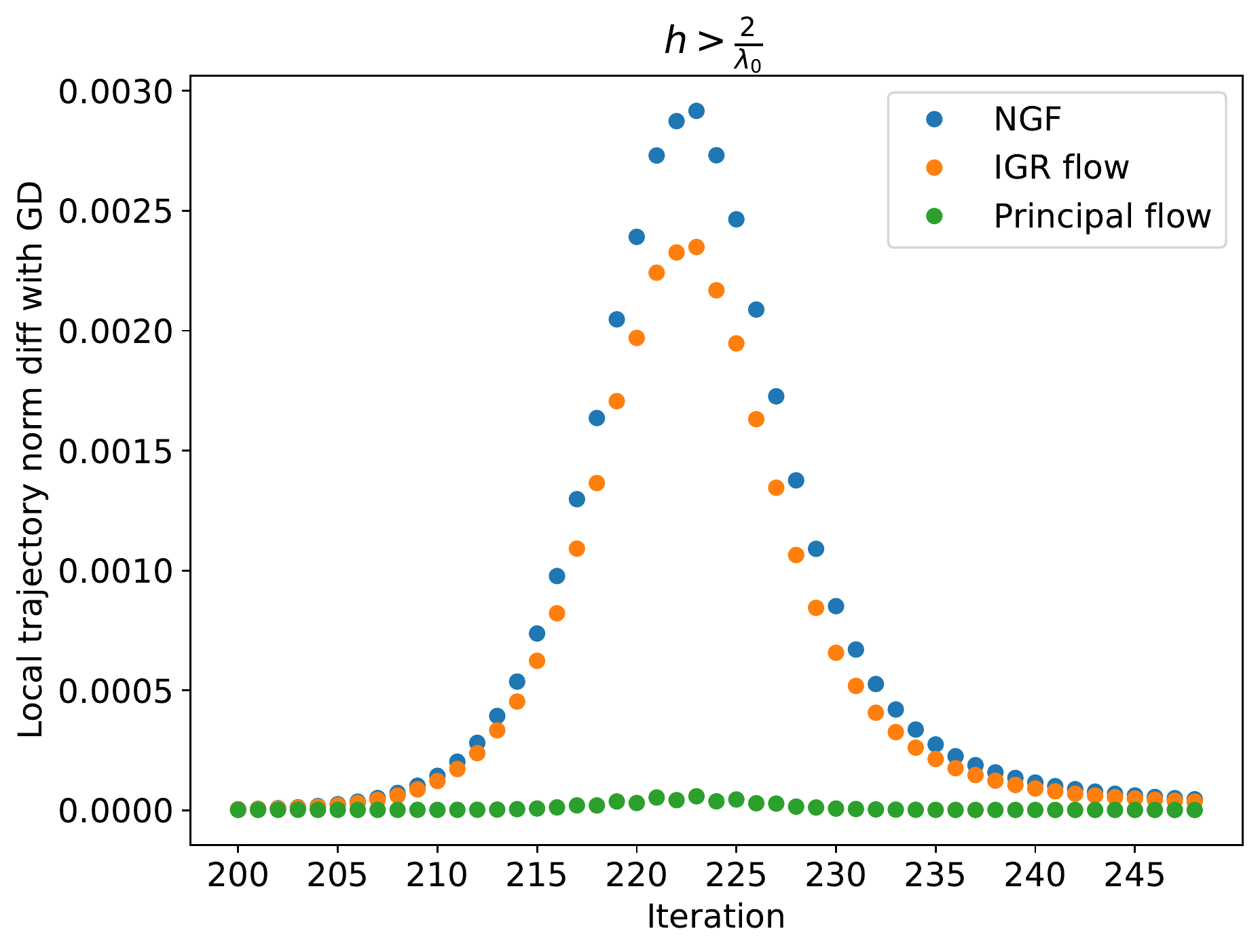}
 \includegraphics[width=0.34\columnwidth]{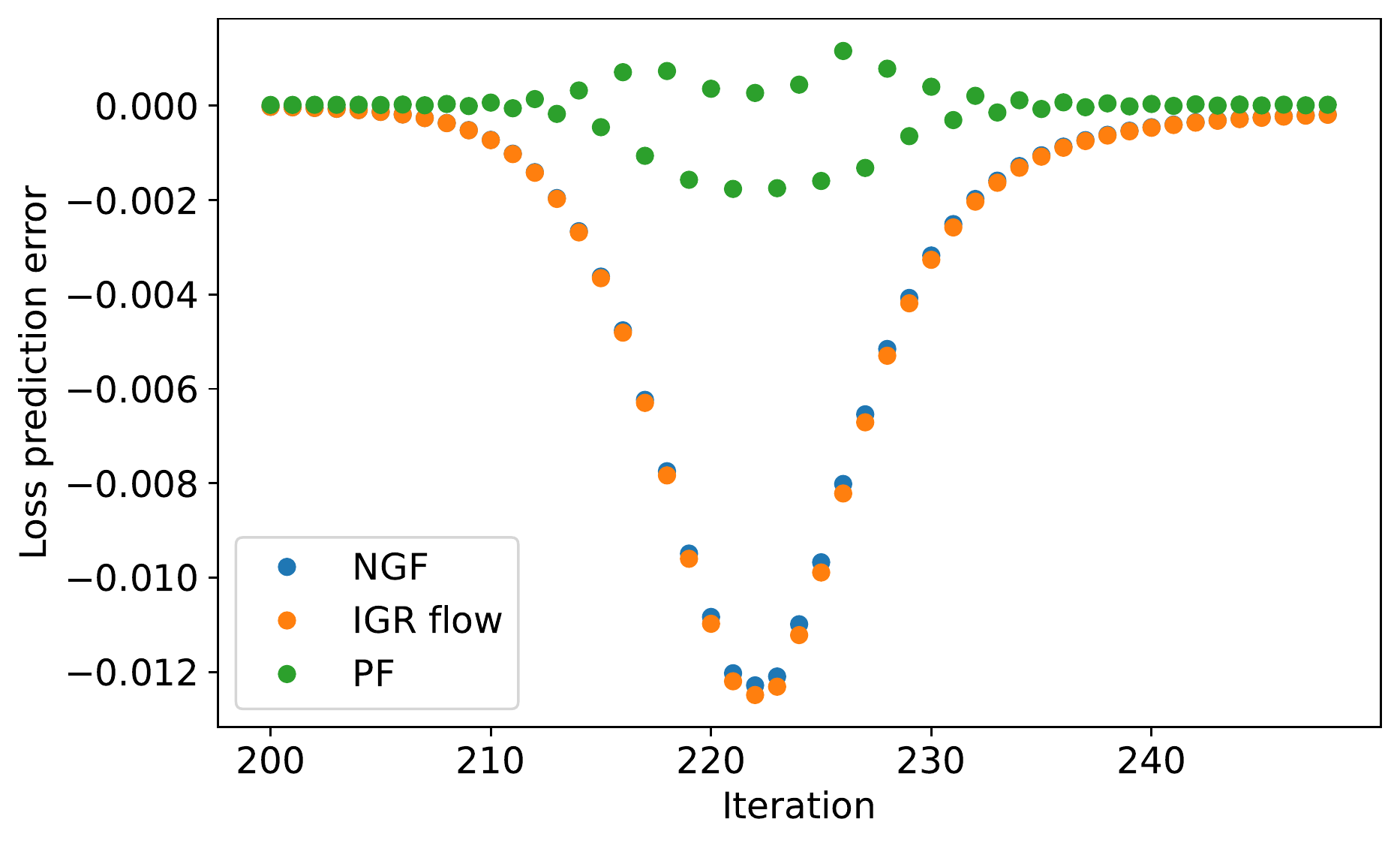}
\caption[Edge of stability on a small 5 layer MLP with 10 units per layer on the UCI Iris dataset. Full batch training. The learning rate used is 0.18.]{Comparing different continuous time models of gradient descent at the edge of stability area on a small 5 layer MLP, with 10 units per layer. We show the local parameter prediction error $\norm{\vtheta_t - \vtheta(h;\vtheta_{t-1})}$ for the NGF, IGR and PF flows (middle), as well as $E(\vtheta(h;\vtheta_{t-1})) - E(\vtheta_t)$ (right).}
\label{fig:model_of_gd_at_edge_of_stability}
\end{figure}

\textbf{Connection with the principal flow: stability coefficients.} The PF captures the key quantities 
observed in the edge of stability phenomenon: the eigenvalues of the Hessian $\lambda_i$ and the threshold $2/h$. These quantities appear in the PF via the stability coefficient $sc_i = \frac{\log(1-h\lambda_i)}{h\lambda_i}\nabla_{\vtheta} E^T \vu_i = \alpha_{PF}(\lambda_i h) \nabla_{\vtheta} E^T \vu_i$ of eigendirection $\vu_i$. Through the PF, by connecting the case analysis in Section~\ref{sec:principal_flow_eigendecom} with existing and new empirical observations, we can shed light on the edge of stability behavior in deep learning. 

\textit{First phase of training (progressive sharpening): $\lambda_0 < 2/h$}. This entails $Re[sc_i] = Re[\alpha_{PF}(h \lambda_i)] \le 0, \forall i$ (Real stable and complex stable cases of the analysis in Section~\ref{sec:principal_flow_eigendecom}). $\sign(\alpha_{NGF}) = \sign(\alpha_{PF}) =-1$ and following the PF minimises $E$ or its real part (Eq~\ref{eq:changes_in_e}). 
To understand the behavior of $\lambda_0$, we now have to make use of empirical observations about the behavior of the NGF early in the training of neural networks.
It has been empirically observed that in early areas of training, $\lambda_0$ increases here when following the NGF~\citep{cohen2021gradient}; we further show this in Figure~\ref{fig:mnist_gradient_flow} in the Appendix. Since in this part of training gradient descent follows closely the NGF, it exhibits similar behavior and $\lambda_0$ increases. We show this case in Figure~\ref{fig:early_training_eigen_loss}. 

\textit{Second phase of training (edge of stability) $\lambda_0 \ge 2/h$}. This entails $Re[sc_0(\vtheta)] = Re[\alpha_{PF}(h \lambda_i)] \ge 0$. (Unstable complex case of the analysis in Section~\ref{sec:principal_flow_eigendecom}). We can no longer say that following the PF minimizes E. $\sign(\alpha_{NGF}(h \lambda_0)) \neq \sign(Re[(\alpha_{PF}(h \lambda_0)])$, since $\alpha_{NGF}(h \lambda_0) = -1$ and $\sign(Re[(\alpha_{PF}(h \lambda_0)])>0$ meaning that in that direction gradient descent resembles the positive gradient flow $\dot{\vtheta} = \nabla_{\vtheta} E$ rather than the NGF. The positive gradient flow component can cause instabilities, and the strength of the instabilities depends on the stability coefficient $sc_0 = \alpha_{PF}(h \lambda_0) \nabla_{\vtheta} E^T \vu_0$. We show in Figures~\ref{fig:early_edge_eigen_loss} and \ref{fig:instabilities_resnet} how the behavior of the loss and $\lambda_0$ are affected by the behavior of the positive gradient flow when $\lambda_0 > 2/h$.

\begin{figure}[tb]
\begin{subfigure}[Early training]{
 \includegraphics[width=0.45\columnwidth]{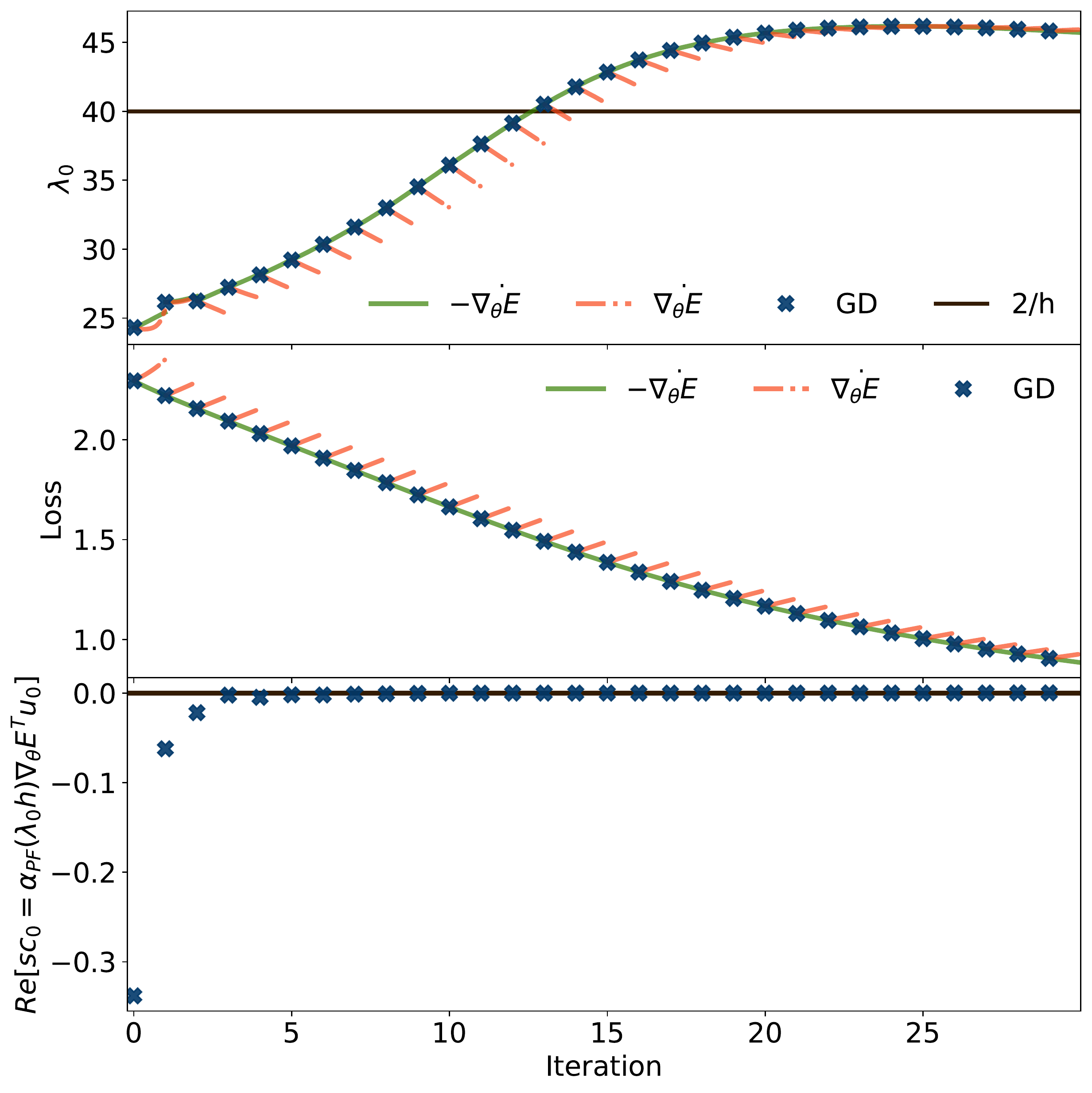}
 \label{fig:early_training_eigen_loss}
 }\end{subfigure}
\begin{subfigure}[Edge of stability.]{
 \includegraphics[width=0.45\columnwidth]{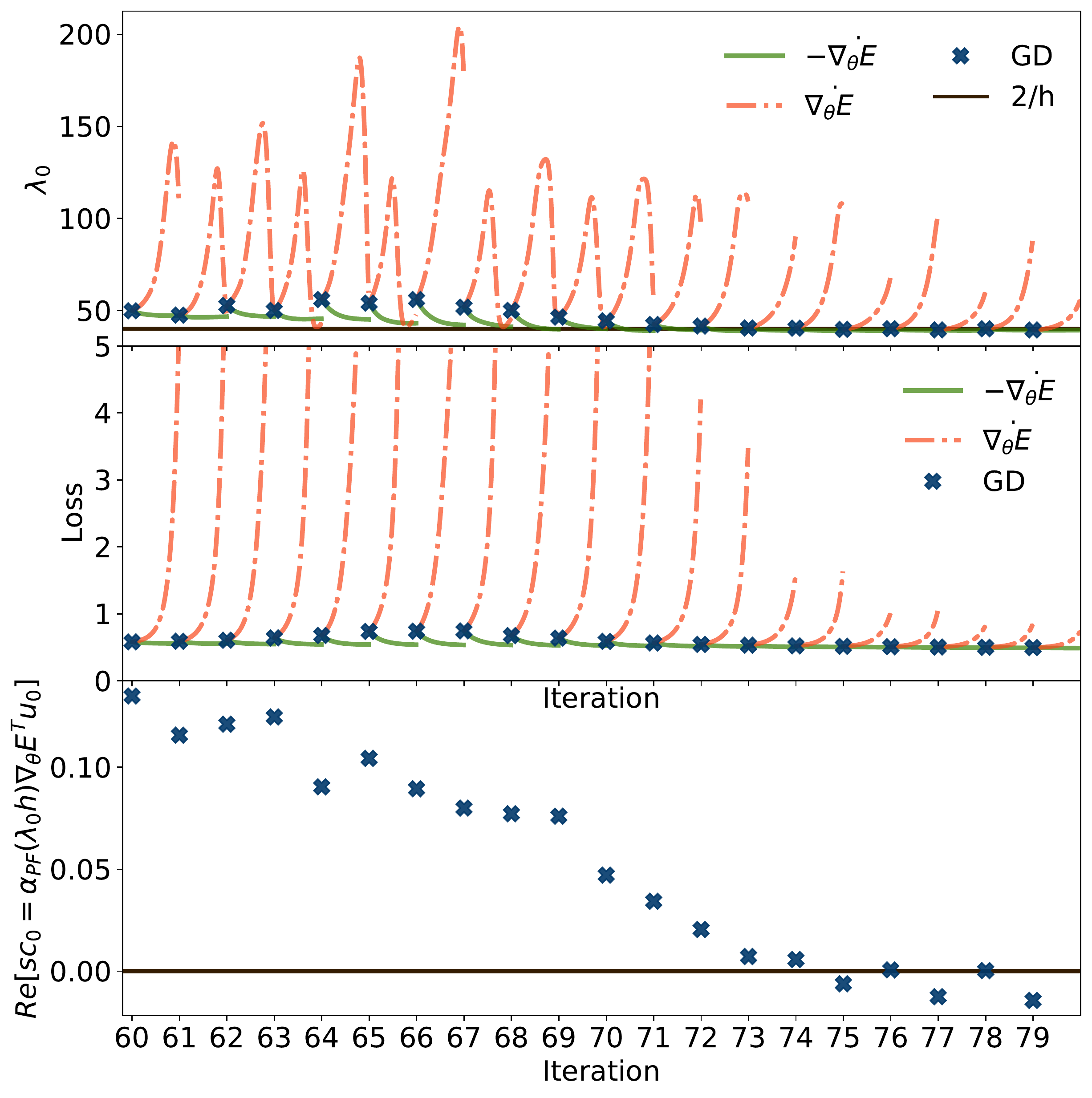}
 \label{fig:early_edge_eigen_loss}
}\end{subfigure}
 \caption[Edge of stability results. MNIST MLP with 4 layers. Learning rate $0.05$. For each model we approximate the NGF and the positive gradient flow initialized at the current gradient descent parameters and measure the value od the loss function and $\lambda_0$.]{Understanding the edge of stability results using the PF on a 4 layer MLP: we plot the behavior of the NGF $\dot{\vtheta} = - \nabla_{\vtheta} E$ and the positive gradient flow $\dot{\vtheta} = \nabla_{\vtheta} E$ initialized at each gradient descent iteration parameters, and see that the behavior of gradient descent is connected to the behavior of the respective flow through the stability coefficient. \rebuttalrone{\textit{Figure~\ref{fig:early_training_eigen_loss} shows that even when $\lambda_0 > 2/h$, if the real part of the stability coefficient $sc_0$ is negative or close to 0, there are no instabilities in the loss and the eigenvalue $\lambda_0$ keeps increasing, as it does when following the NGF in that region.}}}
\label{fig:edge_of_stability_results_local}
\end{figure}

\textbf{More than $\lambda_0$: the importance of stability coefficients}. 
While the sign of the real part of the stability coefficient $sc_0$ is determined by $\lambda_0$, its magnitude is modulated by the dot product $\nabla_{\vtheta} E^T \vu_0$, since $sc_0 = \alpha_{PF}(h \lambda_0) \nabla_{\vtheta} E^T \vu_0$.
The magnitude of $\nabla_{\vtheta} E^T \vu_0$ plays an important role, 
since if $\lambda_0$ is the only eigenvalue greater than $2/h$ training is stable if $\nabla_{\vtheta}E ^T \vu_0 = 0$, as we observe in Figure~\ref{fig:edge_of_stability_results_local}. 
\textit{To understand instabilities, we have to look at stability coefficients, not only eigenvalues.}
 We show in Figure~\ref{fig:instabilities_short} how the instabilities in training can be related with the stability coefficient $sc_0$: the increases in loss occur when the corresponding $Re[sc_0]$ is positive and large. 
 In Figure~\ref{fig:instabilities_resnet} we show results with the behavior of $\lambda_0$: $\lambda_0$ increases or decreases based on the behavior of the corresponding flow and the strength of the stability coefficient and that gets reflected in instabilities in the loss function;
 specifically when $\lambda_0 > 2/h$, we use the positive gradient flow and see how the strength of its fluctuations affect the changes both in the loss value and $\lambda_0$ of gradient descent.
 We show additional results in Figures~\ref{fig:edge_of_stability_results_lambda} and~\ref{fig:edge_of_stability_results_cifar_resnet} in the Appendix.

\begin{figure}[tb]
\centering
 \includegraphics[width=0.45\columnwidth]{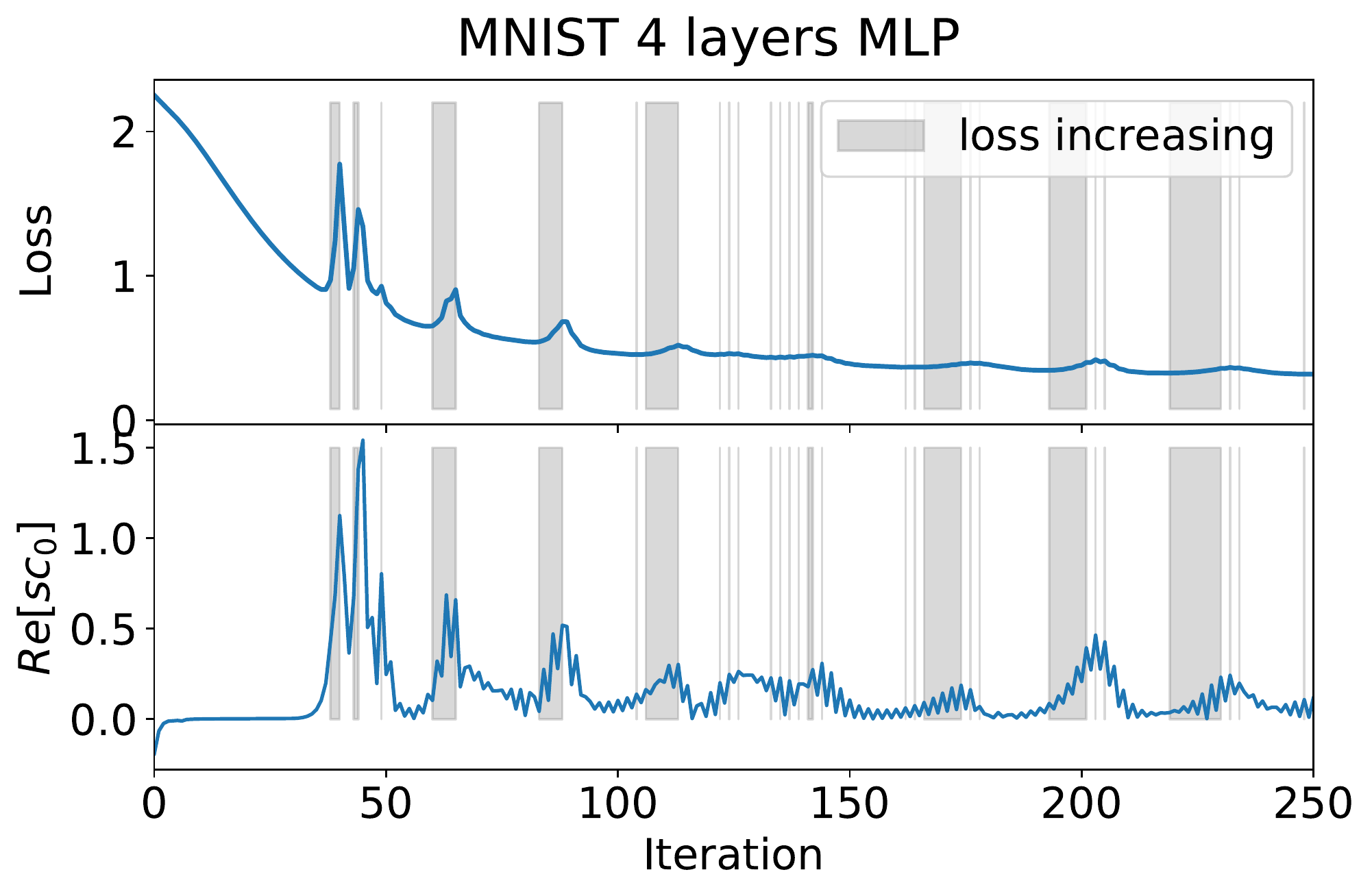}
 \hspace{2em}
 \includegraphics[width=0.45\columnwidth]{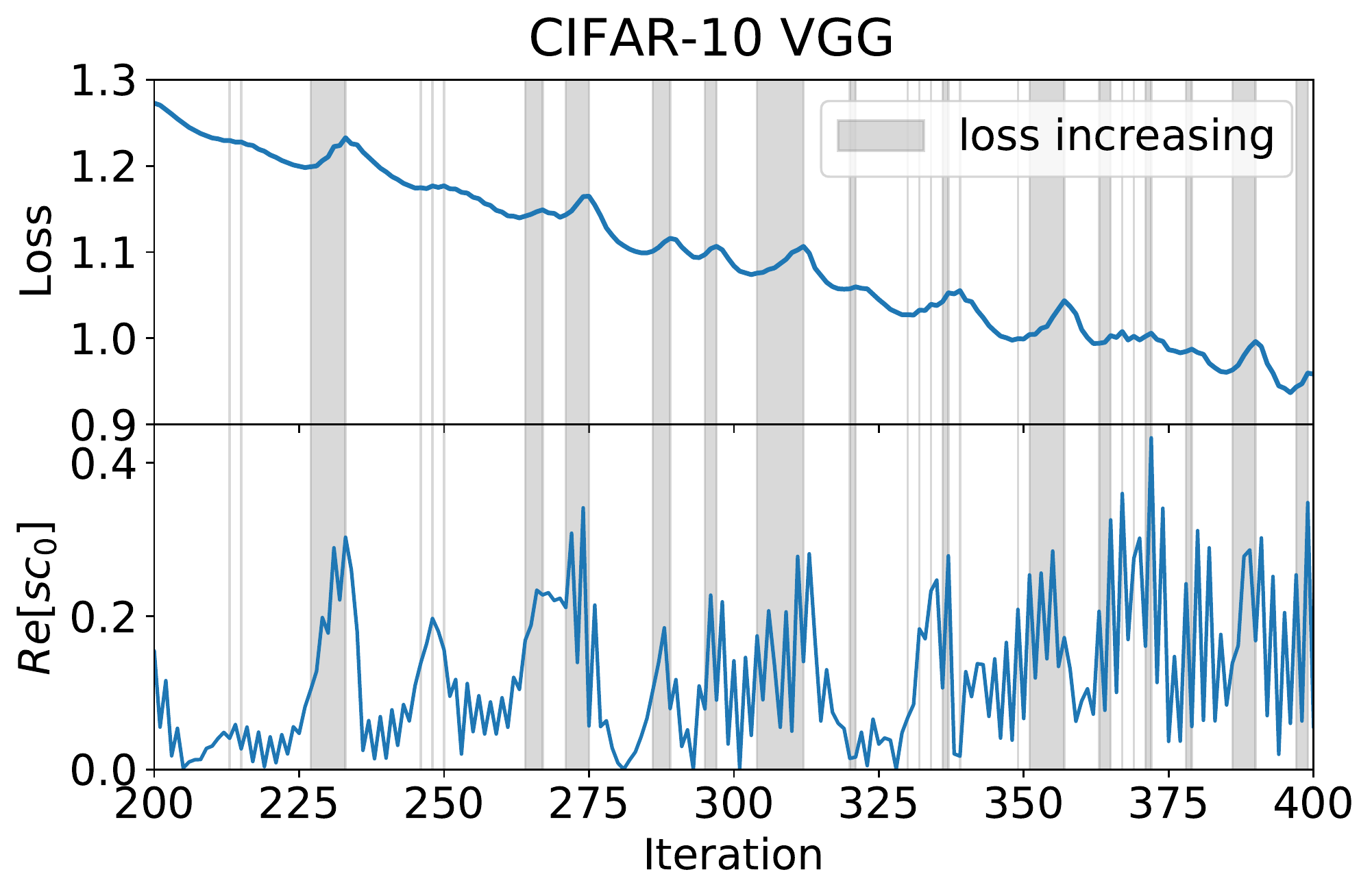}
\caption[Using the PF and stability coefficient to understand instabilities in deep learning. MNIST results use a 4 layer MLP and a $0.05$ learning rate. CIFAR-10 MLP use a VGG network and a $0.02$ learning rate.]{\rebuttalrone{The loss function and stability coefficients}: areas where the loss increases correspond to areas where the $sc_0$ is large. The highlighted areas correspond to regions where the loss increases.}
\label{fig:instabilities_short}
\end{figure}

\begin{figure}[tb]
\centering
 \includegraphics[width=0.9\columnwidth]{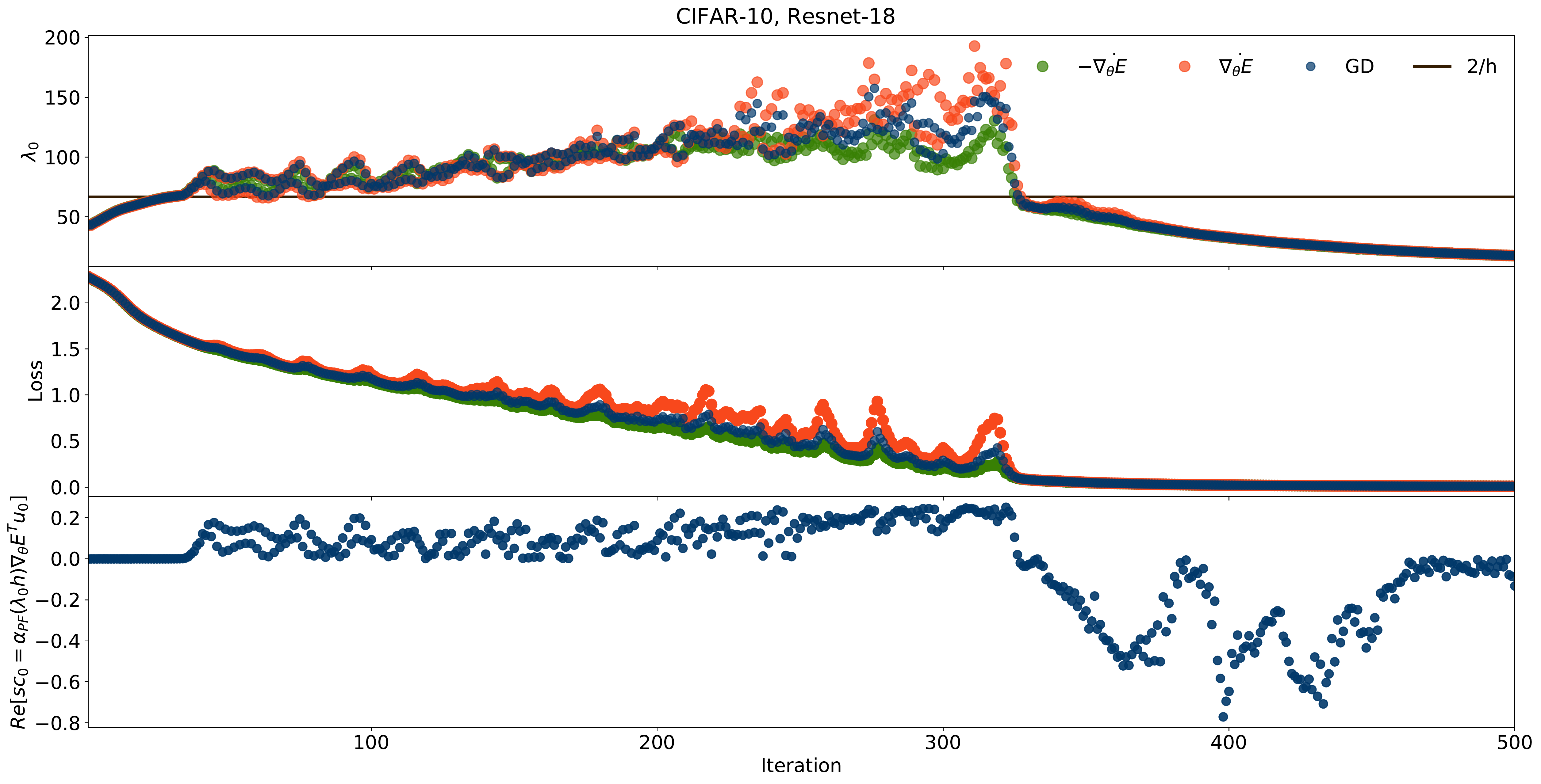}
\caption[Understanding the behavior of the loss through the PF and the behavior of $\lambda_0$. Using a Resnet-18 trained on CIFAR-10 using a learning rate of $0.03$. For each model we approximate the NGF and the positive gradient flow initialized at the current gradient descent parameters and measure the value of the loss function and $\lambda_0$.]{\rebuttalrone{Loss instabilities, $\lambda_0$ and stabilitiy coefficients for CIFAR-10.} Together with the behavior of gradient descent, we plot the behavior of the NGF and positive gradient flow initialized at $\vtheta_t$ and simulated for time $h$ for each iteration $t$. \rebuttalrone{The analysis we performed based on the PF suggests that when $Re[sc_0] > 0$ and large we should expected gradient descent to exhibit behaviors close to those of the positive gradient flow. What we observe empirically is that increases in loss value of gradient descent are proportional to the increase of the positive gradient flow in that area (can be seen best between iterations 200 and 350); the same behavior can be seen in relation to the eigenvalue $\lambda_0$. }}
\label{fig:instabilities_resnet}
\end{figure}

\textbf{Is one eigendirection enough to cause instability?} 
One question that arises from the PF is whether the leading eigendirection $\vu_0$ can be sufficient to cause instabilities, especially in the context of deep networks with millions of parameters. To assess this we train a model with gradient descent until it reaches the edge of stability ($\lambda_0 \approx 2/h$), after which we simulate the continuous flow $\dot{\vtheta} = \nabla_{\vtheta}E^T \vu_0 \vu_0 + \sum_{i=1}^{D-1} -\nabla_{\vtheta}E^T \vu_i \vu_i$. 
The coefficients of the modified vector field of this flow are negative for all eigendirections except from $\vu_0$, which is positive; this is also the case for the PF when $\lambda_0$ is the only eigenvalue greater than $2/h$.
In Figure~\ref{fig:instabilities_change_learning_rate_loss} we empirically show that a positive coefficient for $\vu_0$ can be responsible for an increase in loss value and a significant change in $\lambda_0$ in neural network training.

\begin{figure}[tb]
\centering
 \includegraphics[width=0.4\columnwidth]{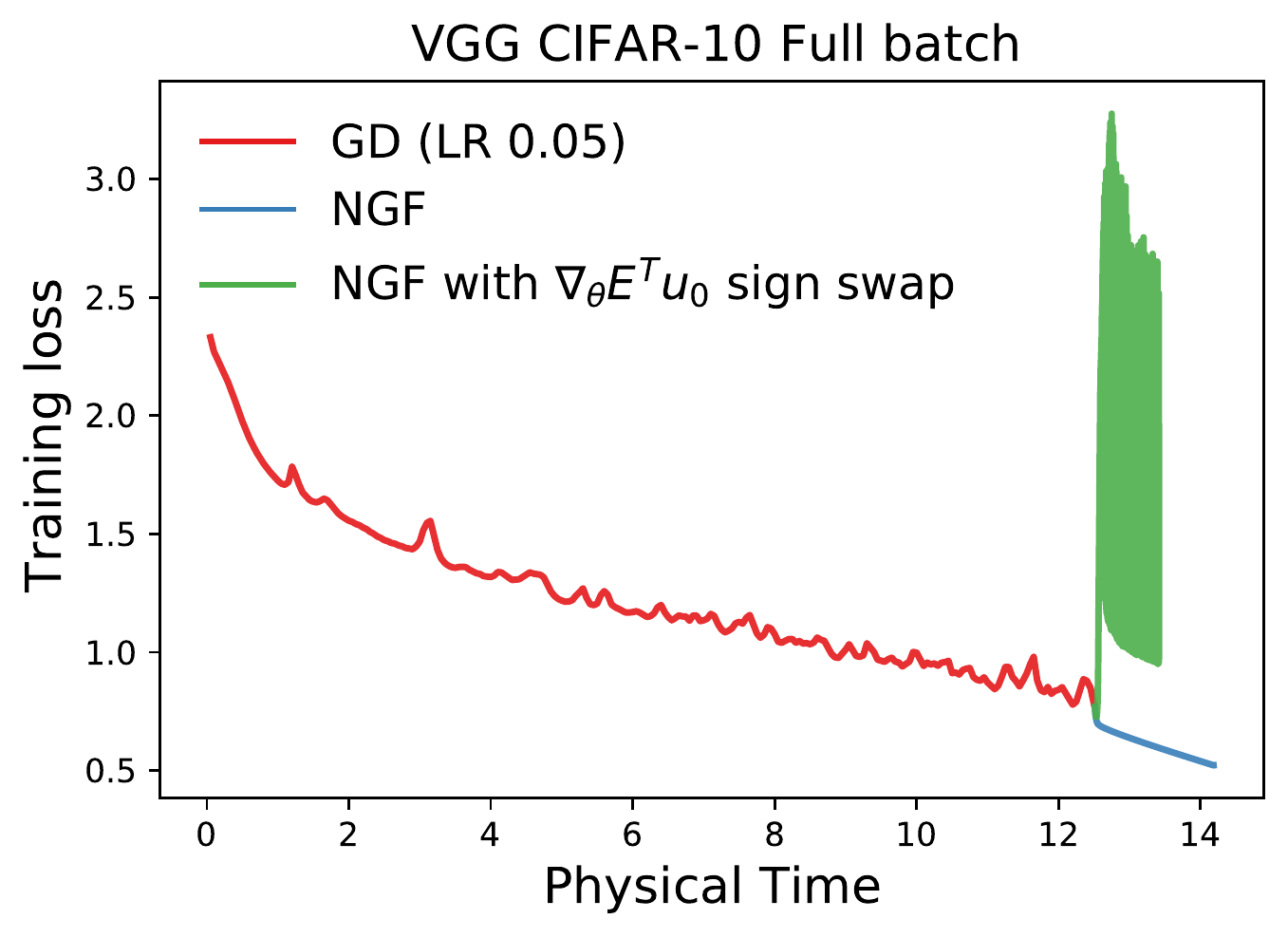}
 \hspace{3em}
 \includegraphics[width=0.4\columnwidth]{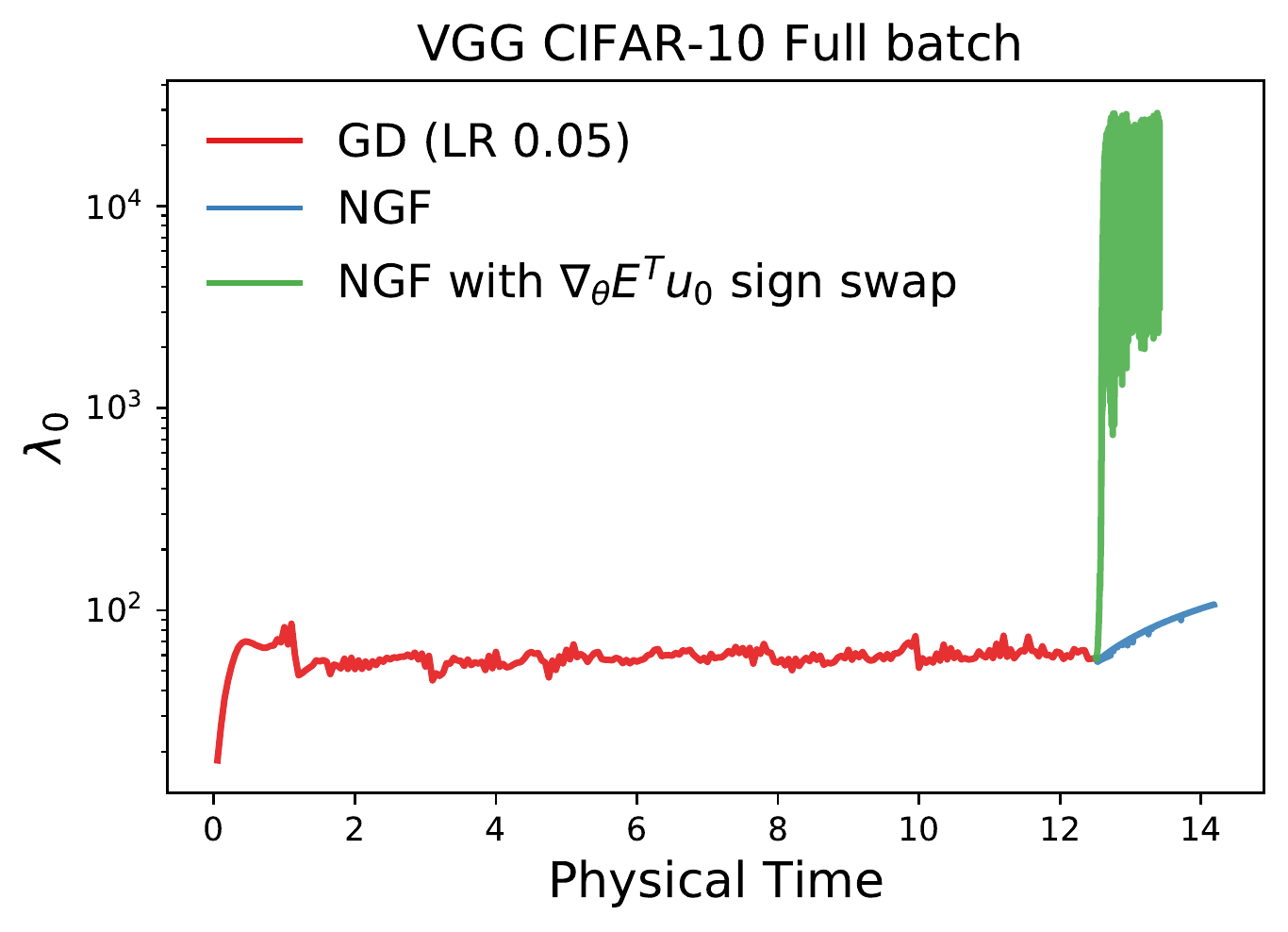}
\caption[Assessing whether 1 dimension is enough to cause instability in a continuous time flow. We train a model with gradient descent until it reaches the edge of stability ($\lambda_0 \approx 2/h$), after which we use a approximate the continuous flow $\dot{\vtheta} = \nabla_{\vtheta}E^T \vu_0 \vu_0 + \sum_{i=1}^{D-1} -\nabla_{\vtheta}E^T \vu_i \vu_i$. The model is a VGG model trained on CIFAR-10 with learning rate $0.05$ and after the edge of stability is reached the flow is approximated using Euler steps of size $5e-5$.]{One eigendirection is sufficient to lead to instabilities. To create a situation similar to that of the PF, we construct a flow given by the NGF in all eigendirections but $\vu_0$; in the direction of $\vu_0$, we change the sign of the flow. This leads to the flow $\dot{\vtheta} = \nabla_{\vtheta}E^T \vu_0 \vu_0 + \sum_{i=1}^{D-1} -\nabla_{\vtheta}E^T \vu_i \vu_i$. We show this flow can be very unstable when initialised in an edge of stability area.}
\label{fig:instabilities_change_learning_rate_loss}
\end{figure}

\textbf{Decreasing the learning rate.} \citet{cohen2021gradient} show that if the edge of stability behavior is reached and the learning rate is decreased, the training stabilizes and $\lambda_0$ keeps increasing (Figure~\ref{fig:instabilities_change_learning_rate_discrete} in the Appendix). The PF tells us that decreasing the learning rate entails going from $Re[sc_0] \ge 0$  to $Re[sc_0] \le 0$ since $\lambda_0 <2/h$ after the learning rate change. Since all stability coefficients are now negative, this reduces instability. The increase in $\lambda_0$ is likely due to the behavior of the NGF in that area (as can be seen in Figure~\ref{fig:instabilities_change_learning_rate_loss} when changing from gradient descent training to the NGF in an edge of stability area leads to an increase of $\lambda_0$).

\textbf{The behavior of $\nabla_{\vtheta}E^T \vu_0$.} The PF also allows us to explain the unstable behavior of $\nabla_{\vtheta}E^T \vu_0$ around edge of stability areas. 
As done in Section~\ref{sec:nns_principal_flow}, we assume that $\lambda_i, \vu_i$ do not change substantially between iterations and write $\dot{\nabla_{\vtheta}E^T \vu_i} = \frac{\log(1 - h \lambda_i)}{h} \nabla_{\vtheta} E^T \vu_i$ under the PF, with solution $(\nabla_{\vtheta}E^T \vu_i)(t) = (\nabla_{\vtheta}E^T \vu_i)(0) e^{\frac{\log(1 - h \lambda_i)}{h}t}$. 
This solution has different behavior depending on the value of $\lambda_0$ relative to $2/h$: decreasing below $2/h$ and increasing above $2/h$. 
We show this theoretically predicted behavior in Figure~\ref{fig:instability_largest_dot_prod}, 
alongside empirical behavior showcasing the fluctuation of $\nabla_{\vtheta}E^T \vu_0$ in the edge of stability area, which confirms the theoretical prediction. We also compute the prediction error of the proposed flow and show it can capture the dynamics of $\nabla_{\vtheta}E^T \vu_0$ closely in this setting. We present a discrete time argument for this observation in Section~\ref{sec:changes_in_dot_prod_discrete}. \rebuttalrthree{We note that the stable behavior early in training together with the oscillatory behavior of $\nabla_{\vtheta}E^T \vu_0$ in the edge of stability area which we predict and observe can explain the results of~\citet{cohen2021gradient} on the behavior of $\vtheta^T \vu_0$, since $\vtheta$ accumulates changes given by gradient updates.}

\begin{figure}[thb]
\centering
 \includegraphics[width=0.3\columnwidth]{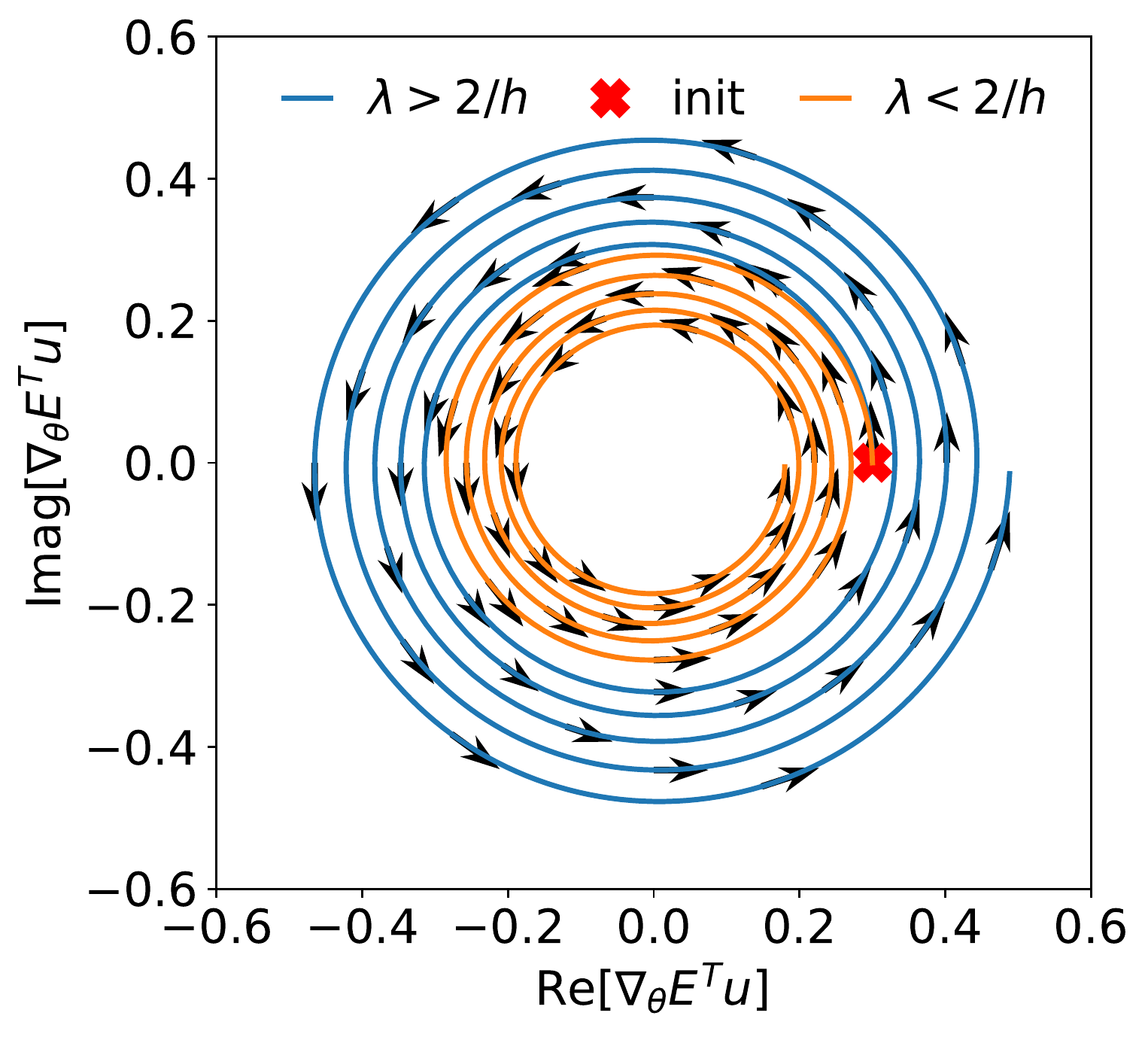}
  \includegraphics[width=0.3\columnwidth]{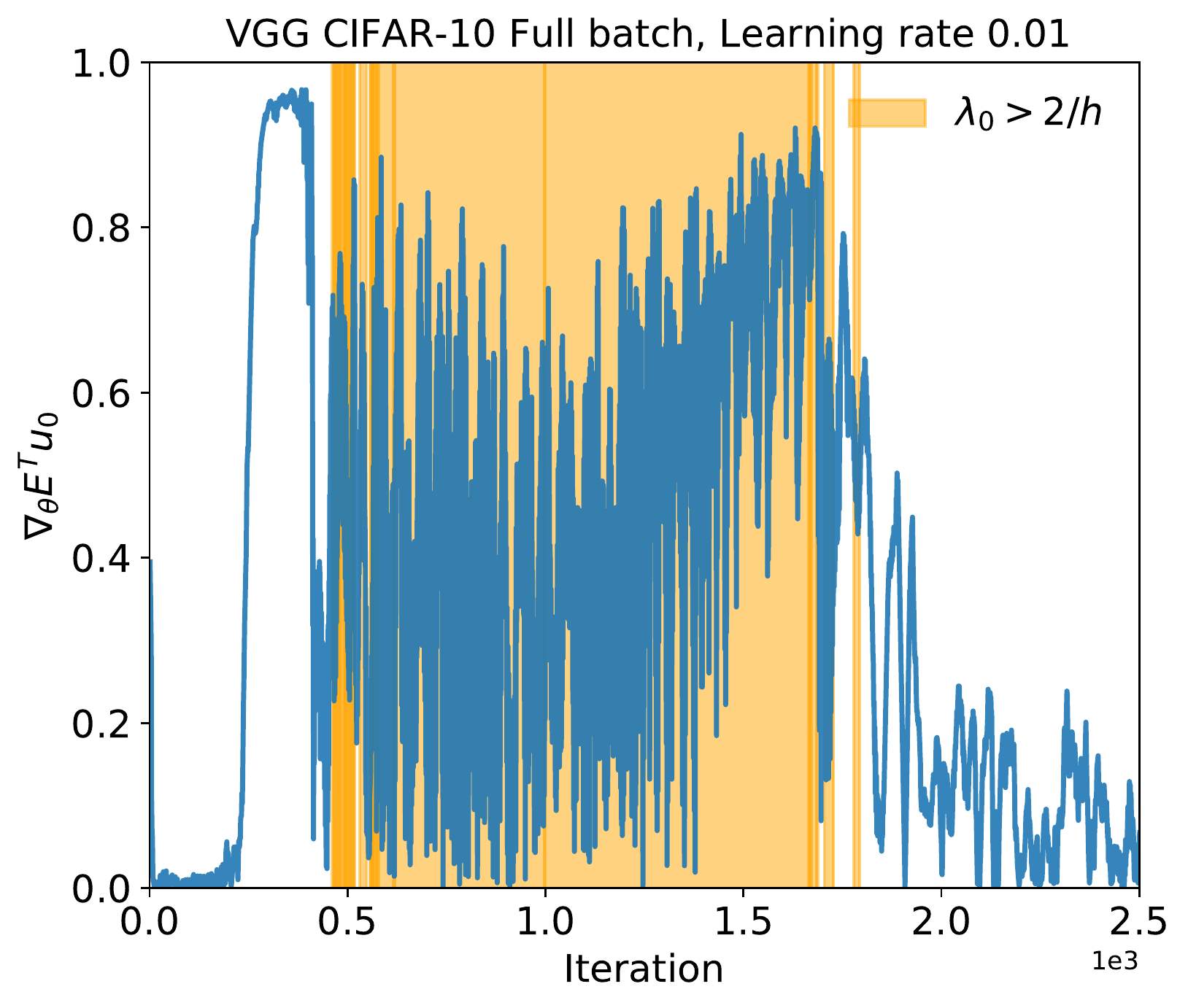}
  \includegraphics[width=0.3\columnwidth]{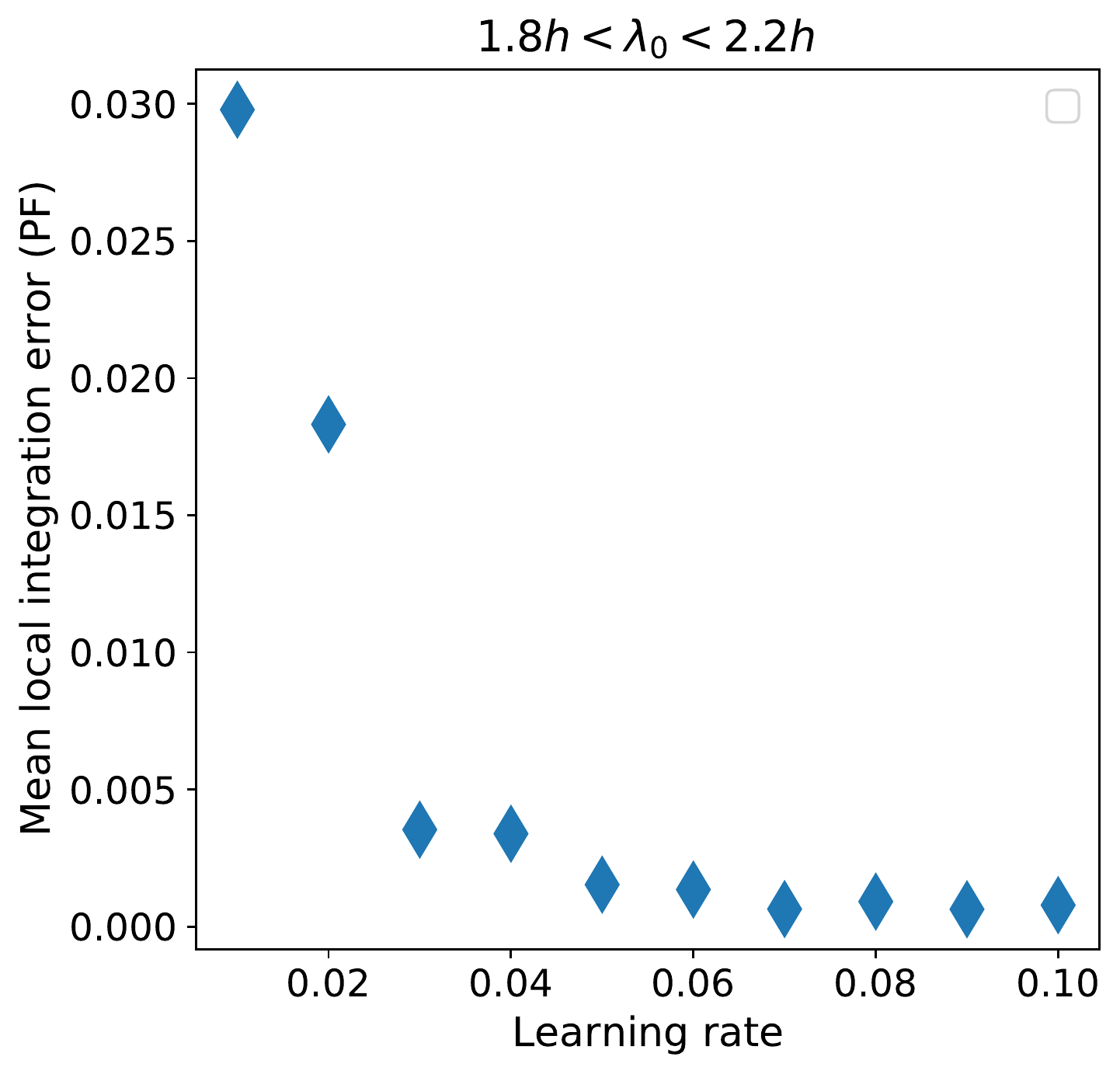}
\caption[The unstable dynamics of $\nabla_{\vtheta}E^T u$ in the edge of stability area. The model results are obtained from a VGG model trained on CIFAR-10 with a learning rate of $0.01$.]{\rebuttalrone{Predicting the unstable dynamics of $\nabla_{\vtheta}E^T u$  in the edge of stability area ($\lambda \approx 2/h)$ using the PF}. Left: the predicted behavior of $\nabla_{\vtheta}E^T u$ under $\dot{\nabla_{\vtheta}E^T \vu_i} = \frac{\log(1 - h \lambda_i)}{h} \nabla_{\vtheta} E^T \vu_i$, with an inflection point at $\lambda = 2/h$. Middle: empirical behavior of $\nabla_{\vtheta}E^T u$ for a model shows instabilities in the edge of stability area (highlighted). Right: the approximation made to derive the flow is suitable around $\lambda \approx 2/h$.}
\label{fig:instability_largest_dot_prod}
\end{figure}

\rebuttalrthree{\textbf{Why not more instability?} To determine why there isn't more instability in the edge of stability area we have to consider that neural networks are not quadratic, which has two effects. Firstly, when following the PF the landscape changes slightly locally; this leads to changes in stability coefficients and thus the behavior of gradient descent as we have consistently seen in the experiments in this section. Secondly, non-principal  terms can have an effect; while we do not know all non-principal terms in Section~\ref{sec:non_principal_ap} in the Appendix we provide a justification for why the non-principal term we do know (Eq~\ref{eq:pf_with_non_principal}) can have a stabilizing effect by inducing a regularisation pressure to minimise $\lambda_i (\nabla_{\vtheta} E^T \vu_i)^2$ in certain parts of the training landscape.}

In this section we have shown the PF closely  predicts the behavior of gradient descent in neural network training. \rebuttalrone{This has led to additional insights, including the importance of stability coefficients in determining instabilities in gradient descent (Figures~\ref{fig:edge_of_stability_results_local}, \ref{fig:instabilities_short}, \ref{fig:instabilities_resnet}), causally showing one eigendirection is sufficient to cause instability (Figure~\ref{fig:instabilities_change_learning_rate_loss}) and change and being able to closely predict the behavior of the dot product between the gradient and the largest eigenvector (Figure~\ref{fig:instability_largest_dot_prod}). This evidence suggests that the PF captures significant aspects of the behavior of gradient descent in deep learning; this is likely due to the specific structure of neural network models.}
While we take a continuous time approach, a discrete time approach can be used to motivate some of our observations (Section \ref{sec:discrete}); this is complementary to our approach but nonetheless related, since it also does not account for higher order derivatives of the loss and further suggests the strength of a quadratic approximation of the loss in the case of neural networks, as observed by \citet{cohen2021gradient}.

\section{Stabilizing training by adjusting discretization drift}
\label{sec:stabilising_training}

The PF allows us to understand not only how gradient descent differs from the trajectory given by the NGF, but also when they follow each other very closely. Understanding when gradient descent behaves like the NGF flow \rebuttalrone{reveals when the existing analyses of gradient descent using the NGF discussed in Section \ref{sec:motivation} are valid}. \rebuttalrone{It also has practical implications, since in areas where gradient descent follows the NGF closely} training can be sped up by increasing the learning rate. Prior works have empirically observed that gradient descent follows the NGF early in neural network training~\citep{cohen2021gradient} and this observation can be used to explain why decaying learning rates ~\citep{cosine_lr_decay} or learning rate warm up \citep{he2019bag} are successful when training neural networks: having a high learning rate in areas where the drift is small will not cause instabilities and can speed up training while decaying the learning rate avoids instabilities later in training when the drift is larger. 

\subsection{$\nabla_{\vtheta}^2 E \nabla_{\vtheta} E$ determines discretization drift}

\label{sec:h_g}

In previous sections we have seen that the Hessian plays an important role in defining the PF \rebuttalrone{ and in training instabilities. We now want to quantify the difference between the NGF and the PF in order to understand when the NGF can be used as a model of gradient descent. We find that:} 
\begin{remark} In a region of the space where $\nabla_{\vtheta}^2 E \nabla_{\vtheta} E= \mathbf{0}$ the PF is the same as the NGF.
\end{remark}

 To see why, we can expand
\begin{align}
\nabla_{\vtheta}^2 E \nabla_{\vtheta} E= \sum_{i=0}^{D-1} \lambda_i \nabla_{\vtheta} E ^T \vu_i \vu_i.
\label{eq:expansion}
\end{align}
If $\nabla_{\vtheta}^2 E \nabla_{\vtheta} E = \mathbf{0}$ we have that $\lambda_j \nabla_{\vtheta} E ^T \vu_j = 0, \forall j \in \{1, .., D\}$, thus either $\lambda_j = 0$ leading to $\alpha_{NGF}(h\lambda_j) = \alpha_{PF}(h\lambda_j) = -1$ or $\nabla_{\vtheta} E ^T \vu_j = 0$. 
Then $\dot{\vtheta} =  \sum_{i=0}^{D-1} \alpha_{PF}(h\lambda_i)(\nabla_{\vtheta} E^T \vu_i)  \vu_i = \sum_{i=0}^{D-1} \alpha_{NGF}(h\lambda_i) (\nabla_{\vtheta} E^T \vu_i)  \vu_i$.

\rebuttalrone{Thus comparing the PF  with the NGF reveals an important quantity: $\nabla_{\vtheta}^2 E \nabla_{\vtheta} E$. Further investigating this quantity reveals it has a connection with the total drift, since:}

\begin{theorem} The discretization drift (error between gradient descent and the NGF) after 1 iteration ${\vtheta_{t} = \vtheta_t - h \nabla_{\vtheta} E(\vtheta_{t-1})}$ is $\frac{h^2}{2} \nabla_{\vtheta}^2 E(\vtheta')  \nabla_{\vtheta} E(\vtheta')$ for a set of parameters $\vtheta'$ in the neighborhood of $\vtheta_{t-1}$.
\label{thm:total_drift}
\end{theorem}

This follows from the Taylor reminder theorem in mean value form (proof in Section~\ref{sec:proofs_total_per_iteration_drift}). This leads to:

\begin{corollary} In a region of space where $\nabla_{\vtheta}^2 E \nabla_{\vtheta} E=\mathbf{0}$ gradient descent follows the NGF.
\end{corollary}

Thus \rebuttalrone{the PF revealed $\nabla_{\vtheta}^2 E \nabla_{\vtheta} E$  as a core quantity in the discretisation drift of gradient descent. To further see the connection between with the PF consider that}
$\norm{\nabla_{\vtheta}^2 E \nabla_{\vtheta} E}^2= \norm{ \sum_{i=0}^{D-1}\lambda_i \nabla_{\vtheta} E ^T \vu_i \vu_i}^2 = \sum_{i=0}^{D-1}\norm{\lambda_i \nabla_{\vtheta} E ^T \vu_i}^2$; the higher each term in the sum, the higher the difference between the NGF and the PF.
To measure the connection between per iteration drift and  $\norm{\nabla_{\vtheta}^2 E \nabla_{\vtheta} E}$  in neural network training we approximate it via $\norm{\vtheta_{t} - \widetilde{NGF}(\vtheta_{t-1}, h)}$ where $\widetilde{NGF}$ is the numerical approximation to the NGF initialised at $\vtheta_{t-1}$. 
Results in Figures~\ref{fig:h_g_drift} and~\ref{fig:h_g_drift_spearman} show the strong correlation between per iteration drift and $\norm{\nabla_{\vtheta}^2 E \nabla_{\vtheta} E}$ throughout training and across learning rates. Since Theorem~\ref{thm:total_drift} tells us the form of the drift but not the exact value of $\vtheta'$, we have used $\vtheta_{t-1}$ instead to evaluate $\norm{\nabla_{\vtheta}^2 E \nabla_{\vtheta} E}$ and thus some error exists.

Understanding this connection is advantageous since computing discretization drift is computationally expensive as it requires simulating the continuous time NGF but computing $\norm{\nabla_{\vtheta}^2 E \nabla_{\vtheta} E}$ via Hessian-vector products is cheaper and approximations are available, such as 
$\nabla_{\vtheta}^2 E \nabla_{\vtheta} E \approx \frac{\nabla_{\vtheta} E(\vtheta + \epsilon \nabla_{\vtheta} E) - \nabla_{\vtheta} E(\vtheta)}{\epsilon}$
which only requires an additional backward pass~\citet{geiping2021stochastic}.

\begin{figure}[tb]
\centering
\includegraphics[width=0.31\columnwidth]{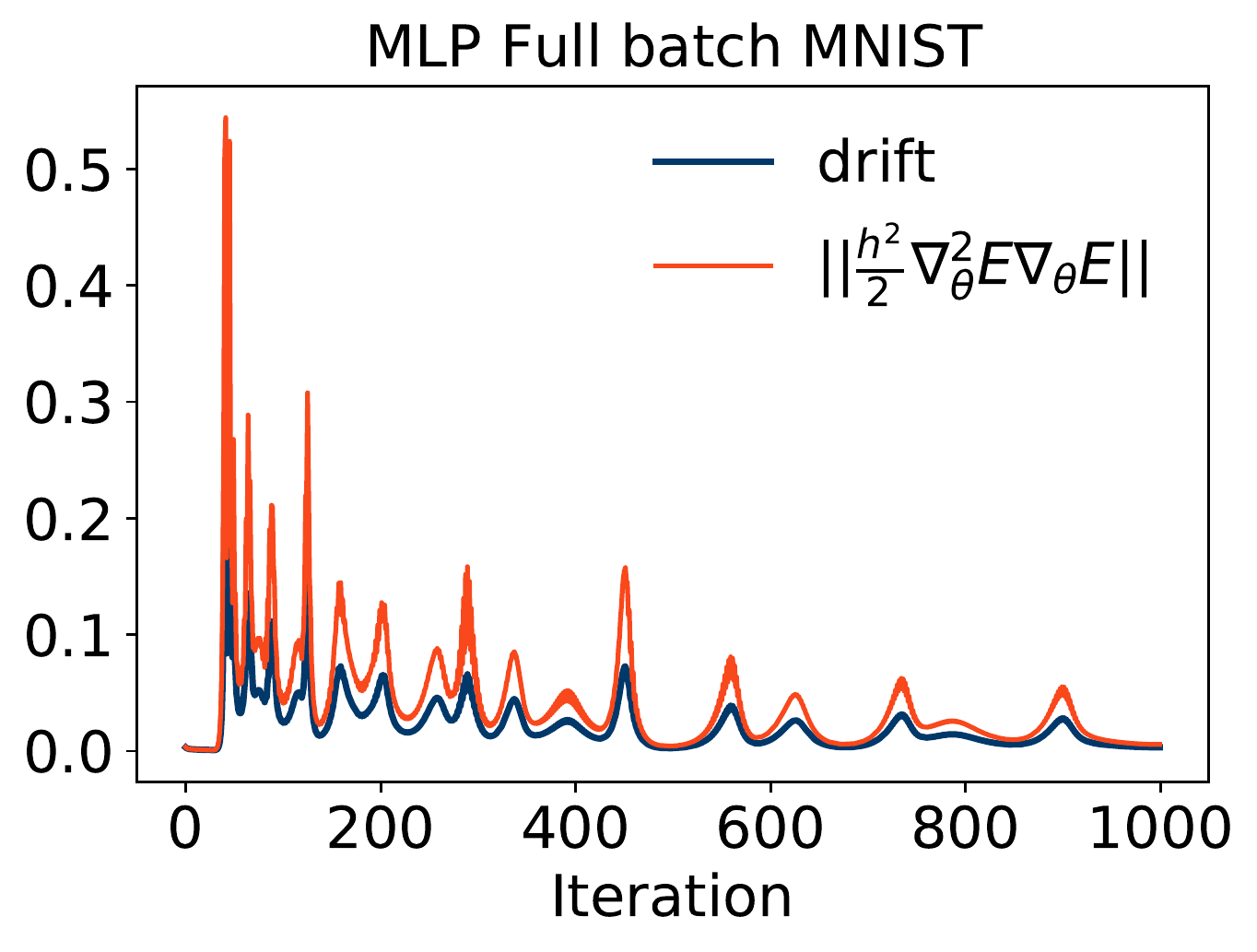}
 \includegraphics[width=0.31\columnwidth]{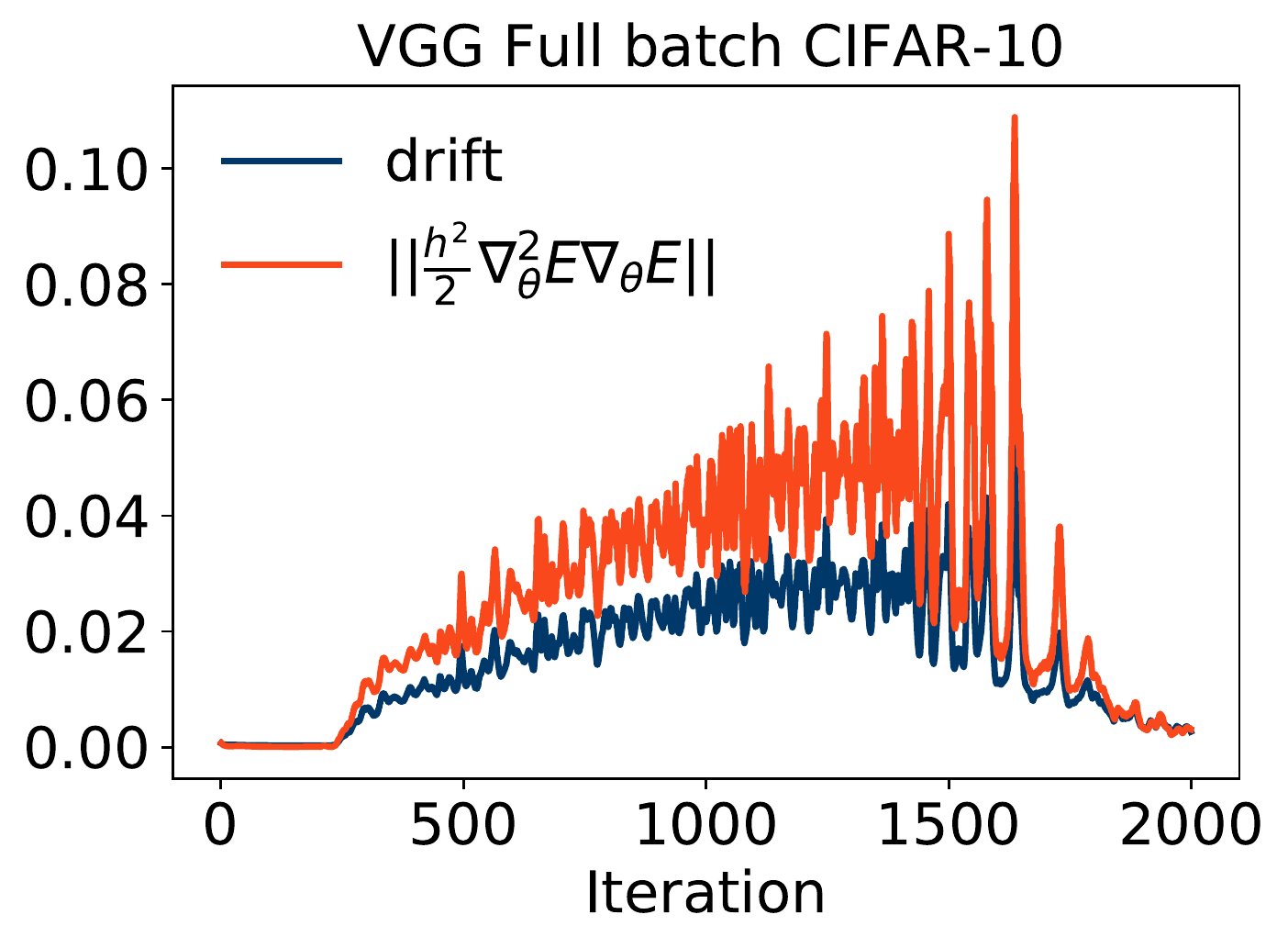}
 \includegraphics[width=0.31\columnwidth]{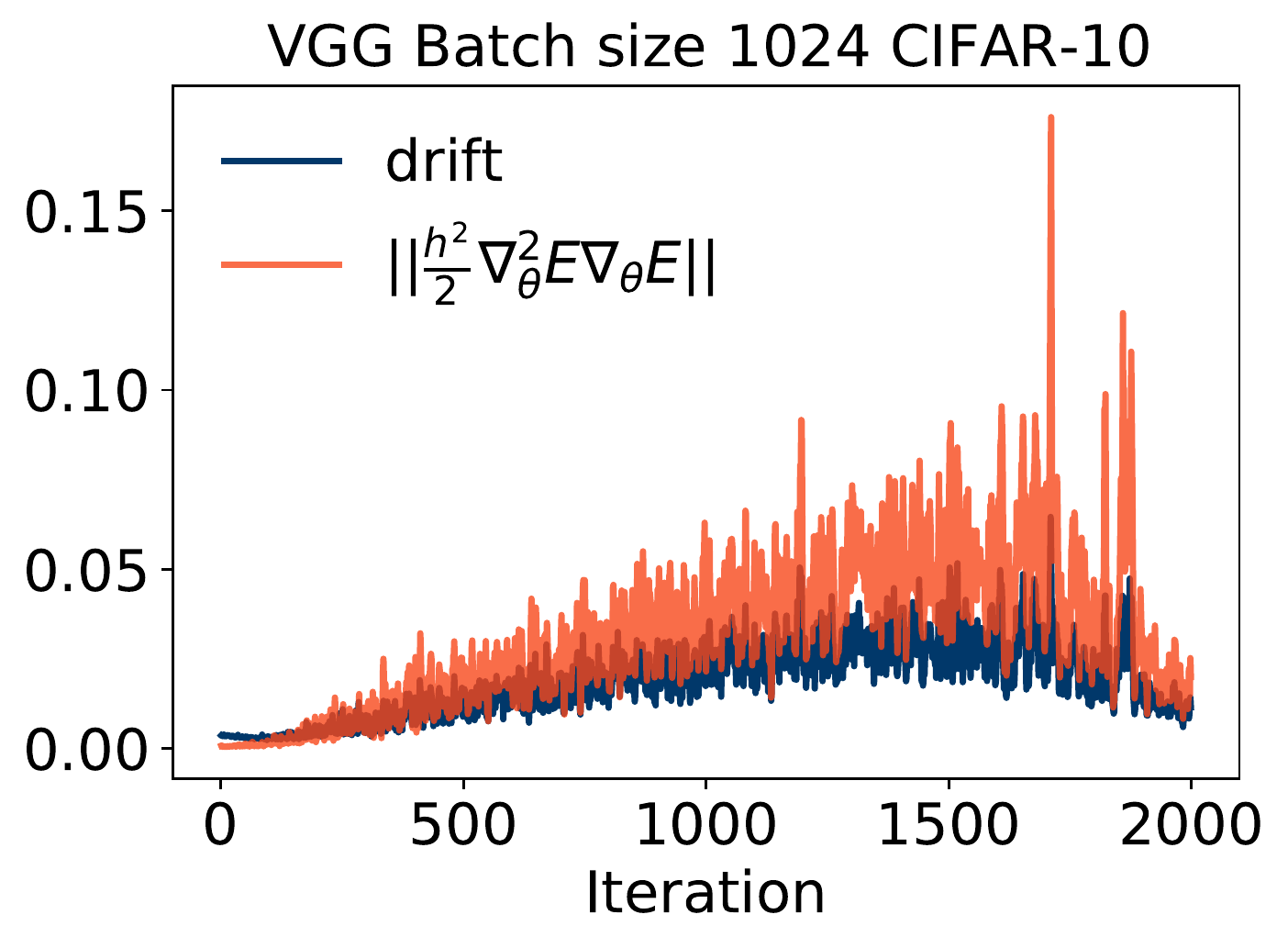}
\caption[$||\nabla_{\vtheta}^2 E \nabla_{\vtheta} E||$ and the per iteration drift as measured during training. The plots are obtained on an MNIST MLP with 4 layers and 100 units per layer, and a VGG model trained on CIFAR-10. The learning rates used are $0.05$, $0.01$ and $0.01$ respectively.]{Connection between $||\nabla_{\vtheta}^2 E \nabla_{\vtheta} E||$ and the per iteration drift as measured during training.}
\label{fig:h_g_drift}
\end{figure}

\begin{figure}[tb]
\centering
\includegraphics[width=0.31\columnwidth]{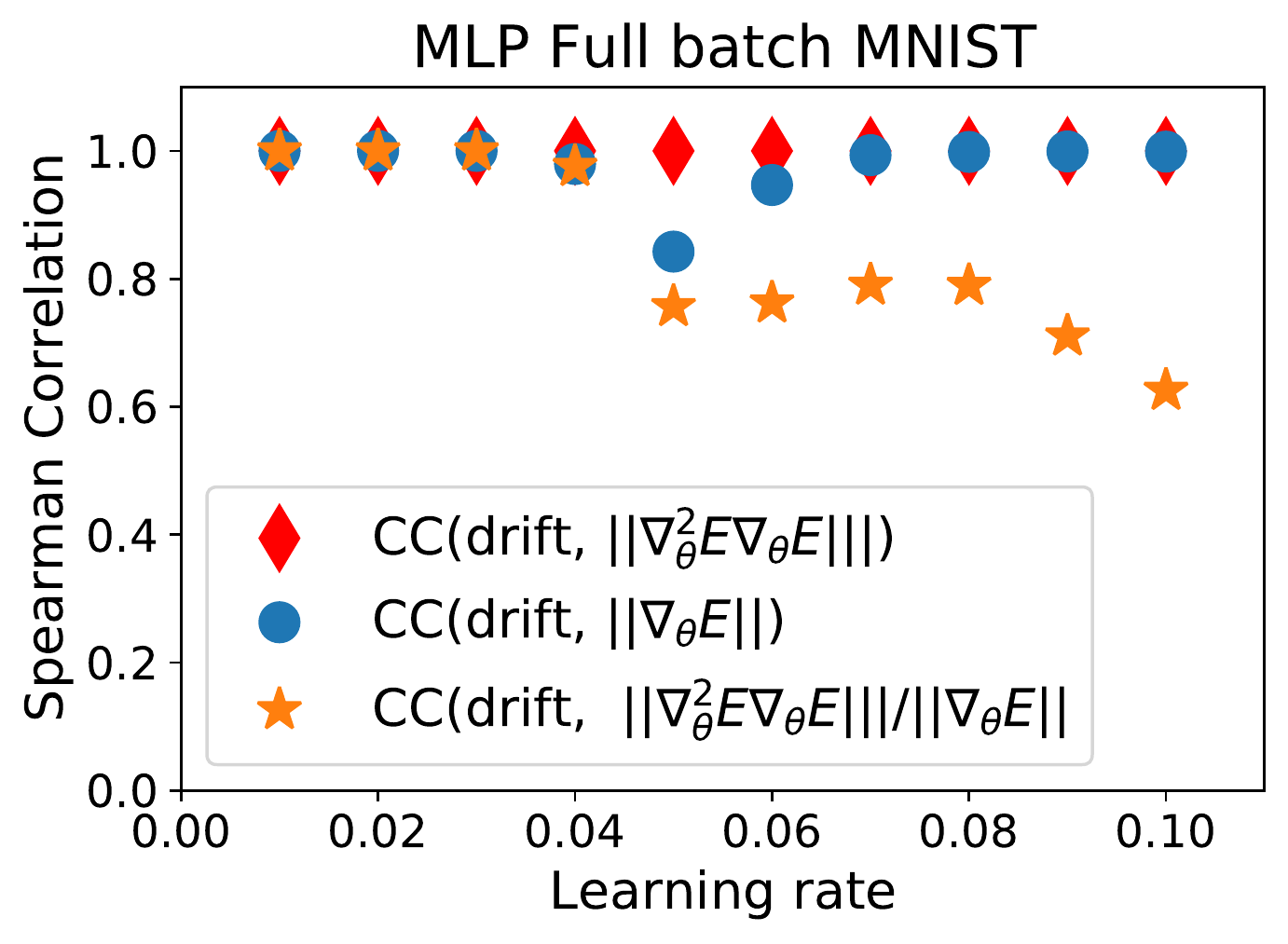}
 \includegraphics[width=0.31\columnwidth]{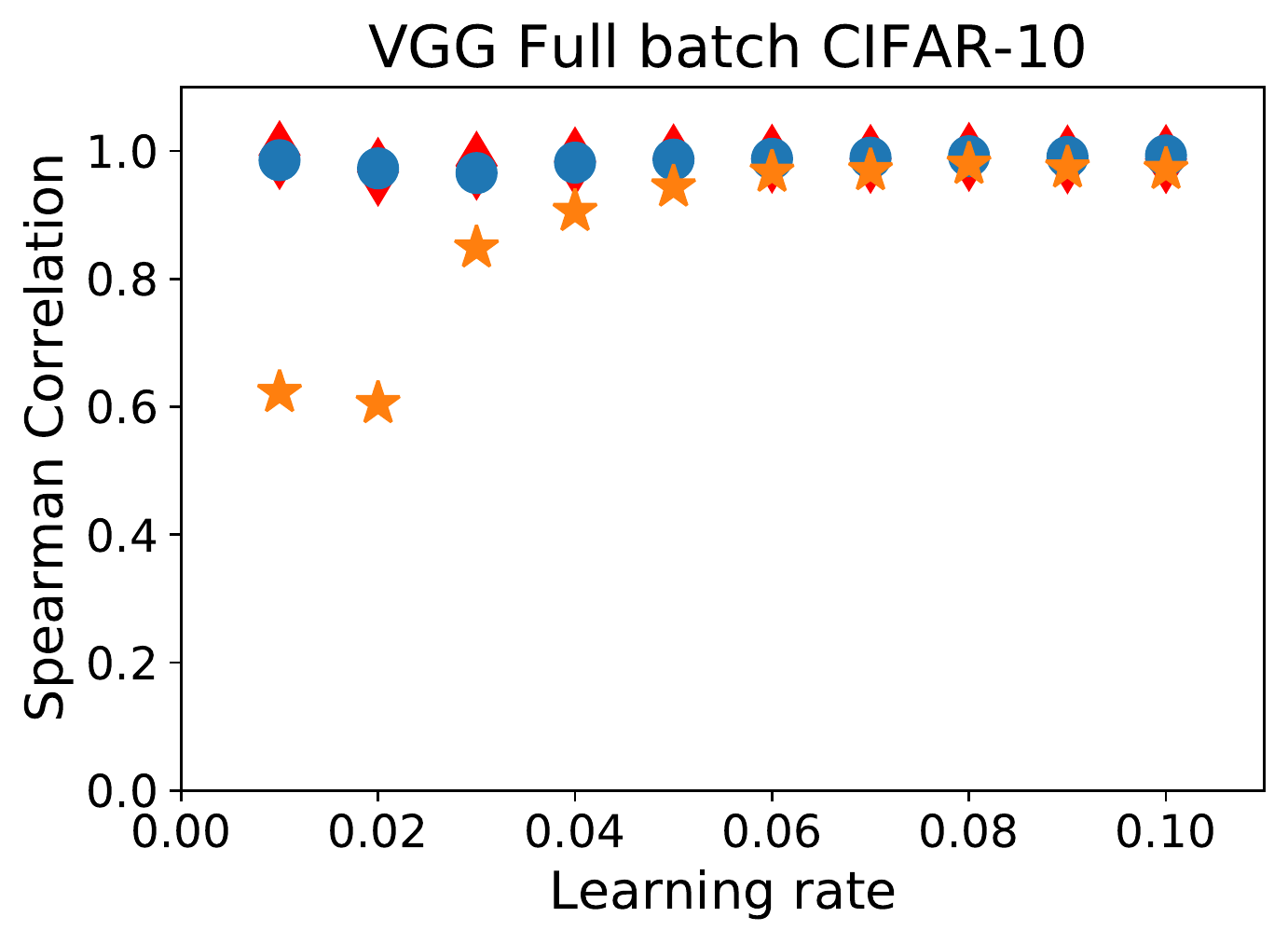}
 \includegraphics[width=0.31\columnwidth]{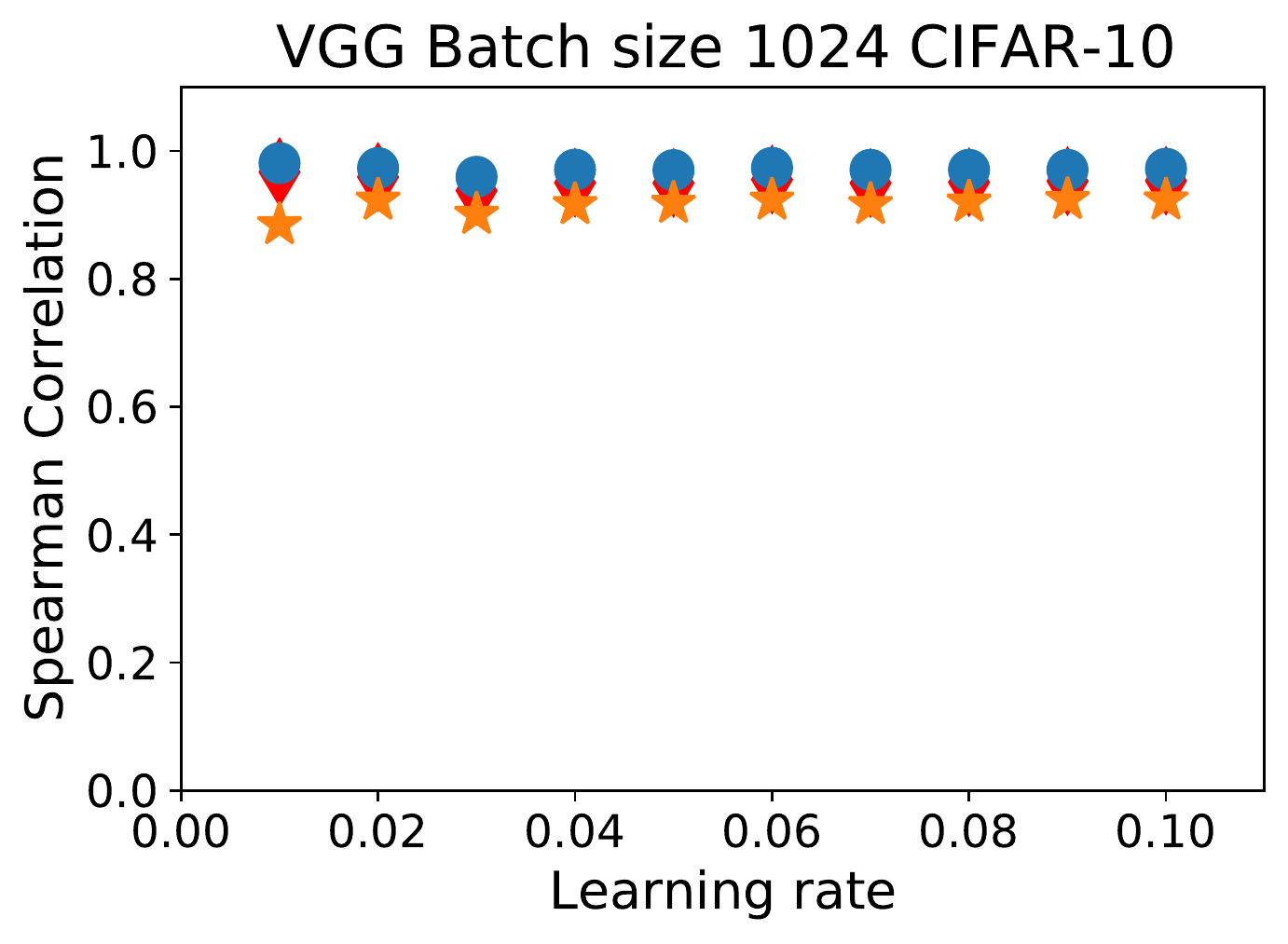}
\caption[Correlation between $||\nabla_{\vtheta}^2 E \nabla_{\vtheta} E||$ and the iteration drift. The plots are obtained on an MNIST MLP with 4 layers and 100 units per layer, and a VGG model trained on CIFAR-10. Results are obtained over a learning rate sweep of 10 learning rates $0.01$, $0.02$... $0.1$.]{Correlation between $||\nabla_{\vtheta}^2 E \nabla_{\vtheta} E||$ and the per iteration drift.  Since  $||\nabla_{\vtheta}^2 E \nabla_{\vtheta} E|| = \left(||\nabla_{\vtheta}^2 E \nabla_{\vtheta} E||\right) / || \nabla_{\vtheta} E|| || \nabla_{\vtheta} E||$, we plot the correlation with the individual terms as well.}
\label{fig:h_g_drift_spearman}
\end{figure}

\subsection{Drift adjusted learning rate (DAL)}
\label{sec:dal}

A natural question to ask is how to use the correlation between  $\norm{\nabla_{\vtheta}^2 E \nabla_{\vtheta} E}$ and the iteration drift to improve training stability; $\norm{\nabla_{\vtheta}^2 E \nabla_{\vtheta} E}$ captures all the quantities we have shown to be relevant to instability highlighted by the PF: $\lambda_i$ and $\nabla_{\vtheta} E ^T \vu_i$ (Eq.~\ref{eq:expansion}).
One way to use this information is to adapt the learning rate of the gradient descent update, such as using $\frac 2 {\norm{\nabla_{\vtheta}^2 E \nabla_{\vtheta} E}}$ as the learning rate. 
This learning rate slows down training when the drift is large --- areas where instabilities are likely to occur --- and it speeds up training in regions of low drift --- areas where instabilities are unlikely to occur.
Computing the norm of the update provided by this learning rate shows a challenge however since $2/\norm{\nabla_{\vtheta}^2 E \nabla_{\vtheta} E} \ge \frac{2}{\lambda_0 \norm{\nabla_{\vtheta} E}}$; this implies that when using this learning rate the norm of the gradient descent update will never be 0 and thus training will not result in convergence.
Furthermore, the magnitude of the parameter update will be independent of the gradient norm.
To reinstate the gradient norm, we propose using the learning rate
\begin{align}
h(\vtheta) = \frac{2}{\norm{\nabla_{\vtheta}^2 E \nabla_{\vtheta} E}/ \norm{\nabla_{\vtheta} E}} = \frac{2}{\norm{\nabla_{\vtheta}^2 E \hat{\vg}(\vtheta)}} \label{eq:dal}
\end{align}
where $\hat{\vg}(\vtheta)$ is the unit normalised gradient $\nabla_{\vtheta} E / \norm{\nabla_{\vtheta} E}$. We will call this learning rate \textbf{DAL} (Drift Adjusted Learning rate).
As shown in Figure~\ref{fig:h_g_drift},  $\norm{\nabla_{\vtheta}^2 E \hat{\vg}(\vtheta)}$ has a strong correlation with the per iteration drift.
 \rebuttalrthree{Another interpretation of DAL can be provided through a signal to noise perspective: the size of the learning signal obtained by minimising $E$ is that of the update $h \norm{\nabla_{\vtheta}E}$, while the norm of the noise coming from the drift can be approximated as $\frac{h^2}{2} \norm{\nabla_{\vtheta}^2E \nabla_{\vtheta}E}$, thus the `signal to noise ratio' can be approximated as $h \norm{\nabla_{\vtheta}E}/({\frac{h^2}{ 2} \norm{\nabla_{\vtheta}^2E \nabla_{\vtheta}E}}) =2/({h \norm{\nabla_{\vtheta}^2 E \hat{\vg}(\vtheta)}})$, which when using DAL (Eq~\ref{eq:dal}) is 1; thus DAL can be seen as balancing the gradient signal and the regularising drift noise in gradient descent training.}

We use DAL to set the learning rate and show results  across architectures, models and datasets in Figures~\ref{fig:dal} (with additional results in Figure~\ref{fig:imagenet_lr_scaling_across_batch_sizes} in the Appendix). \textit{Despite not requiring a learning rate sweep, DAL is stable compared to using fixed learning rates}. 
To provide intuition about DAL, 
we show the learning rate and the update norm in Figure~\ref{fig:l_r_scaling_quantities}:
for DAL the learning rate decreases in training after which it slowly increases when reaching areas with low drift. Compared to larger learning static learning rates where the update norm can increase in the edge of stability area with DAL the update norm steadily decreases in training.

\begin{figure}[tb]
\centering
 \includegraphics[width=0.32\columnwidth]{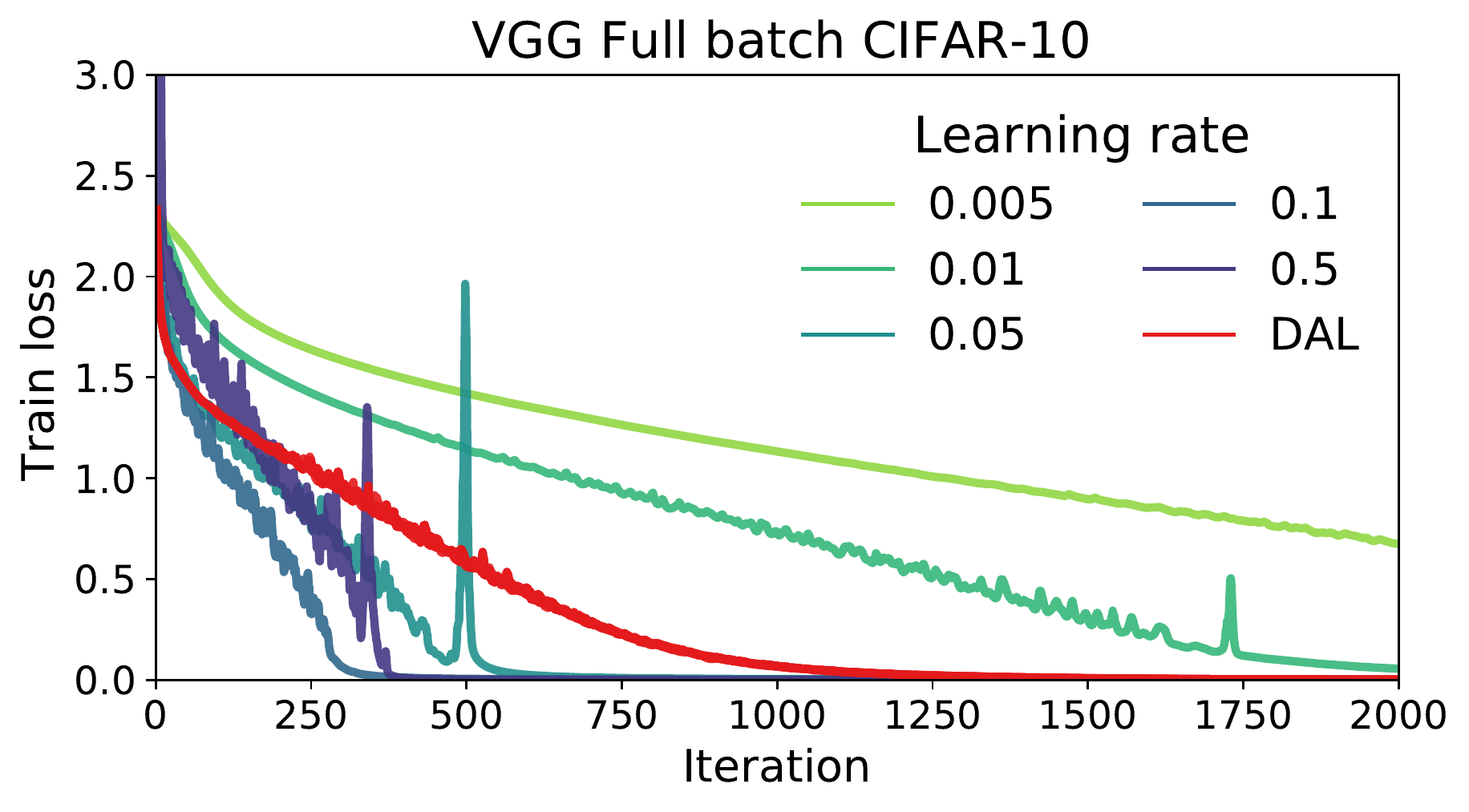}
 \includegraphics[width=0.32\columnwidth]{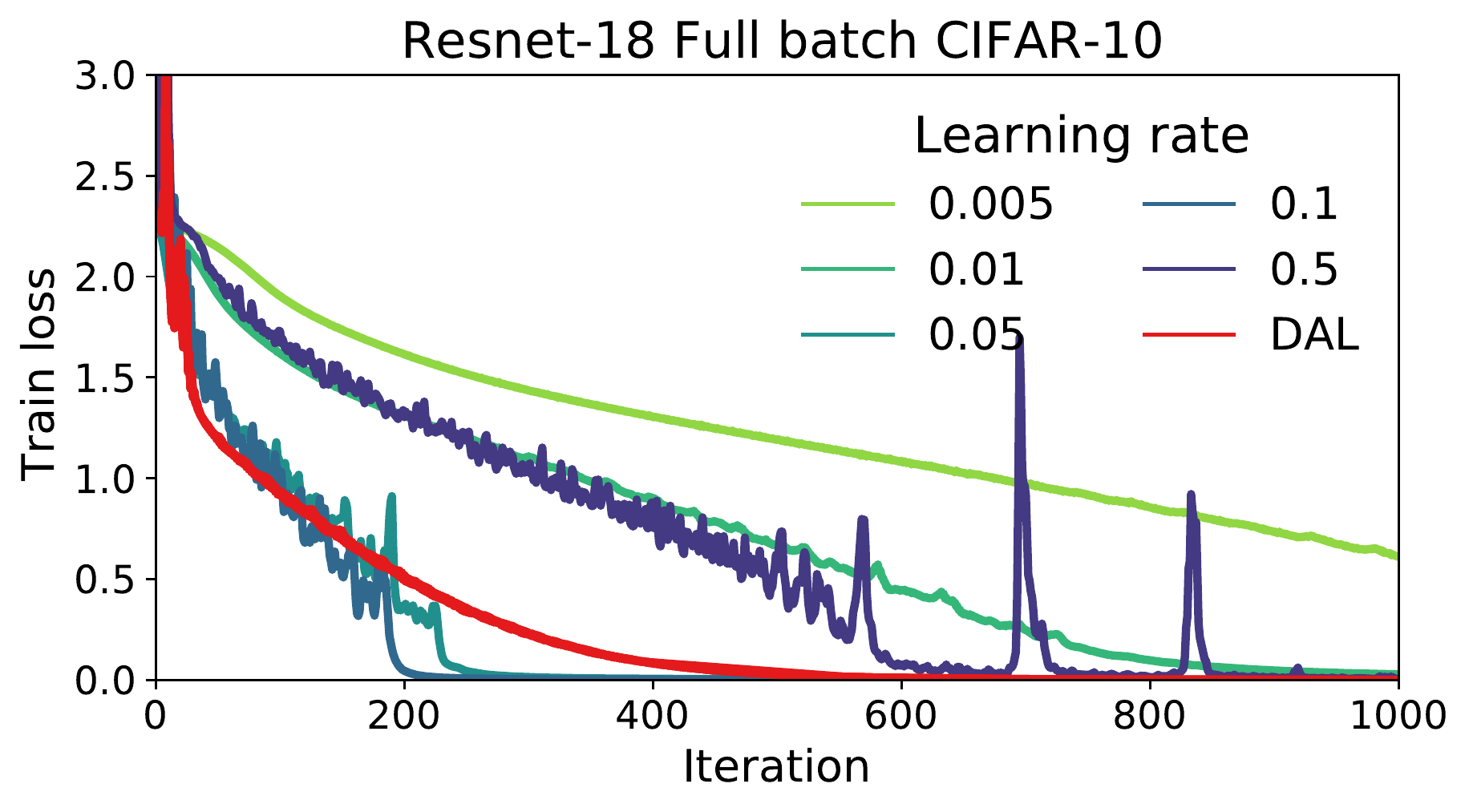}
  \includegraphics[width=0.32\columnwidth]{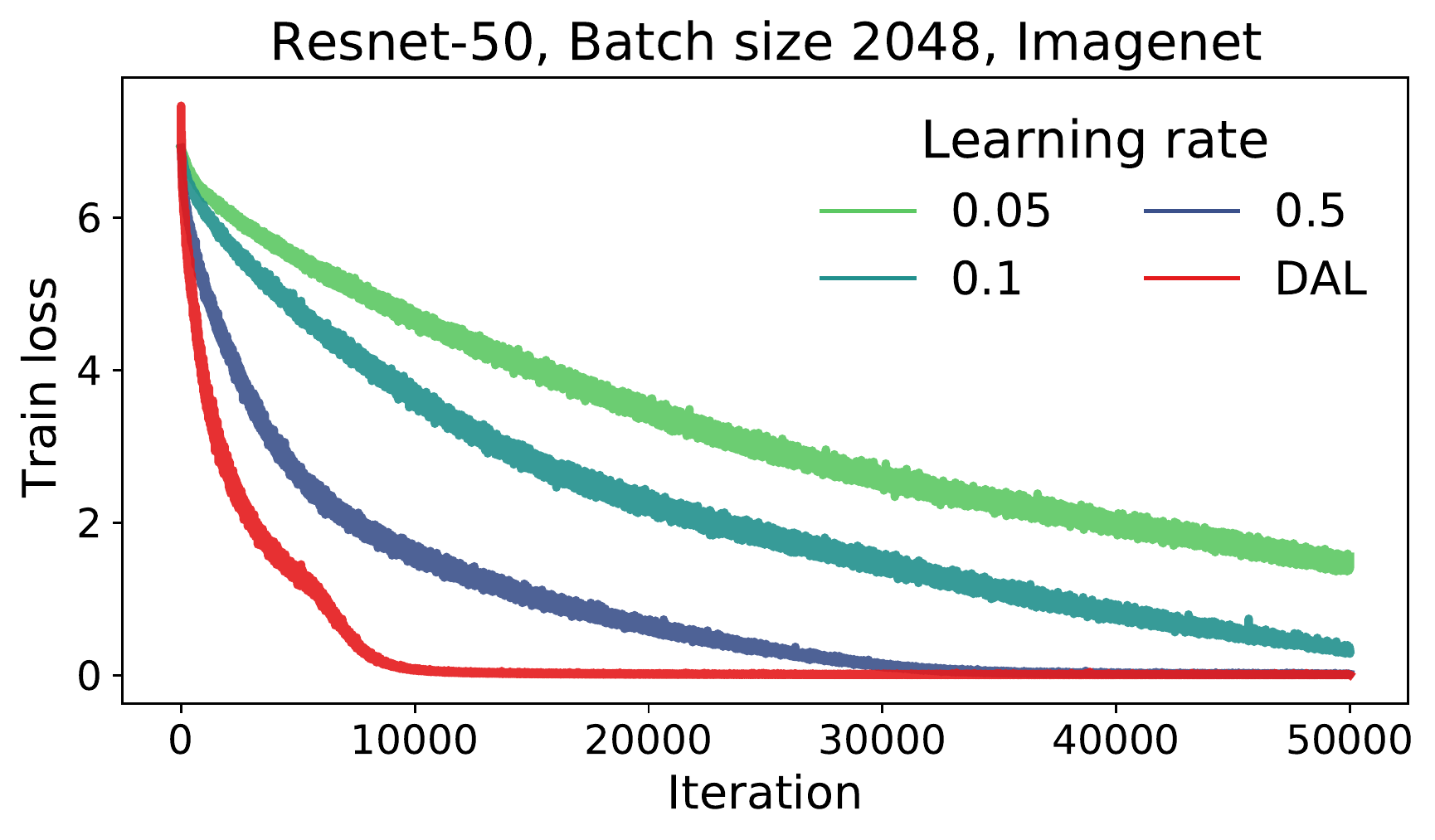}
\caption[Models trained using a learning rate sweep or DAL on CIFAR-10 and Imagenet. Models are VGG, Resnet-18 and Resnet-50 respectively.]{DAL: using the learning rate $\frac{2}{\norm{\nabla_{\vtheta}^2 E \hat{\vg}(\vtheta)}}$ results in improved stability without requiring a hyperparameter sweep. }
\label{fig:dal}
\end{figure}

\begin{figure}[tb]
\centering
 \includegraphics[width=0.32\columnwidth]{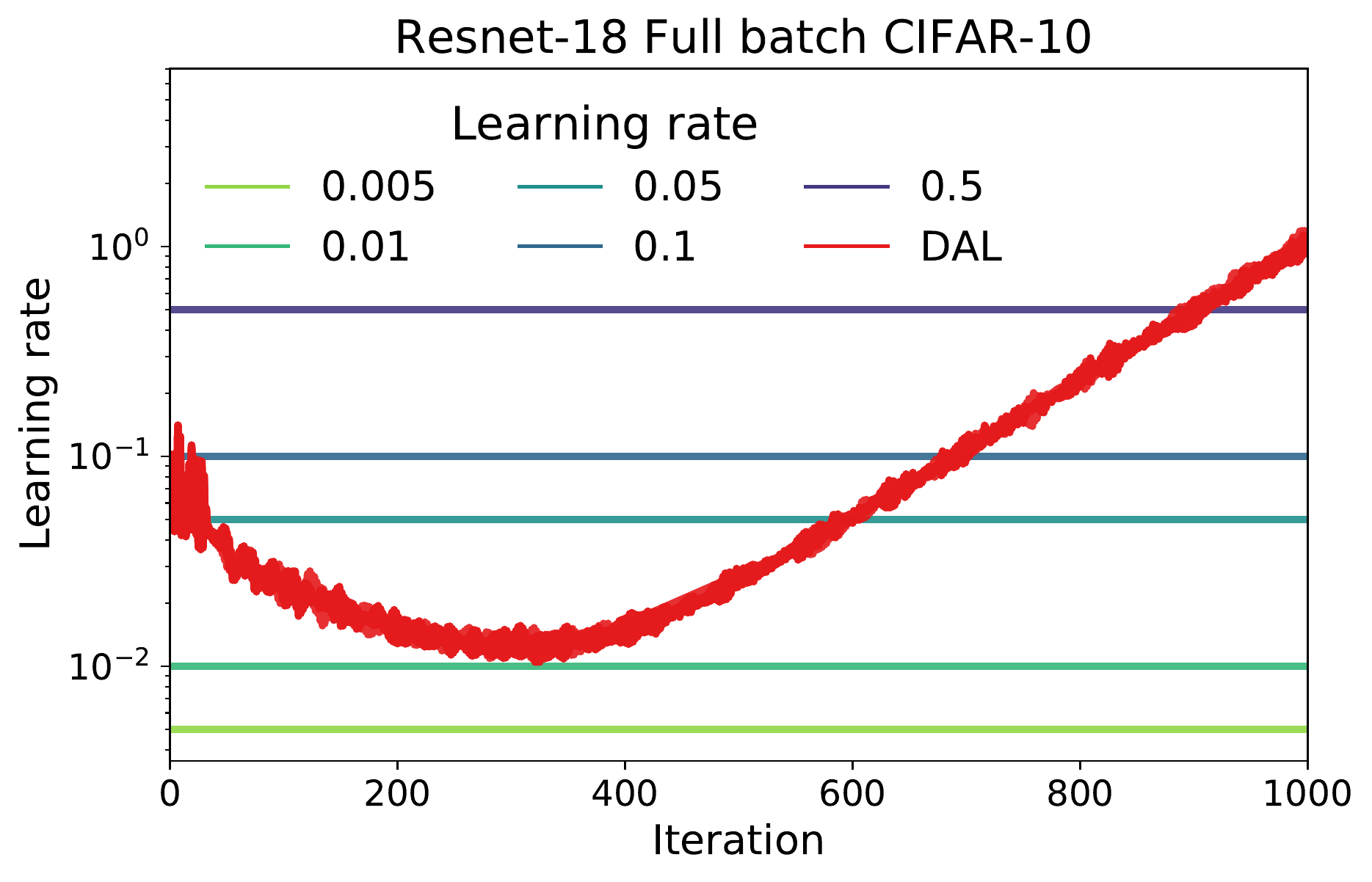}
 \includegraphics[width=0.32\columnwidth]{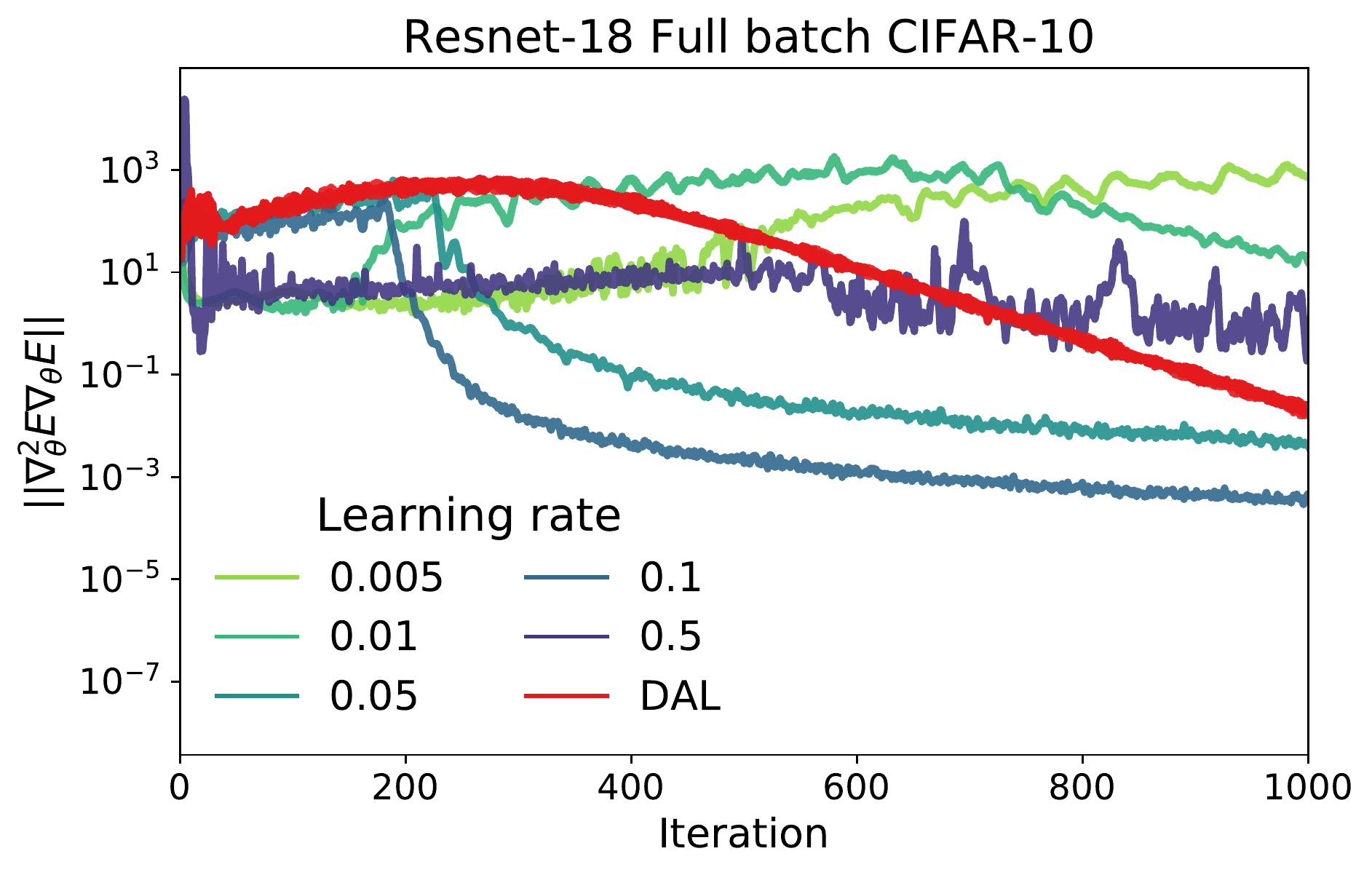}
 \includegraphics[width=0.32\columnwidth]{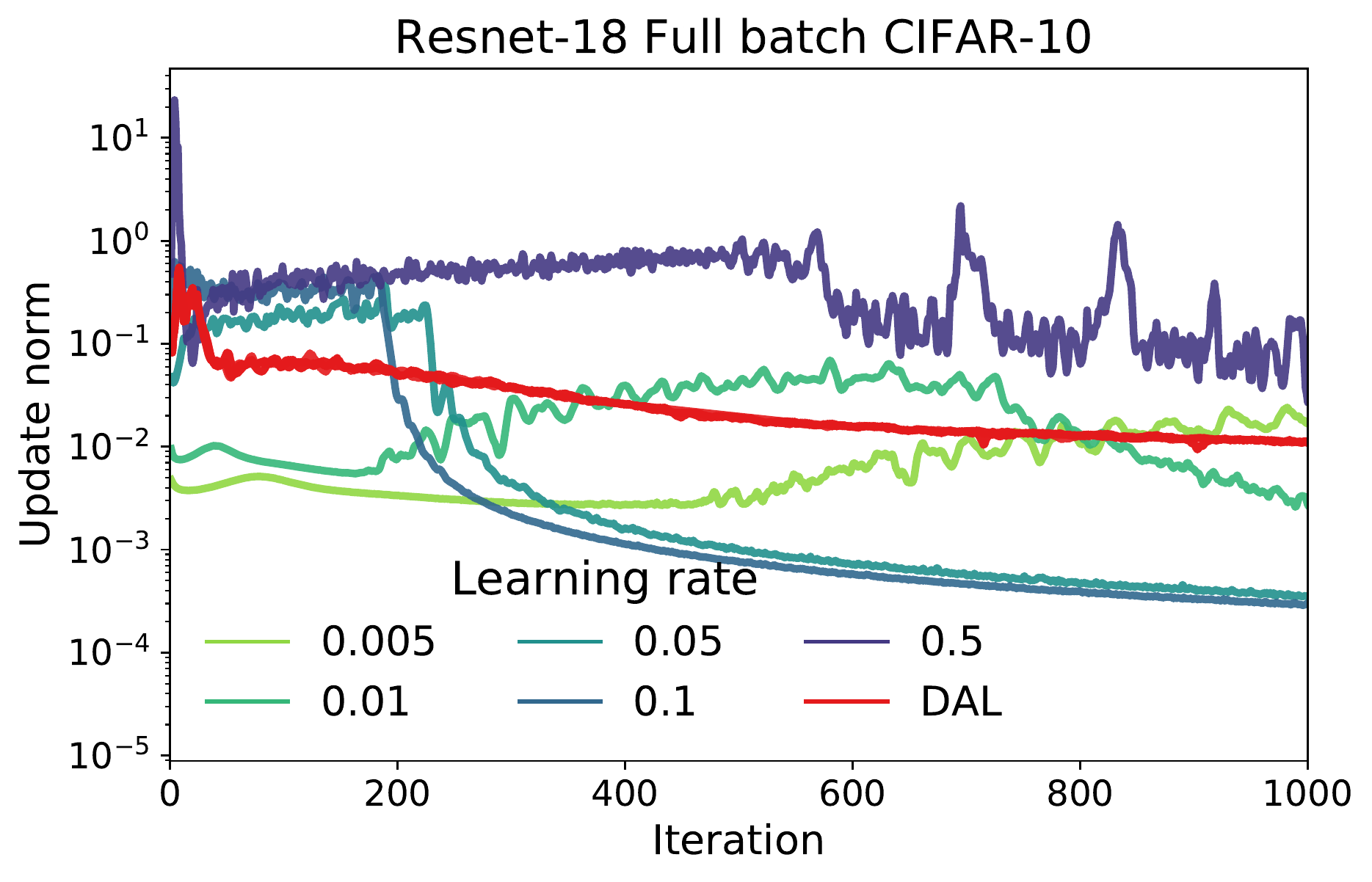}
\caption[Key quantities in DAL versus fixed learning rate training: learning rate, and update norms. Results obtained on a Resnet-18 model trained on CIFAR-10.]{Key quantities in DAL versus fixed learning rate training: learning rate, and update norms.}
\label{fig:l_r_scaling_quantities}
\end{figure}

\subsection{The trade-off between stability and performance}
\label{sec:trade_off}

Since we are interested in understanding the  optimisation dynamics of gradient descent, we have so far focused on training performance. We now try to move our attention to test performance and generalization. Previous works~\citep{li2019towards,igr,jastrzebski2019break} have shown that higher learning rates lead to better generalization performance. We now try to further connect this information with the per iteration drift and the PF.
To do so, we use learning rates with various degrees of sensitivity to iteration drift using DAL-$p$:
\begin{align}
 h_p(\vtheta) = \frac{2}{\left(\norm{\nabla_{\vtheta}^2 E \hat{\vg}(\vtheta)}\right)^p}
\end{align}

The higher $p$, the slower the training and less drift there is; the lower $p$, there is more drift. We start with extensive experiments with $p=0.5$, which we show in Figure~\ref{fig:l_r_scaling_sqrt}, and show more results in Figure~\ref{fig:imagenet_lr_sqrt_scaling_across_batch_sizes}. Compared to $p=1$ (DAL), there is faster training but at times also more instability. Performance on the test set shows that DAL-$0.5$ performs as well or better than when using fixed learning rates.

\begin{remark}
We find that across datasets and batch sizes, DAL-$0.5$ performs best in terms of the stability generalization trade-off and in these settings can be used as a drop in replacement for a learning rate sweep.
\end{remark}

\begin{figure}[tb]
  \includegraphics[width=0.33\columnwidth]{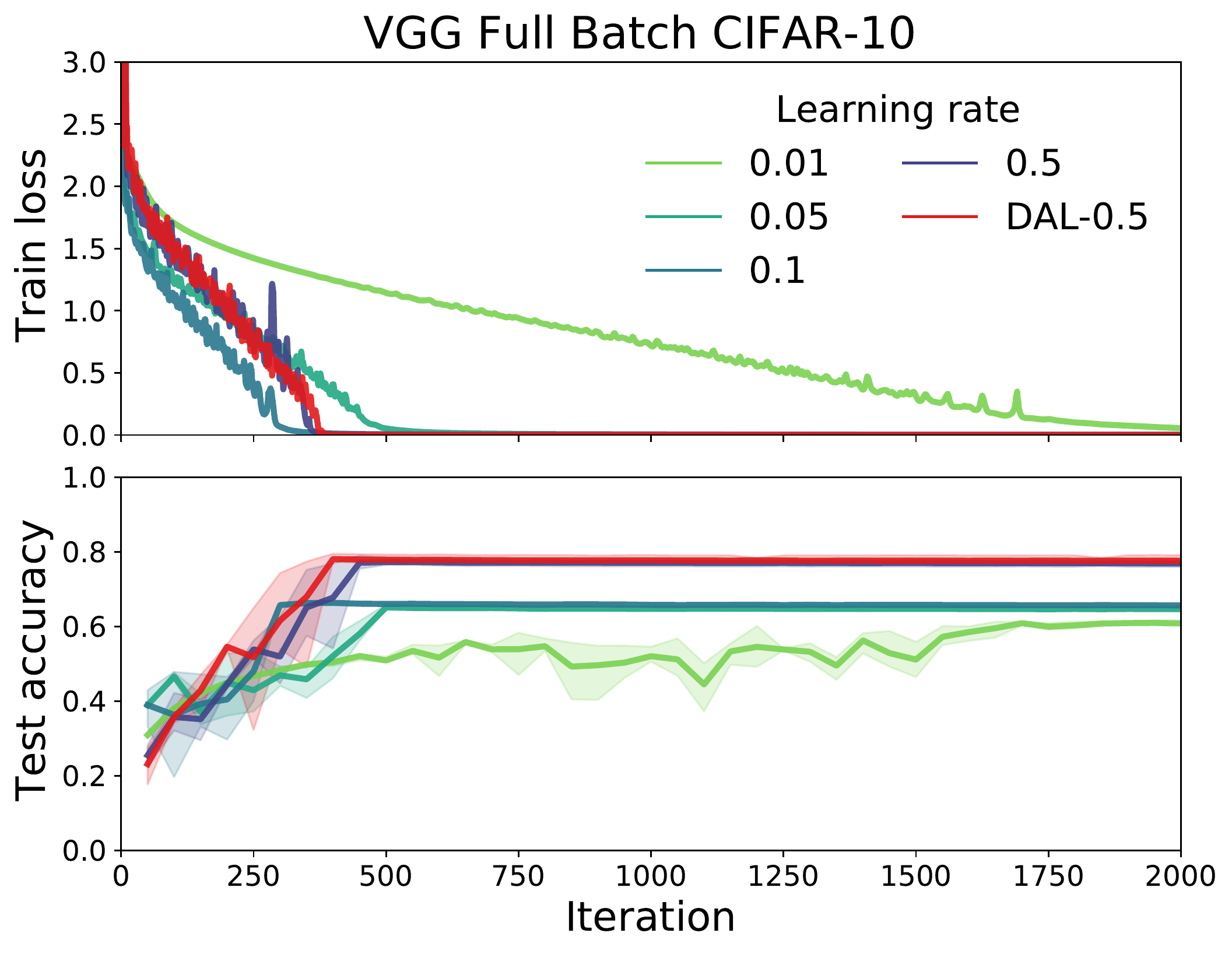}
  \includegraphics[width=0.33\columnwidth]{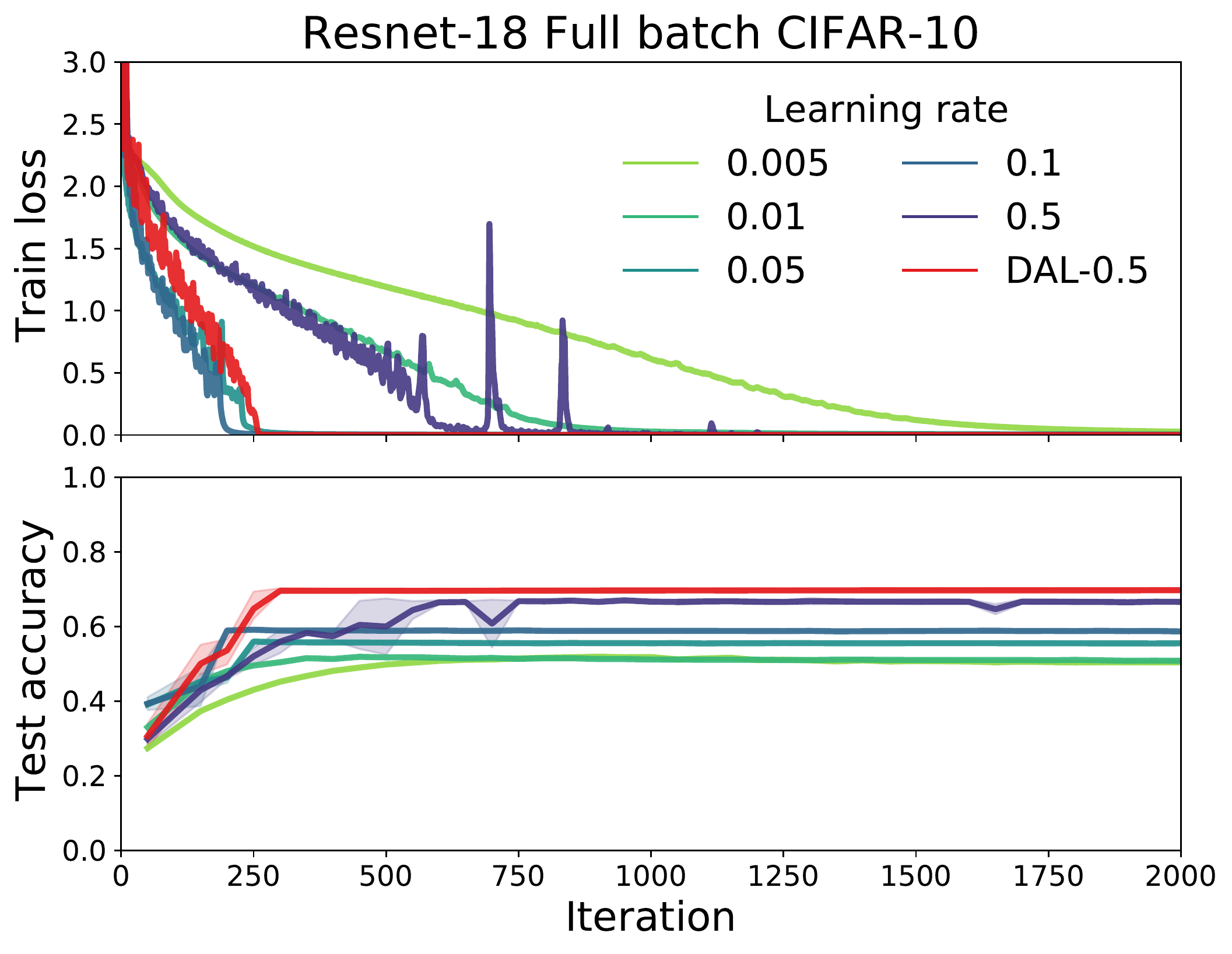}
  \includegraphics[width=0.33\columnwidth]{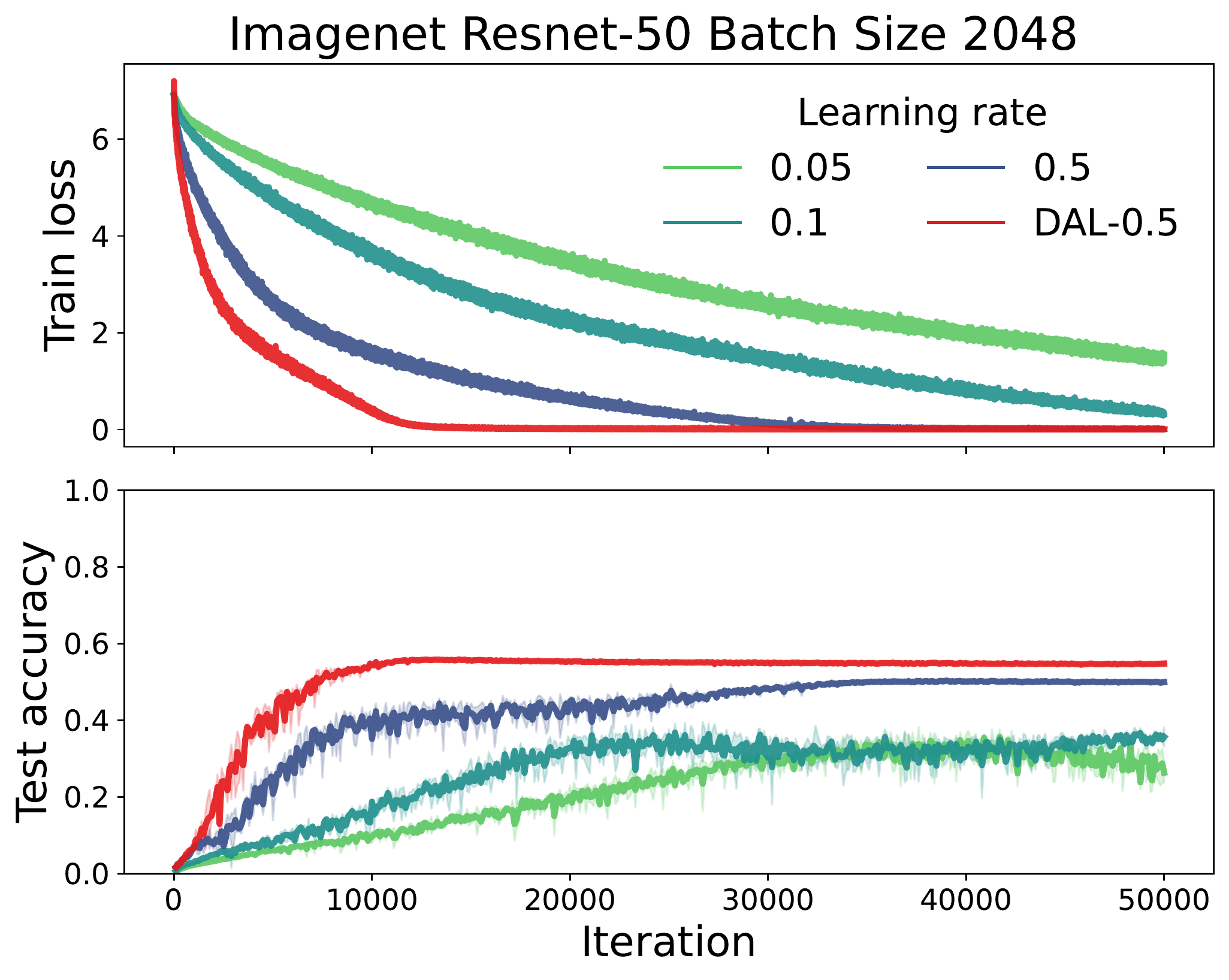}
  \caption[DAL-0.5 results on a CIFAR-10 and Imagenet using  VGG, Resnet-18 and Resnet-50 model respectively.]{DAL-0.5: increased training speed and generalization compared to a sweep of fixed learning rates.}
\label{fig:l_r_scaling_sqrt}
\end{figure}

To further investigate the connection between drift and test set performance, we perform a set of sweeps over the power $p$ and show results in Figure~\ref{fig:power_sweeps}. These results show that the higher the drift (the smaller $p$), the more generalization; additional results across batch sizes showing the same trend are shown in Figure~\ref{fig:power_sweeps_batch_sizes_vgg} in the Appendix. We also show in Figure~\ref{fig:stability_lambda_performance} the correlation between mean per iteration drift and test accuracy both for learning rate and DAL-$p$ sweeps. The results consistently show that the higher the mean iteration drift, the higher the test accuracy. We also show that the mean iteration drift has a connection to the largest eigenvalue $\lambda_0$: the higher the drift, the smaller $\lambda_0$. These results add further evidence to the idea that discretization drift is beneficial for generalization performance in the deep learning setting. We also notice that DAL-$p$ with smaller values of $p$ leads to a small $\lambda_0$ compared to vanilla gradient descent even when large learning rates are used for the latter; this could explain its generalisation capabilities as lower sharpness has been connected to generalisation in previous works~\citep{keskar2016large,jastrzkebski2018relation,foret2020sharpness}. \rebuttalrone{To consolidate these results, we use the method of~\citet{li2018visualizing} to visualise the loss landscape learned by DAL-$p$ compared to that learned using gradient descent, and observe that \textit{even when reporting similar accuracies, DAL-$p$ converges to a flatter landscape}; this is observed even when small batch sizes are used. Results are shown in Figures~\ref{fig:DAL_p_64_landscape}, \ref{fig:DAL_p_full_batch_landscape}, \ref{fig:DAL_p_imagenet_landscape} in the Appendix.}

\begin{figure}[tb]
  \includegraphics[width=0.33\columnwidth]{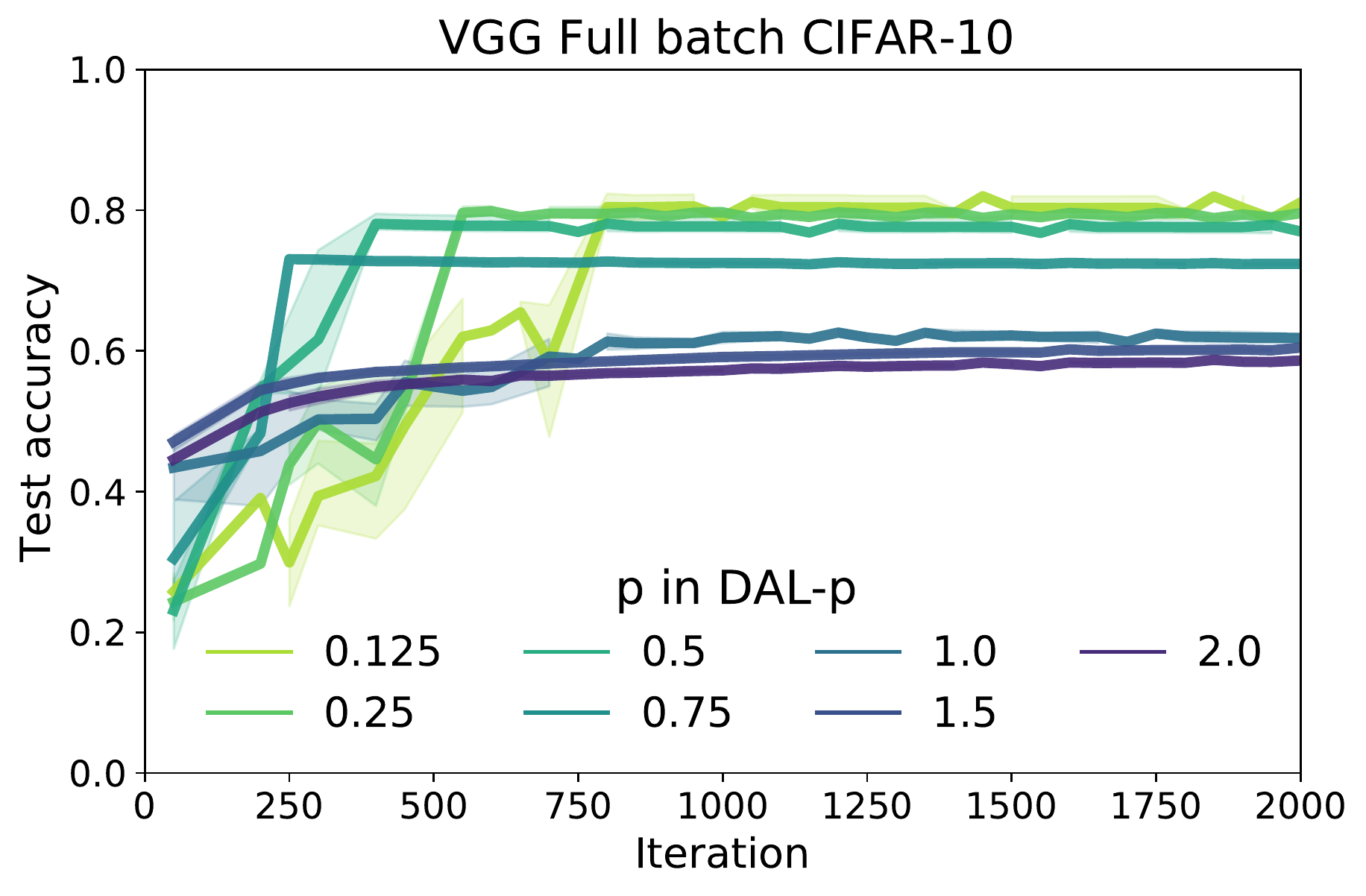}
 \includegraphics[width=0.33\columnwidth]{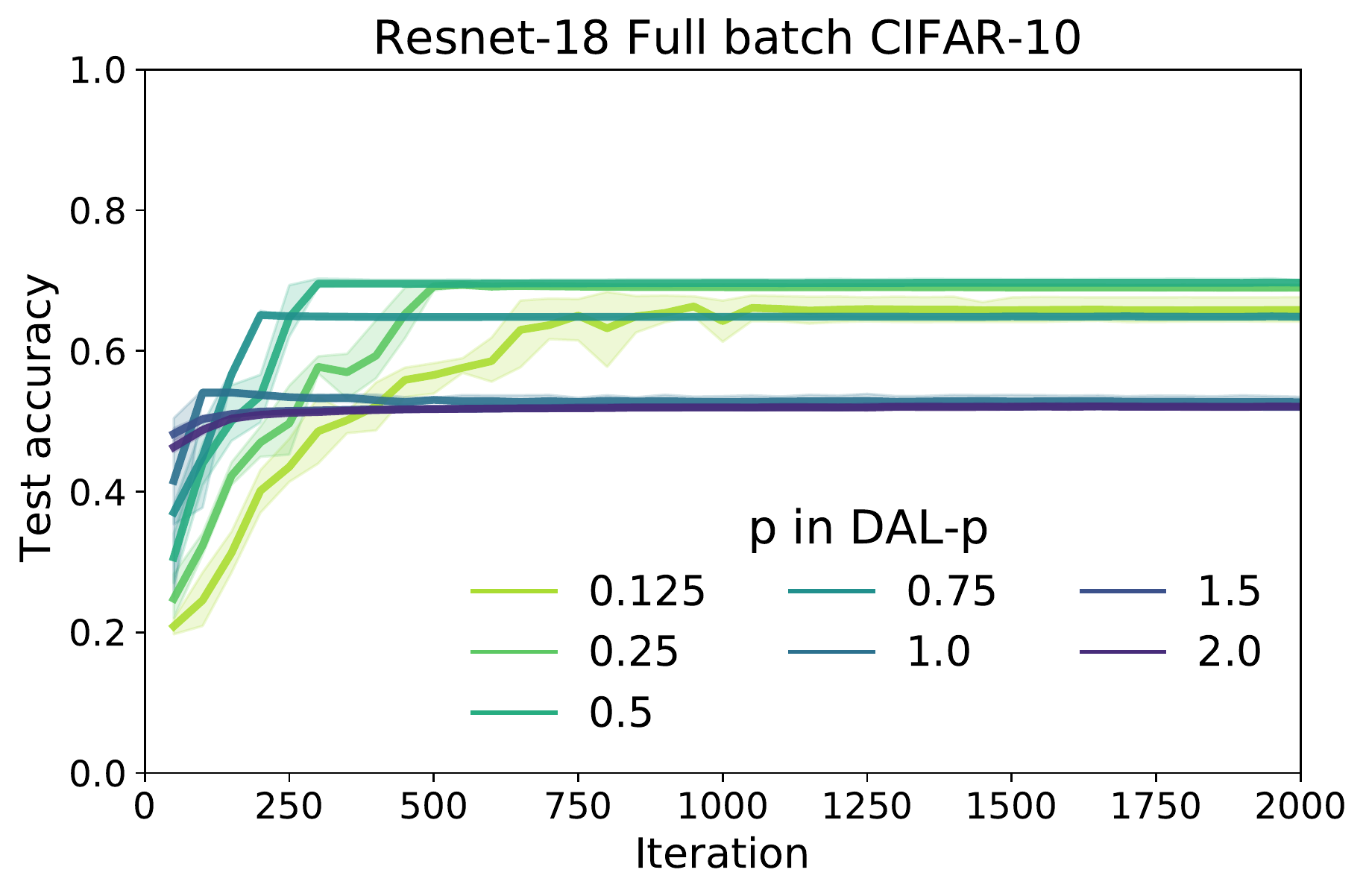}
 \includegraphics[width=0.33\columnwidth]{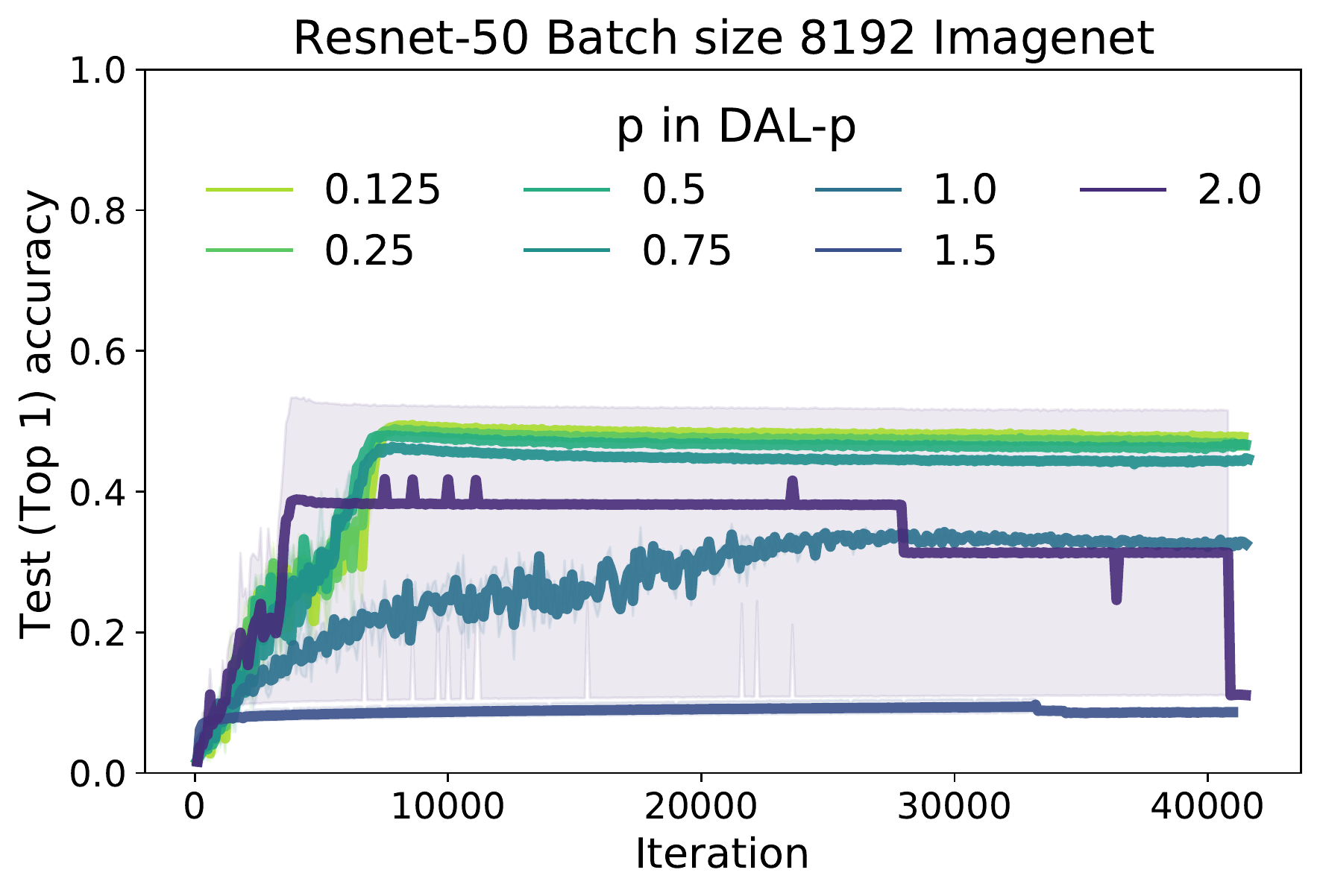}
\caption[DAL-$p$ sweep on full batch training on CIFAR-10 with VGG and Resnet-18 models and a Resnet-50 model trained on Imagenet.]{DAL-$p$ sweep: discretization drift helps test performance at the cost of stability. Corresponding training curves and loss functions are present in the Figure~\ref{fig:power_sweeps_all} in the Appendix; results showing the same trends across various batch sizes are shown in Figure~\ref{fig:power_sweeps_batch_sizes_vgg}.}
\label{fig:power_sweeps}
\end{figure}

\begin{figure}[tb]
\begin{subfigure}[Fixed learning rate sweep.]{
 \includegraphics[width=0.24\columnwidth]{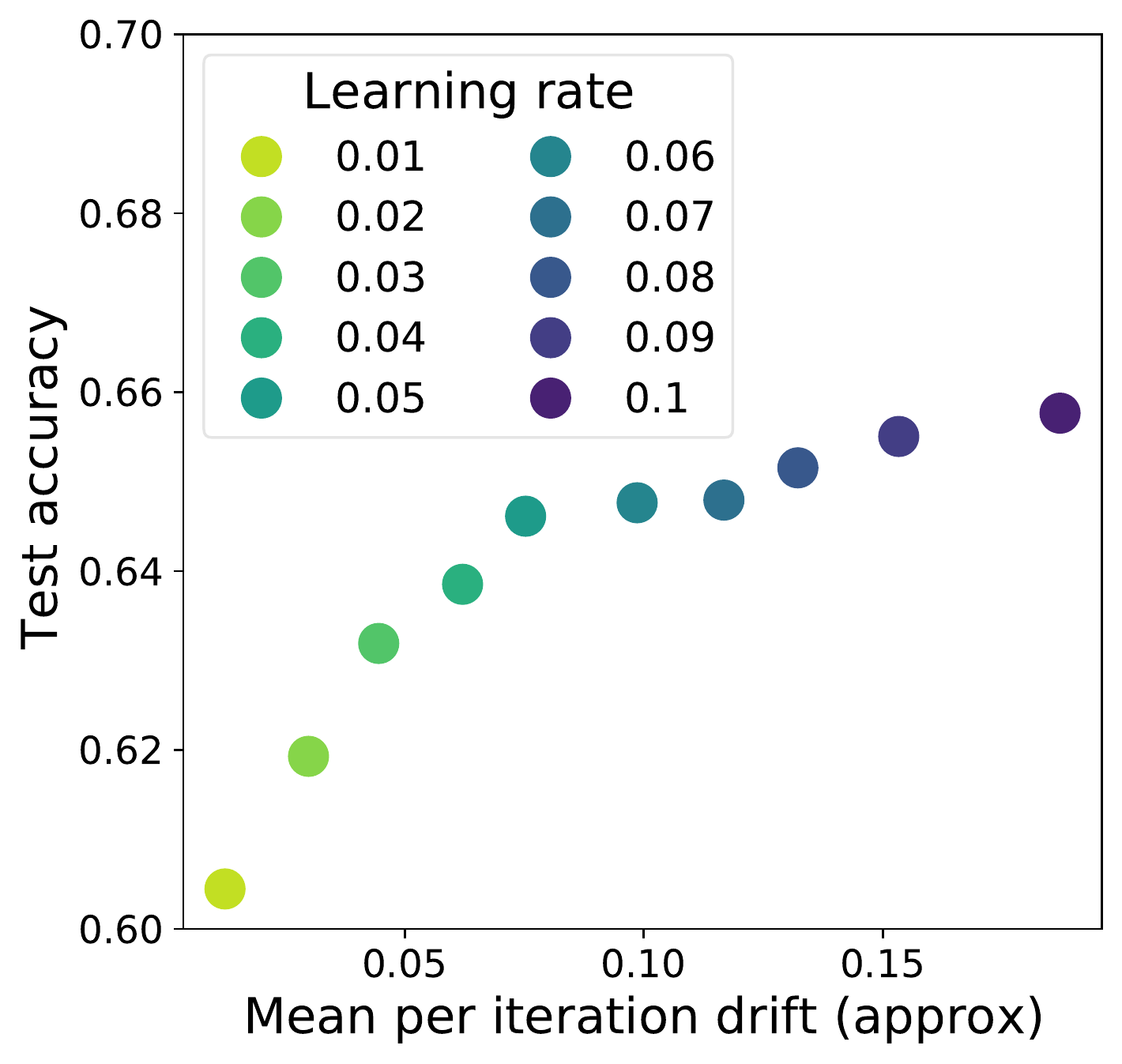}
  \includegraphics[width=0.24\columnwidth]{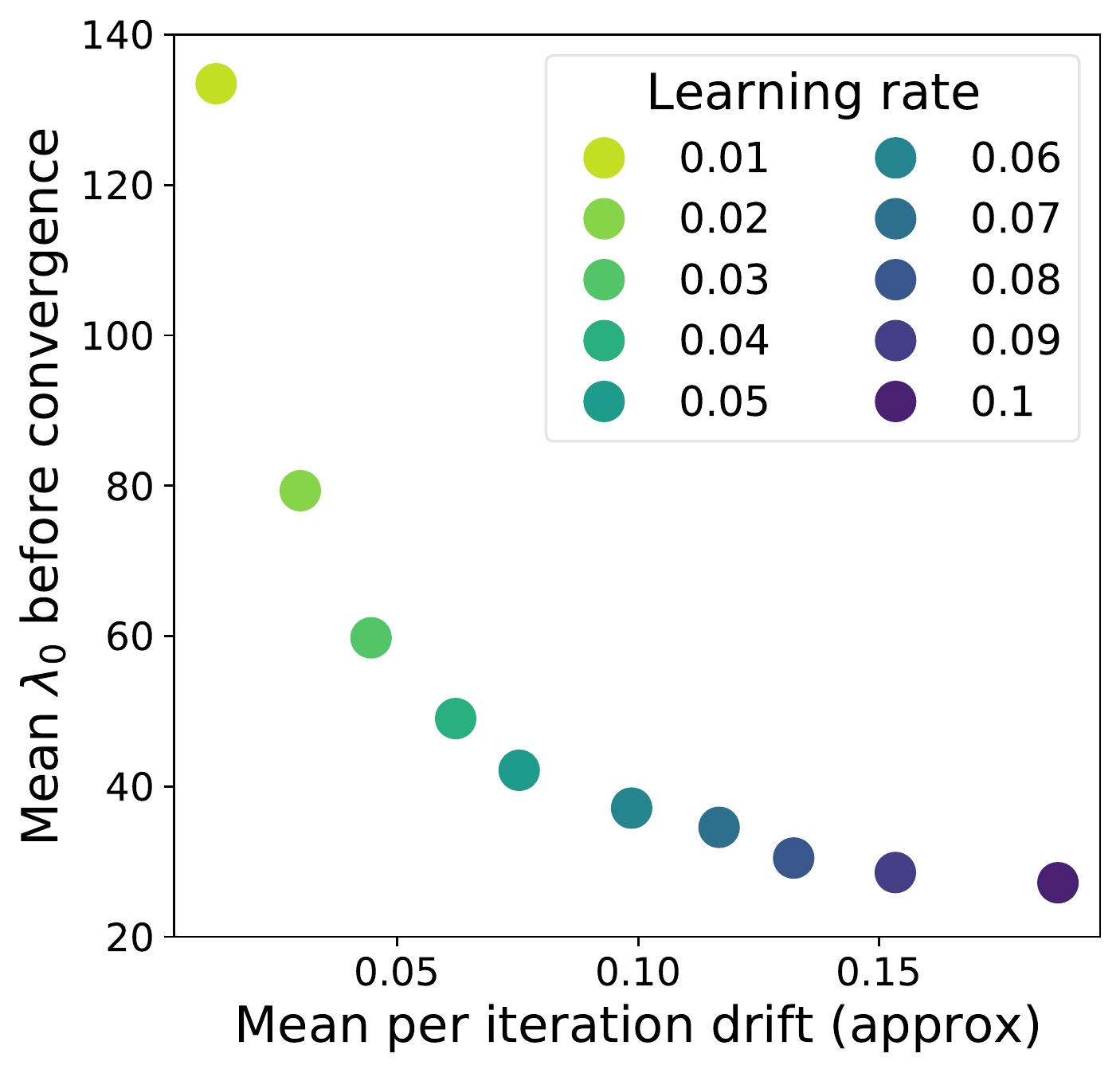}
}\end{subfigure}
\begin{subfigure}[DAL-$p$ sweep.]{
  \includegraphics[width=0.24\columnwidth]{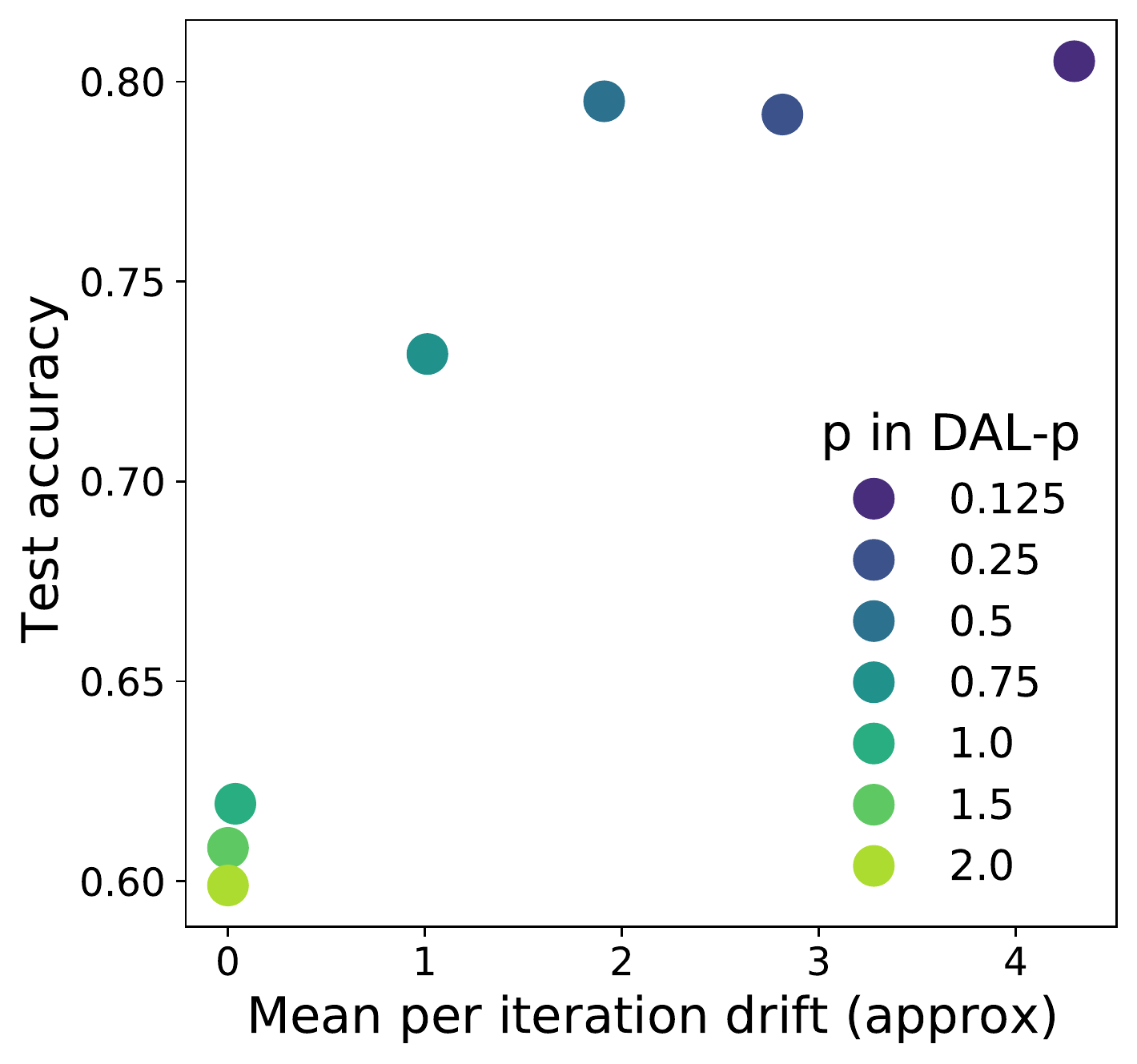}
  \includegraphics[width=0.24\columnwidth]{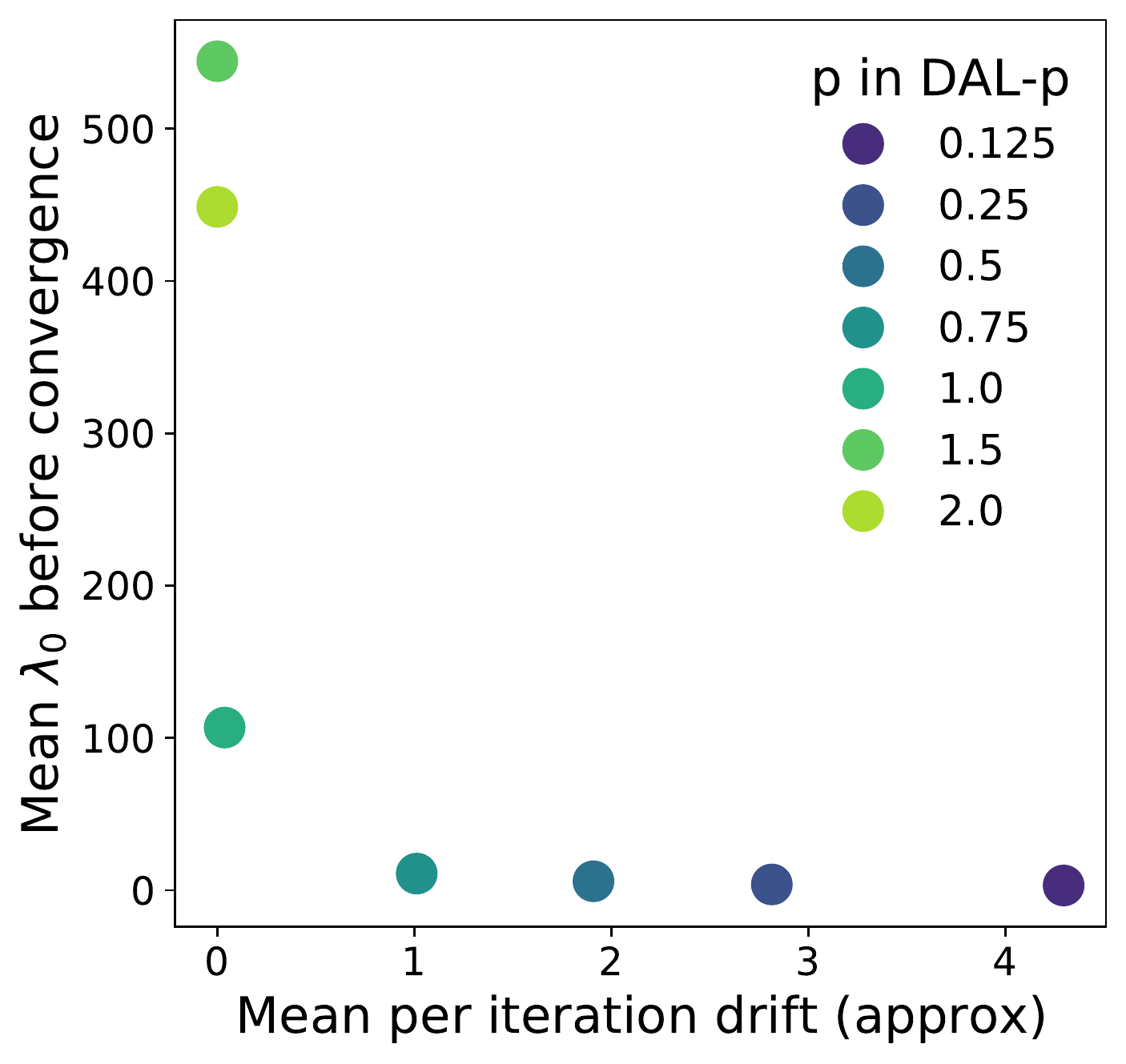}
}\end{subfigure}
\caption[VGG model trained on CIFAR-10 using full batch gradient descent, either with a fixed learning rate in a sweep or a DAL-$p$ sweep. The drift per iteration is approximated using $h^2/2 \nabla_{\vtheta}E^2 \nabla_{\vtheta}E$ or $h(\vtheta)^2/2 \nabla_{\vtheta}E^2 \nabla_{\vtheta}E$ (the adaptive learning rate) in case of DAL-p.]{The correlation between drift, test set performance and $\lambda_0$ in full batch training on CIFAR-10. The same pattern can be seen in SGD results in Figure \ref{fig:drift_lambda_sgd_connection}.} 
\label{fig:stability_lambda_performance}
\end{figure}

\rebuttalrthree{Inspired by understanding when the PF is close to the NGF, in this section we investigated the total discretisation drift of gradient descent. This led us to DAL-$p$, a method to automatically set the learning rate based on approximation to the per iteration drift of gradient descent; we have seen that DAL produces stable training and further connected discretisation drift, generalisation and flat landscapes as measured by leading Hessian eigenvalues.}

\section{Future work}
\label{sec:future_work}

\textbf{Beyond gradient descent}.
In this work we focused on understanding vanilla gradient descent. 
Understanding discretization drift via the PF can be beneficial for improving other gradient based optimization algorithms as well,
as we briefly illustrate for momentum updates with decay $m$ and learning rate $h$:
\begin{align}
\vv_t &= m \vv_{t-1} - h \nabla_{\vtheta} E(\vtheta_{t-1}); \hspace{10em}
\vtheta_t = \vtheta_{t-1} + \vv_t 
\end{align}

We can scale $\nabla_{\vtheta} E(\vtheta_{t-1})$ in the above not by a fixed learning rate $h$, but by adjusting the learning rate according to the approximation to the drift. This has two advantages: it removes the need for a learning rate sweep and it uses local landscape information in adapting the moving average, such that in areas of large drift the contribution is decreased, while it is increased in areas where the drift is small (a more formal justification is provided in Section~\ref{sec:proofs_total_per_iteration_drift}). This leads to the following updates:
\begin{align}
\vv_t &= m \vv_{t-1} - \frac{1}{2 ||{\nabla_{\vtheta}^2 E(\vtheta_{t-1})} \hat{\vg}(\vtheta)(\vtheta_{t-1})||}\nabla_{\vtheta} E(\vtheta_{t-1}) \hspace{10em}
\vtheta_t = \vtheta_{t-1} + \vv_t 
\end{align}

As with DAL-$p$, we can use powers to control the stability performance trade-off: the lower $p$, the more the current update contribution is reduced in high drift (instability) areas. We tested this approach on Imagenet and show results in Figure \ref{fig:momentum_imagenet}. The results show that integrating drift information improves the speed of convergence compared to standard gradient descent (Figure~\ref{fig:power_sweeps}), and leads to more stable training compared to using a fixed learning rate. We present additional experimental results in the Appendix.

\begin{figure}[thb]
 \includegraphics[width=0.45\columnwidth]{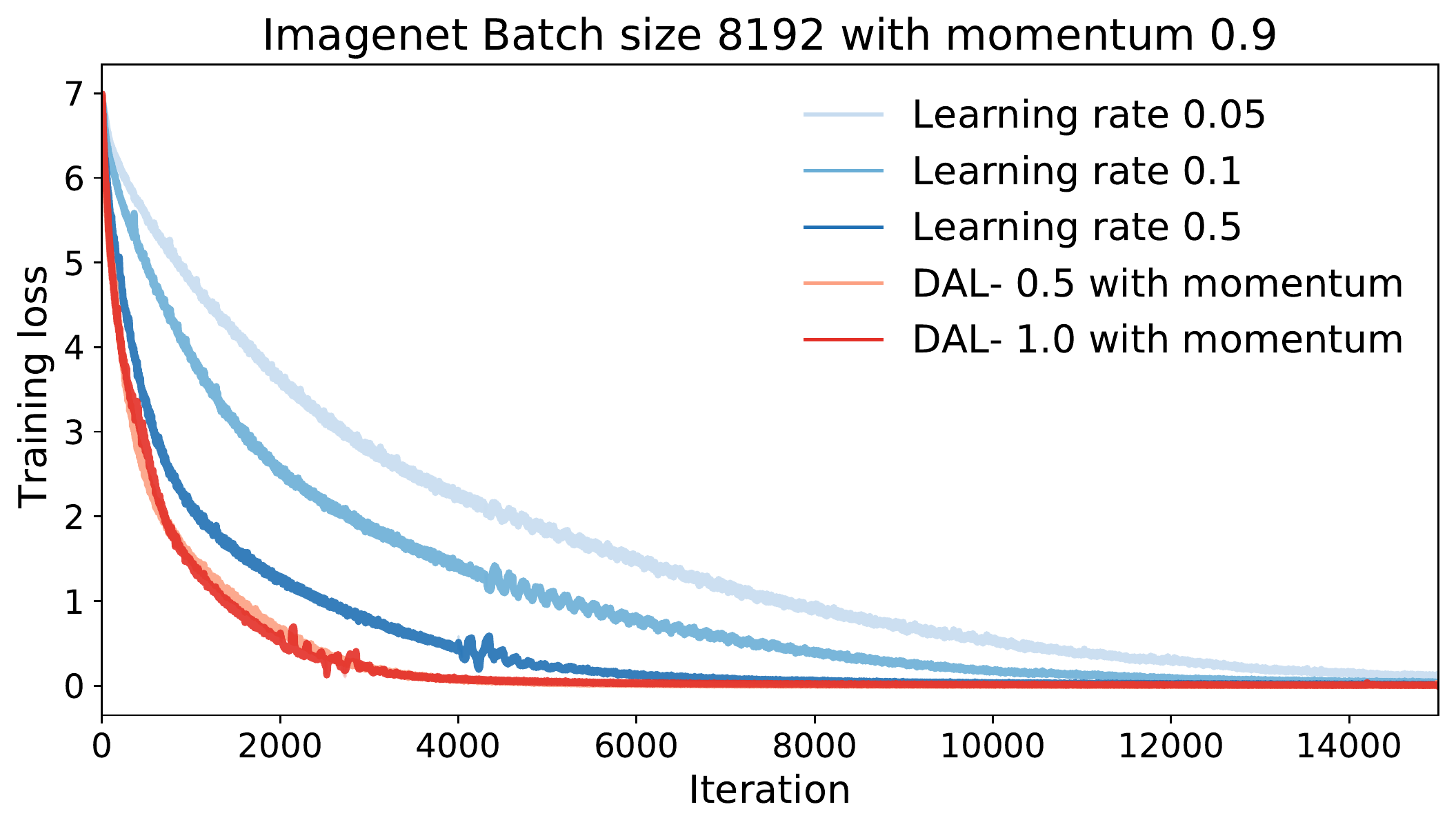}
  \includegraphics[width=0.45\columnwidth]{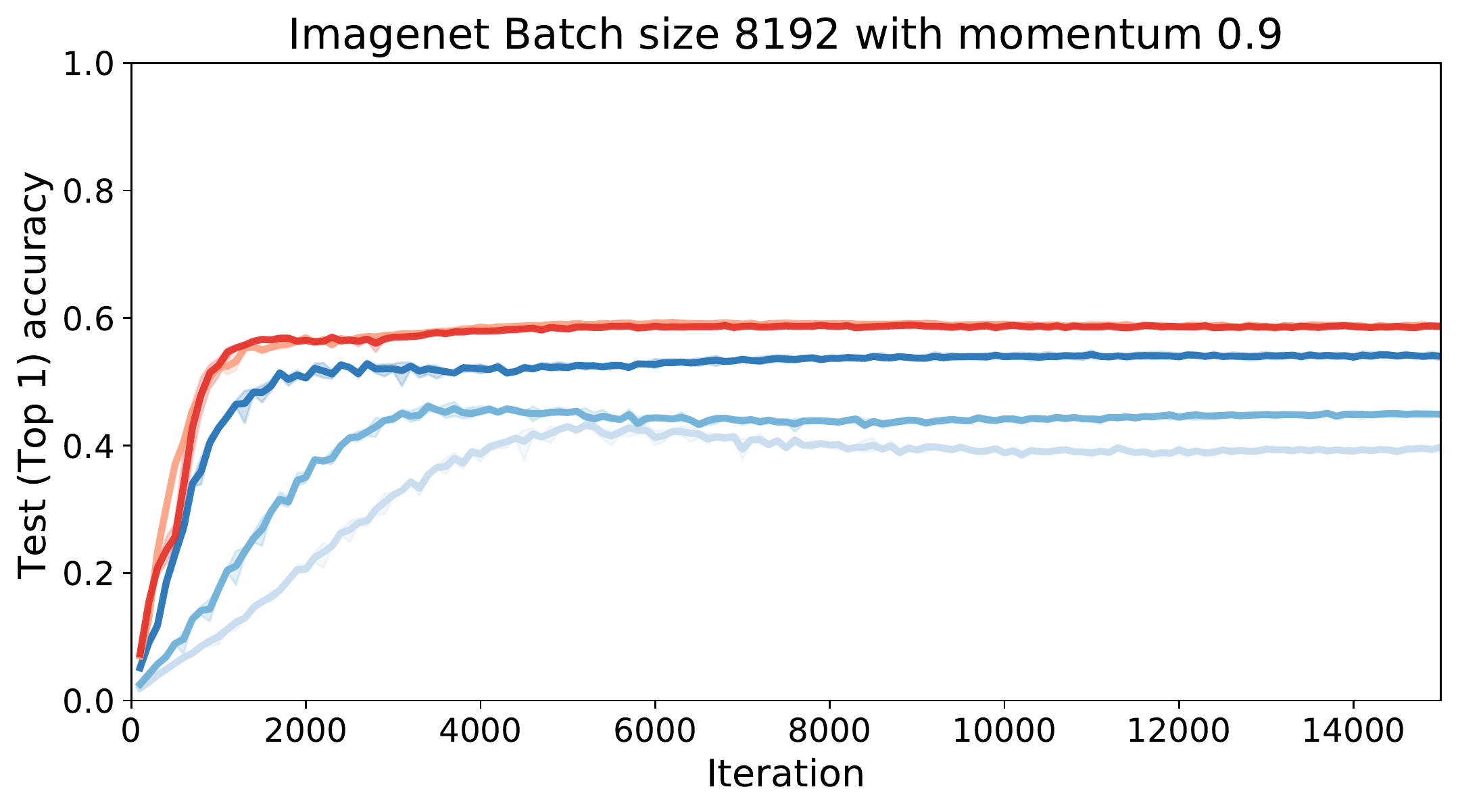}
\caption[Resnet-50 results with $0.9$ momentum and a learning rate sweep or DAL. Models trained on Imagenet.]{DAL with momentum: integrating drift information results in faster and more stable training compared to a fixed learning rate sweep. Compared to vanilla gradient descent there is also a significant performance and convergence speed boost.}
\label{fig:momentum_imagenet}
\end{figure}

Just as momentum is a common staple of optimization algorithms, so are adaptive schemes such as Adam~\citep{kingma2014adam} and Adagrad~\citep{duchi2011adaptive}, which adjust the step taken for each parameter independently. We can also use the knowledge from the PF to set a per parameter learning rate: instead of using $\norm{{\nabla_{\vtheta}^2 E(\vtheta_{t-1})} \nabla_{\vtheta} E(\vtheta_{t-1})}$ to set a global learning rate, we can use the per parameter information provided by $\nabla_{\vtheta}^2 E(\vtheta_{t-1}) \nabla_{\vtheta} E(\vtheta_{t-1})$ to adapt the learning rate of each parameter. 
We present preliminary results in the Appendix (Figures~\ref{fig:vgg_lr_scaling_per_parameter} and~\ref{fig:imagenet_lr_scaling_per_parameter}). The above two approaches (momentum and per-parameter learning rate adaptation) can be combined, bringing us closer to the most commonly used deep learning optimization algorithms.
While we do not explore this avenue here, we are hopeful that this understanding of discretization drift can be leveraged further to stabilize and improve deep learning optimization. 

\textbf{Non-principal terms}.
This work focuses on understanding the effects of the PF on the behavior of gradient descent.  
The principal terms however are not the only terms in the discretization drift: we have found one non-principal term (Eq~\ref{eq:third_order_modified_vector_field}) and have seen that it can have a stabilising effect (Figure~\ref{fig:intuition_banana}).
We provide a preliminary explanation for the stabilising effect of this non-principal term together with results measuring its value in neural network training in Section~\ref{sec:non_principal_ap} in the Appendix. \rebuttalrone{One  promising avenue of non-principal terms is theoretically modelling the change of the eigenvalues $\lambda_i$ in time;} \rebuttalrthree{another promising direction is that of implicit regularisation: while existing work which uses BEA in deep learning has found important implicit regularisation effects~\citep{igr,igr_sgd,rosca2021discretization}, we have shown here that considering only effects of $\mathcal{O}(h^3)$ is not sufficient to capture the intricacies of gradient descent, which suggests that other implicit regularisation effects could be uncovered using the non-principal terms.}

\rebuttalrthree{\textbf{Neural network theory}}.
\rebuttalrthree{Many theoretical works studying at gradient descent in the neural network context use the NGF~\citep{du2018algorithmic,elkabetz2021continuous,symmetry,jacot2018neural}. We posit that replacing NGF in these theoretical contexts with PF may yield interesting results. In contrast to the NGF, the PF allows the incorporation of the learning rate into the analysis, and unlike existing continuous time models of gradient descent, it can model unstable behaviors observed in the discrete case.}
\rebuttalrthree{An example can be seen using the Neural Tangent Kernel: \citet{jacot2018neural} model gradient descent using the NGF to show that in the infinite wide limit gradient descent for neural networks follows kernel gradient descent. The PF can be incorporated in this analysis either by replacing the NGF with the PF as a model of gradient descent or by studying the difference in the PF for infinitely wide and finite width networks, since discretisation drift could be responsible for the observed gap between finite and infinite networks in the large learning rate case \citep{lee2020finite}.}

\section{Related work}

\textbf{Modified flows for deep learning optimization}.
~\cite{igr} found the first order correction modified flow for gradient descent using BEA and uncovered its regularization effects; they were the first to show the power of BEA in the deep learning context. ~\citet{igr_sgd} find the first order error correction term in expectation during one epoch of stochastic gradient descent. 
 Modified flows have also been used for other optimizers than vanilla gradient descent: \citet{franca2020,shi2021understanding} compare momentum and Nesterov accelerated momentum; \citet{symmetry} study the symmetries of deep neural networks and use modified vector fields to show commonly used discrete updates break conservation laws present when using the NGF (for gradient descent they use the IGR flow while for momentum and weight decay they introduce different flows); \citet{kovachki2021continuous} use modified flows to understand the behavior of momentum by approximating Hamiltonian systems; \citet{dissipative} construct optimizers controlling their stability and convergence rates while \citet{stochastic_adaptive_sgd} construct optimizers with adaptive learning rates in the context of stochastic differential equations. In the context of two-player games, ~\citet{rosca2021discretization} compute the first order BEA correction terms while~\citet{chavdarova2021last} use high-resolution differential equations to shed light on the properties of different saddle point optimizers.

In concurrent work~\citet{miyagawatoward} use BEA to find a modified flow coined `Equations of Motion' (EOM) to describe gradient descent and find higher order terms, including non-principal terms; their focus is however on EOM(1), which is the IGR flow, which they use to understand scale and translation invariant layers. Their approach does not expand to complex space and does not capture the instabilities studied here (see also the discussion on the difference between the full modified flow provided by BEA and the PF in Section~\ref{sec:principal_flow}).

\textbf{Edge of stability and the importance of the Hessian}.
There have been a number of empirical studies on the Hessian in gradient descent. \cite{cohen2021gradient} observed the edge of stability behavior and performed an extensive study which led to many empirical observations used in this work. ~\citet{jastrzkebski2018relation} performed a similar study in the context of stochastic gradient descent.~\citet{sagun2017empirical,ghorbani2019investigation,papyan2018full} approximate the entire spectrum of the Hessian, and show that there are only a few negative eigenvalues, plenty of eigenvalues centered around 0, and a few positive eigenvalues with large magnitude. Similarly, ~\cite{gur2018gradient} discuss how gradient descent operates in a small subspace. ~\cite{lewkowycz2020large} discuss the large learning rate catapult in deep learning when the largest eigenvalue exceeds $2/h$.~\cite{gilmer2021loss} assess the effects of the largest Hessian eigenvalue in a large number of empirical settings. 

There have been a series of concurrent works aimed at theoretically explaining the empirical results above. \citet{ahn2022understanding} connect the edge of stability behavior with what they coin as the `relative progress ratio':
$\frac{E(\vtheta - h \nabla_{\vtheta} E) - E(\vtheta)}{h \norm{\nabla_{\vtheta} E}^2}$, which they empirically show is 0 in stable areas of training and 1 in the edge of stability areas. To see the connection between the relative progress ratio and the quantities discussed in this paper, one can perform a Taylor expansion on ${\frac{E(\vtheta - h \nabla_{\vtheta} E) - E(\vtheta)}{h \norm{\nabla_{\vtheta} E}^2} \approx \frac{- h \nabla_{\vtheta} E^T \nabla_{\vtheta} E + h^2/2 \nabla_{\vtheta} E^T \nabla_{\vtheta}^2 E \nabla_{\vtheta} E}{h \norm{\nabla_{\vtheta} E}^2} = -1 + h/2 \frac{\nabla_{\vtheta} E^T \nabla_{\vtheta}^2 E \nabla_{\vtheta} E}{\norm{\nabla_{\vtheta} E}^2}}$. 
While this ratio is related to the quantities we discuss, we also note significant differences: it is a scalar, and not a parameter length vector and thus does not capture per eigendirection behavior as we see with the stability coefficients (Section~\ref{sec:instability_deep_learning}). \citet{arora2022understanding} prove the edge of stability result occurs under certain conditions either on the learning rate or on the loss function. \citet{ma2022multiscale} empirically observe the multi-scale structure of the loss landscape in neural networks and use it to theoretically explain the edge of stability behavior of gradient descent. \citet{chen2022gradient} use low dimensional theoretical insights around a local minima to understand the edge of stability behavior. \citet{damian2022self} use a cubic Taylor expansion to show that gradient descent follows the trajectory of a projected method which ensures that $\lambda_0 < 2/h$ and $\nabla_{\vtheta} E^T \vu = 0$; their work is what inspired us to write the third order non-principal term in the form of Eq~\ref{eq:non_principal_minimisation} in the Appendix, after we had previously noted its stabilizing properties.
These important works are complementary to our own work; they do not use continuous time approaches and tackle primarily the edge of stability problem or its subcases, while we focus on understanding gradient descent and applying that understanding broadly, including but not limited to the edge of stability phenomenon.

\rebuttalrthree{\textbf{Discrete models of gradient descent}. The desire to understand learning rate specific behavior in gradient descent has been a motivation in the construction of discrete time analyses. These analyses have provided great insights, from studying noise in the stochastic gradient descent setting~\citep{liu2021noise,ziyin2021strength}, the study of overparametrized neural models and their convergence~\citep{gunasekar2018implicit,du2019gradient,allen2019convergence}, providing examples when gradient descent can converge to local maxima~\citep{ziyin2021sgd}, the importance of width for proving convergence in deep linear networks~\citep{du2019width}. We differ from these studies both in motivation and execution: we are looking for a continuous time flow which will increase the applicability of continuous time analysis of gradient descent. We do so by incorporating discretisation drift using BEA and showing that the resulting flow is a useful model of gradient descent, which captures instabilities and escape of local minima and saddle points.}

\textbf{Understanding the difference between the negative gradient flow and gradient descent}. \citet{elkabetz2021continuous} recently examined the differences between gradient descent and the NGF in the deep learning context; their work examines the importance of the Hessian in determining when gradient descent follows the NGF. Their theoretical results show that neural networks are roughly convex and thus for reasonably sized learning rates one can expect that gradient descent follows the NGF flow closely. Their results complement ours and their approach might be extended to help us understand why the PF is sufficient to shed light on many instability observations in the neural network training.

\textbf{Second-order optimization.} By using second order information (or approximations thereof) to set the learning rate, DAL is related to second-order approaches used in deep learning. 
Many second-order methods can be seen as approximates of Newton's method $\vtheta_t = \vtheta_{t-1} - \nabla^2_{\vtheta} E ^{-1}(\vtheta_{t-1}) \nabla_{\vtheta} E(\vtheta_{t-1})$.
Since computing the inverse of Hessian can be prohibitively expensive for large models, many practical methods approximate it with tractable alternatives~\citep{martens2015optimizing}. 
\citet{foret2020sharpness} propose an optimisation scheme directly aimed at minimising sharpness, and show this can improve generalisation.

\textbf{Connection between drift and generalization}.
We have made the connection between increased drift and increased generalization. This connection was first made by ~\citet{igr} through the IGR flow. Generalization has also been connected to the largest eigenvalue $\lambda_0$\citep{hochreiter1997flat,keskar2016large,jastrzkebski2018relation,lewkowycz2020large}; recently \citet{kaur2022maximum} however showed a more complex picture, primarily in the context of stochastic gradient descent. The largest eigenvalue could be a confounder to the drift as we have observed in Section~\ref{sec:trade_off}; we hope that future work can deepen these connections.

\section{Conclusion}

We have expanded on previous works which used Backward Error Analysis in deep learning to find a new continuous time flow, called \textbf{the Principal Flow}, to analyze the behavior of gradient descent. \rebuttalrthree{Unlike existing flows, the principal flow operates in complex space which enables it to better capture the behavior of gradient descent compared to existing flows, including but not limited to instability and oscillatory behavior}. 
\rebuttalrthree{We use the form of the Principal Flow to find new quantities relevant to the stability of gradient descent, and shed light on newly observed empirical phenomena, such as the edge of stability results.}
After understanding the core quantities connected to instabilities in deep learning we  devised an automatic learning rate schedule, DAL, which exhibits stable training. We concluded by cementing the connection between large discretization drift and increased generalization performance.
\rebuttalrthree{We ended by highlighting future work avenues including incorporating the principal flow in existing theoretical analyses of gradient descent which use the negative gradient flow, incorporating our understanding of the drift of gradient descent in other optimization approaches and specializing the PF for neural network function approximators.}

\vspace{1em}
\textbf{Acknowledgments}. We would like to thank the TMLR anonymous reviewers and the TMLR Action Editor for their useful feedback and comments. We would also like to thank Soham De and Michael Munn for discussions and feedback; and Frederic Besse, Marc Deisenroth, Patrick Cole, Shakir Mohamed and Timothy Lillicrap for their support.

\clearpage
\bibliography{references}

\clearpage
\appendix

\appendixpage
\startcontents[sections]
\printcontents[sections]{l}{1}{\setcounter{tocdepth}{2}}

\section{Proofs}
\label{sec:all_proofs}

\subsection{BEA proof structure}
\label{sec:bea_proof_structure}

\begin{figure}[h]
\begin{tikzpicture}[every text node part/.style={align=center,inner sep=0,outer sep=0}][overlay]
\coordinate (theta_t_minus_1) at (0,0);
\coordinate (theta_t) at (2.2,-1);

\node(draw) at ($(theta_t_minus_1) + (+0.3,-0.32)$) {$\vtheta_{t-1}$};
\node(draw) at ($(theta_t) + (0.5,-0.2)$) {$\vtheta_{t}$};
\node(draw) at ($(theta_t) + (5.8,-0.2)$) {$= {\color{purple} \vtheta_{t-1}}  {\color{blue}- h \nabla_{\vtheta} E(\vtheta_{t-1})}$};

\coordinate (first_time_transition) at ($(theta_t_minus_1) + (+0.55,1.6)$);

\coordinate (mod_cont_theta_t) at (2.2,0.8);
n!
\node(draw) at ($(mod_cont_theta_t) + (5.8,-0.)$) {$\tilde{\vtheta}(h) \hspace{1.8em}= \sum_{p=0}^{\infty} \frac{h^p}{p!} \tilde{\vtheta}^{(p)}  = {\color{purple} \vtheta_{t-1}} {\color{blue}- h \nabla_{\vtheta} E(\vtheta_{t-1})} +  \sum_{i=2}^{n+1} h^i \underbrace{l_i(\vtheta_{t-1})}_{{\color{red}{\mathbf{0}}}} + {\color{red}\mathcal{O}(h^{n+2})} $};

\draw [thick,dashed] (theta_t_minus_1) -- (theta_t);

\draw [thick]  (theta_t_minus_1) to[out=50,in=180]  node[near end,above] {$\hatdotvtheta$} (mod_cont_theta_t);

\draw [
    thick,
    decoration={
        brace,
        mirror,
        raise=0.1cm
    },
    decorate
] (theta_t) -- (mod_cont_theta_t)
node [pos=0.5,anchor=west,xshift=0.15cm,yshift=-0.0cm] { \scriptsize ${\color{red}\mathcal{O}(h^{n+2})}$};
\end{tikzpicture}
 \caption{BEA finds continuous modified flows which describe the gradient descent update with learning rate $h$ with an error of $\mathcal{O}(h^{n+2})$. We identify $f_1, \cdots f_n$ so that terms of order $\mathcal{O}(h^p), 2 \le  p \le n+1$ in $\tilde{\vtheta}(h)$ are $\mathbf{0}$.}
  \label{fig:bea_proof_structure}
\end{figure}

\textbf{General structure.}
The goal of BEA is to find the functions $f_1$, ... $f_{n}$ such that the flow
\begin{align}
    \hatdotvtheta =  - \nabla_{\vtheta} E + h f_1(\vtheta) + \cdots + h^n f_n(\vtheta)
\end{align}

has an error $\| \vtheta_t - \tilde{\vtheta}(h)\|$ of order $\mathcal{O}(h^{n+2})$ after 1 gradient descent step of learning rate $h$. To do so requires multiple steps (visualized in Figure~\ref{fig:bea_proof_structure}): 
\begin{enumerate}
    \item Expand $\tilde{\vtheta}(h)$ via a Taylor expansion in $h$: $\tilde{\vtheta}(h) = \sum_{p=0}^{\infty} \frac{h^p}{p!} \tilde{\vtheta}^{(p)}$;
    \item Expand each $\tilde{\vtheta}^{(p)}$ up to order $\mathcal{O}(h^{n+2-p})$ as a function of $f_1$, ... $f_{n}$ via the chain rule;
    \item Group together terms of the same order in $h$ in the expansion, up to order $n+2$. 
    \begin{align}
    \tilde{\vtheta}(h) = \sum_{i=0}^{n+1} h^i l_i(\vtheta) + \mathcal{O}(h^{n+2}) = \vtheta_{t-1} - h \nabla_{\vtheta} E(\vtheta_{t-1}) + \sum_{i=2}^{n+1} h^i l_i(\vtheta_{t-1}) + \mathcal{O}(h^{n+2})
    \end{align}
    \item Compare the above update with the gradient descent update $\vtheta_t = \vtheta_{t-1} - h \nabla_{\vtheta} E(\vtheta_{t-1})$ and conclude that $l_i=\mathbf{0}$, $\forall i \in \{2, n + 1\}$. Use this to identify $f_1$, ... $f_n$.
\end{enumerate}

\textbf{Notation and context}: all proofs below apply to general Euler updates not only gradient descent. We thus assume an update function $f$ with the Euler step $\vtheta_t = \vtheta_{t-1} + h f(\vtheta_{t-1})$. We can then use BEA to find the higher order correction terms describing the Euler discrete update up to a certain order, and replace $f = - \nabla_{\vtheta} E$ to obtain the corresponding results for gradient descent. When we perform a Taylor expansion in $h$ we often drop in notation the evaluation at $h=0$ and we make that implicit.

\subsection{Third order flow}
\label{sec:third_order_flow_proof}

\begin{theorem} The modified flow 
\begin{align}
 \dot{\vtheta} = f - h \frac{1}{2} \nabla_{\vtheta} f f + h^2 \left(\frac{1}{3}  (\nabla_{\vtheta} f)^2 f  + \frac {1} {12} f^T (\nabla^2_{\vtheta} f) f\right)
 \label{eq:third_order_euler_app}
\end{align}

with $\vtheta(0) = \vtheta_{t-1}$ follows an Euler update $\vtheta_{t} = \vtheta_{t-1} + hf(\vtheta_{t-1})$ with an error of $\mathcal{O}(h^4)$ after 1 gradient descent step.
\end{theorem}
\begin{proof}

Since we are using BEA, we wil be looking for functions $f_1$ and $f_2$ such that the modified flow:
\begin{align}
 \dot{\vtheta} = f + h f_1 + h^2 f_2
 \label{eq:bea_structure_third_order}
\end{align}

follow the steps of GD with an error up to $\mathcal{O}(h^4)$.
We now perform a Taylor expansion of step size $h$ of the above modified flow to be able to see the displacement in that time up to order $\mathcal{O}(h^4)$.

We obtain (all function evaluations of $f$  and $f_i$ are at $\vtheta_{t-1}$ which we omit for simplicity, and annotate proof steps, CR denotes Chain Rule): 
\begin{align} 
   \vtheta&(h) = \vtheta(0) + h \dot{\vtheta}(0) + \frac{1}{2} h^2 \dot{\dot{\vtheta}}(0) + \frac{1}{6} h^3 \dot{\dot{\dot{\vtheta}}}(0) + \mathcal{O}(h^4)\\
                     &= \vtheta_{t-1} + h \left[ (f + h f_1 + h^2 f_2) \right] + \frac{1}{2} h^2 \frac{d}{dt}(f + h f_1 + h^2 f_2) + \frac{1}{6} h^3 \dot{\dot{\dot{\vtheta}}} + \mathcal{O}(h^4)  \owntag{using Eq \ref{eq:bea_structure_third_order}} \\
                     &= \vtheta_{t-1} + h f + h^2 f_1 + h^3 f_2 + \frac{1}{2} h^2 (\nabla_{\vtheta} (f + h f_1 + h^2 f_2))(f + h f_1 + h^2 f_2) + \frac{1}{6} h^3 \dot{\dot{\dot{\vtheta}}} + \mathcal{O}(h^4) \owntag{CR} \\
                    &= \vtheta_{t-1} + h f + h^2 f_1 + h^3 f_2 + \frac{1}{2} h^2 (\nabla_{\vtheta} f f + h  \nabla_{\vtheta} f f_1 + h  \nabla_{\vtheta}f_1 f) + \frac{1}{6} h^3 \dot{\dot{\dot{\vtheta}}} + \mathcal{O}(h^4) \\
                     &= \vtheta_{t-1} + h f + h^2 f_1 + h^3 f_2 + \frac{1}{2} h^2 (\nabla_{\vtheta} f f + h  \nabla_{\vtheta} f f_1 + h  \nabla_{\vtheta}f_1 f) + \frac{h^3}{6} \frac{d}{dt} \left[\nabla_{\vtheta} f f + h  \nabla_{\vtheta} f f_1 + h  \nabla_{\vtheta}f_1 f\right] + \mathcal{O}(h^4) \\
                    &= \vtheta_{t-1} + h f + h^2 f_1 + h^3 f_2 + \frac{1}{2} h^2 (\nabla_{\vtheta} f f + h  \nabla_{\vtheta} f f_1 + h  \nabla_{\vtheta}f_1 f) + \frac{1}{6} h^3 \frac{d}{dt} \left[\nabla_{\vtheta} f f\right] + \mathcal{O}(h^4) \\
                    &= \vtheta_{t-1} + h f + h^2 f_1 + h^3 f_2 + \frac{1}{2} h^2 (\nabla_{\vtheta} f f + h  \nabla_{\vtheta} f f_1 + h  \nabla_{\vtheta}f_1 f) + \frac{1}{6} h^3  \nabla_{\vtheta}\left[\nabla_{\vtheta} f f\right](f + h f_1 + h^2 f_2) + \mathcal{O}(h^4)  \owntag{CR} \\
                    &= \vtheta_{t-1} + h f + h^2 f_1 + h^3 f_2 + \frac{1}{2} h^2 (\nabla_{\vtheta} f f + h  \nabla_{\vtheta} f f_1 + h  \nabla_{\vtheta}f_1 f) + \frac{1}{6} h^3  \nabla_{\vtheta}\left[\nabla_{\vtheta} f f\right] f + \mathcal{O}(h^4)
\label{eq:modified_flow_exp}
\end{align}

If we match the terms with the Euler update, we obtain that:

\begin{align}
f_1 = -\frac{1}{2} \nabla_{\vtheta} f f
\label{eq:f_1_third_order}
\end{align}
and

\begin{align}
f_2 &= - \frac{1}{2}  \left[\nabla_{\vtheta} f  f_1  + \nabla_{\vtheta}f_1 f \right] - \frac{1}{6} \nabla_{\vtheta}\left[\nabla_{\vtheta} f f\right]f \\
f_2 &= - \frac{1}{2}  \left[\nabla_{\vtheta} f (-\frac{1}{2} \nabla_{\vtheta} f f)  + \nabla_{\vtheta} (-\frac{1}{2} \nabla_{\vtheta} f f) f\right] - \frac{1}{6} \nabla_{\vtheta}\left[\nabla_{\vtheta} f f\right] f \\
f_2 &=  \frac{1}{4}  (\nabla_{\vtheta} f)^2 f + \frac 1 4 \nabla_{\vtheta} (\nabla_{\vtheta} f f) f - \frac{1}{6} \nabla_{\vtheta}\left[\nabla_{\vtheta} f f\right] f \\
f_2 &=  \frac{1}{4}  (\nabla_{\vtheta} f)^2 f + \frac {1} {12}  \nabla_{\vtheta} (\nabla_{\vtheta} f f) f \\
f_2 &=  \frac{1}{4}  (\nabla_{\vtheta} f)^2  f + \frac {1} {12}  \left({(\nabla_{\vtheta} f)}^2 f + f^T \nabla^2_{\vtheta} f  \right) f\\
f_2 &=  \frac{1}{3}  (\nabla_{\vtheta} f)^2  f + \frac {1} {12} f^T (\nabla^2_{\vtheta} f)  f
\label{eq:f_2_third_order}
\end{align}

Replacing $f_1$ and $f_2$ in Eq~\ref{eq:bea_structure_third_order} leads to the desired Eq~\ref{eq:third_order_euler_app}.

\end{proof}

In the case of gradient descent, substituting $f = -\nabla_{\vtheta} E$ in $f_1$ and $f_2$, we obtain (Eq~\ref{eq:third_order_modified_vector_field}):

\begin{align}
\dot{\vtheta} = -\nabla_{\vtheta}E   -\frac{h}{2} \nabla_{\vtheta}^2E  \nabla_{\vtheta}E - h^2 \left( \frac{1}{3}  (\nabla_{\vtheta}^2E)^2 \nabla_{\vtheta} E + \frac {1} {12} \nabla_{\vtheta}E^T (\nabla_{\vtheta}^3E) \nabla_{\vtheta}E\right)
\end{align}

\subsection{Higher order terms of the form $(\nabla_{\vtheta} f)^n f$ or $\left(\nabla_{\vtheta}^2 E\right)^n \nabla_{\vtheta} E$}
\label{sec:higher_order_proofs}

\begin{theorem}
The modified flow with an error of order $\mathcal{O}(h^{p+1})$ to the Euler update  ${\vtheta_t = \vtheta_{t-1} + h f(\vtheta_{t-1})}$
has the form:
\begin{align}
    \dot{\vtheta} = \sum_{n=0}^{p} \frac{(-1)^{n}}{n+1} (\nabla_{\vtheta} f)^n f + \mathcal{C}(\nabla_{\vtheta}^2 f)
\end{align}

where $\mathcal{C}(\nabla_{\vtheta}^2 f)$ denotes a class of functions defined as a sum of individual terms each containing higher than first order derivatives applied to $f$.

\label{thm:general_bea}
\end{theorem}
\begin{proof}
Consider the modified flow given by BAE:
\begin{align}
\dot{\vtheta} = \tilde{f} = f + h f_1 + h^2 f_2 + ... + h^n f_n + ...
\label{eq:modified_bea_app}
\end{align}

 For the proof we need to show that the only terms not involving higher order derivatives are of the form $(\nabla_{\vtheta} f)^{n}f $  and we need to find the coefficients of $(\nabla_{\vtheta} f)^{n} f$ in $f_n$, which we will call $c_n$. We already know that $c_0 = 1$, $c_1 = -\frac{1}{2}$ (Eq.~\ref{eq:f_1_third_order}) and $c_3 = \frac 1 3$ (Eq.~\ref{eq:f_2_third_order}).
We use the general structure provided in Section~\ref{sec:bea_proof_structure}.

We start by  expanding $\vtheta(h)$ via Taylor expansion (Step 1 in the proof structure in Section~\ref{sec:bea_proof_structure}):
    \begin{align}
\vtheta(h) = \vtheta(0) + h \dot{\vtheta} + \sum_{k=2}^{\infty} \frac{1}{k!} \vtheta^{(k)} 
                 = \vtheta_{t-1} +  h f(\vtheta_{t-1}) + \sum_{n=1}^{\infty} h^{n+1} f_n + \sum_{k=2}^{\infty} \frac{h^{k}}{k!}  \vtheta^{(k)} 
\label{eq:modified_ode_expansion}
\end{align}

We now express in ${\vtheta}^{(k)}$ in terms of the wanted quantities $f_n$ via  Lemma~\ref{lemma:theta_n_expansion}:
    \begin{align}\vtheta^{(k)} = \sum_{m \ge 0}h^m  \frac{d^{k-2}}{dt^{k-2}} \sum_{i+j=m}  \nabla_{\vtheta} f_j f_i \hspace{1em} k\ge2
    \label{eq:theta_k_expansion_result}
    \end{align}

By replacing Eq~\ref{eq:theta_k_expansion_result} in Eq~\ref{eq:modified_ode_expansion} we obtain (Step 2 in the proof structure in Section~\ref{sec:bea_proof_structure}):
       \begin{align}
       \vtheta(h) &= \vtheta_{t-1} +  h f(\vtheta_{t-1}) + \sum_{n=1}^{\infty} h^{n+1} f_n(\vtheta_{t-1}) + \sum_{k=2}^{\infty} \frac{h^k}{k!} \sum_{m \ge 0}h^m  \frac{d^{k-2}}{dt^{k-2}} \sum_{i+j=m}  \nabla_{\vtheta} f_j(\vtheta_{t-1}) f_i(\vtheta_{t-1})  \\
       & =  \vtheta_{t-1} +  h f(\vtheta_{t-1}) + \sum_{n=1}^{\infty} h^{n+1} f_n(\vtheta_{t-1}) +  \sum_{m \ge 0} \sum_{k=2}^{\infty} \frac{h^{k+m }}{k!}  \frac{d^{k-2}}{dt^{k-2}} \sum_{i+j=m}  \nabla_{\vtheta} f_j(\vtheta_{t-1}) f_i(\vtheta_{t-1})
       \label{eq:theta_h_exp_f_i}
\end{align}

 We note that due to the derivative wrt to $t$ and a dependence of $h$ in the modified vector field, there will be a dependence of $h$ in $\frac{d^{k-2}}{dt^{k-2}} \sum_{i+j=m}  \nabla_{\vtheta} f_j f_i$, since the chain rule will expose $\frac{d \vtheta}{d t}$ which depends on $h$ (Eq~\ref{eq:modified_bea_app}). To highlight this dependence on $h$ we introduce the notation:
       \begin{align} 
        \frac{d^{k-2}}{dt^{k-2}} \sum_{i+j=m}  \nabla_{\vtheta} f_j f_i = \sum_p h^p A_{k,m}^{p}
        \label{eq:f_n_based_on_a}
     \end{align}

    where $A_{k,m}^{p}$ is defined to be the term of order $h^p$ in $\frac{d^{k-2}}{dt^{k-2}} \sum_{i+j=m}  \nabla_{\vtheta} f_j f_i$; thus by definition $A_{k,m}^{p}$ does not depend on $h$. Replacing the above in Eq~\ref{eq:theta_h_exp_f_i}, we conclude Step 3 outlined in the proof structure in Section~\ref{sec:bea_proof_structure}:
    \begin{align}
       \vtheta(h)  =  \vtheta_{t-1} +  h f(\vtheta_{t-1}) + \sum_{n=1}^{\infty} h^{n+1} f_n(\vtheta_{t-1}) +  \sum_{m \ge 0} \sum_{k=2}^{\infty} \frac{h^{k+m }}{k!} \sum_p  h^p A_{k,m}^{p}(\vtheta_{t-1})
\end{align}

Since the Euler update $\vtheta_t = \vtheta_{t-1} +hf(\vtheta_{t-1})$ does not involve terms of order higher than $h$ the $f_n$ terms are obtained by cancelling the terms of order $\mathcal{O}(h^{n+1})$  in $\vtheta(h)$ (Step 4 outlined in the proof structure in Section~\ref{sec:bea_proof_structure}), we have that for $k+m+p = n+1$:
     \begin{align}
        f_n = - \sum_{m \ge 0} \sum_{k=2}^{\infty} \frac{1}{k!}A_{k,m}^{n+1-(k+m)}
     \end{align}

   Since $A_{k,m}^{n+1-(k+m)}$ depends on the functions $f_n$, we have found a recursive relation for $f_n$. We now use induction to show that
     \begin{align}
      f_n = c_n  (\nabla_{\vtheta} f)^{n} f + \mathcal{C}(\nabla^2_{\vtheta} f)
     \end{align}
    where $\mathcal{C}(\nabla^2_{\vtheta} f)$ denotes a class of functions defined as a sum of individual terms each containing higher than first order derivatives applied to $f$  (Lemma~\ref{lemma:induction}). The same proof shows the recurrence relation for $c_n$:
    \begin{align}
        c_n = - \sum_{k=2}^{n+1} \frac{1}{k!} \sum_{l_1+ l_2+ ...+l_k = n - k + 1} c_{l_1} c_{l_2} ... c_{l_k}
    \end{align}

with solution $c_n = \frac{(-1)^{n}}{n+1}$ (Lemma~\ref{lemma:c_n_recurrence_proof})

We have found that each $f_n$ can be written as $\frac{(-1)^{n}}{n+1} (\nabla_{\vtheta} f)^n f + \mathcal{C}(\nabla_{\vtheta}^2 f)$. Replacing this in Eq.\ref{eq:modified_bea_app} concludes the proof.

\end{proof}

In the case of gradient descent, where $f = - \nabla_{\vtheta} E$ we have that the modified flow has the form 
\begin{align}
    \dot{\vtheta} &= \sum_{n=0}^{p} \frac{(-1)^{n}}{n+1}  (-\nabla_{\vtheta} f)^n (- \nabla_{\vtheta} E) + \mathcal{C}(\nabla_{\vtheta}^2 f) \\
                &= \sum_{n=0}^{p} \frac{(-1)^{n}}{n+1}  (-\nabla_{\vtheta}^2 E)^n (- \nabla_{\vtheta} E) + \mathcal{C}(\nabla_{\vtheta}^2 f) \\
                &= \sum_{n=0}^{p} \frac{-1}{n+1}  (\nabla_{\vtheta}^2 E)^n (\nabla_{\vtheta} E) + \mathcal{C}(\nabla_{\vtheta}^2 f) 
\end{align}

which is the statement of Theorem~\ref{thm:order_n_flow} in the main paper.

\begin{lemma} 
    \begin{align}
    \vtheta^{(k)} = \sum_{m \ge 0}h^m  \frac{d^{k-2}}{dt^{k-2}} \sum_{i+j=m}  \nabla_{\vtheta} f_j f_i \hspace{1em} k\ge2
    \end{align}
\label{lemma:theta_n_expansion}
\end{lemma}

\begin{proof}

\begin{align}
\dot{\vtheta} &= \tilde{f} \\
\dot{\dot{\vtheta}} &= \nabla_{\vtheta}\tilde{f}  \tilde{f} \\
{\vtheta}^{(k)} &= \frac{d^{k-2}}{dt^{k-2}} \dot{\dot{\vtheta}}= \frac{d^{k-2}}{dt^{k-2}}  \nabla_{\vtheta}\tilde{f} \tilde{f}
\end{align}

We have that
\begin{align}
 \nabla_{\vtheta}\tilde{f} \tilde{f} = \left( \sum_i h^i \nabla_{\vtheta} f_i \right) \left( \sum_j h^j f_j \right)= \sum_{m \ge 0} h^m \sum_{i+j=m}  \nabla_{\vtheta} f_i f_j
\end{align}

which concludes the proof.
We note that this is an adaptation of Lemma A.2 from~\citep{igr}. We need this adaptation to avoid assuming that $f$  is a symmetric vector field - i.e that $\tilde{f} \nabla_{\vtheta}\tilde{f} = \frac{1}{2} \nabla_{\vtheta} \norm{\tilde{f}}^2$.
\end{proof}

\begin{lemma} By induction we have that:
    \begin{align}
      f_n = c_n  (\nabla_{\vtheta} f)^{n} f + \mathcal{C}(\nabla^2_{\vtheta} f)
     \end{align}
with 
\begin{align}
c_n = - \sum_{k=2}^{n+1} \frac{1}{k!} \sum_{l_1+ l_2+ ...+l_k = n - k + 1} c_{l_1} c_{l_2} ... c_{l_k}
\end{align}
\label{lemma:induction}
\end{lemma}

\begin{proof}

\textbf{Base step}: This is true for $n=1, 2$ and $3$, for which we computed all terms in the BEA expansion.

\textbf{Induction step}:
We know that (Eq~\ref{eq:f_n_based_on_a}):
    \begin{align}
        f_n = - \sum_{m \ge 0} \sum_{k=2}^{\infty} \frac{1}{k!}A_{k,m}^{n+1-(k+m)}
        \label{eq:f_n_as_function_a}
     \end{align}

where  $A_{k,m}^{p}$ is defined as the coefficient of $h^p$ in:
\begin{align}
        \frac{d^{k-2}}{dt^{k-2}} \sum_{i+j=m}  \nabla_{\vtheta} f_j f_i = \sum_p h^p A_{k,m}^{p}
\end{align}

 Under the induction hypothesis we can use that 
 \begin{align}
      f_i = c_i  (\nabla_{\vtheta} f)^{i} f + \mathcal{C}(\nabla^2_{\vtheta} f) \hspace{2em} \forall i < n
\end{align}

And thus that the flow is of the form:
 \begin{align}
\dot{\vtheta} =  f + \sum_{i=1}^{n-1} c_i(\nabla_{\vtheta} f)^{i} f + \mathcal{C}(\nabla^2_{\vtheta} f) + \mathcal{O}(h^n)
\label{eq:flow_under_ih}
 \end{align}

To find $f_n$ (Eq~\ref{eq:f_n_as_function_a}) we need to find $A_{k,m}^{n+1-(k+m)}$ with $k\ge 2$ and $m\ge 0$. We thus need to find  $A_{k,m}^{p}$ with $p \le n-1$ (since $p \le (n+1) - (k+m) \text{ and } k\ge 2, m\ge 0 \implies p \le n -1$). We do so expanding $\frac{d^{k-2}}{dt^{k-2}} \sum_{i+j=m}  \nabla_{\vtheta} f_j f_i$ under the induction step. Since $i+j = m$ and $k+m \le n+1$ and $k\ge 2$, we have that $m \le n-1$ and thus $i, j \le n-1$ and thus we can apply the induction hypothesis for $f_i$ and $f_j$.

We annotate the steps with IH (Induction Hypothesis, often by using the form in Eq. \ref{eq:flow_under_ih}) and CR (Chain Rule) and S (Simplifying or grouping terms of $\mathcal{C}(\nabla^2_{\vtheta} f))$ by using that products and sum of terms in this function class also belong in this function class).

\begin{align}
&\frac{d^{k-2}}{dt^{k-2}} \sum_{i+j=m}  \nabla_{\vtheta} f_j f_i \\
&= 
 \frac{d^{k-2}}{dt^{k-2}} \sum_{i+j=m} \left(\nabla_{\vtheta} (c_j (\nabla_{\vtheta} f)^{j} f + \mathcal{C}(\nabla^2_{\vtheta} f)\right)(c_i (\nabla_{\vtheta} f)^{i} f+ \mathcal{C}(\nabla^2_{\vtheta} f)) \hspace{5em} \owntag{IH}\\
&= \frac{d^{k-2}}{dt^{k-2}} \sum_{i+j=m} \left(c_j (\nabla_{\vtheta} f)^{j+1} +c_j  f^T \nabla_{\vtheta} (\nabla_{\vtheta} f)^{j}  +   \mathcal{C}(\nabla^2_{\vtheta} f)\right)(c_i (\nabla_{\vtheta} f)^{i} f + \mathcal{C}(\nabla^2_{\vtheta} f)) \owntag{CR} \\
&= \frac{d^{k-2}}{dt^{k-2}} \sum_{i+j=m} \left(c_j (\nabla_{\vtheta} f)^{j+1} +   \mathcal{C}(\nabla^2_{\vtheta} f)\right)(c_i  (\nabla_{\vtheta} f)^{i} f + \mathcal{C}(\nabla^2_{\vtheta} f)) \owntag{S} \\
&= \frac{d^{k-2}}{dt^{k-2}} \sum_{i+j=m} c_i c_j (\nabla_{\vtheta} f)^{m+1} f+ 
    \mathcal{C}(\nabla^2_{\vtheta} f)  \owntag{S} \\
    &= \frac{d^{k-3}}{dt^{k-3}} \sum_{i+j=m} \frac{d}{dt} c_i c_j (\nabla_{\vtheta} f)^{m+1} f+ 
    \mathcal{C}(\nabla^2_{\vtheta} f) \\
        &= \frac{d^{k-3}}{dt^{k-3}} \sum_{i+j=m} \frac{d}{d\vtheta} \left(c_i c_j (\nabla_{\vtheta} f)^{m+1} f +  
    \mathcal{C}(\nabla^2_{\vtheta} f) \right)\frac{d\vtheta}{d t} \owntag{CR}\\
&= \frac{d^{k-3}}{dt^{k-3}}  \sum_{i+j=m}  \frac{d}{d \vtheta} \left(c_i c_j (\nabla_{\vtheta} f)^{m+1} f + \mathcal{C}(\nabla^2_{\vtheta} f) \right) \left(\sum_{l_1=0}^{n-1} h^{l_1} f_{l_1} + \mathcal{O}(h^n)\right)\owntag{Eq. \ref{eq:flow_under_ih}} \\
&= \frac{d^{k-3}}{dt^{k-3}}  \sum_{i+j=m}  \frac{d}{d \vtheta} \left(c_i c_j (\nabla_{\vtheta} f)^{m+1} f + \mathcal{C}(\nabla^2_{\vtheta} f) \right) (\sum_{l_1=0}^{n-1} h^{l_1} f_{l_1})   + \mathcal{O}(h^n)  \\
&= \frac{d^{k-3}}{dt^{k-3}}  \sum_{i+j=m}   \left(c_i c_j (\nabla_{\vtheta} f)^{m+2} + \mathcal{C}(\nabla^2_{\vtheta} f) \right) (\sum_{l_1=0}^{n-1} h^{l_1} f_{l_1})  + \mathcal{O}(h^n) \owntag{CR, S}\\
&= \frac{d^{k-3}}{dt^{k-3}}  \sum_{i+j=m}   \left(c_i c_j (\nabla_{\vtheta} f)^{m+2} + \mathcal{C}(\nabla^2_{\vtheta} f)\right) \left(\sum_{l_1=0}^{n-1} h^{l_1} c_{l_1}  (\nabla_{\vtheta} f) ^{l_1} f + \mathcal{C}(\nabla^2_{\vtheta} f )\right) + \mathcal{O}(h^n) \owntag{IH}\\
&= \frac{d^{k-3}}{dt^{k-3}}  \sum_{i+j=m} (\sum_{l_1=0}^{n-1} h^{l_1} c_i c_j c_{l_1} (\nabla_{\vtheta} f)^{l_1+m+2} f + \mathcal{C}(\nabla^2_{\vtheta} f)) + \mathcal{O}(h^n)
\owntag{S} \\
&= \frac{d^{k-4}}{dt^{k-4}}  \sum_{i+j=m} (\sum_{l_1=0}^{n-1} h^{l_1} c_i c_j c_{l_1}  \frac{d}{d \vtheta} \left( (\nabla_{\vtheta} f)^{l_1+m+2} f + \mathcal{C}(\nabla^2_{\vtheta} f))\right)(\sum_{l_2=0}^{n-1} h^{l_2} f_{l_2})  + \mathcal{O}(h^n) \owntag{CR, Eq. \ref{eq:flow_under_ih}}\\
&= \frac{d^{k-4}}{dt^{k-4}}  \sum_{i+j=m} (\sum_{l_1=0}^{n-1} h^{l_1} c_i c_j c_{l_1} \left(( \nabla_{\vtheta} f )^{l_1+m+3} + \mathcal{C}(\nabla^2_{\vtheta} f)\right) \left(\sum_{l_2=0}^{n-1} h^{l_2} c_{l_2}  (\nabla_{\vtheta}f)^{l_2} f + \mathcal{C}(\nabla^2_{\vtheta} f)\right)  + \mathcal{O}(h^n) \owntag{CR, IH}\\
&= \frac{d^{k-4}}{dt^{k-4}}  \sum_{i+j=m} \sum_{l_1=0}^{n-1} \sum_{l_2=0}^{n-1} h^{l_1+l_2} c_i c_j c_{l_1} c_{l_2}  (\nabla_{\vtheta} f)^{l_1+ l_2+ m+3} f + \mathcal{C}(\nabla^2_{\vtheta} f) + \mathcal{O}(h^n) \owntag{S}\\
&= \sum_{i+j=m} \sum_{l_1=0}^{n-1} \sum_{l_2=0}^{n-1}  ... \sum_{l_{k-2}=0}^{n-1} h^{l_1+l_2 + ...+l_{k-2}} c_i c_j c_{l_1} c_{l_2} ... c_{l_k-2} (\nabla_{\vtheta}f)^{l_1+ l_2+ ...+{l_{k-2}} + m+k-1} f + \mathcal{C}(\nabla^2_{\vtheta} f)  + \mathcal{O}(h^n) \\
&= \sum_{i+j=m} \sum_{\substack{0 \le l_1,..., l_{k-2}\le n-1,\\ l_1 + l_2 +...l_{k-2} = p}}  h^{p} c_i c_j c_{l_1} c_{l_2} ... c_{l_k-2}  (\nabla_{\vtheta} f) ^{p+  m+k-1} f+ \mathcal{C}(\nabla^2_{\vtheta} f)   + \mathcal{O}(h^n)
\end{align}

By extracting the terms of $\mathcal{O}(h^p)$ with $p \le n -1$ from the above we can now conclude that (note that since we are now only concerned with $p \le n -1$ we drop the $l_i \le n-1$ since it is implied from $ l_1 + l_2 +...l_{k-2} = p$):
\begin{align}
A_{k,m}^{p} = \sum_{i+j=m} \sum_{\substack{l_1,..., l_{k-2}\ge 0,\\ l_1 + l_2 +...l_{k-2} = p}}  c_i c_j c_{l_1} c_{l_2} ... c_{l_k-2}  (\nabla_{\vtheta} f)^{p+ m+k-1} f+ \mathcal{C}(\nabla^2_{\vtheta} f) \hspace{2em} \forall p \le n-1
\end{align}

Replacing this in $f_n$:
     \begin{align}
       f_n &= - \sum_{m \ge 0} \sum_{k=2}^{\infty} \frac{1}{k!}A_{k,m}^{n+1-(k+m)} \\
            &= - \sum_{m \ge 0} \sum_{k=2}^{\infty} \frac{1}{k!} \sum_{i+j=m} \sum_{\substack{l_1,..., l_{k-2}\ge 0,\\ l_1 + l_2 +...l_{k-2} = n+1-(k+m)}}  c_i c_j c_{l_1} c_{l_2} ... c_{l_k-2}  (\nabla_{\vtheta}f)^{n+1-(k+m)+ k+m-1} f+ \mathcal{C}(\nabla^2_{\vtheta} f)\\
             &= - \sum_{m \ge 0} \sum_{k=2}^{\infty} \frac{1}{k!} \sum_{i+j=m} \sum_{\substack{l_1,..., l_{k-2}\ge 0,\\ l_1 + l_2 +...l_{k-2} = n+1-(k+m)}}  c_i c_j c_{l_1} c_{l_2} ... c_{l_k-2}  (\nabla_{\vtheta} f)^n f+ \mathcal{C}(\nabla^2_{\vtheta} f)\\
               &= -  \sum_{k=2}^{\infty} \frac{1}{k!} \sum_{m \ge 0} \sum_{\substack{l_1,..., l_{k}\ge 0,\\ l_1 + l_2 +...l_k = n+1-k}}  c_{l_1} c_{l_2} ... c_{l_k} (\nabla_{\vtheta} f)^n f + \mathcal{C}(\nabla^2_{\vtheta} f)
     \end{align}
We can now say that $f_n = c_n  (\nabla_{\vtheta} f)^{n} f + \mathcal{C}(\nabla^2_{\vtheta} f)$. Not only that, we have now found the recurrence relation we were seeking:
\begin{align}
c_n = - \sum_{k=2}^{n+1} \frac{1}{k!} \sum_{l_1+ l_2+ ...+l_k = n - k + 1} c_{l_1} c_{l_2} ... c_{l_k}
\label{eq:rec_coeffs}
\end{align}

\end{proof}

\begin{lemma} The solution to the recurrence relation
\begin{align}
c_n = - \sum_{k=2}^{n+1} \frac{1}{k!} \sum_{l_1+ l_2+ ...+l_k = n - k + 1} c_{l_1} c_{l_2} ... c_{l_k}
\end{align}
with $c_0 = 1$ is 
\begin{align}
c_n = \frac{(-1)^{n}}{n+1}
\end{align}
\label{lemma:c_n_recurrence_proof}
\end{lemma}
\begin{proof}

We will use generating functions to solve the recurrence. Let $c(x)$ be the generating function of $c_n$:
\begin{align}
c(x) = c_0  + \sum_{n=1}^{\infty} c_n x^n
\end{align}

We have that:
\begin{align}
c(x) &= 1 - \sum_{n=1}^{\infty} x^n \left(\sum_{k=2}^{n+1}\frac{1}{k!} \sum_{l_1+ l_2+ ...+l_k = n - k + 1} c_{l_1} c_{l_2} ... c_{l_k}\right)\\
     &= 1 -\frac{1}{x} \sum_{n=1}^{\infty} x^{n+1} \left(\sum_{k=2}^{n+1}\frac{1}{k!} \sum_{l_1+ l_2+ ...+l_k = n - k + 1} c_{l_1} c_{l_2} ... c_{l_k}\right)\\
        &= 1 -\frac{1}{x} \sum_{n=1}^{\infty}  \left(\sum_{k=2}^{n+1}\frac{x^k}{k!} x^{n+1-k} \sum_{l_1+ l_2+ ...+l_k = n - k + 1} c_{l_1} c_{l_2} ... c_{l_k}\right)\\
     &= 1 - \frac{1}{x} \sum_{n=1}^{\infty} \left(\sum_{k=2}^{n+1}\frac{x^k}{k!} \sum_{l_1+ l_2+ ...+l_k = n - k + 1} c_{l_1}x^{l_1} c_{l_2}x^{l_2} ... c_{l_k} x^{l_k}\right)\\
     &= 1 - \frac{1}{x} \sum_{k=2}^{\infty}\frac{x^k}{k!}\left( \sum_{l_1=0}^{\infty}\sum_{l_2=0}^{\infty}\sum_{l_k=0}^{\infty} c_{l_1}x^{l_1} c_{l_2}x^{l_2} ... c_{l_k} x^{l_k}\right) \label{eq:sum_transition}\\
     &= 1 - \frac{1}{x} \sum_{k=2}^{\infty}\frac{x^k}{k!}\left( \sum_{l_1=0}^{\infty} c_{l_1}x^{l_1} \sum_{l_2=0}^{\infty}c_{l_2}x^{l_2} \sum_{l_k=0}^{\infty}  c_{l_k} x^{l_k}\right)\\
      &= 1 - \frac{1}{x} \sum_{k=2}^{\infty}\frac{x^k}{k!} c(x)^k\\
      &= 1 - \frac{1}{x} (e^{x c(x)} -1 - x c(x))
      \label{eq:c_x_rec}
\end{align}

Where Eq \ref{eq:sum_transition} can be obtained via the following reasoning:
\begin{align}
& \sum_{k=2}^{\infty}\frac{x^k}{k!}\left( \sum_{l_1=0}^{\infty}\sum_{l_2=0}^{\infty}\sum_{l_k=0}^{\infty} c_{l_1}x^{l_1} c_{l_2}x^{l_2} ... c_{l_k} x^{l_k}\right) \\
& = \sum_{k=2}^{\infty}\frac{x^k}{k!}\left(\sum_{m=0}^{\infty} \sum_{l_1 + l_2 +...l_k = m} c_{l_1}x^{l_1} c_{l_2}x^{l_2} ... c_{l_k} x^{l_k}\right)  \\
& = \sum_{k=2}^{\infty}\frac{x^k}{k!}\left(\sum_{n=k-1}^{\infty} \sum_{l_1 + l_2 +...l_k = n+1-k} c_{l_1}x^{l_1} c_{l_2}x^{l_2} ... c_{l_k} x^{l_k}\right)  \\
& = \sum_{k=2}^{\infty}\sum_{n=k-1}^{\infty}\frac{x^k}{k!} \sum_{l_1 + l_2 +...l_k = n+1-k} c_{l_1}x^{l_1} c_{l_2}x^{l_2} ... c_{l_k} x^{l_k}  \\
& = \sum_{n=1}^{\infty}\sum_{k=2}^{n+1}\frac{x^k}{k!} \sum_{l_1 + l_2 +...l_k = n+1-k} c_{l_1}x^{l_1} c_{l_2}x^{l_2} ... c_{l_k} x^{l_k}
\end{align}

Solving for $c(x)$ in Eq~\ref{eq:c_x_rec} we obtain: 
\begin{align}
c(x) &=  1 - \frac{1}{x} (e^{x c(x)} -1 - x c(x)) \\
c(x) &=  1 - \frac{1}{x} e^{x c(x)} + \frac 1 x + c(x) \\
0 &=  1 - \frac{1}{x} e^{x c(x)} + \frac 1 x\\
 e^{x c(x)} &=  x + 1\\
 x c(x) &=  \log(x + 1)\\
 c(x) &=  \frac{\log(x + 1)}{x}
\end{align}

We now perform the series expansion for $c(x)$:
\begin{align}
 c(x) =  \frac{\log(x + 1)}{x}
      =  \frac{ \sum_{n=1}^{\infty} \frac{(-1)^{n-1}}{n} x^n}{x}
      =  \sum_{n=1}^{\infty} \frac{(-1)^{n-1}}{n} x^{n-1}
      =  \sum_{n=0}^{\infty} \frac{(-1)^{n}}{n+1} x^{n}
\end{align}

We thus have that $c_n = \frac{(-1)^{n}}{n+1}$ which finishes the proof.

\end{proof}

\subsection{The Jacobian of the principal flow at critical points}
\label{sec:jacobian}

Under the principal flow, we have that the Jacobian at a critical point $\vtheta^*$ with the eigenvalues and eigenvector of $\nabla_{\vtheta}^2E(\vtheta^*)$ denoted as $\lambda_i^*$ and $\vu_i^*$, respectively. We will use that $\nabla_{\vtheta}E(\vtheta^*) = \myvecsym{0}$.
 
\begin{align}
J_{PF}(\vtheta^*) &= \sum_{i=0}^{D-1} \frac{d (\nabla_{\vtheta} E^T \vu_i)  \frac{\log(1 - h \lambda_i)}{h \lambda_i} \vu_i}{d\vtheta} (\vtheta^*) \\
            &= \sum_{i=0}^{D-1} \frac{d (\nabla_{\vtheta} E^T \vu_i)  \frac{\log(1 - h \lambda_i)}{h \lambda_i}}{d\vtheta} (\vtheta^*) {\vu_i^*}^T + \sum_{i=0}^{D-1} \underbrace{ \nabla_{\vtheta} E(\vtheta^*)^T}_{\mathbf{0}}  \vu_i^*  \frac{\log(1 - h \lambda_i^*)}{h \lambda_i^*} \frac{d \vu_i}{d\vtheta}\\
              &= \sum_{i=0}^{D-1} \frac{d (\nabla_{\vtheta} E^T \vu_i)  \frac{\log(1 - h \lambda_i)}{h \lambda_i}}{d\vtheta} (\vtheta^*) {\vu_i^*}^T \\
              &= \sum_{i=0}^{D-1} \frac{\log(1 - h \lambda_i^*)}{h \lambda_i^*} \frac{d (\nabla_{\vtheta} E^T \vu_i) }{d\vtheta}(\vtheta^*) {\vu_i^*}^T  + \sum_{i=0}^{D-1} (\underbrace{\nabla_{\vtheta} E(\vtheta^*)^T}_{\mathbf{0}} \vu_i^*) \frac{d   \frac{\log(1 - h \lambda_i)}{h \lambda_i}}{d\vtheta} {\vu_i^*}^T \\
                &= \sum_{i=0}^{D-1} \frac{\log(1 - h \lambda_i^*)}{h \lambda_i^*} \frac{d (\nabla_{\vtheta} E^T \vu_i) }{d\vtheta}(\vtheta^*) {\vu_i^*}^T\\
                &= \sum_{i=0}^{D-1} \frac{\log(1 - h \lambda_i^*)}{h \lambda_i^*} \left(\frac{d \nabla_{\vtheta} E}{d\vtheta}(\vtheta^*)\right)^T \vu_i^* {\vu_i^*}^T +  \sum_{i=0}^{D-1} \frac{\log(1 - h \lambda_i^*)}{h \lambda_i^*} \left(\frac{d \vu_i }{d\vtheta}(\vtheta^*)\right)^T  \underbrace{\nabla_{\vtheta} E(\vtheta^*)}_{\mathbf{0}}  {\vu_i^*}^T \\
                &= \sum_{i=0}^{D-1} \frac{\log(1 - h \lambda_i^*)}{h \lambda_i^*} \nabla_{\vtheta}^2E(\vtheta^*) \vu_i^* {\vu_i^*}^T \\
                &= \sum_{i=0}^{D-1} \frac{\log(1 - h \lambda_i^*)}{h \lambda_i^*} \lambda_i^* \vu_i^* {\vu_i^*}^T \\
                &= \sum_{i=0}^{D-1} \frac{\log(1 - h \lambda_i^*)}{h} \vu_i^* {\vu_i^*}^T
\end{align}

This shows that the eigenvalues of the Jacobian for the PF are $\frac{\log(1 - h \lambda_i^*)}{h}$, thus local minima where $\lambda_i^* \ge 0 $ are only attractive if $h \lambda_i^* <2$.

\subsubsection{Comparison with the Jacobian of other flows at critical points}
\label{sec:jacobian_igr_ngf}

We show here the Jacobian of the NGF:
\begin{align}
J_{NGF}(\vtheta^*) &= - \nabla_{\vtheta}^2 E(\vtheta^*) = - \sum_{i=0}^{D-1} \lambda_i^* \vu_i^* {\vu_i^*}^T
\end{align}

and thus the eigenvalues of the Jacobian for the NGF are $-\lambda_i^*$. Thus local minima where $\lambda_i^* > 0 \implies -\lambda_i^* < 0$  are attractive for the NGF.

For the IGR flow, we have:
\begin{align}
J_{IGR}(\vtheta^*) &= -\frac{d \left(\nabla_{\vtheta} E + \frac{h^2}{2} \nabla_{\vtheta}^2 E \nabla_{\vtheta} E \right)}{d\vtheta}\left( \vtheta^*\right) \\
                  &= -\frac{d \nabla_{\vtheta} E}{d \vtheta}\left( \vtheta^*\right) - \frac{h^2}{2} \frac{d \nabla_{\vtheta}^2 E \nabla_{\vtheta} E }{d\vtheta}\left( \vtheta^*\right) \\
                  &= - \sum_{i=0}^{D-1} \lambda_i^* \vu_i^* {\vu_i^*}^T - \frac{h^2}{2} \left( \nabla_{\vtheta} E^T \nabla_{\vtheta}^3 E   + \nabla_{\vtheta}^2 E \frac{d \nabla_{\vtheta} E}{d\vtheta} \right) \left( \vtheta^*\right)\\ 
                 &= - \sum_{i=0}^{D-1} \lambda_i^* \vu_i^* {\vu_i^*}^T  - \frac{h^2}{2}  \underbrace{ \nabla_{\vtheta} E(\vtheta^*)^T}_{\mathbf{0}}\nabla_{\vtheta}^3 E(\vtheta^*)   - \frac{h^2}{2}  \nabla_{\vtheta}^2 E \left( \vtheta^*\right) \nabla_{\vtheta}^2 E\left( \vtheta^*\right) \\
                  &= - \sum_{i=0}^{D-1} \lambda_i^* \vu_i^* {\vu_i^*}^T - \frac{h^2}{2}  \nabla_{\vtheta}^2 E \left( \vtheta^*\right) \nabla_{\vtheta}^2 E\left( \vtheta^*\right) \\
                  &= - \sum_{i=0}^{D-1} \lambda_i^* \vu_i^* {\vu_i^*}^T - \sum_{i=0}^{D-1} \frac{h^2}{2} {\lambda_i^*}^2 \vu_i^* {\vu_i^*}^T\\
                  &= - \sum_{i=0}^{D-1} \left(\lambda_i^* + \frac{h^2}{2} {\lambda_i^*}^2\right) \vu_i^* {\vu_i^*}^T
\end{align}

and thus the eigenvalues of the Jacobian for the IGR flow are $- \left(\lambda_i^* + \frac{h^2}{2} {\lambda_i^*}^2\right)$. Thus local minima where $\lambda_i^* > 0 \implies - \left(\lambda_i + \frac{h^2}{2} {\lambda_i^*}^2\right) < 0$  are attractive for the IGR flow.

\subsection{Linearization results around critical points}
\label{sec:linearization}

Assume we are around a critical point $\vtheta^*$, i.e. $\nabla_{\vtheta} E(\vtheta^*)= \mathbf{0}$. Via the Hartman–Grobman theorem~\citep{grobman,hartman1960lemma} we can write:
\begin{align}
 \nabla_{\vtheta} E \approx  \nabla_{\vtheta} E(\vtheta^*) + \nabla_{\vtheta}^2 E(\vtheta^*) (\vtheta - \vtheta^*) = \nabla_{\vtheta}^2 E(\vtheta^*) (\vtheta - \vtheta^*)
\label{eq:critical_point_approx}
\end{align}

Replacing this in the gradient descent updates:
\begin{align}
\vtheta_t = \vtheta_{t-1} - h \nabla_{\vtheta} E(\vtheta_{t-1}) \implies \vtheta_t - \vtheta^* \approx  (I - h \nabla_{\vtheta}^2 E(\vtheta^*)) (\vtheta_{t-1} - \vtheta^*) \approx  (I - h \nabla_{\vtheta}^2 E(\vtheta^*))^t (\vtheta_{0} - \vtheta^*)
\end{align}

Since $\left(I - h \nabla_{\vtheta}^2 E(\vtheta^*)\right)^t (\vtheta_{0} - \vtheta^*) = \sum_{i=0}^{D-1} (1 - h \lambda_i^*)^t {\vu_i^*} {\vu_i^*}^T (\vtheta_{0} - \vtheta^*)$, gradient descent converges in the limit of $t \to \infty$ when $|1 - h \lambda_i^*|< 1$, i.e. $\lambda_i < 2/h$. Thus, we obtain the same conclusion as obtained with the PF around a local minima.

We now also show how we can use these results to derived a specialized version of the PF, one that is only valid around critical points and requires the additional assumption that the Hessian of the critical point is invertible. 

We start with:
\begin{align}
\vtheta_t - \vtheta^*  \approx (I - h \nabla_{\vtheta}^2 E(\vtheta^*))^t (\vtheta_{0} - \vtheta^*) = \sum_{i=0}^{D-1} (1 - h \lambda_i^*)^t {\vu_i^*} {\vu_i^*}^T (\vtheta_{0} - \vtheta^*)
\end{align}

Under the above approximation we have:
\begin{align}
\vtheta_t = \vtheta^* + \sum_{i=0}^{D-1} (1 - h \lambda_i^*)^t {\vu_i^*} {\vu_i^*}^T  (\vtheta_{0} - \vtheta^*)
\label{eq:linearisation_expansion_theta_t}
\end{align}

which we can write in continuous form as (we use that iteration $n$ is at time $t =nh$): 
\begin{align}
\vtheta(t) &= \vtheta^* + \sum_{i=0}^{D-1} (1 - h \lambda_i^*)^{t/h} {\vu_i^*} {\vu_i^*}^T (\vtheta_{0} - \vtheta^*) \\          & =\vtheta^* + \sum_{i=0}^{D-1} e^{\log(1 - h \lambda_i^*){t/h}} {\vu_i^*} {\vu_i^*}^T (\vtheta_{0} - \vtheta^*)
\end{align}

and taking the derivative wrt to $t$:
\begin{align}
\dot{\vtheta} = \sum_{i=0}^{D-1} \frac{\log(1 - h \lambda_i^*)}{h} \underbrace{e^{\log{(1 - h \lambda_i^*)} {t/h}} {\vu_i^*} {\vu_i^*}^T (\vtheta_{0} - \vtheta^*)}_{Eq. \ref{eq:linearisation_expansion_theta_t}} = \sum_{i=0}^{D-1} \frac{\log(1 - h \lambda_i^*)}{h} (\vtheta - \vtheta^*)
\end{align}

From here, if we assume the Hessian at the critical point is invertible we can write from the approximation above (Eq. \ref{eq:critical_point_approx}) that 
$\vtheta - \vtheta^* = (\nabla_{\vtheta}^2 E)^{-1} \nabla_{\vtheta} E = \sum_{i=0}^{D-1} \frac{1}{\lambda_i^*} \nabla_{\vtheta} E^T {\vu_i^*} {\vu_i^*}^T$ and replacing this in the above we obtain:

\begin{align}
\dot{\vtheta} = \sum_{i=0}^{D-1} \frac{\log(1 - h \lambda_i^*)}{h \lambda_i^*}  \nabla_{\vtheta} E^T {\vu_i^*} {\vu_i^*}^T
\end{align}

which is the principal flow.

\textbf{Differences with the BEA.} The above analysis only applies around a critical point. What is crucial about the BEA derivation is that it allows us to specify the form of the modified flow for \textit{any iteration of gradient descent},
something we leveraged intensely in our experiments.

\subsection{The solution of the principal flow for quadratic losses}

We derive the solution of the PF in the case of quadratic losses - Remark~\ref{remark:quadractic_remark} in the main paper. We have $E(\vtheta) = \frac{1}{2} \vtheta^t \vA \vtheta + \vb^T \vtheta$  with $\vA$ symmetric. $\nabla_{\vtheta} E = \vA \vtheta + \vb$ and $\nabla_{\vtheta}^2 E = \vA$. Thus $\lambda_i$ and $\vu_i$ are the eigenvalues and eigenvectors of $\vA$ and do not change based on $\vtheta$. Replacing this in the PF leads to:
\begin{align}
\dot{\vtheta} &= \sum_{i=0}^{D-1} \frac{\log (1 - h\lambda_i)}{h \lambda_i} \left(A \vtheta + b \right)^T \vu_i  \vu_i \\
 &= \sum_{i=0}^{D-1} \frac{\log (1 - h\lambda_i)}{h \lambda_i} (A \vtheta )^T \vu_i \vu_i + \sum_{i=0}^{D-1} \frac{\log (1 - h\lambda_i)}{h \lambda_i} b^T \vu_i  \vu_i \\
  &= \sum_{i=0}^{D-1} \frac{\log (1 - h\lambda_i)}{h \lambda_i} \vtheta^T A \vu_i \vu_i + \sum_{i=0}^{D-1} \frac{\log (1 - h\lambda_i)}{h \lambda_i} b^T \vu_i  \vu_i \\
    &= \sum_{i=0}^{D-1} \frac{\log (1 - h\lambda_i)}{h \lambda_i}\lambda_i \vtheta^T  \vu_i \vu_i + \sum_{i=0}^{D-1} \frac{\log (1 - h\lambda_i)}{h \lambda_i} b^T \vu_i  \vu_i \\
    &= \sum_{i=0}^{D-1} \frac{\log (1 - h\lambda_i)}{h } \vtheta^T  \vu_i \vu_i + \sum_{i=0}^{D-1} \frac{\log (1 - h\lambda_i)}{h \lambda_i} b^T \vu_i  \vu_i
\end{align}

From here
\begin{align}
\dot{(\vtheta^T \vu_i)} &= \frac{\log (1 - h\lambda_i)}{h } \vtheta^T  \vu_i +  \frac{\log (1 - h\lambda_i)}{h \lambda_i} b^T \vu_i 
\end{align}

The solution of this ODE is $(\vtheta^T \vu_i)(t)  = e^{\frac{\log (1 - h\lambda_i)}{h } t} \vtheta_0^T  \vu_i + t \frac{\log (1 - h\lambda_i)}{h \lambda_i} \vb^T \vu_i$.

Since $\vu_i$ form a basis of $\mathbb{R}^D$, we can now write
\begin{align}
\vtheta(t) &= \sum_{i=0}^{D-1} \vtheta(t)^T \vu_i \vu_i  \\
          &= \sum_{i=0}^{D-1} \left(e^{\frac{\log (1 - h\lambda_i)}{h } t} \vtheta_0^T  \vu_i + t \frac{\log (1 - h\lambda_i)}{h \lambda_i} b^T \vu_i \right) \vu_i 
\end{align}

\subsection{Complex matrices and the Jordan normal form}
\label{sec:bea_pf_jordan}

The approach we took in deriving the PF had two steps: first, we developed the principal series with BEA (Theorem~\ref{thm:order_n_flow}) and second, we used the principal series given by BEA to find the PF based on the eigen-decomposition of the Hessian (Corollary~\ref{col:principal_series} and Theorem~\ref{thm:principal_ode}). The first step can be readily translated to games as it applies to any vector field, using  Theorem~\ref{thm:general_bea}, which Theorem~\ref{thm:order_n_flow} is a corollary of. From there, we derived the PF. Since the eigenvalues of a Hessian $\nabla_{\vtheta}^2 E$ can be complex, the PF, and the parameters $\vtheta$, can be complex valued.  In that case, the eigenvectors $\vu_i$ of $\nabla_{\vtheta}^2 E$ no longer need to form a basis. If we want to retrain that property, we have to work with the Jordan normal for of  $\nabla_{\vtheta}^2 E$ instead. We show how to derive that here. Consider the Jordan normal form of $\nabla_{\vtheta}^2 E$:

\begin{align}
\nabla_{\vtheta}^2 E = \vP^{-1} \vJ \vP
\end{align}
where $\vJ$ is a block diagonal matrix. We then have
\begin{align}
\nabla_{\vtheta}^2 E^p = \vP^{-1} \vJ^p \vP,
\end{align}
leading to
\begin{align}
  \dot{\vtheta
  }
   &=  \sum_{p=0}^{\infty} \frac{-1}{p+1} \vP^{-1} (h\vJ)^p \vP \nabla_{\vtheta} E + \mathcal{C}(\nabla_{\vtheta}^3 E) ,
\end{align}
We are interested in finding a form for the series
\begin{align}
\sum_{p=0}^{\infty} \frac{-1}{p+1} (h \vJ)^p ,
\end{align}
which if converges can be written as another block diagonal matrix,
\begin{align}
\frac{1}{h} \log(\vI - h \vJ) \vJ^{-1}.
\end{align}
This result leads to the following flow:
\begin{align}
  \dot{\vtheta}
   &= \frac{1}{h}  \vP^{-1}  \log(\vI - h\vJ) \vJ^{-1} \vP \nabla_{\vtheta} E + \mathcal{C}(\nabla_{\vtheta}^3 E) 
\end{align}
We note that while the Jordan normal form of $\nabla_{\vtheta}^2 E$ is not unique due to permutations of the order of the blocks in $\vJ$, $\vP^{-1}  \log(\vI - h\vJ) \vJ^{-1} \vP$ will be the same regardless of the block order chosen.

We note that when the Hessian is real, this reverts to the eigenvalue decomposition. Thus, this does not affect out results in the main paper, as the analysis was done around a gradient descent iteration, and stability analysis is done around a real equilibruim.

\subsection{Comparison with a discrete time approach}
\label{sec:discrete}

\subsubsection{Changes in loss function}
\label{sec:changes_in_loss_discrete}

We aim to obtain similar intuition to what we have obtained from the PF by discretising the NGF using Euler steps. We have:

\begin{align}
 E(\vtheta_{t+1}) &\approx  E(\vtheta_{t}) - h \nabla_{\vtheta} E(\vtheta_t)^T\nabla_{\vtheta} E(\vtheta_t) + h^2/2 \nabla_{\vtheta} E(\vtheta_t)^T \nabla_{\vtheta}^2 E(\vtheta_t) \nabla_{\vtheta} E(\vtheta_t) \\
                 &\approx  E(\vtheta_{t}) - h \nabla_{\vtheta} E(\vtheta_t)^T\nabla_{\vtheta} E(\vtheta_t) + h^2/2 \sum_i \lambda_i  \left(\nabla_{\vtheta} E(\vtheta_t)^T \vu_i \right)^2\\
                  &\approx  E(\vtheta_{t}) - h \sum_i \left(\nabla_{\vtheta} E(\vtheta_t)^T \vu_i \right)^2 + h^2/2 \sum_i \lambda_i  \left(\nabla_{\vtheta} E(\vtheta_t)^T \vu_i \right)^2\\
                   &\approx  E(\vtheta_{t}) +  \sum_i (1 - h/(2\lambda_i))\left(\nabla_{\vtheta} E(\vtheta_t)^T \vu_i \right)^2
\end{align}

thus under this approximation loss to decrease between iterations one requires $1 - h/(2\lambda_i) \le 0$. This is consistent with the observations of the loss obtained using the PF in Eq~\ref{eq:changes_in_e}.

\subsubsection{The dynamics of $\nabla_{\vtheta} E^T \vu_i$}
\label{sec:changes_in_dot_prod_discrete}

We now show to obtain the approximated dynamics of $\nabla_{\vtheta} E^T \vu_i$ obtained from the PF in Section~\ref{sec:instability_deep_learning} using a discrete time approach.

We have that:
\begin{align}
\nabla_{\vtheta} E(\vtheta_{t+1}) &\approx \nabla_{\vtheta} E\left(\vtheta_t - h \nabla_{\vtheta} E(\vtheta_t)\right) \\
                        &\approx \nabla_{\vtheta} E(\vtheta_t) - h \nabla_{\vtheta}^2 E(\vtheta_t)\nabla_{\vtheta} E(\vtheta_t)
\end{align}

If we assume that the Hessian eigenvectors do not change between iterations, we have:
\begin{align}
\nabla_{\vtheta} E(\vtheta_{t+1}) ^T {(\vu_{i})}_{t+1} \approx \nabla_{\vtheta} E(\vtheta_t)^T {(\vu_{i})}_{t+1}  - h \nabla_{\vtheta}^2 E(\vtheta_t)\nabla_{\vtheta} E(\vtheta_t) {(\vu_{i})}_{t+1}  =   ( 1 - h \lambda_i )\nabla_{\vtheta} E(\vtheta_t)^T {(\vu_{i})}_{t+1} 
\end{align}

and obtain the same behavior predicted in Section~\ref{sec:instability_deep_learning}.

\subsubsection{\rebuttalrthree{The connection between DAL and Taylor expansion optimal learning rate}}

\rebuttalrthree{In Section~\ref{sec:dal} introduced DAL as a way to control the drift of gradient descent; to do so, we set the learning rate of gradient descent to be inverse proportional to $\norm{\nabla_{\vtheta}^2 E\nabla_{\vtheta}E}/\norm{\nabla_{\vtheta}E} = \norm{\nabla_{\vtheta}^2 E \hat{\vg}}$}.

\rebuttalrthree{The learning rate set by DAL is similar, but distinct to that obtained by using a Taylor expansion of the loss $E$, as follows:
\begin{align}
E(\vtheta - h \nabla_{\vtheta}E) = E(\vtheta) - h \nabla_{\vtheta}E^T \nabla_{\vtheta}E + \frac{h^2}{2} \nabla_{\vtheta}E ^T \nabla_{\vtheta}^2 E \nabla_{\vtheta}E + \mathcal{O}(h^3)
\end{align}
}

\rebuttalrthree{
solving for the optimal $h$ in the above leads to:
\begin{align}
h = \frac{\nabla_{\vtheta}E^T \nabla_{\vtheta}E}{\nabla_{\vtheta}E ^T \nabla_{\vtheta}^2 E \nabla_{\vtheta}E} = \frac{1}{\hat{\vg}^T \nabla_{\vtheta}^2 E \hat{\vg}}
\end{align}
}

\rebuttalrthree{
While this optimal learning rate contains similarities to DAL, since $\hat{\vg}^T \nabla_{\vtheta}^2 E \hat{\vg} = \sum_i \lambda_i (\hat{\vg}^T \vu_i)^2$ and $\norm{\nabla_{\vtheta}^2 E \hat{\vg}} = |{\sum_i \lambda_i \hat{\vg}^T \vu_i}|$.
}

\rebuttalrthree{
We also note a few significant differences:
\begin{itemize}
    \item Since DAL is derived from discretisation drift, we can use it in a variety of settings. We have shown in Section~\ref{sec:future_work} how DAL can be adapted to be used with momentum, where the main idea is to increase the contribution of local gradients to the momentum moving average when the drift is large, and decrease it otherwise. 
    \item While we focused on using one learning rate for all parameters in the main paper, and thus use the norm operator in DAL, this is not necessary. Instead we can use the \textit{per parameter drift} to construct per parameter adaptive learning rates. We show preliminary results in Figures~\ref{fig:vgg_lr_scaling_per_parameter} and~\ref{fig:imagenet_lr_scaling_per_parameter}.
    \item The aim of DAL is to control the discretisation drift of gradient descent and show that this can affect both the stability and performance of gradient descent, through DAL-p. Beyond providing a training tool, DAL-p is another tool to use to understand gradient descent. Importantly, the ability to control the drift also enhances the generalisation capabilities of the method.
\end{itemize}
}

\subsection{Multiple gradient descent steps}
\label{sec:multiple_steps_proof}

We now show that if we use BEA to construct a flow that has an error of order $\mathcal{O}(h^p)$
after 1 gradient descent step, then the error will be of the same order after 2 steps. The same argument can be applied to multiple steps, though the error scales proportionally to the number of steps and the bound becomes vacuous as the number of steps increases.
We have $\dot{\vtheta}$ with $\vtheta(0) = \vtheta_{t-1}$ tracks $\vtheta_{t} = \vtheta_{t-1} - h \nabla_{\vtheta} E(\vtheta_{t-1})$  with an error $\mathcal{O}(h^p)$ (i.e $\norm{\vtheta(h) - \vtheta_t}$ is $\mathcal{O}(h^p)$),  then $\dot{\vtheta}$ tracks two steps of gradient descent $\vtheta_{t} = \vtheta_{t-1} - h \nabla_{\vtheta} E(\vtheta_{t-1})$ and $\vtheta_{t+1} = \vtheta_t - h \nabla_{\vtheta} E(\vtheta_t)$  with an error of the same order $\mathcal{O}(h^p)$ (i.e. $\norm{\vtheta(2h) - \vtheta_{t+1}}$ is $\mathcal{O}(h^p)$). We prove this below. For the purpose of this proof we use the following notation $\vtheta(h; \vtheta')$ is the value of the flow at time $h$ with initial condition $\vtheta'$. Making the initial condition explicit is necessary for the proof.

\begin{align}
 \norm{\vtheta(2h; \vtheta_{t-1}) - \vtheta_{t+1}} 
 &= \norm{\vtheta(2h; \vtheta_{t-1}) - \vtheta(h; \vtheta_{t}) + \vtheta(h; \vtheta_{t}) - \vtheta_{t+1}} \\
 & \le \norm{\vtheta(2h; \vtheta_{t-1}) - \vtheta(h; \vtheta_{t})} + \norm{\vtheta(h; \vtheta_{t}) - \vtheta_{t+1}}\\
 &\le \norm{\vtheta(2h; \vtheta_{t-1}) - \vtheta(h; \vtheta_{t})} + \mathcal{O}(h^p) \\
 & =  \norm{\vtheta(h; \vtheta(h; \vtheta_{t-1})) - \vtheta(h; \vtheta_{t})} + \mathcal{O}(h^p)
\end{align}

We thus have to bound how the flow changes after time $h$ when starting with two different initial conditions $\vtheta(h; \vtheta_{t-1})$ and $\vtheta_{t})$. We also know that $\norm{\vtheta(h; \vtheta_{t-1}) - \vtheta_{t}}$ is $\mathcal{O}(h^p)$. Thus by expanding the Taylor series and the mean value theorem we obtain:

\begin{align}
 \norm{\vtheta(h; \vtheta(h; \vtheta_{t-1})) - \vtheta(h; \vtheta_{t})} &= \norm{\sum_{i=0}^\infty \frac{h^i}{i!} \vtheta^{(i)}\left(\vtheta(h; \vtheta_{t-1})\right) -\sum_{i=0}^\infty \frac{h^i}{i!} \vtheta^{(i)}(\vtheta_t))} \\
  & = \norm{\sum_{i=0}^\infty \frac{h^i}{i!} \left(\vtheta^{(i)}\left(\vtheta(h; \vtheta_{t-1})\right) - \vtheta^{(i)}(\vtheta_t))\right)} \\
  &= \norm{\sum_{i=0}^\infty \frac{h^i}{i!} \frac{d}{d\vtheta}\vtheta^{(i)}(\vtheta') \left(\vtheta(h; \vtheta_{t-1}) - \vtheta_t\right)} \\
    &\le \sum_{i=0}^{\infty} \frac{h^i}{i!} \left|\frac{d}{d\vtheta}\vtheta^{(i)}(\vtheta')\right| \underbrace{\norm{(\vtheta(h; \vtheta_{t-1})) - \vtheta_t}}_{\mathcal{O}(h^p)} = \mathcal{O}(h^p)
\end{align}

This tells us that we can construct a bound:
\begin{align}
 \norm{\vtheta(2h; \vtheta_{t-1}) - \vtheta_{t+1}} \le \mathcal{O}(h^p) + \mathcal{O}(h^p) = \mathcal{O}(h^p)
\end{align}

\textbf{Dependence on the number of steps}.
While the order in learning rate is the same, 
we note that as the number of steps increases the errors are likely to accumulate with the number of steps (for $n$ discrete steps we will sum $n$ terms of order $\mathcal{O}(h^p)$ in the above bound). For example when taking the number of steps $n \rightarrow \infty$ the above no longer provides a bound.

\textbf{A comment on stochastic gradient descent.}
Using BEA we can construct flows which approximate one step of a discrete update. The same flows will approximate multiple discrete steps with the same level of error.
In the case of deep learning however, it is common to use stochastic gradient descent where individual discrete gradient updates minimize different losses depending on the data batch:  $\vtheta_{t} = \vtheta_{t-1} - h \nabla_{\vtheta} E(\vtheta_{t-1}; x_{t-1})$ and $\vtheta_{t+1} = \vtheta_t - h \nabla_{\vtheta} E(\vtheta_t; \vx_{t})$. Thus if we perform BEA on each stochastic gradient descent update separately, we obtain different modified flows.
To use BEA in the case of stochastic gradient descent we have to choose one of the following options: expand the two discrete gradient descent updates into one update and perform BEA on the combined discrete update; work in expectation over batches in an epoch as done in ~\citet{igr_sgd}; or focus on each gradient descent update independently, and try to understand the per iteration behavior of gradient descent.
In this work, we took the latter approach and thus for simplicity we excluded the dependence of the function $E$ on the data batch.

\subsection{Approximations to per iteration drift for gradient descent and momentum}
\label{sec:proofs_total_per_iteration_drift}

We now prove Thm~\ref{thm:total_drift}, on the per iteration drift of gradient descent. We apply the Taylor reminder theorem on the NGF initialised at gradient descent iteration parameters $\vtheta_{t-1}$. We have that there exists $h'$ such that 
\begin{align}
\vtheta(h)   = & \vtheta(0) + h \dot{\vtheta}(0) + \frac{h^2}{2}\dot{\dot{\vtheta}}(h') \\
     = & \vtheta(0) - h \nabla_{\vtheta} E {\vtheta(0)} + \frac{h^2}{2} \nabla_{\vtheta}^2 E(\vtheta(h')) \nabla_{\vtheta} E(\vtheta(h')) \\
           =&  \vtheta_{t-1} - h \nabla_{\vtheta} E (\vtheta_{t-1}) + \frac{h^2}{2} \nabla_{\vtheta}^2 E(\vtheta') \nabla_{\vtheta} E(\vtheta')
\end{align}

with $h' \in (0, h)$. The above proof measures the drift of gradient descent, which we then used to construct DAL.
In Section~\ref{sec:future_work}, we also used DAL for momentum, and used an intuitive justification. We now show a more theoretical justification, by measuring the total per iteration drift of momentum. 
To do so, we use the following flow to describe momentum, provided by ~\citet{symmetry}:
\begin{align}
\dot{\vtheta} = - \frac{1}{1-m}\nabla_{\vtheta} E
\end{align}

and we measure the discretization drift of a momentum update compared to this flow.

If we apply the same strategy as above, we have that 
\begin{align}
\vtheta(h) &= \vtheta(0) + h \dot{\vtheta}(0) + \frac{h^2}{2}\dot{\dot{\vtheta}}(h') \\
&= \vtheta(0) - \frac{h}{1-m} \nabla_{\vtheta} E {\vtheta(0)} + \frac{h^2}{2 (1-m)^2} \nabla_{\vtheta}^2 E(\vtheta(h')) \nabla_{\vtheta} E(\vtheta(h')) \\
        &= \vtheta_{t-1} - \frac{h}{1-m} \nabla_{\vtheta} E (\vtheta_{t-1}) + \frac{h^2}{2 (1-m)^2}  \nabla_{\vtheta}^2 E(\vtheta') \nabla_{\vtheta} E(\vtheta')
\end{align}

 When we compare this with the discrete update 
\begin{align}
    \vtheta_t = \vtheta_{t-1} + m \vv_{t-1} - h \nabla_{\vtheta} E(\vtheta_{t-1})
\end{align}
and obtain that the per iteration drift is
\begin{align}
& \left(\vtheta_{t-1} - \frac{h}{1-m} \nabla_{\vtheta} E (\vtheta_{t-1}) + \frac{h^2}{2 (1-m)^2}  \nabla_{\vtheta}^2 E(\vtheta') \nabla_{\vtheta} E(\vtheta') \right) - \left( \vtheta_{t-1} + m \vv_{t-1} - h \nabla_{\vtheta} E(\vtheta_{t-1})\right) \\
& =\frac{h^2}{2 (1-m)^2}  \nabla_{\vtheta}^2 E(\vtheta') \nabla_{\vtheta} E(\vtheta') + m \vv_{t-1} + h \nabla_{\vtheta} E(\vtheta_{t-1}) -  \frac{h}{1-m} \nabla_{\vtheta} E (\vtheta_{t-1}) \\
& =\frac{h^2}{2 (1-m)^2}  \nabla_{\vtheta}^2 E(\vtheta') \nabla_{\vtheta} E(\vtheta') + m \vv_{t-1} -\frac{h m}{1-m} \nabla_{\vtheta} E(\vtheta_{t-1})
\end{align}

 Thus there is one part of the drift that comes from the same source as gradient descent and another part of the drift which comes from the alignment of the current gradient and the moving average obtained using momentum. Thus when using DAL-$p$ with momentum, we only focus on one of the sources of drift (the first term).

\section{Non-principal terms}
\label{sec:non_principal_ap}

This work focuses on understanding the effects of the PF on the behavior of gradient descent.  
The principal terms however are not the only terms in the discretization drift: we have found one non-principal term of the form $\nabla_{\vtheta} E^T (\nabla^3_{\vtheta} E) \nabla_{\vtheta} E$ (Eq~\ref{eq:third_order_modified_vector_field}) and have seen that it can have a stabilising effect (Figure~\ref{fig:intuition_banana}). 
To provide some intuition, while $\nabla_{\vtheta} E^T (\nabla^3_{\vtheta}E) \nabla_{\vtheta} E$ cannot be written as a gradient operator we can write:
\begin{align}
\nabla_{\vtheta}E ^T \nabla_{\vtheta}^3 E \nabla_{\vtheta}E = \sum_{i=0}^{D-1} (\nabla_{\vtheta}E ^T \vu_i) ^2 \left(\vu_i ^T \nabla_{\vtheta}^3 E \vu_i \right) \approx \sum_{i=0}^{D-1} (\nabla_{\vtheta}E ^T \vu_i) ^2 \nabla_{\vtheta} \lambda_i
\end{align}

where we used that $\vu_i$ does not change substantially around a GD iteration to write the derivative of $\lambda_i$, namely $\nabla_{\vtheta} \lambda_i = \vu_i ^T \nabla_{\vtheta}^3 E \vu_i$, since $\nabla_{\vtheta}^2 E$ is real and symmetric around a gradient descent iteration.
We can now write, again assuming $\vu_i$ is locally constant: 

\begin{align}
\sum_{i=0}^{D-1} (\nabla_{\vtheta}E ^T \vu_i) ^2 \nabla_{\vtheta} \lambda_i = \sum_{i=0}^{D-1} \nabla_{\vtheta} \left( \lambda_i (\nabla_{\vtheta}E ^T \vu_i) ^2 \right) - \sum_{i=0}^{D-1} 2 \lambda_i^2 (\nabla_{\vtheta}E ^T \vu_i) \vu_i.
\label{eq:non_principal_minimisation}
\end{align}

From here we can see that if $\lambda_i$ or $\nabla_{\vtheta}E ^T \vu_i$ are close to 0, the non principal term leads to a pressure to minimise $ \lambda_i (\nabla_{\vtheta}E ^T \vu_i) ^2 $, since the non-gradient term $\lambda_i^2 (\nabla_{\vtheta}E ^T \vu_i) \vu_i$ in Eq~\ref{eq:non_principal_minimisation} diminishes. We note that the work of~\citet{damian2022self} is what inspired us to write the third order non-principal term in the form of Eq~\ref{eq:non_principal_minimisation} in the Appendix, after we had previously noted its stabilizing properties.

Since $\lambda_i = 0$ or $\nabla_{\vtheta} E ^T \vu_i=0$ results in $\alpha_{PF}(h\lambda_i) = \alpha_{NGF}(h \lambda_i) = -1$, minimising $\lambda_i (\nabla_{\vtheta} E ^T \vu_i)^2$ has a stabilising effect by reducing the instability from the PF since in those directions the PF has the same behavior as the NGF.
We think this can shed light on observations about neural network training behaviour: the NGF tends to increase the eigenvalues, but as they reach close to 0 there is a pressure to minimise the above, which leads to $\lambda_i$ staying around 0; this is consistent with observations in the literature~\citep{sagun2017empirical,ghorbani2019investigation,papyan2018full}. This observation also explains the pressure for the dot product $\nabla_{\vtheta} E ^T \vu_0$ to stay small if it is initialised around 0, as we see is the case in the early in neural network training (see Figures~\ref{fig:early_training_eigen_loss} and~\ref{fig:instabilities_short}). Furthermore, as we show in Figure~\ref{fig:instability_largest_dot_prod} and previously argued, the dot product $\nabla_{\vtheta} E ^T \vu_0$ fluctuates around 0 in the edge of stability areas, making it likely for the minimisation effect induced by Eq~\ref{eq:non_principal_minimisation} to kick-in.

We plot the value of this non-principal term in neural network training in Figure~\ref{fig:non_principal}.
These results show that the non-principal third order term is very small outside the edge of stability area, but has a larger magnitude around the edge of stability, where the largest eigenvalues stop increasing but the magnitude of $(\nabla_{\vtheta} E(\vtheta_0)^T \vu_i)^2$ fluctuates and can be large (Figure~\ref{fig:instability_largest_dot_prod}).
This observation is consistent with the above interpretation, but more theoretical and empirical work is needed to understand the effects of non-principal terms on training stability and generalization.

\begin{figure}[tbh]
\centering
\includegraphics[width=0.43\columnwidth]{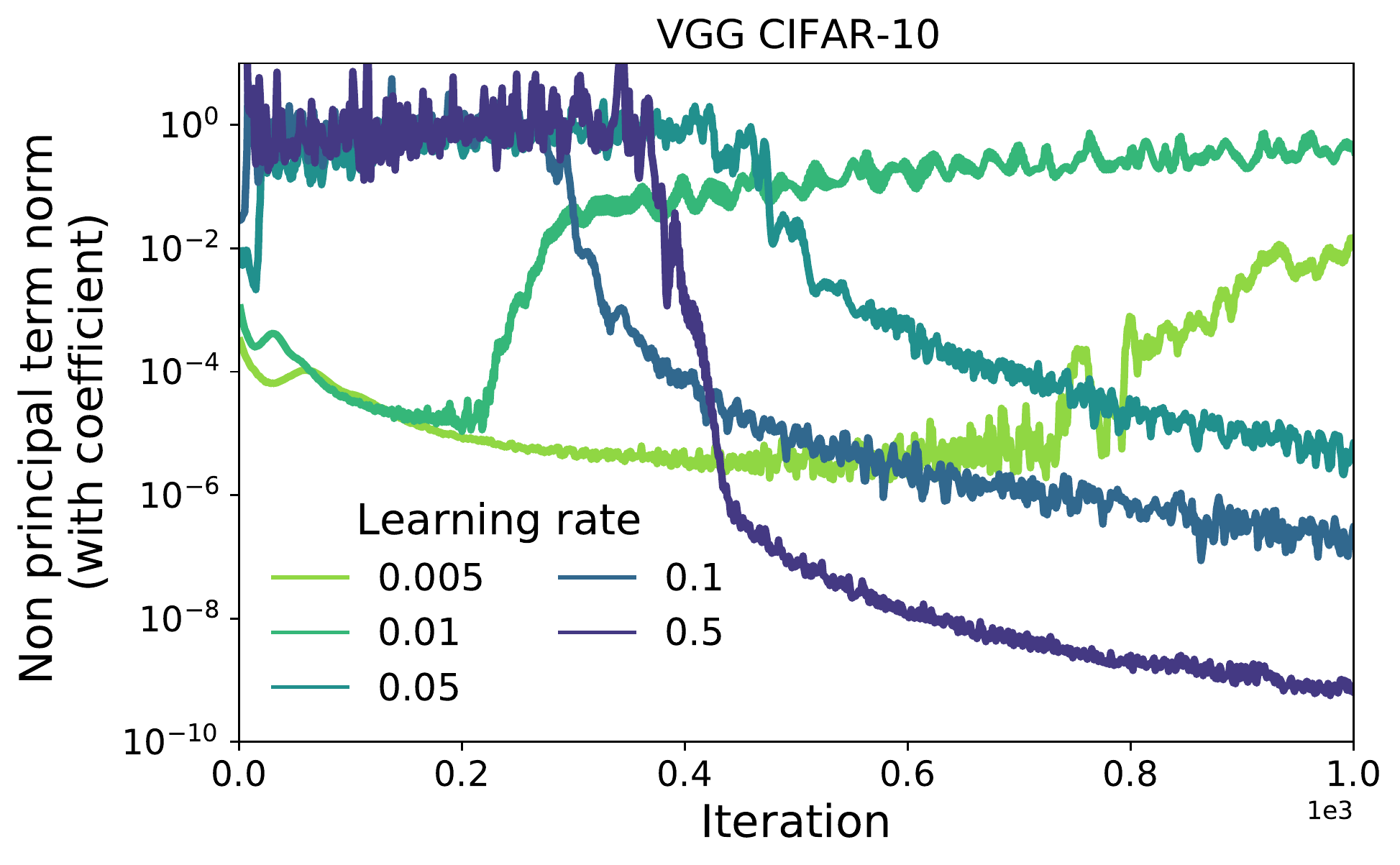}\hspace{2em}
\includegraphics[width=0.43\columnwidth]{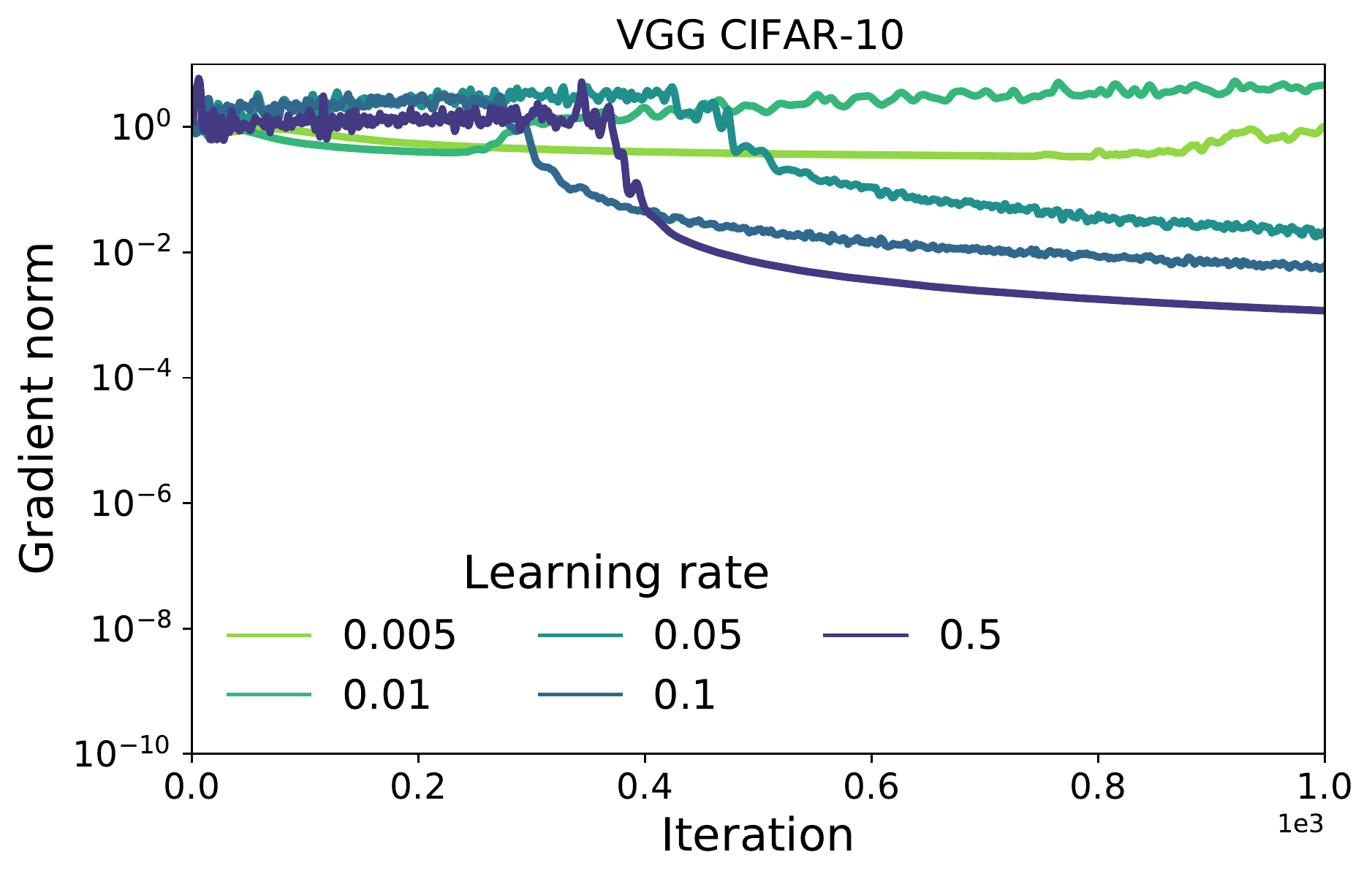}
\caption[Learning rate sweep showing the value of the non-principal third order term in training, on a Resnet-18 model trained on CIFAR-10. The coefficients used for the non-principal term are those in equation~\ref{eq:third_order_modified_vector_field}.]{The value of the non-principal third order term in training: outside the edge of stability areas (best seen for learning rates $0.005$ and $0.01$), the non-principal third order term is very small compared to the gradient. Inside the edge of stability areas, the non-principal term has a higher magnitude. Corresponding loss and $\lambda_0$ plots are provided in Figure~\ref{fig:reproduce_edge_of_stability}.}
\label{fig:non_principal}
\end{figure}

\section{Experimental details}

\textbf{Estimating continuous time trajectories of flows.} To estimate the continuous time trajectories we use Euler integration with a small learning rate $5 \times 10^{-5}$. We reached this learning rate through experimentation: further decreasing it did not change the obtained results. We note that this approach can be computationally expensive: to estimate the trajectory of the NGF of time give by the gradient descent learning rate $10^{-2}$, we need to do 5000 gradient steps. It is common in the literature to use RungeKutta4 to approximate continuous time flows, but we noticed that approximating a flow for time $h$ using RungeKutta4 with learning rate $h$ still introduced significant drift: if the learning rate was further reduced and multiple steps were taken the results were significantly different. Thus RungeKutta4 also needs multiple steps to estimate continuous trajectories. Given that one RungeKutta4 update requires computing 4 gradients, we found that using Euler integration with small learning rates is both sufficient and more efficient in practice.

\textbf{Datasets}. When training neural networks with primarily used three standard datasets: MNIST \citep{lecun1995learning}, CIFAR-10 \citep{cifar10} and Imagenet \citep{deng2009imagenet}. On the small NN example in Section \ref{sec:principal_flow}, we used a dataset of 5 examples, where the input is 2D sampled from a Gaussian distribution and the regression targets are also sample randomly. We used the UCI breast cancer dataset \citep{asuncion2007uci} in Figure \ref{fig:breast_cancer_principal_flow}.

\textbf{Architectures}. We use standard architectures: MLPs for MNIST, VGG \citep{simonyan2014very} or Resnet-18 (Version 1) \citep{he2016deep} for CIFAR-10, Resnet-50 for Imagenet (Version 1). We do not use any form of regularisation or early stopping. We use the Elu activation function \citep{elu} to ensure that the theoretical setting we discussed applies directly (we thus avoid discontinuities caused by Relus\citep{agarap2018deep}). We note that in the CIFAR-10 experiments, we did not adapt the Resnet-18 architecture to the dataset; doing so will likely increase performance.

\textbf{Losses}. Unless otherwise specified we used a cross entropy loss.

\textbf{Computing eigenvalues and eigenvectors.} We use the Lanczos algorithm to compute $\lambda_i$ and $\vu_i$.

\textbf{Computing $\norm{\nabla_{\vtheta}^2 E \nabla_{\vtheta} E}$}. We use Hessian vector products to compute $\norm{\nabla_{\vtheta}^2 E \nabla_{\vtheta} E}$. We experimented with  using the approximation $\nabla_{\vtheta}^2 E \nabla_{\vtheta} E = \frac{1}{2}\nabla_{\vtheta} \norm{\nabla_{\vtheta} E}^2 \approx \frac{\nabla_{\vtheta}E(\vtheta + \epsilon \nabla_{\vtheta} E) - \nabla_{\vtheta}E(\vtheta)}{\epsilon}$ with $\epsilon = 0.01/\norm{\nabla_{\vtheta} E}$ as suggested by~\citet{geiping2021stochastic} and saw no decrease in performance in the full batch setting when using it in DAL, but a slight decrease in performance when using it in stochastic gradient descent. We show experimental results in Figures~\ref{fig:dal_cifar_full_batch}, \ref{fig:dal_cifar_sgd},~\ref{fig:dal_approx_comp} and~\ref{fig:dal_0_5_approx_comp}. More experimentation is needed to see how to efficiently leverage this approximation, for example whether other values of $\epsilon$ are more suited for large networks. Alternatively, even when small batches are used for training a larger batch can be used to approximate $\norm{\nabla_{\vtheta}^2 E \nabla_{\vtheta} E}$; we did not experiment with this. We note however that all the conclusions we observed in the paper regarding the relative performance of $p$ values of DAL still hold, with and without the approximation and that indeed we see less of a difference when higher values of $p$ are used.

\begin{figure}
 \includegraphics[width=0.33\columnwidth]{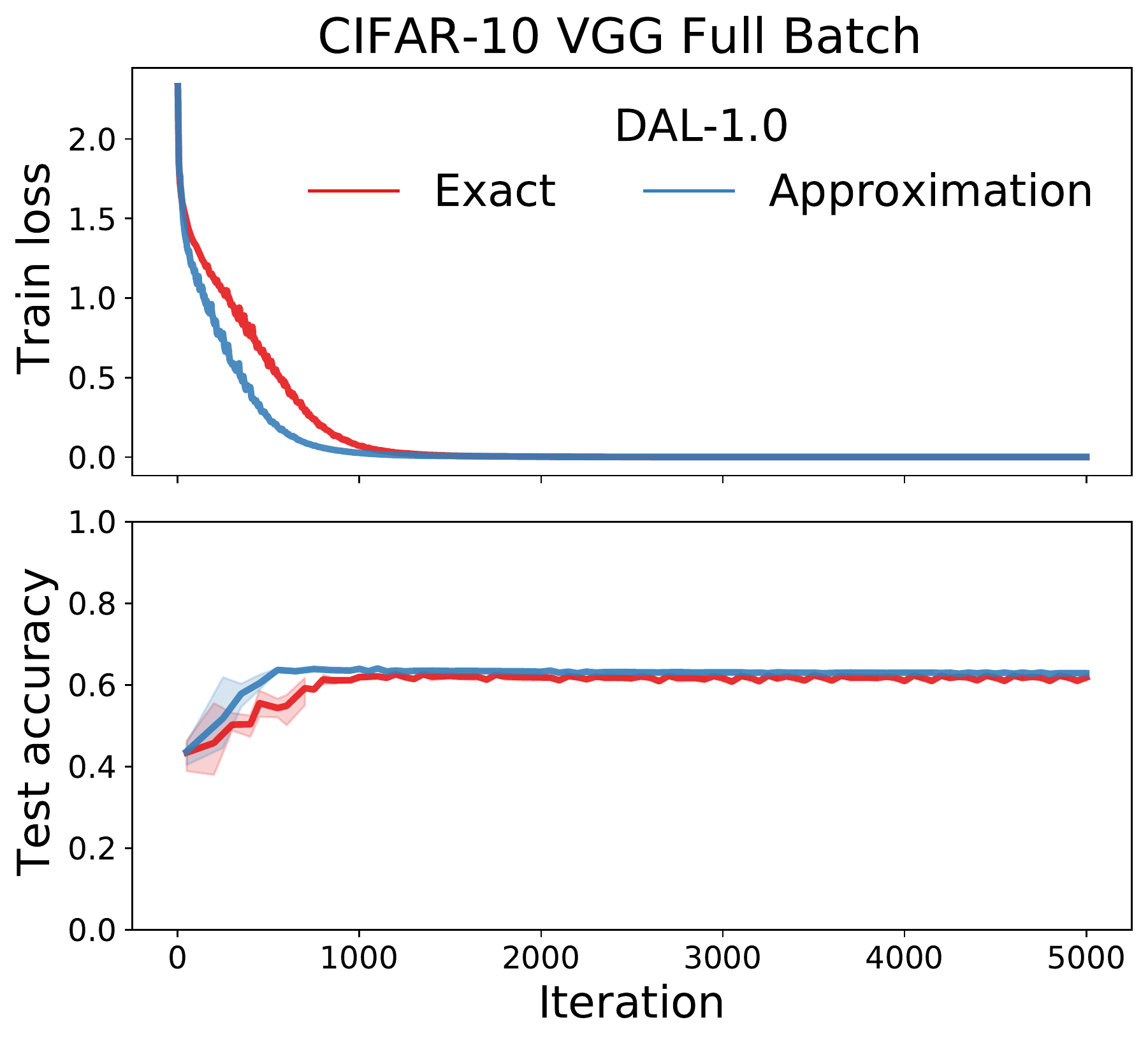}%
 \includegraphics[width=0.33\columnwidth]{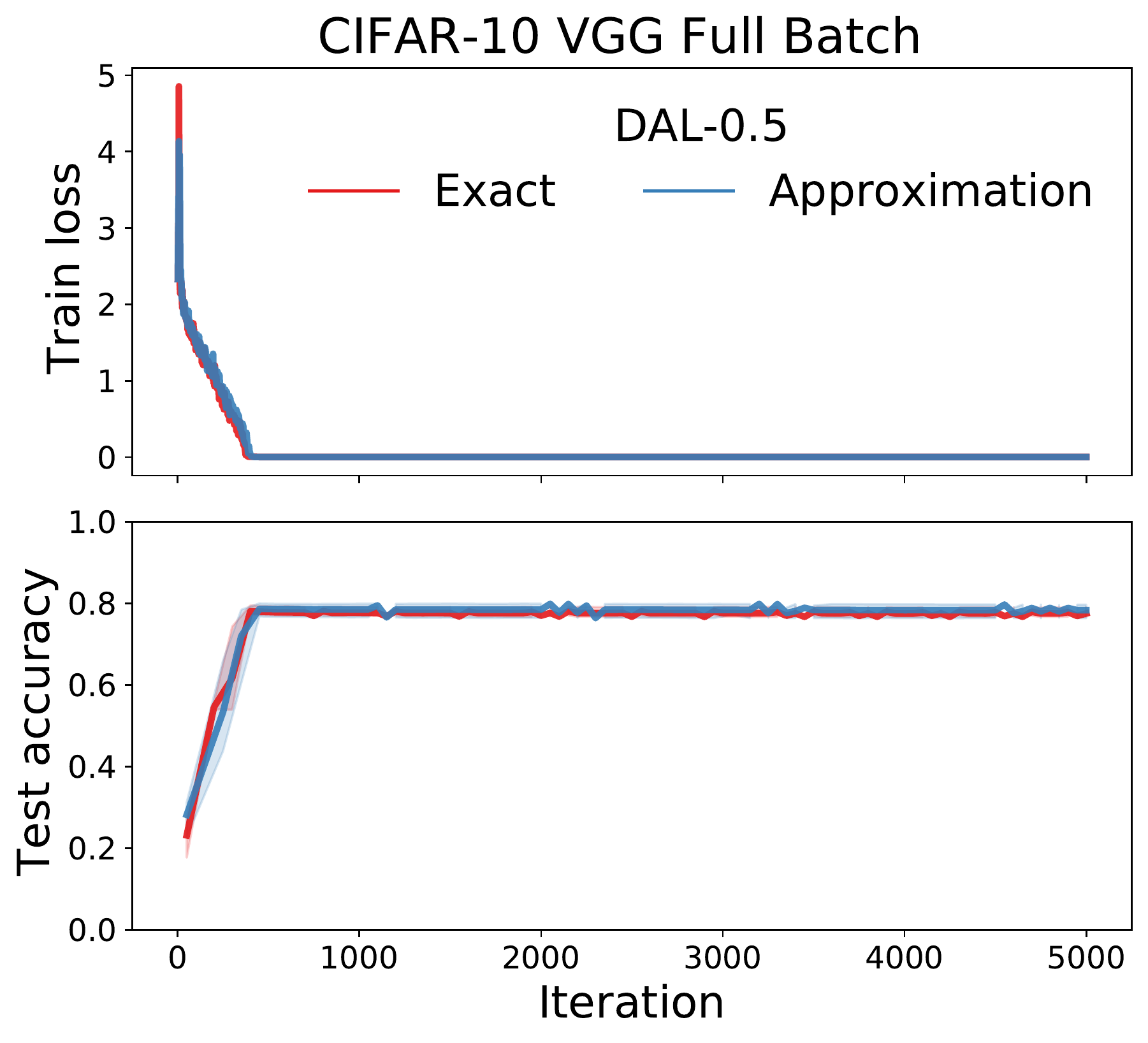}%
 \includegraphics[width=0.33\columnwidth]{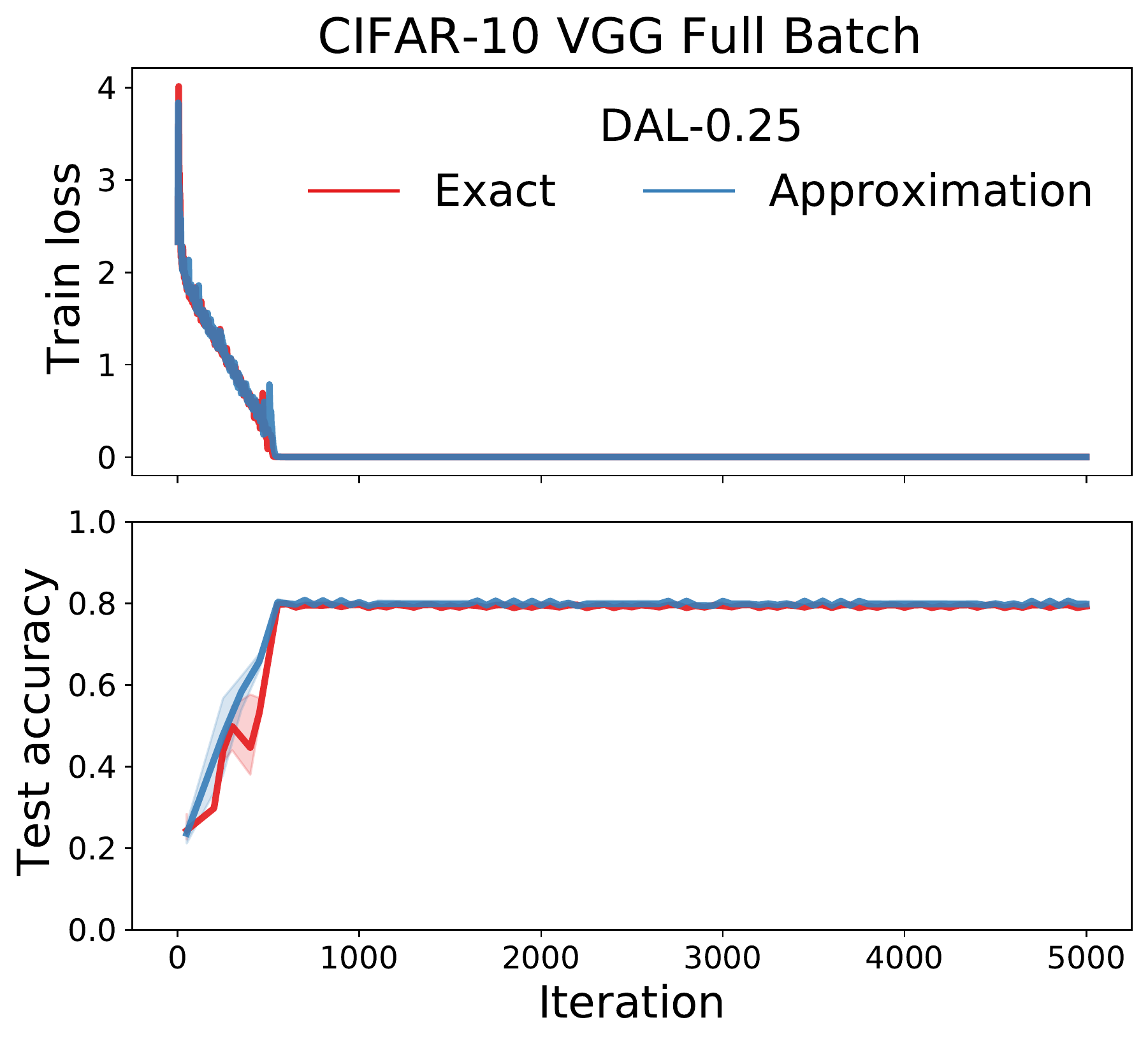}%
\caption[CIFAR-10 DAL results with a Hessian vector product computation of $\nabla_{\vtheta}^2 E \nabla_{\vtheta} E$ compared to an approximation. We used $\epsilon = 0.01$ in the approximation. Full batch training. ]{CIFAR-10 DAL results with a Hessian vector product computation of $\nabla_{\vtheta}^2 E \nabla_{\vtheta} E$ compared to an approximation. We used $\epsilon = 0.01$ in the approximation.}
\label{fig:dal_cifar_full_batch}
\end{figure}

\begin{figure}
 \includegraphics[width=0.32\columnwidth]{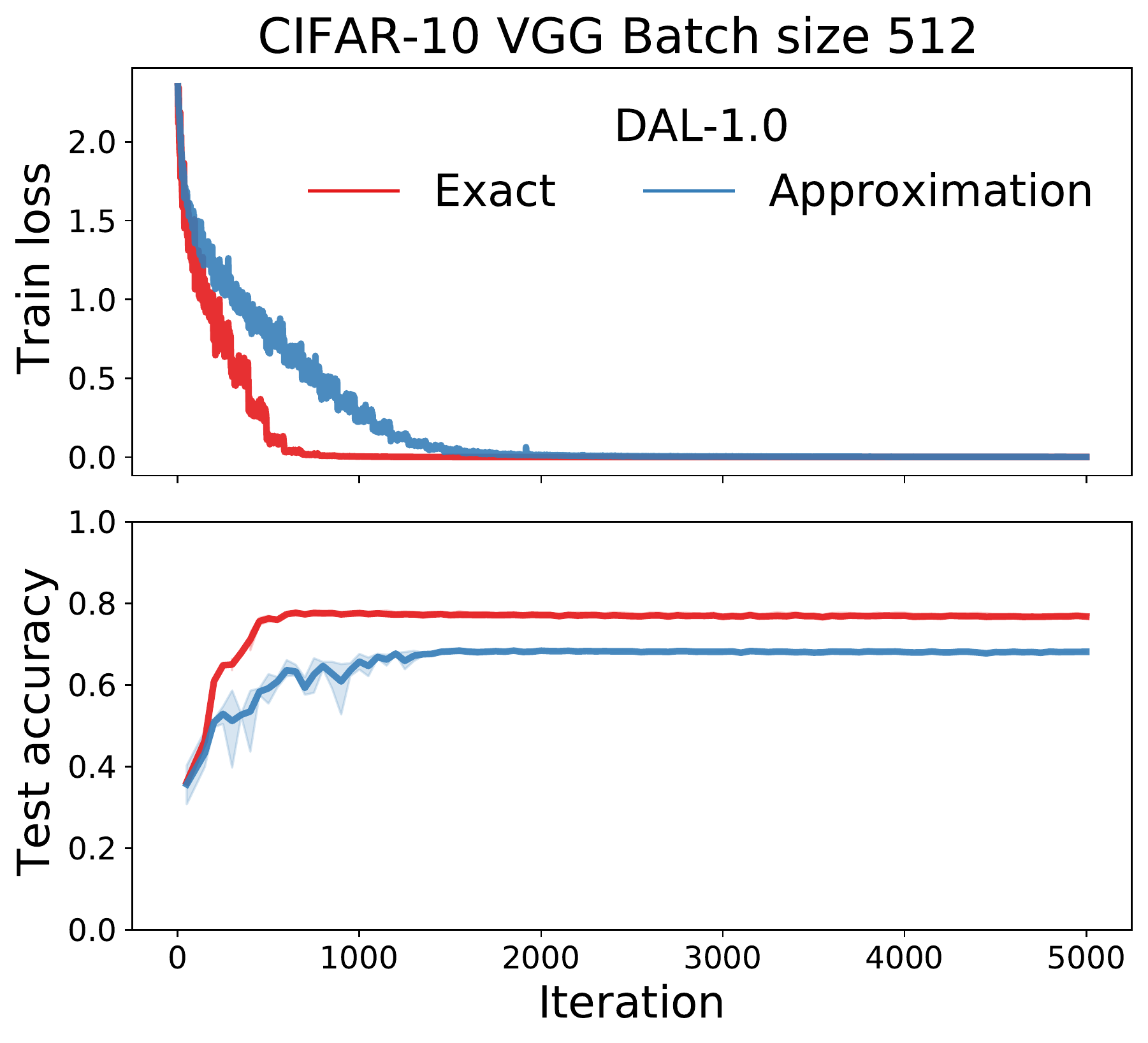}
 \includegraphics[width=0.32\columnwidth]{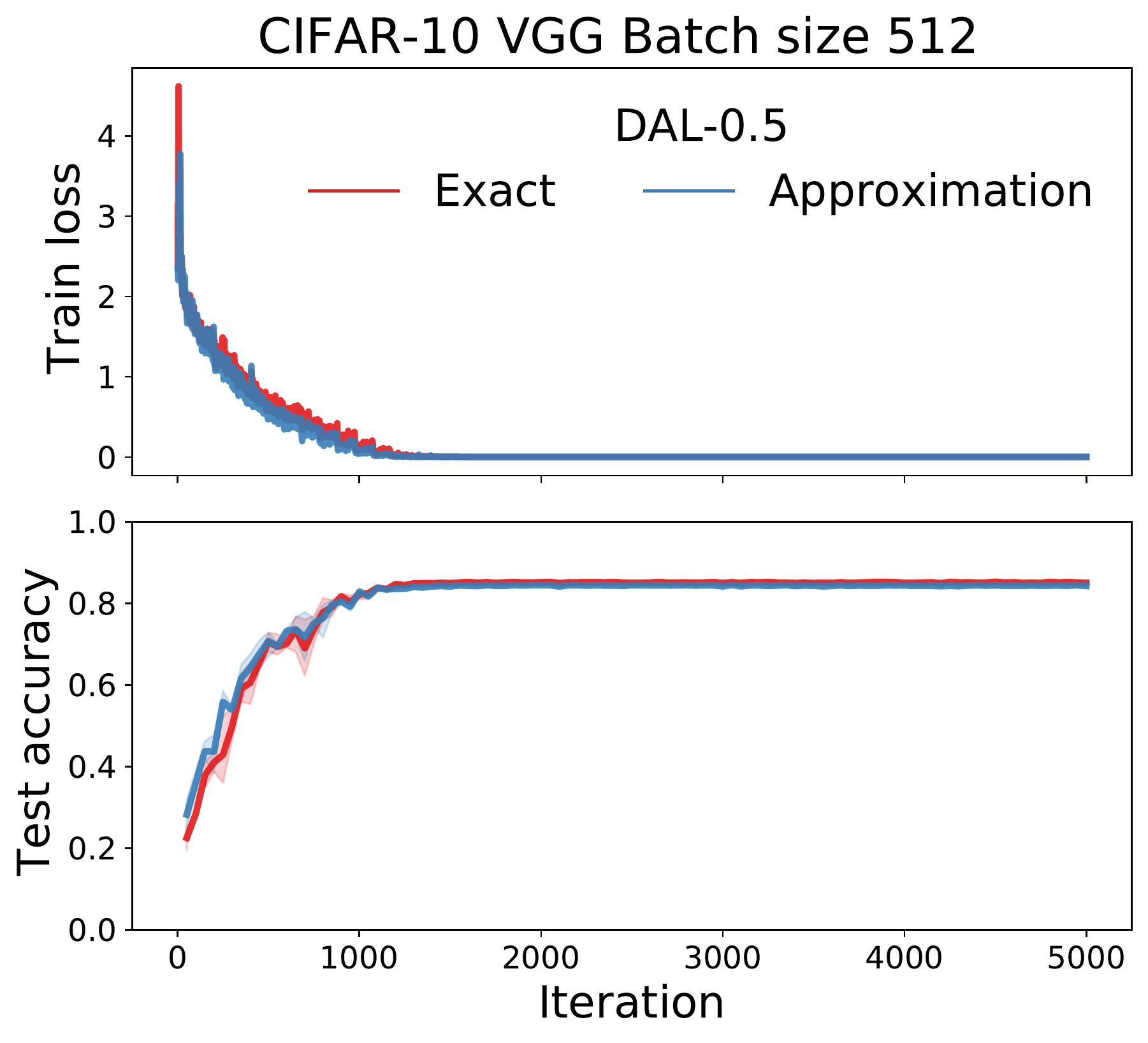}
 \includegraphics[width=0.32\columnwidth]{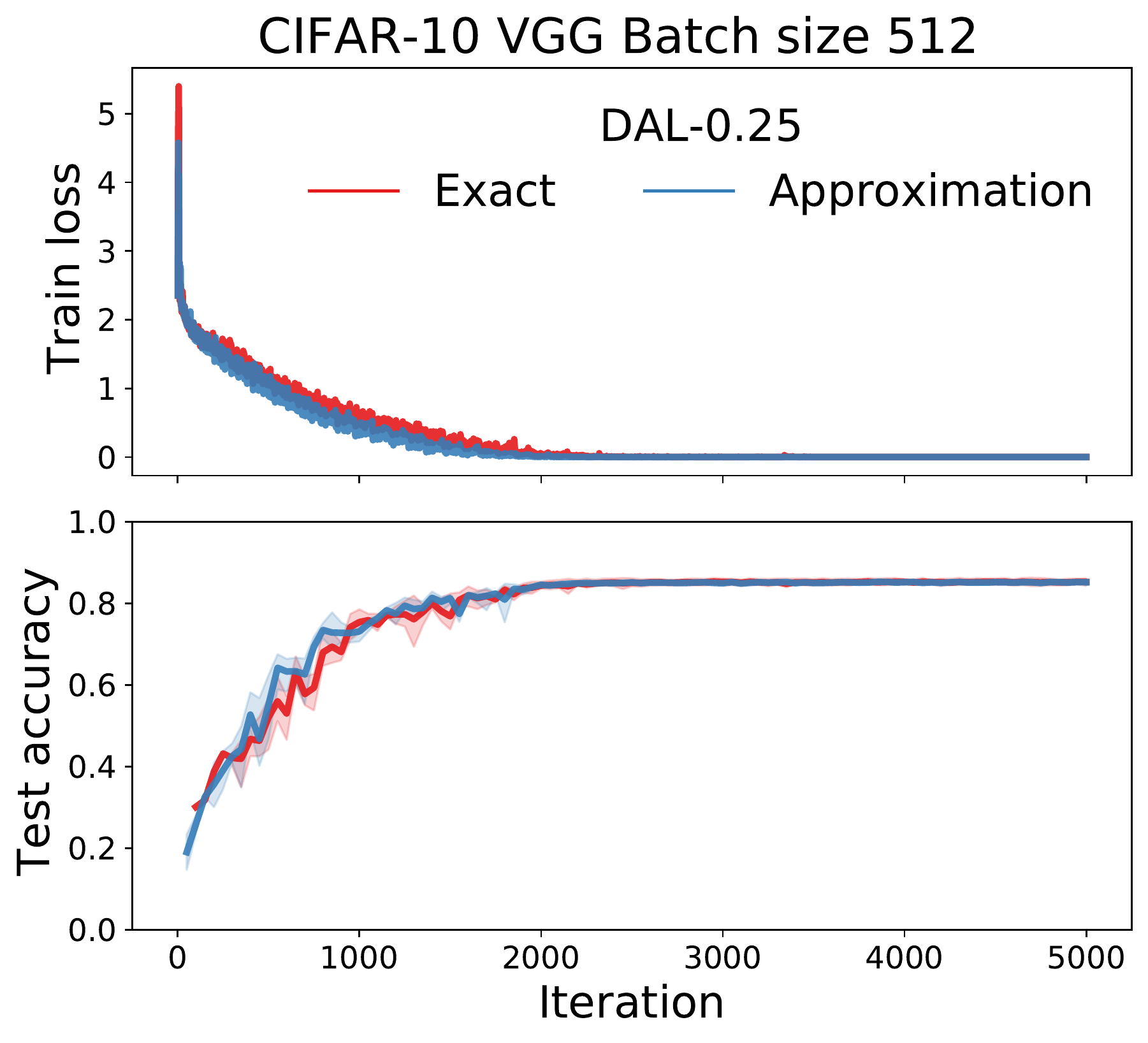}
\caption[CIFAR-10 DAL results with a Hessian vector product computation of $\nabla_{\vtheta}^2 E \nabla_{\vtheta} E$ compared to an approximation. We used $\epsilon = 0.01$ in the approximation. Batch size 512. ]{CIFAR-10 DAL results with a Hessian vector product computation of $\nabla_{\vtheta}^2 E \nabla_{\vtheta} E$ compared to an approximation. We used $\epsilon = 0.01$ in the approximation.}
\label{fig:dal_cifar_sgd}
\end{figure}

\begin{figure}
 \includegraphics[width=0.32\columnwidth]{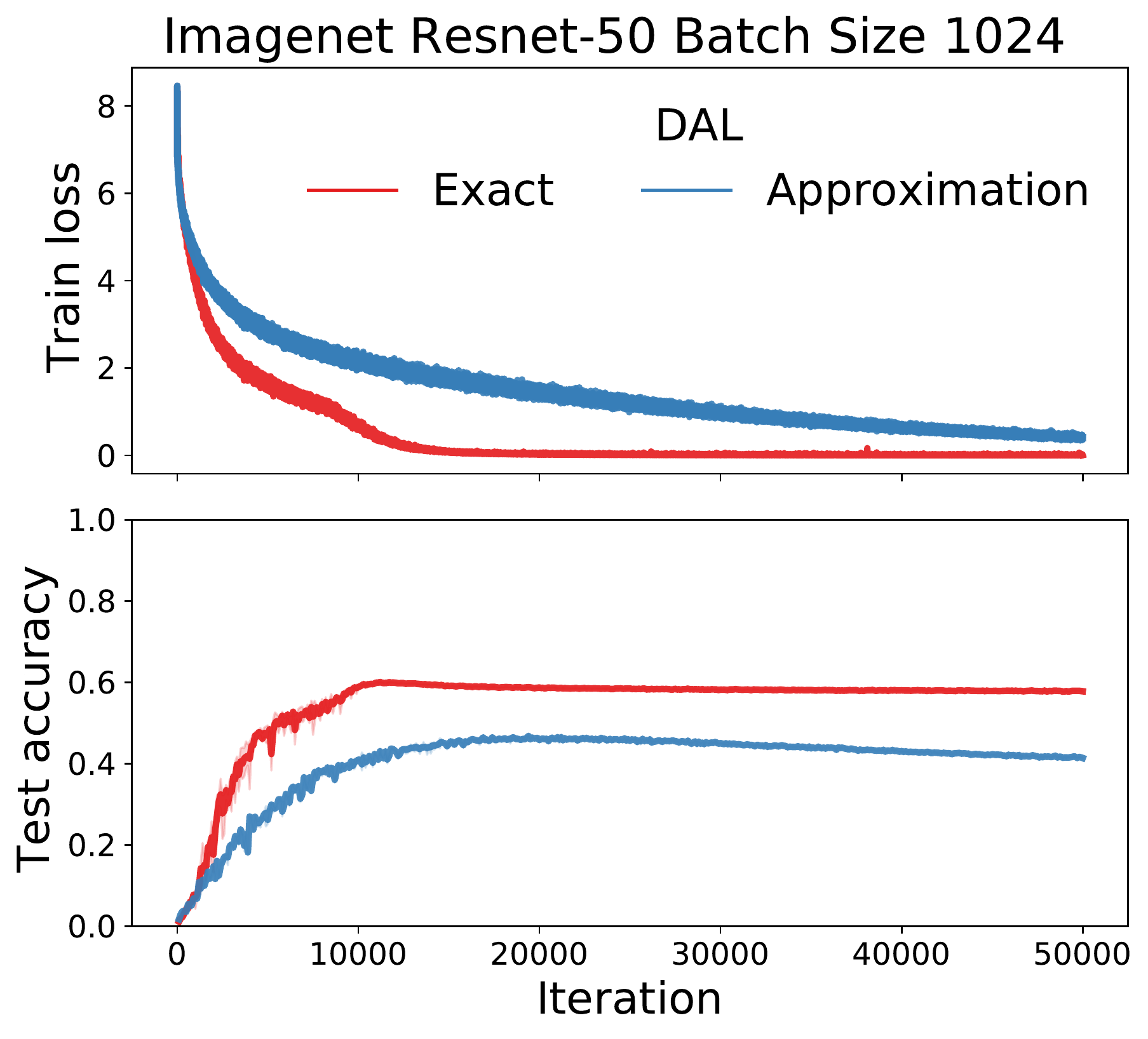}
 \includegraphics[width=0.32\columnwidth]{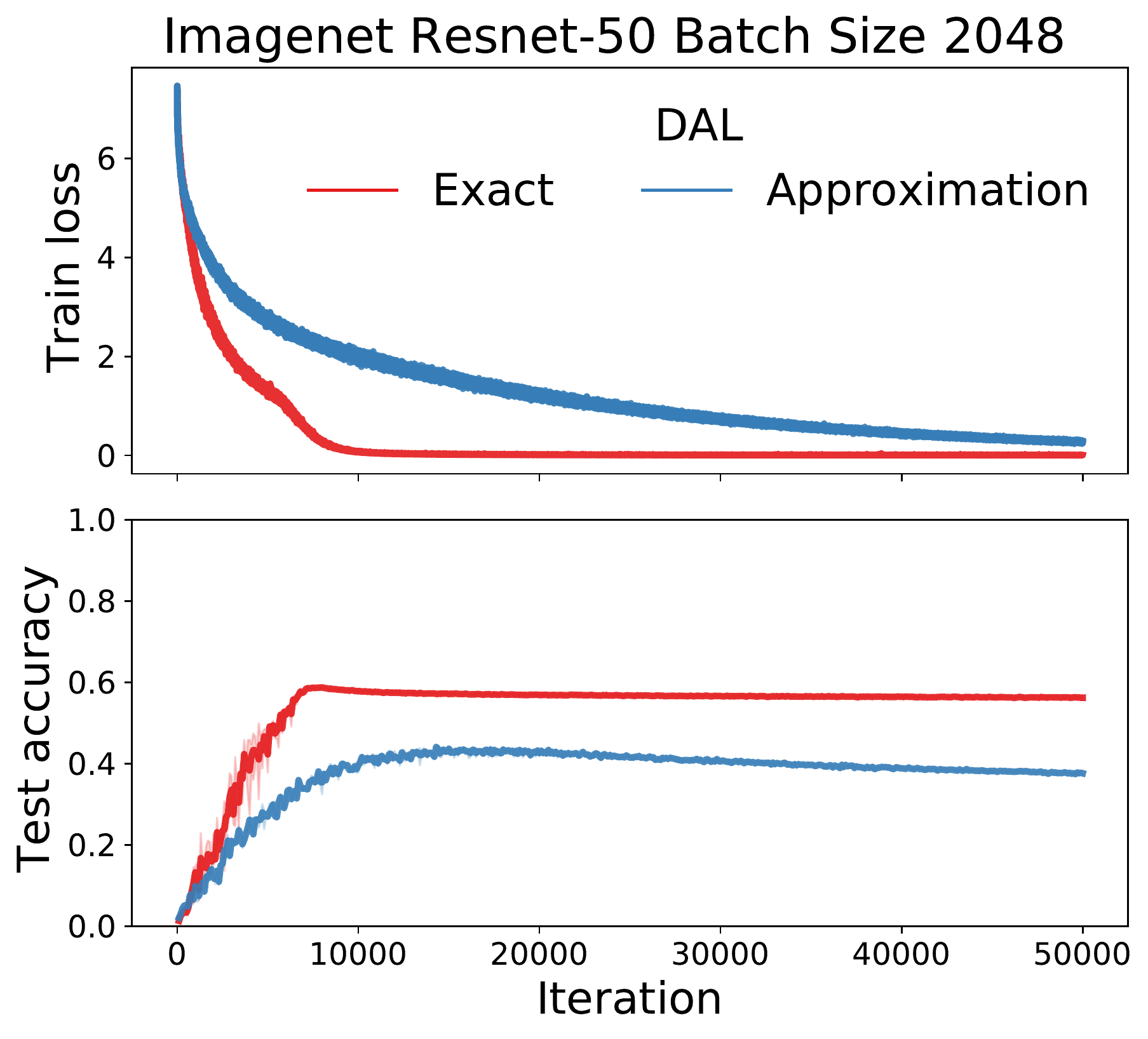}
 \includegraphics[width=0.32\columnwidth]{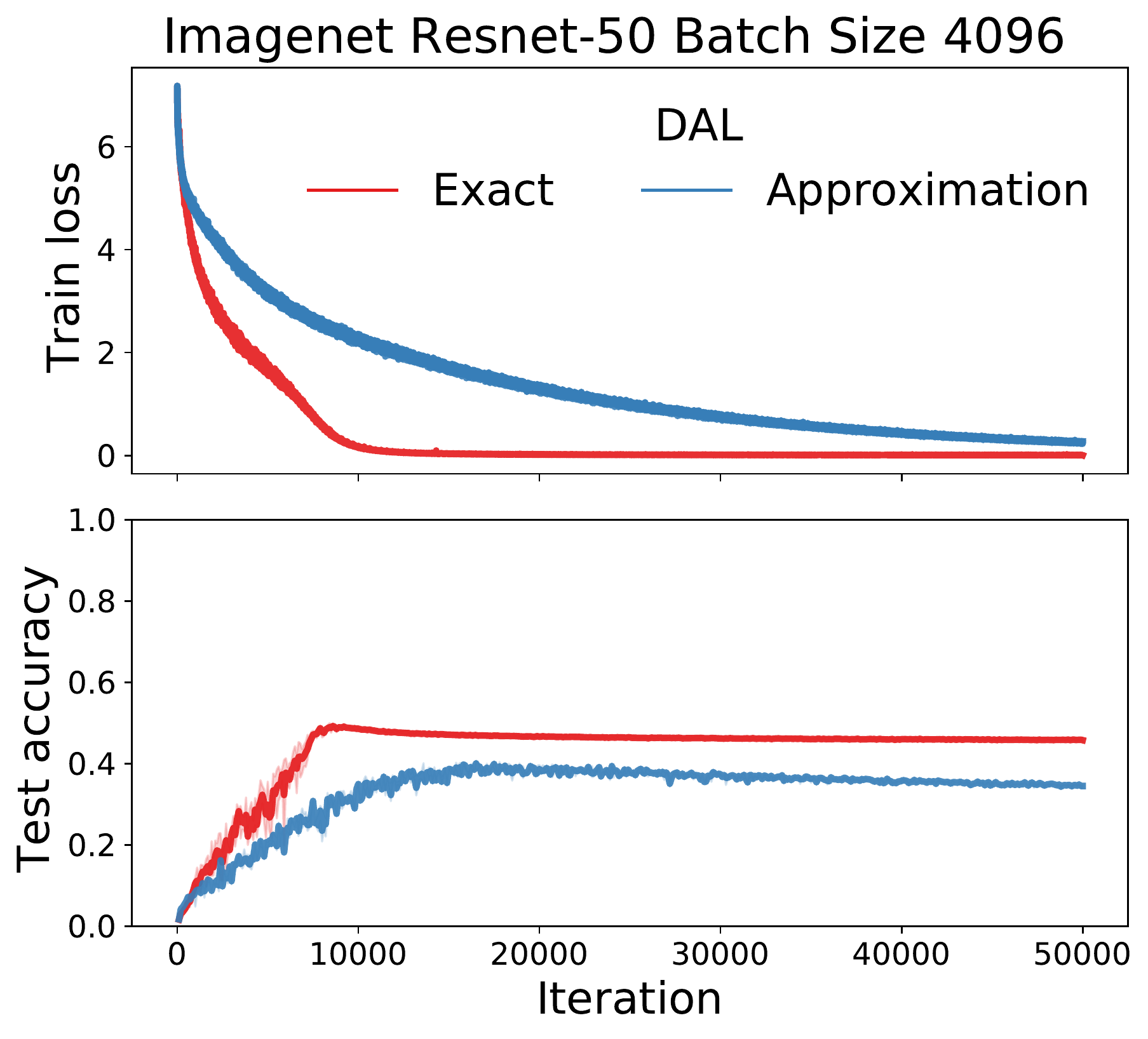}
\caption[Imagenet DAL results with a Hessian vector product computation of $\nabla_{\vtheta}^2 E \nabla_{\vtheta} E$ compared to an approximation. We used $\epsilon = 0.01$ in the approximation. ]{Imagenet DAL results with a Hessian vector product computation of $\nabla_{\vtheta}^2 E \nabla_{\vtheta} E$ compared to an approximation. We used $\epsilon = 0.01$ in the approximation.}
\label{fig:dal_approx_comp}
\end{figure}

\begin{figure}
 \includegraphics[width=0.32\columnwidth]{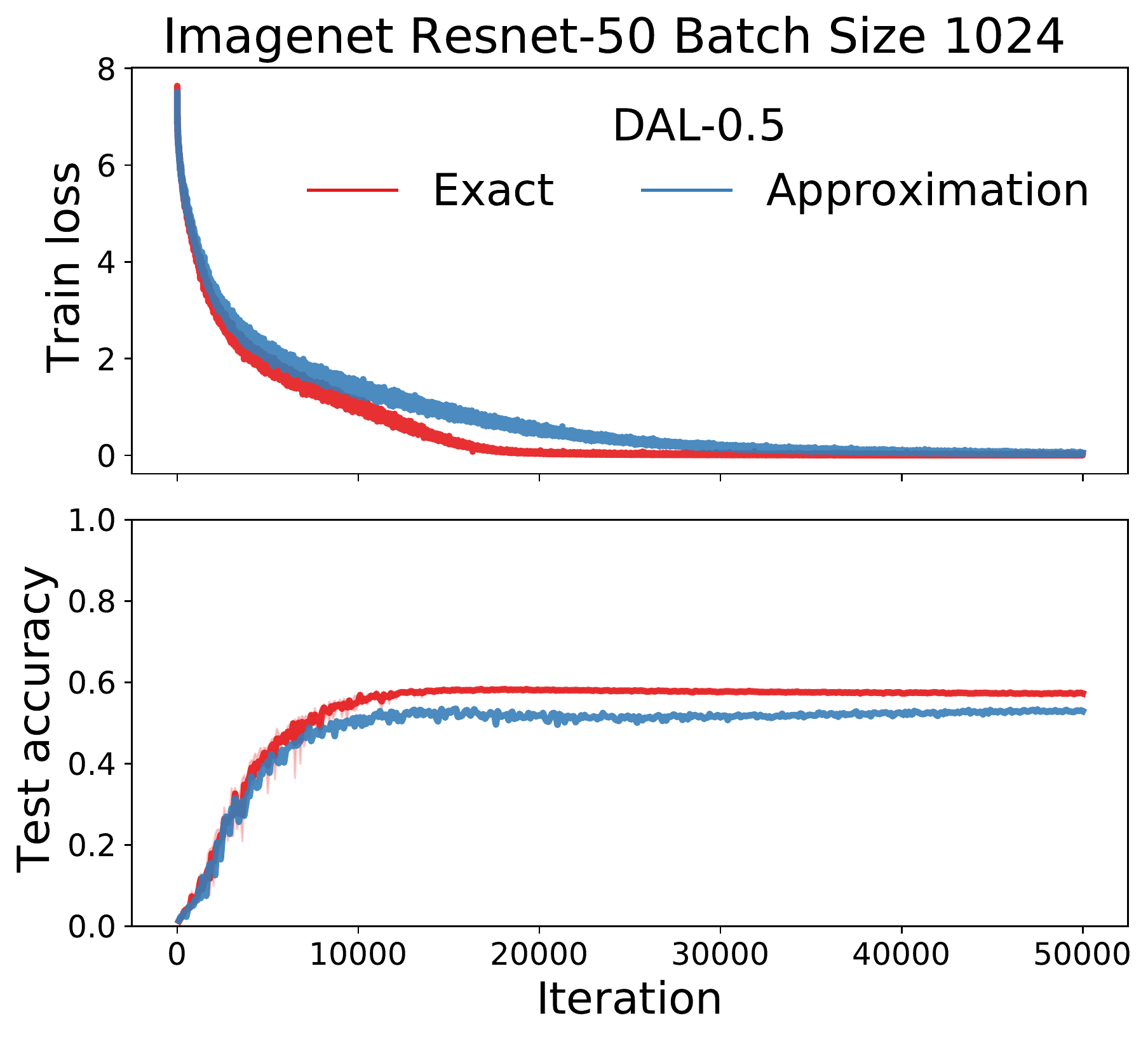}
 \includegraphics[width=0.32\columnwidth]{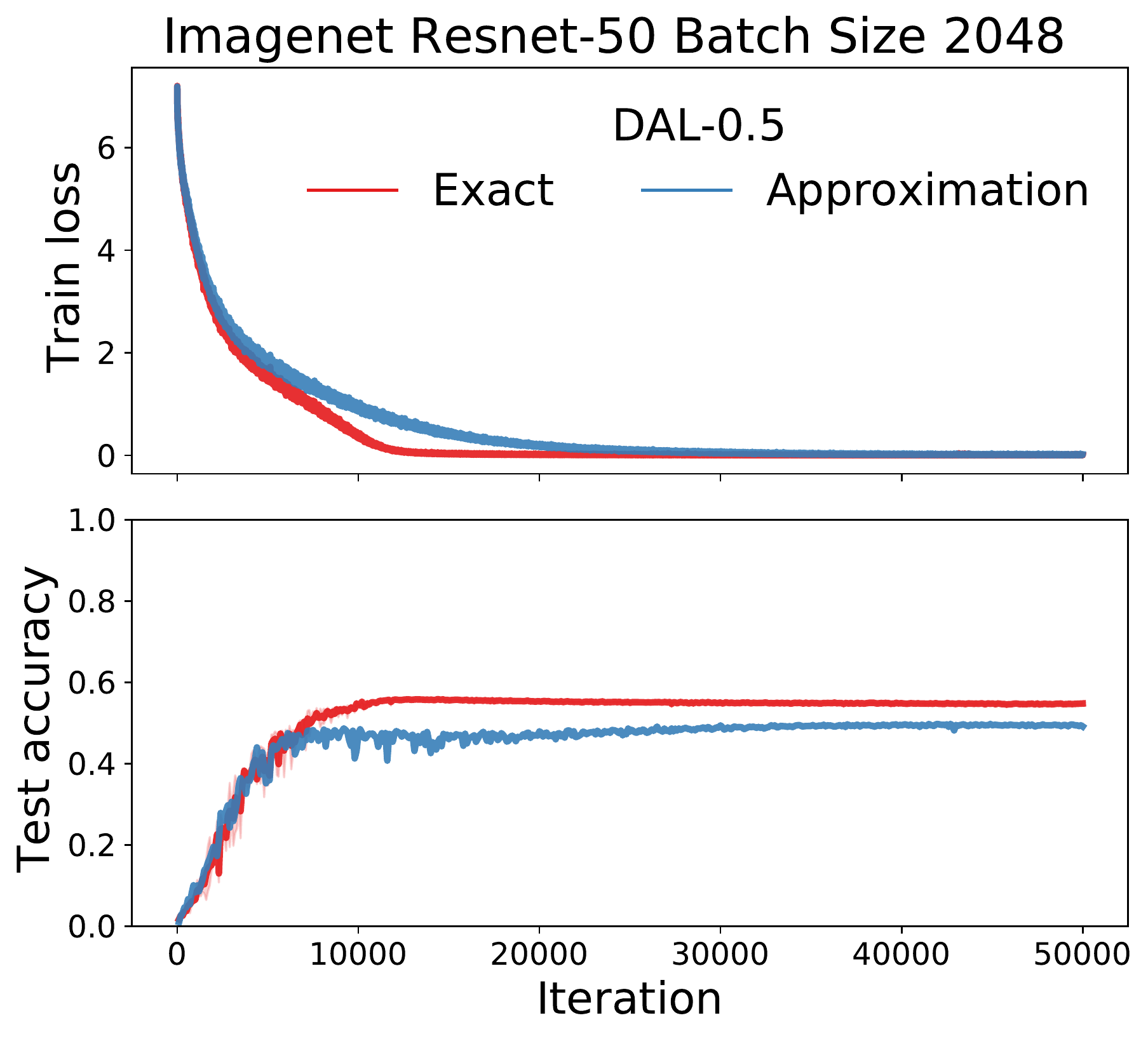}
 \includegraphics[width=0.32\columnwidth]{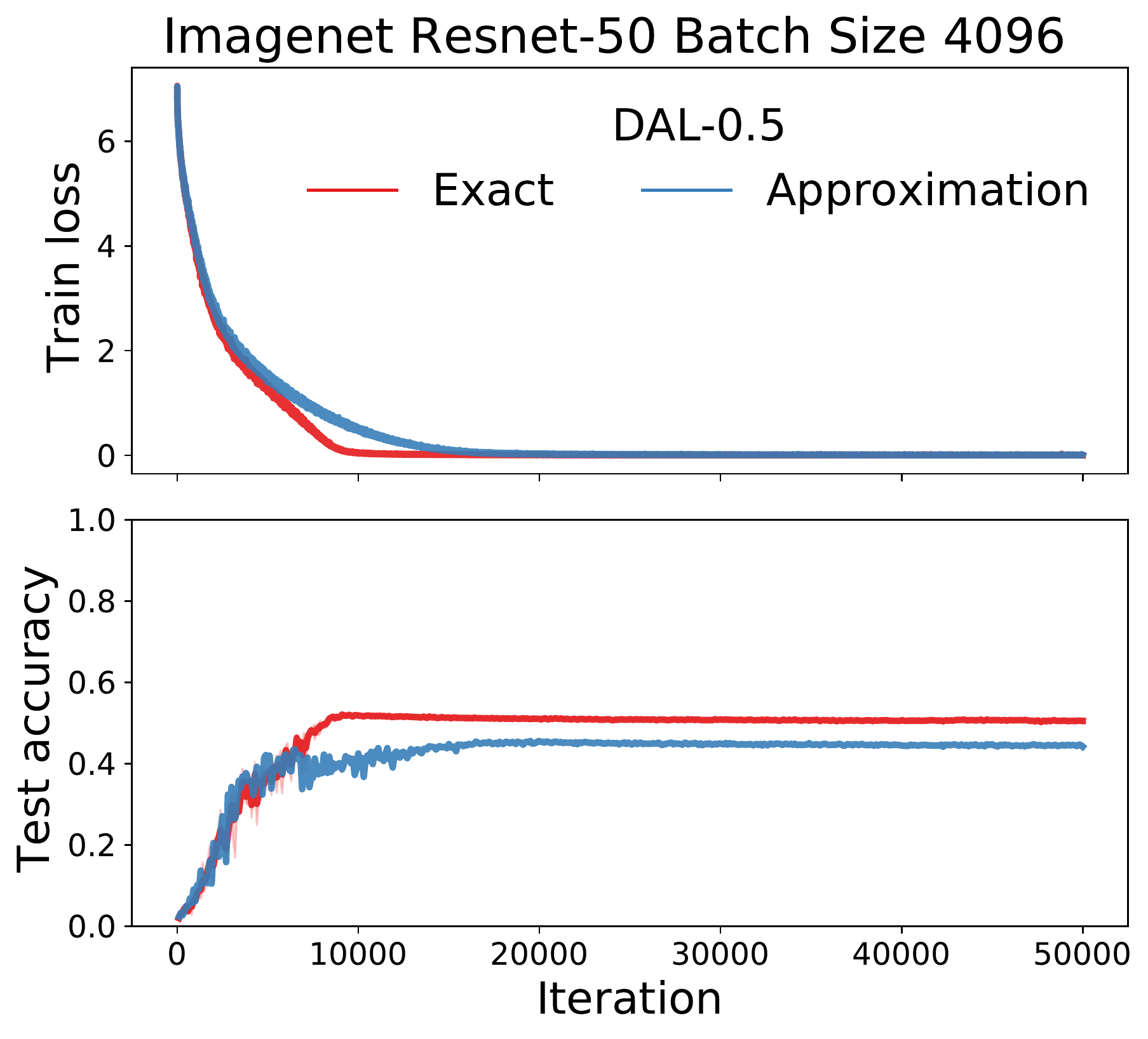}
\caption[Imagenet DAL-$0.5$ results with a Hessian vector product computation of $\nabla_{\vtheta}^2 E \nabla_{\vtheta} E$ compared to an approximation. We used $\epsilon = 0.01$ in the approximation. ]{Imagenet DAL$-0.5$ results with a Hessian vector product computation of $\nabla_{\vtheta}^2 E \nabla_{\vtheta} E$ compared to an approximation. We used $\epsilon = 0.01$ in the approximation.}
\label{fig:dal_0_5_approx_comp}
\end{figure}

\textbf{Seeds}. All test accuracies are shown averaged from 3 seeds. For training curves, we compare models across individual training runs to be able to observe the behavior of gradient descent. We did not observe variability across seeds with any of the behaviors reported in this paper.

\textbf{DAL.} When using DAL we set a maximum learning rate of 5 to avoid any potential instabilities. We did not experiment with other values.

\section{Additional experimental results}

\textbf{Oscillations for non linear functions}. 
We show that in the case of non-linear functions in Figure~\ref{fig:1d_cosine}, using $\cos$. Here too, the PF is better at describing the behavior of GD than the NGF and IGR flow, which stop at local minima.

\begin{figure}
 \includegraphics[width=0.48\columnwidth]{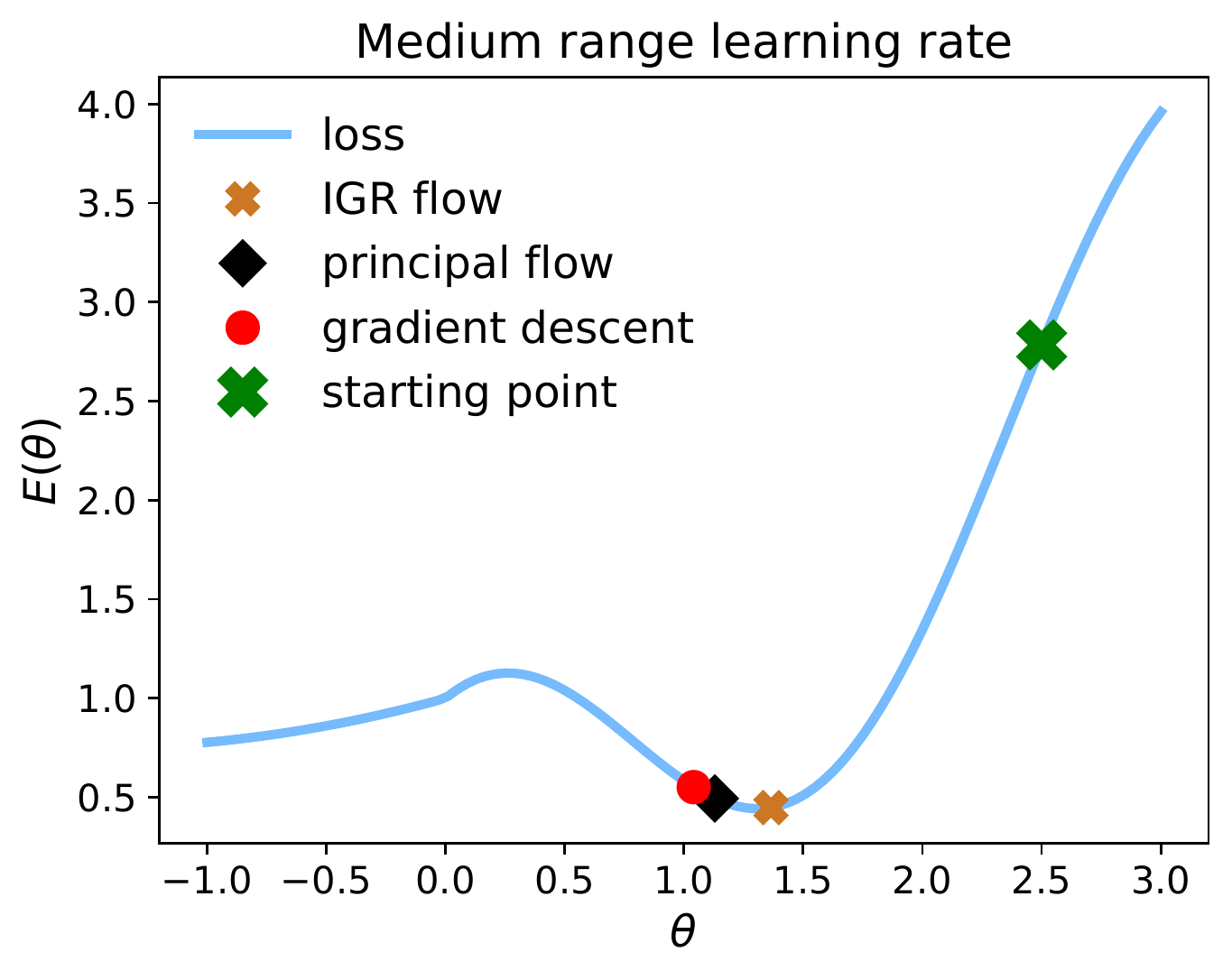}
 \includegraphics[width=0.48\columnwidth]{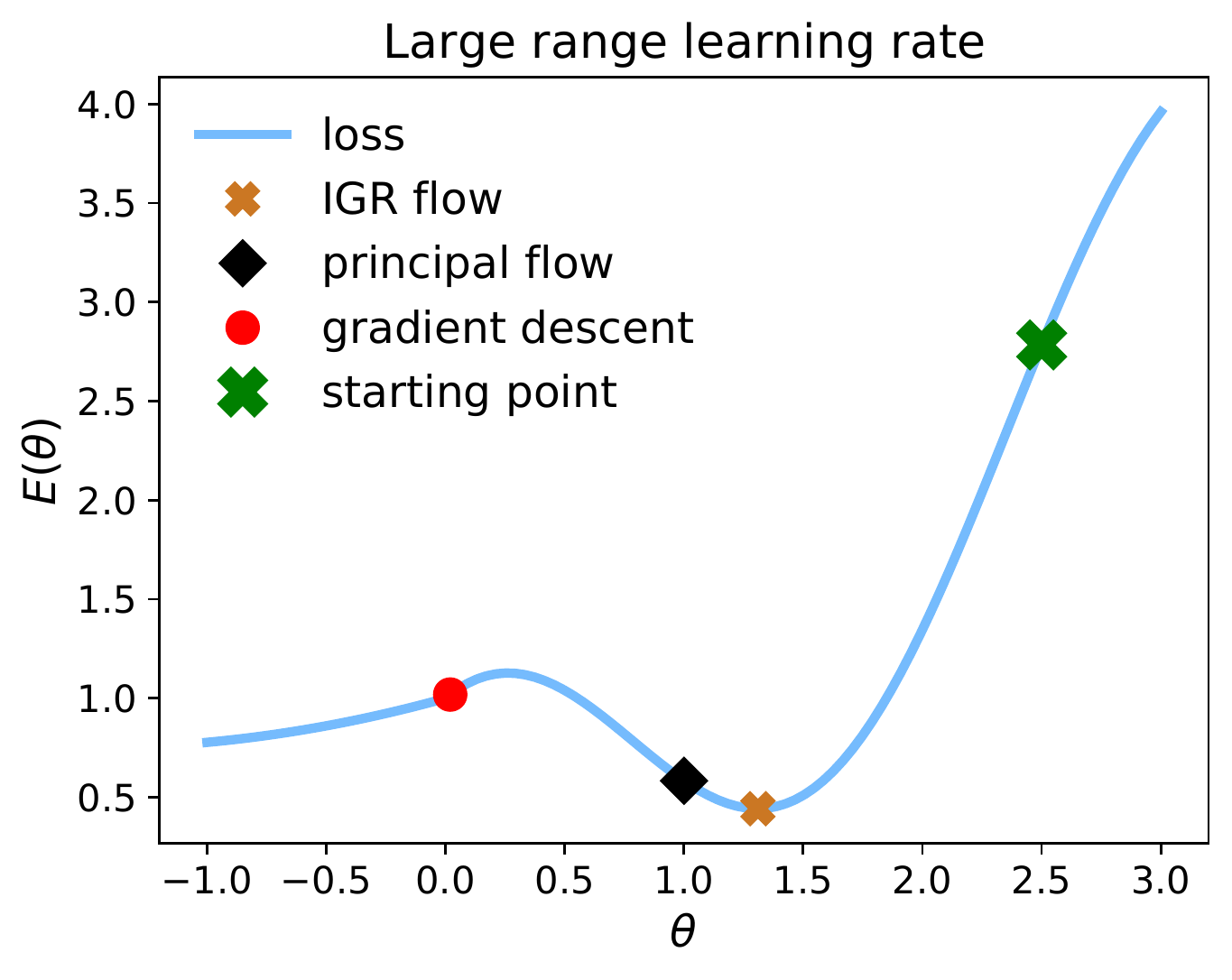}
\caption[Results with a non quadratic function, including $\cos$ and $\sin$. The function used is: $E(\theta) = \cos(\theta) + \theta, \text{if } \theta < 0; (\theta/3)^22 + 1 + \theta/3)$ otherwise. The principal flow and the IGR flow stay on the cosine branch. The main reason for using this function was to show the escape of gradient descent into a flatter valley and to assess what the corresponding flows do. Learning late on the left: $0.5$. On the right: $0.85$.]{Results with a non quadratic function, including $\cos$ and $\sin$. The function used is: $E(\theta) = \cos(\theta) + \theta, \text{if } \theta < 0; (\theta/3)^22 + 1 + \theta/3)$ otherwise. The principal flow and the IGR flow stay on the cosine branch. The main reason for using this function was to show the escape of gradient descent into a flatter valley and to assess what the corresponding flows do. Learning late on the left: $0.5$. On the right: $0.85$.}
\label{fig:1d_cosine}
\end{figure}

\textbf{Per iteration drift for neural networks with IGR and third order flows}.
We show results estimating the per iteration error between GD and the IGR flow and the third order flow in Figure~\ref{fig:validating_bea}.
While the IGR and third order flows reduce error compared to the NGF, a significant gap remains. This result shows that in neural network training, we need to look at higher orders in learning to explain neural network behavior.

\begin{figure}
 \includegraphics[width=0.33\columnwidth]{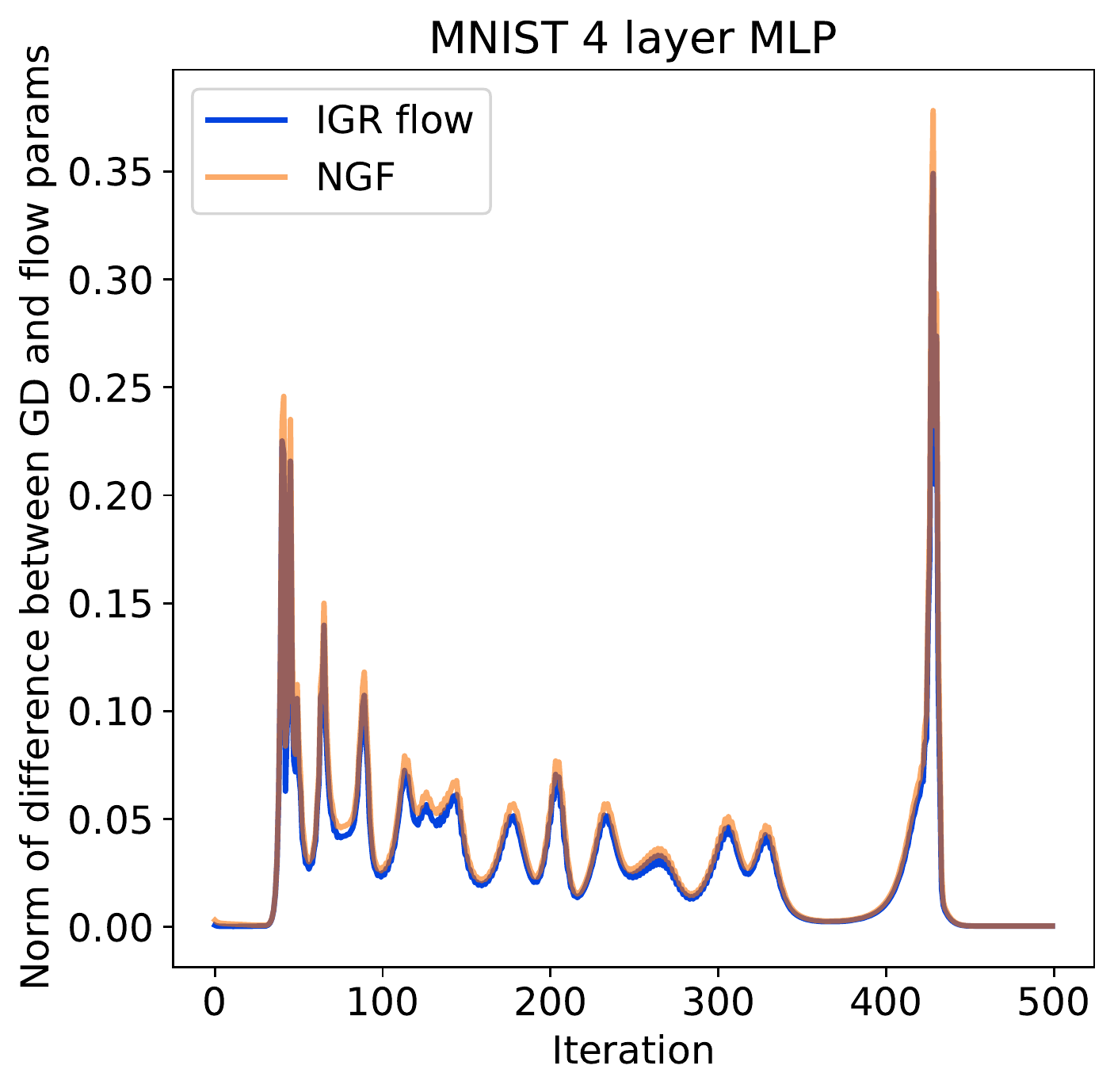}
 \includegraphics[width=0.33\columnwidth]{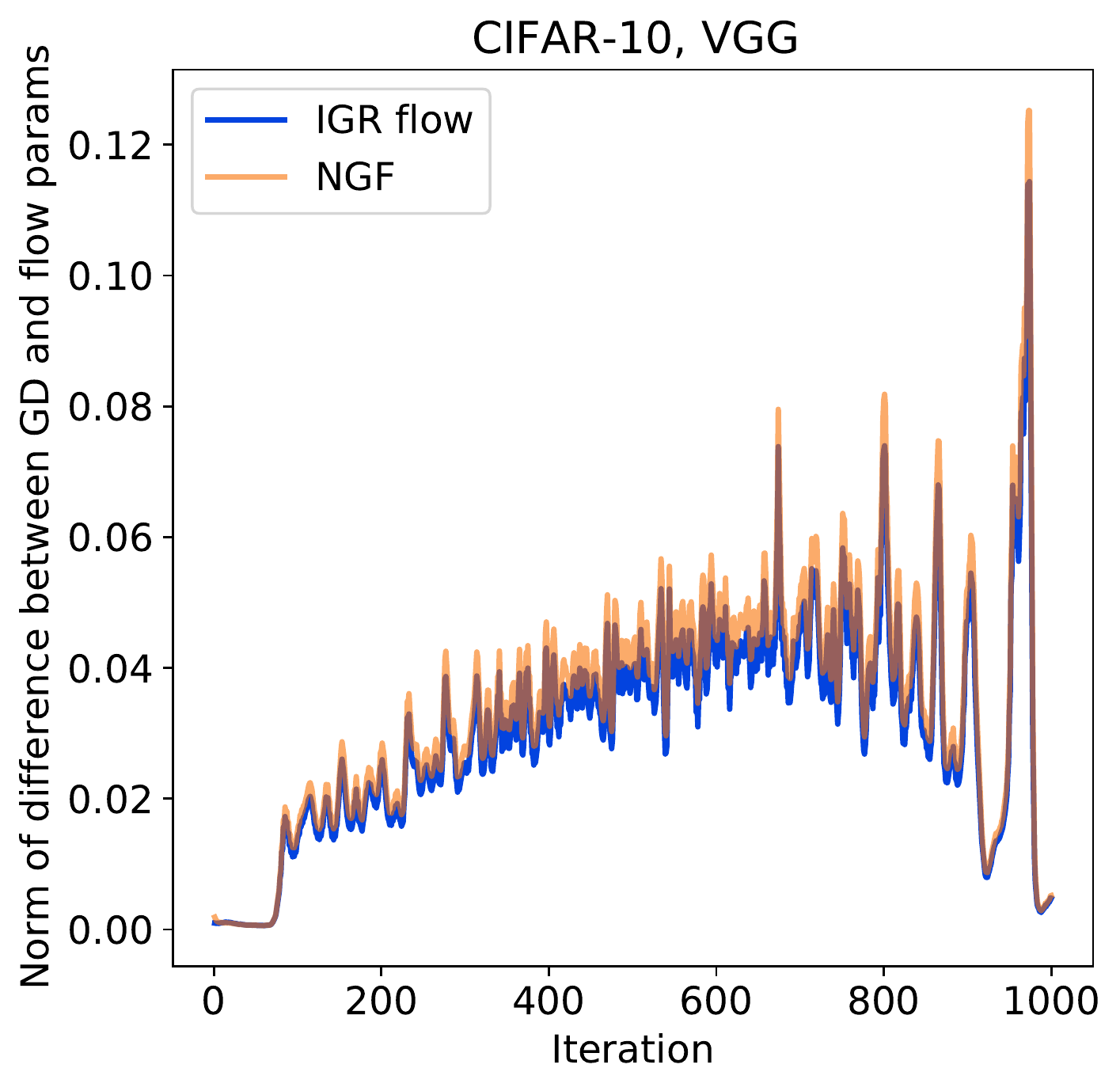}
 \includegraphics[width=0.33\columnwidth]{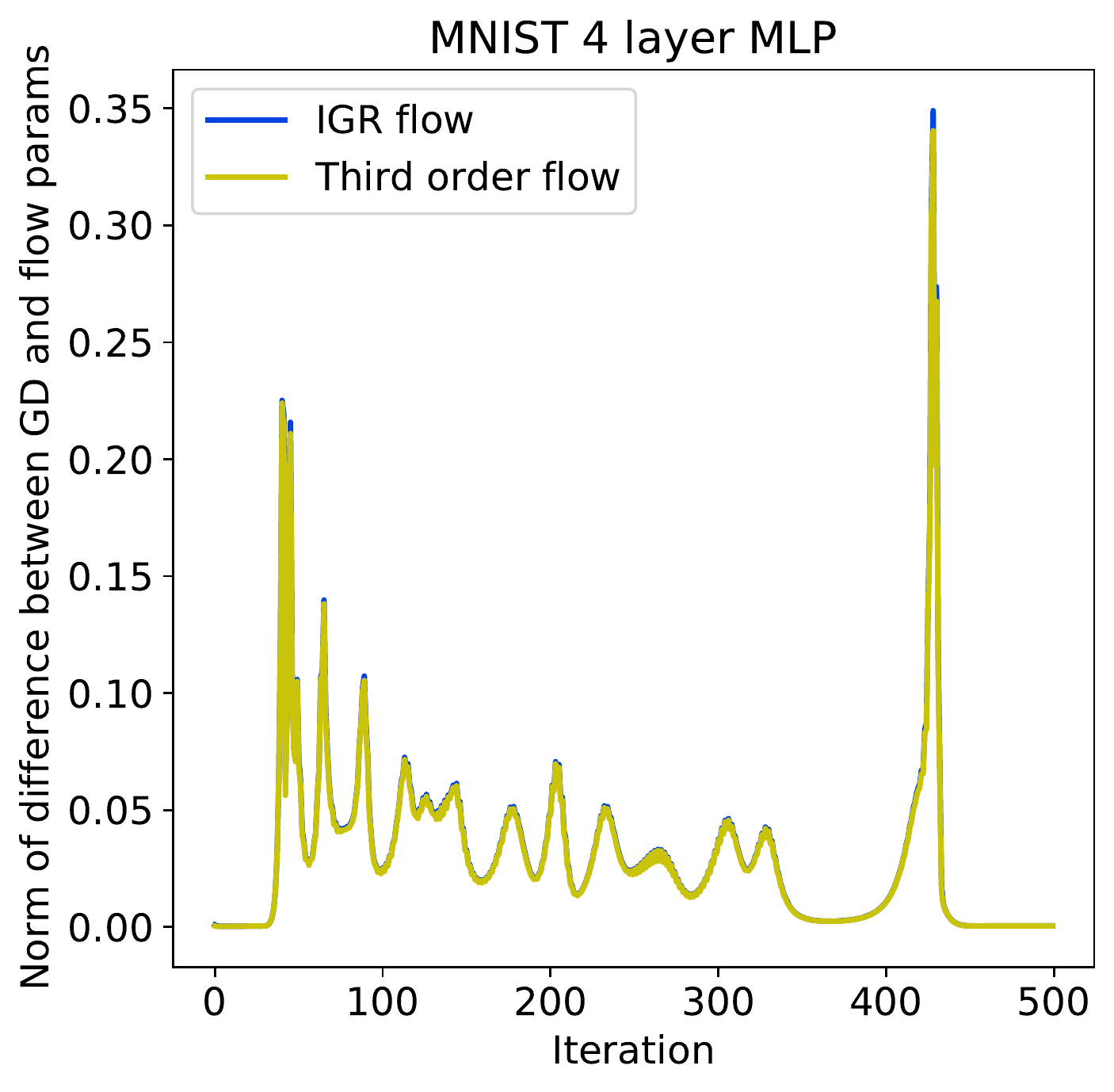}
\caption[Per iteration error between the NGF, IGR and PF and gradient descent. MNIST learning rate $0.05$. CIFAR learning rate $0.02$. We obtain the results by training a model with gradient descent $\vtheta_t = \vtheta_{t-1} - h \nabla_{\vtheta} E(\vtheta_{t-1})$ and at each iteration $t$ of gradient descent approximating the respective flows for time $h$ and computing the difference in norms between the resulting flow parameters and the gradient descent parameters at the next iteration: $\norm{\vtheta_t - \vtheta(h)}$ with $\vtheta(h) = \vtheta_{t-1}$. Results obtained on full batch training.]{\textbf{Motivation}. While modified  flows reduce the per iteration discretization error for neural networks trained with gradient descent, there is still a large gap. We obtain the results by training a model with gradient descent $\vtheta_t = \vtheta_{t-1} - h \nabla_{\vtheta} E(\vtheta_{t-1})$ and at each iteration $t$ of gradient descent approximating the respective flows for time $h$ and computing the difference in norms between the resulting flow parameters and the gradient descent parameters at the next iteration: $\norm{\vtheta_t - \vtheta(h)}$ with $\vtheta(h) = \vtheta_{t-1}$.}
\label{fig:validating_bea}
\end{figure}

\textbf{Global error for the UCI breast cancer dataset.}
We show in Figure~\ref{fig:breast_cancer_principal_flow} that the PF is better at tracking the behavior of GD in the case of NNs both in the stability and edge of stability areas.

\begin{figure}
\centering
\begin{subfigure}[Stability]{
 \includegraphics[width=0.33\columnwidth]{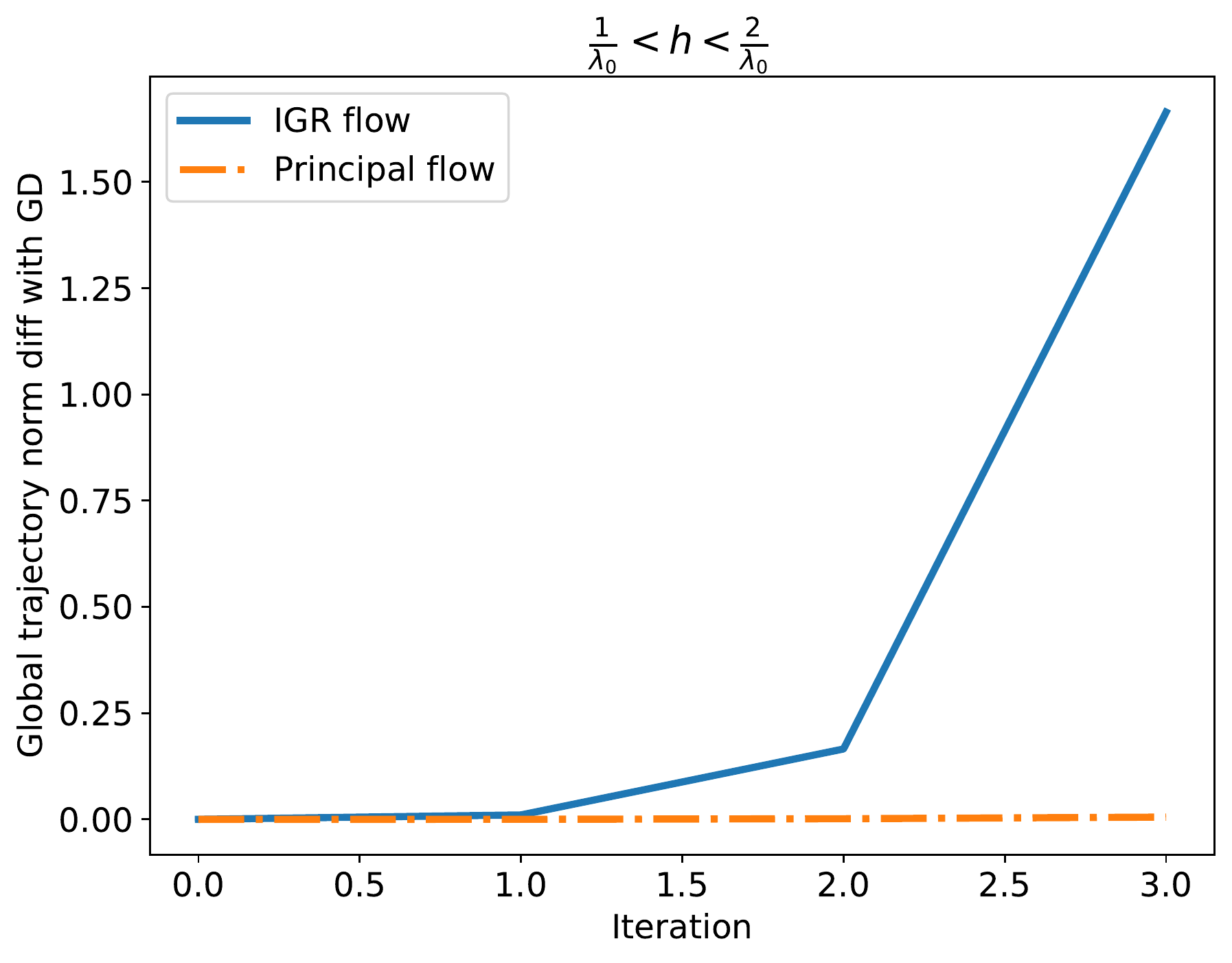}
}\end{subfigure}
\begin{subfigure}[Edge of stability]{
 \includegraphics[width=0.33\columnwidth]{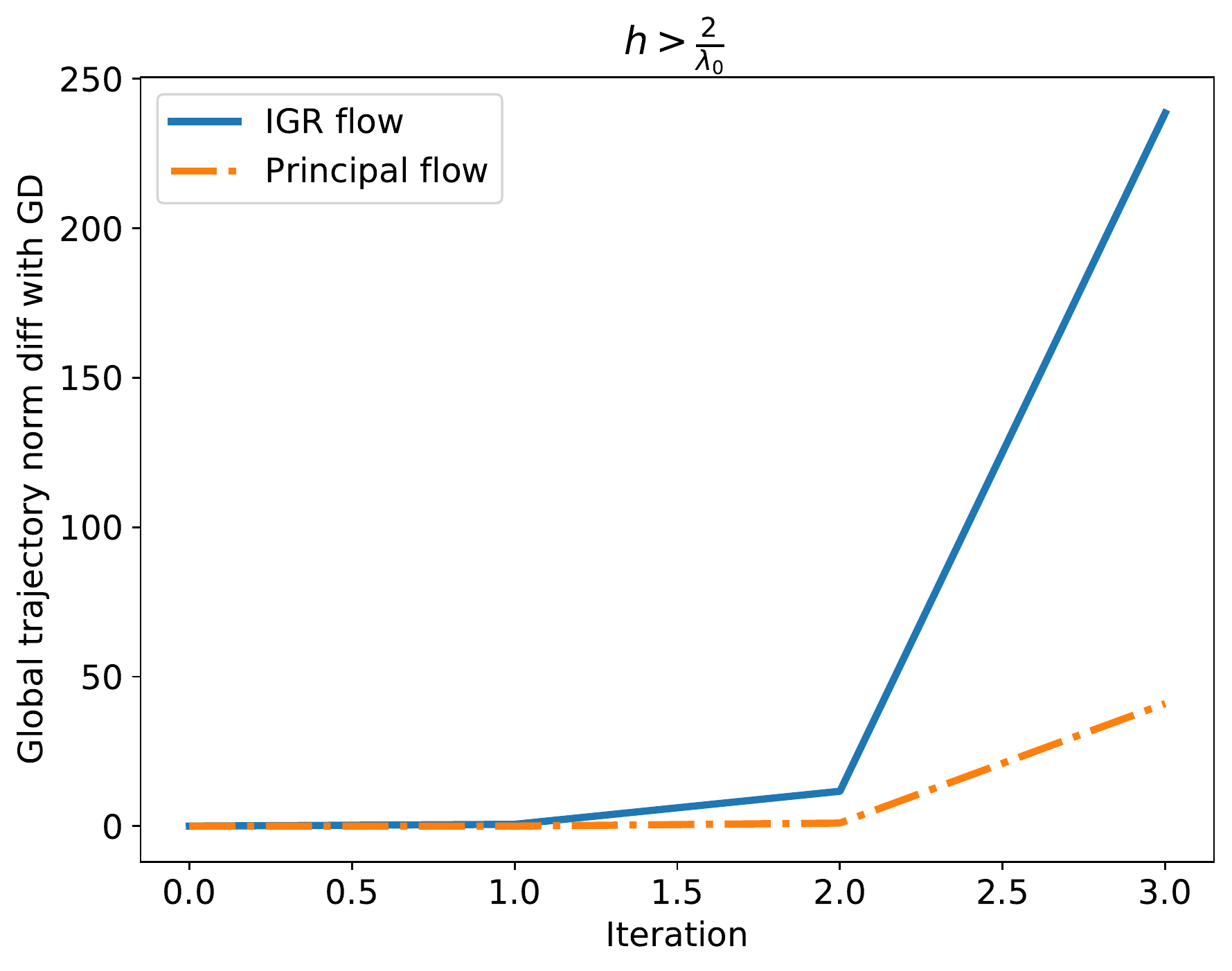}
}
\end{subfigure}
\caption[Global parameter error for IGR ad PF flows compared  on an MLP trained on the UCI breast cancer dataset. MLP with 10 hidden units and 1 output unit. Learning rates $1.5$ and $5.$ respectively.]{Global parameter error for IGR ad PF flows compared  on an MLP trained on the UCI breast cancer dataset. We initialize both gradient descent and the IGR and PF flows at a set of initial parameters, and run them for time step $2h$ and compare the error after each time step $h$. The PF is better at tracking the trajectory of GD than the IGR flow by a significant margin.}
\label{fig:breast_cancer_principal_flow}
\end{figure}

\textbf{Not attractive sharp minima experiment.} We perform an experiment to assess whether local minima with $\lambda_0^*> 2/h$ are attractive to gradient descent with learning rate $h$. Since in the case of neural networks we cannot easily construct local minima, we perform an experiment where we first train a neural network with learning rate $h'<h$ close to convergence, and then change the learning rate. We show results in Figure~\ref{fig:stability_analysis_nn} and observe that despite being in an area with a small gradient norm, increasing the learning rate slightly leads to instabilities and exiting the area around the local minima. We note that while this is \textit{not} the same as assessing local attraction in a stability sense, since we do not know if the point of the learning rate change is in the neigbourhood provided by the existance conditions of local stability, it is the closest practical approximation that we can provide in the context of neural networks. We also note that one of our goals with this experiment is to try to disentangle the training trajectory from the stability behavior, since the training trajectory also has a connection with the largest eigenvalue of the Hessian, through the edge of stability results.

\begin{figure}
\centering
 \includegraphics[width=0.31\columnwidth]{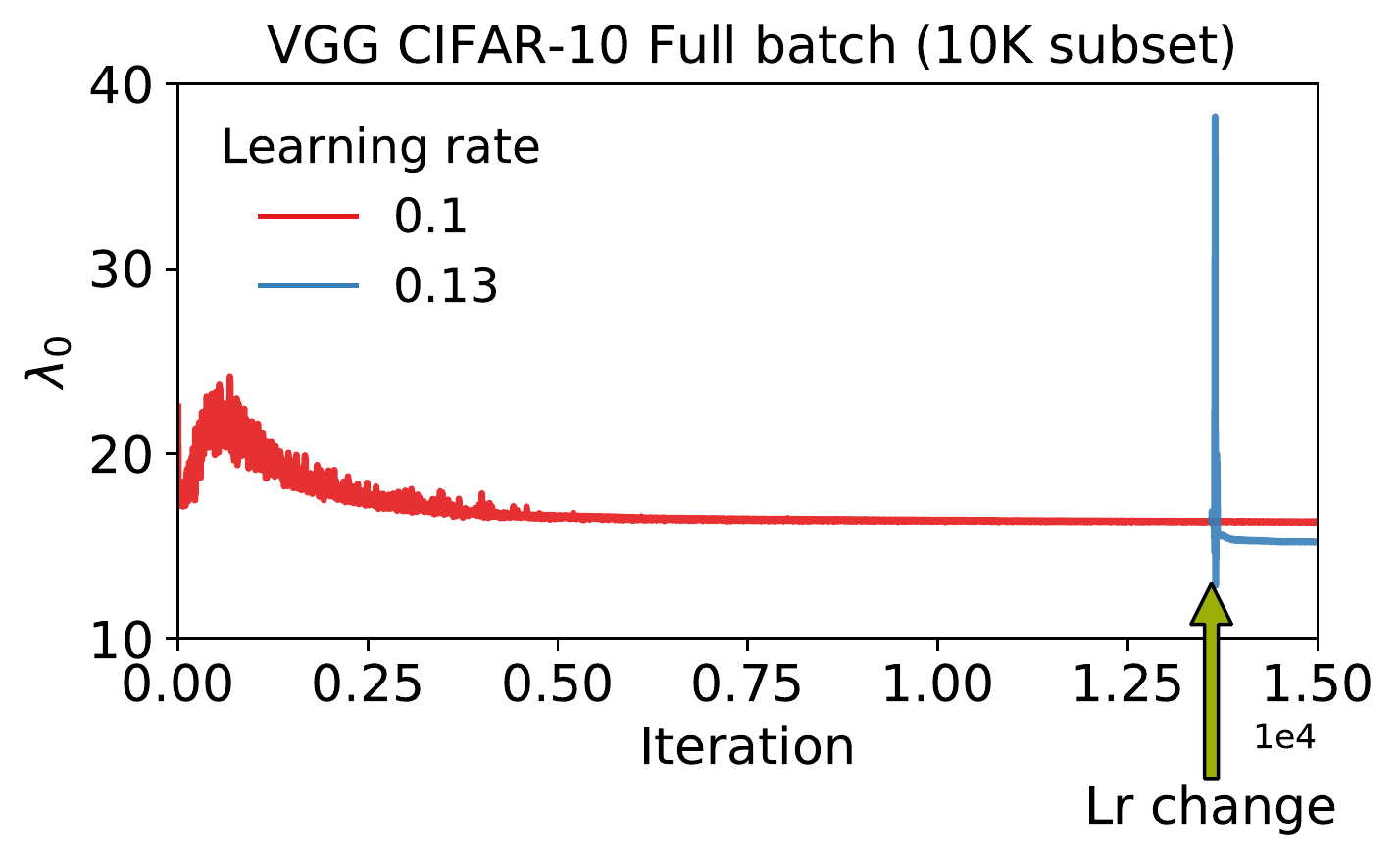}
 \includegraphics[width=0.31\columnwidth]{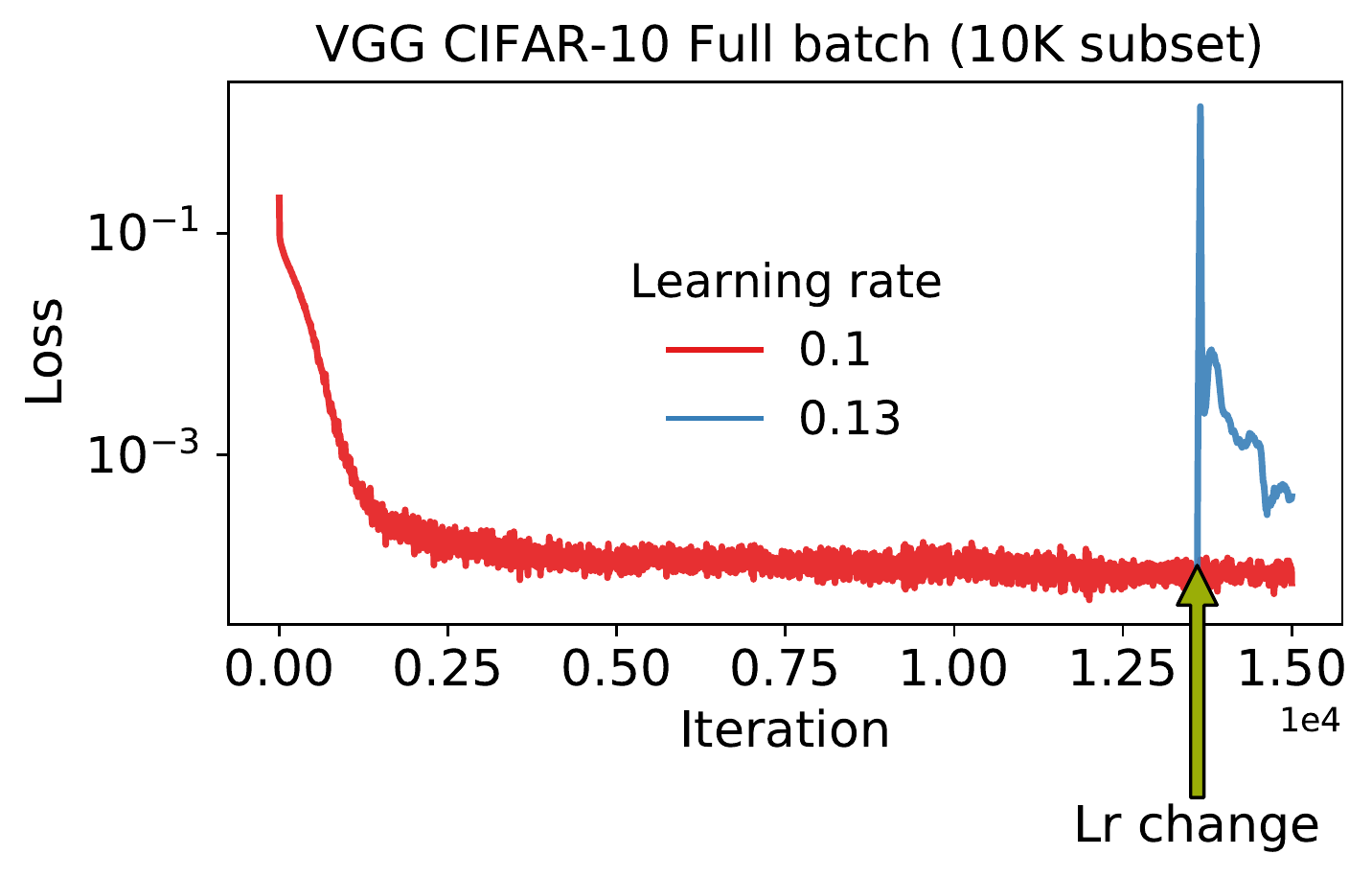}
 \includegraphics[width=0.31\columnwidth]{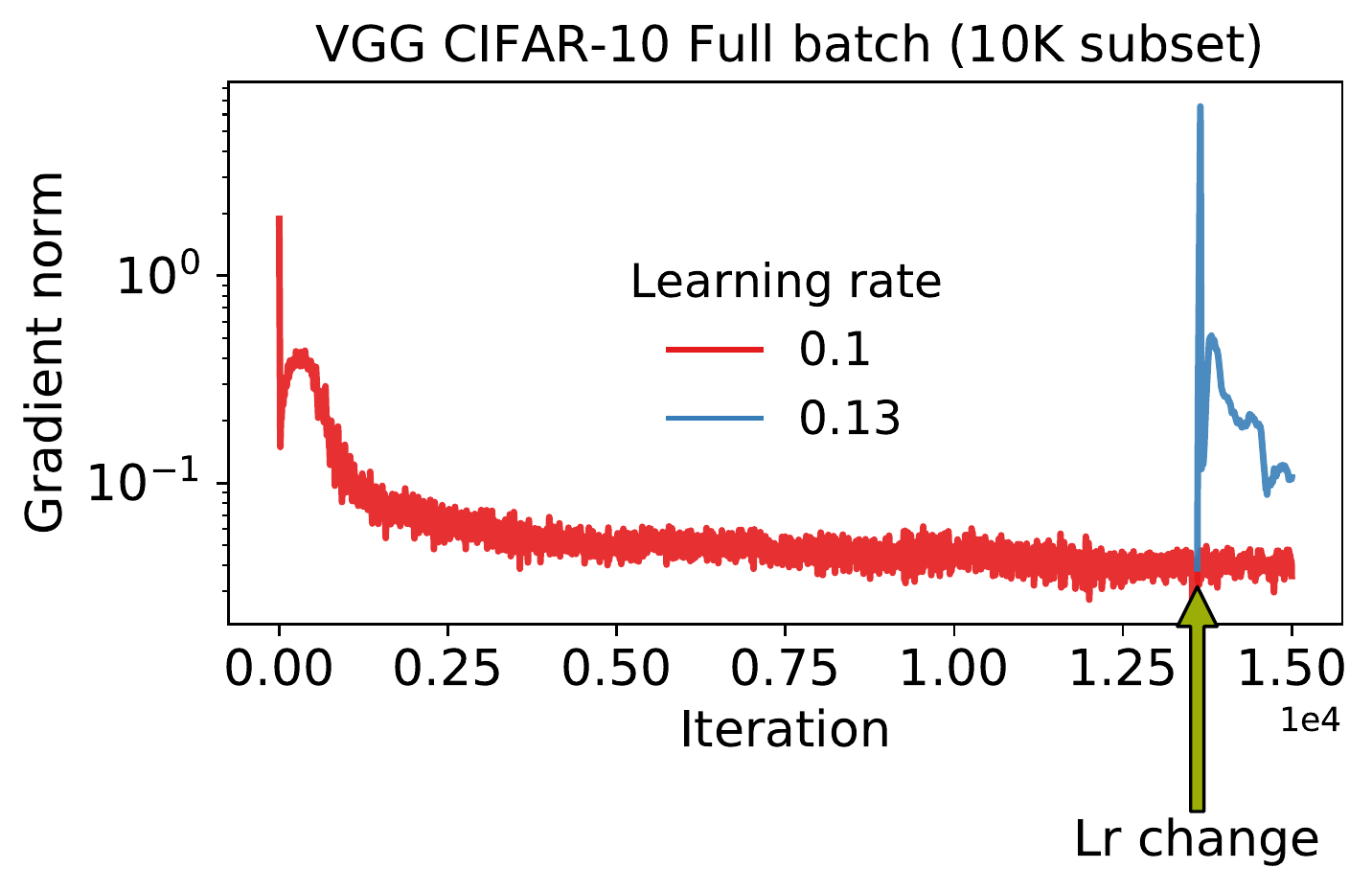}
\caption[We train a model on CIFAR-10 with learning rate $0.1$, and as it reaches converges increase the learning rate slightly so that $\lambda_0 > 2/h$. ]{Local minima are not attractive to gradient descent if $\lambda_0^*> 2/h$. We use a square loss here to avoid training the decrease in eigenvalues later in training associated with cross entropy losses.}
\label{fig:stability_analysis_nn}
\end{figure}

\textbf{Additional edge of stability experiments}.
Results showing the decrease in learning rate in an edge of stability area stabilizes training are shown in Figure \ref{fig:instabilities_change_learning_rate_discrete}. 
We show that one dimension is sufficient to cause instability in training on MNIST in Figure \ref{fig:mnist_continuous_time_swap} (as we have done in Figure \ref{fig:instabilities_change_learning_rate_loss} for CIFAR-10). 
We show the connection between the stability coefficient $sc_0$ and areas where the loss increases in Figure \ref{fig:cifar_peak_commons}: this shows that when the loss increases, it is in an area where $sc_0$ is also increasing. We note that to get a full picture of the behavior of the PF we would need to use continuous time computation, but since that is computationally prohibitive we use the incomplete but easily available discrete data. Figure~\ref{fig:edge_of_stability_results_lambda} zooms in the behavior of $\lambda_0$ as dependent of the behavior of the corresponding flows on MNIST and CIFAR-10 and Figure~\ref{fig:edge_of_stability_results_cifar_resnet} shows additional Resnet-18 results.

\begin{figure}
\centering
 \includegraphics[width=0.31\columnwidth]{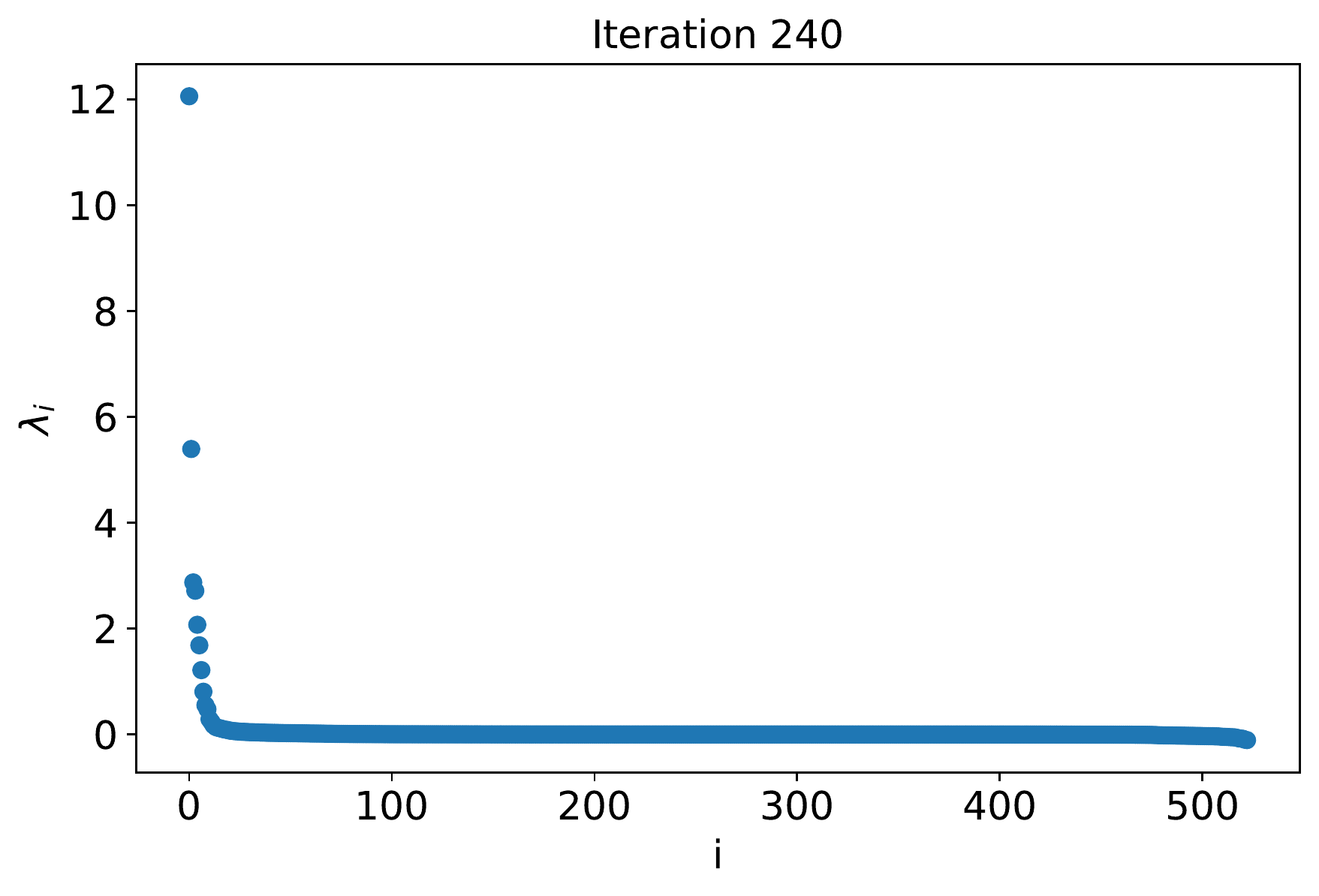}
 \includegraphics[width=0.31\columnwidth]{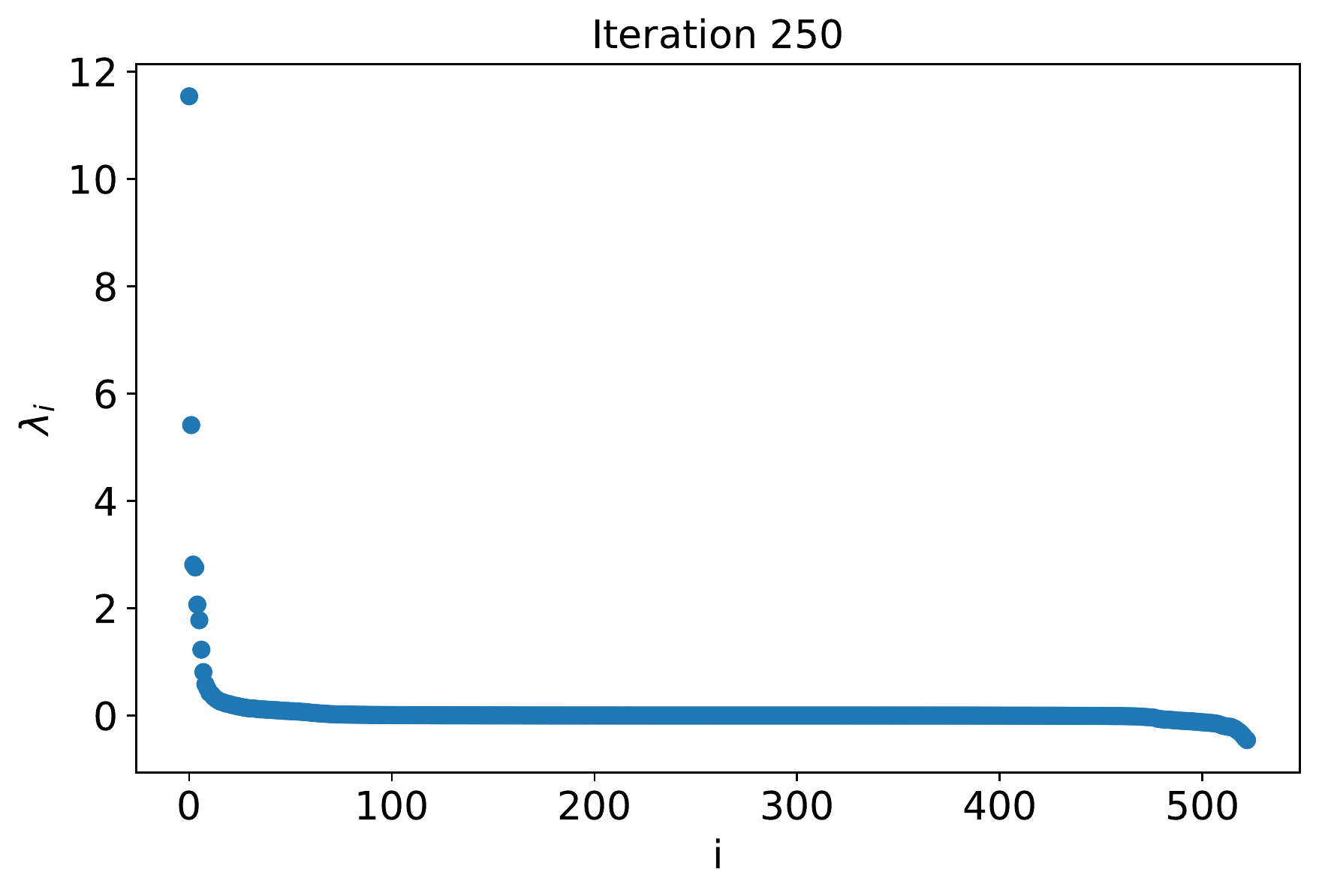}
 \includegraphics[width=0.31\columnwidth]{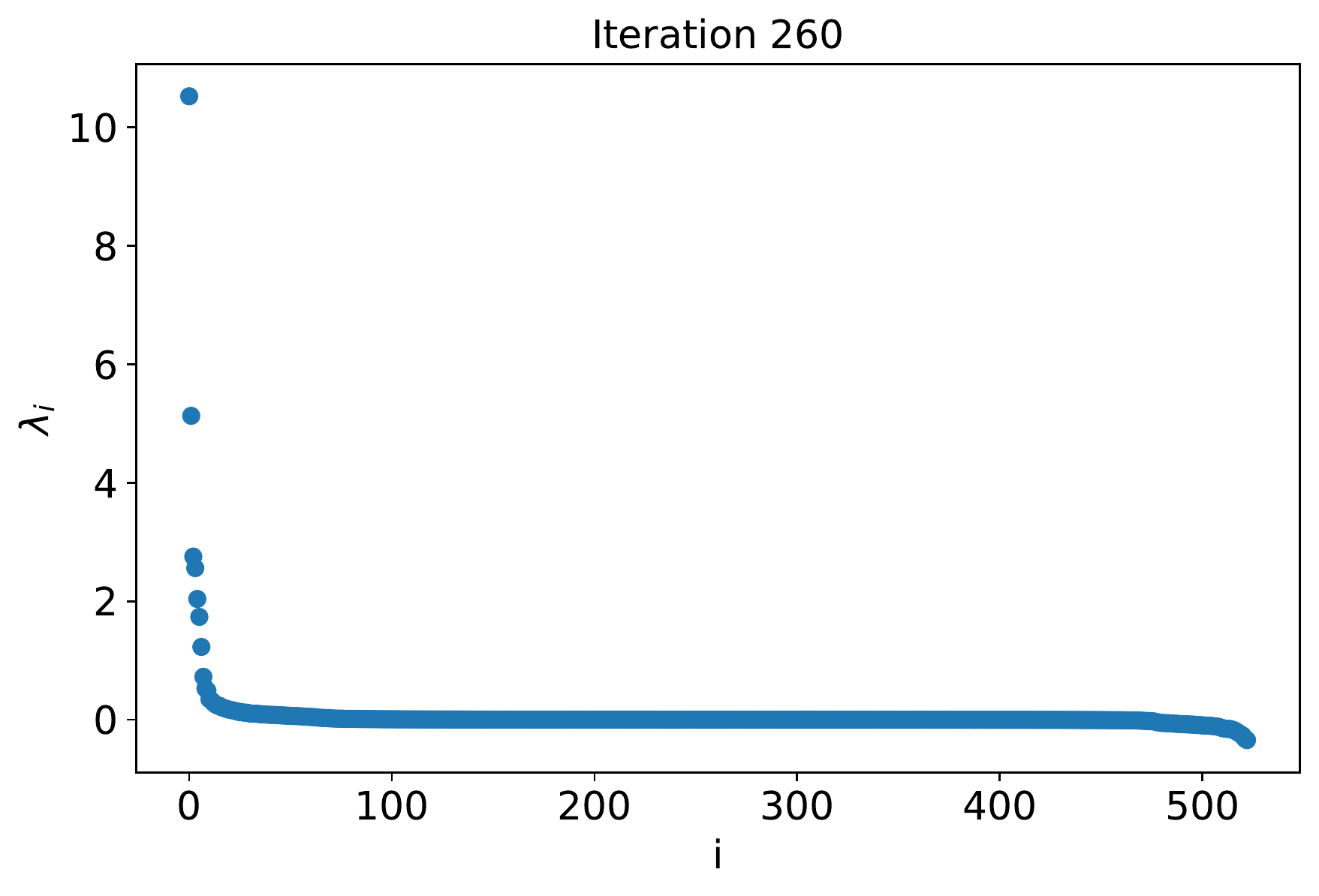}
\caption[The eigenspectrum obtained along the gradient descent path of a 5 layer small MLP trained on the UCI Iris dataset. The Hessian eigenvalues are largely positive. The result accompanies Figure~\ref{fig:model_of_gd_at_edge_of_stability} in the main paper.]{The eigenspectrum obtained along the gradient descent path of a 5 layer small MLP trained on the UCI Iris dataset. The Hessian eigenvalues are largely positive. The result accompanies Figure~\ref{fig:model_of_gd_at_edge_of_stability} in the main paper.}
\label{fig:eigspectrum_small_mlp}
\end{figure}

\begin{figure}[tbh]
 \includegraphics[width=0.45\columnwidth]{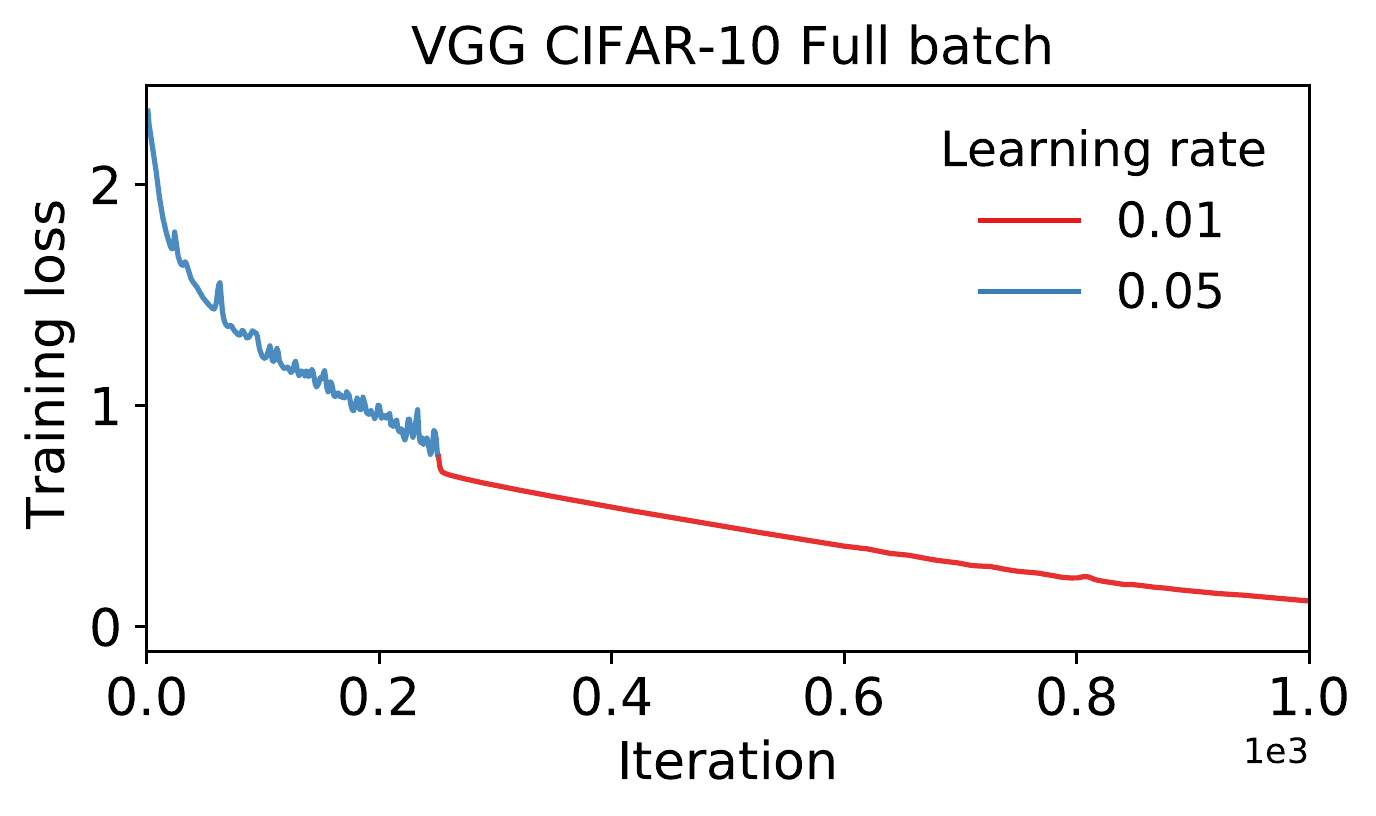}
 \includegraphics[width=0.45\columnwidth]{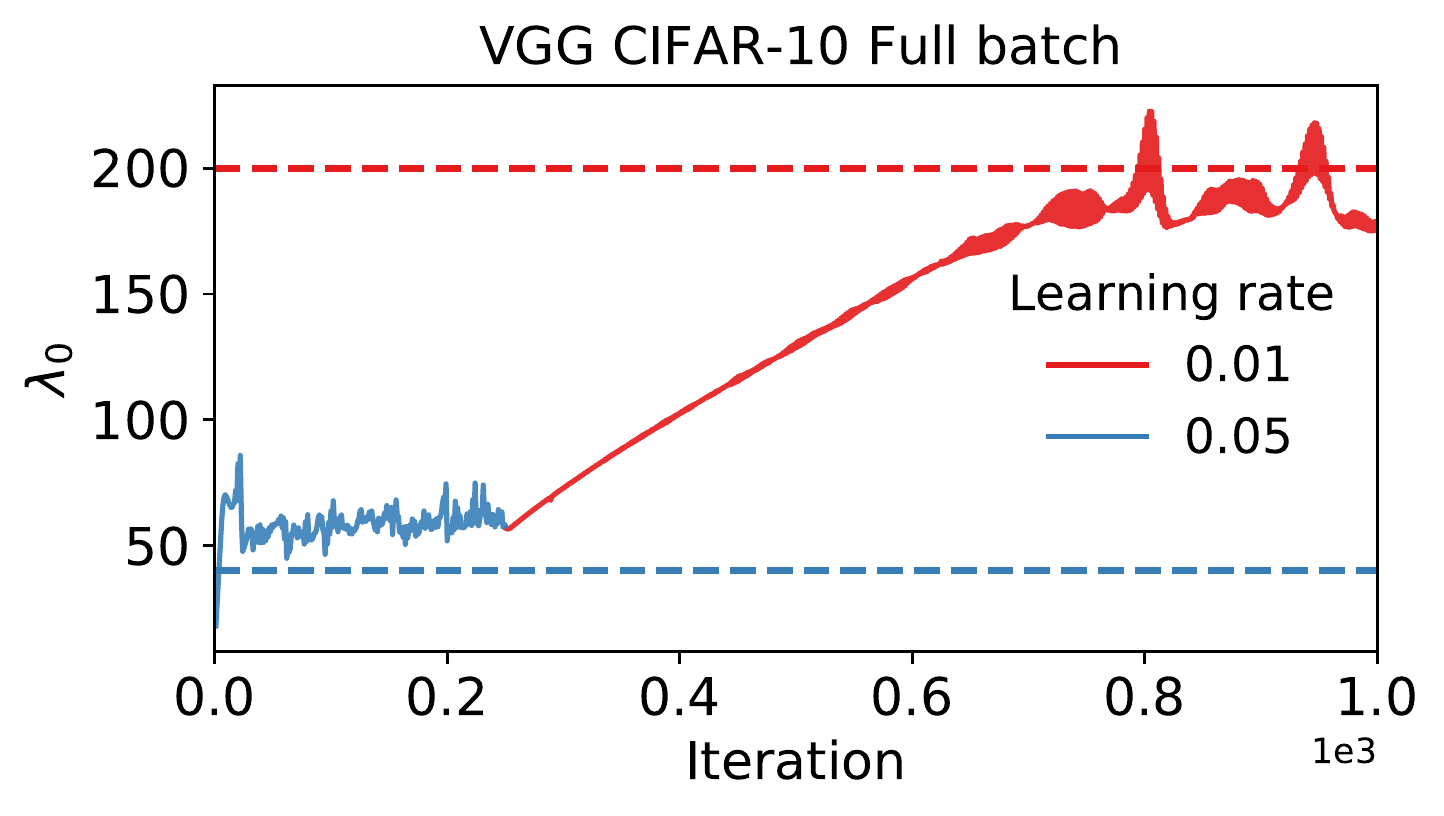}
\caption[Training a VGG model on CIFAR-10 and training with learning rate $0.05$ after which switching to $0.01$.]{The effect of learning rate changes: as discovered by~\citet{cohen2021gradient}, reducing the learning rate in training leads to increased stability while the largest eigenvalue increases. The lack of instability with a smaller learning rate is explained by the principal flow, and the increase in the largest eigenvalue can be thought of through the lens of the principal flow arguments.}
\label{fig:instabilities_change_learning_rate_discrete}
\end{figure}

\begin{figure}[tbh]
  \includegraphics[width=0.45\columnwidth]{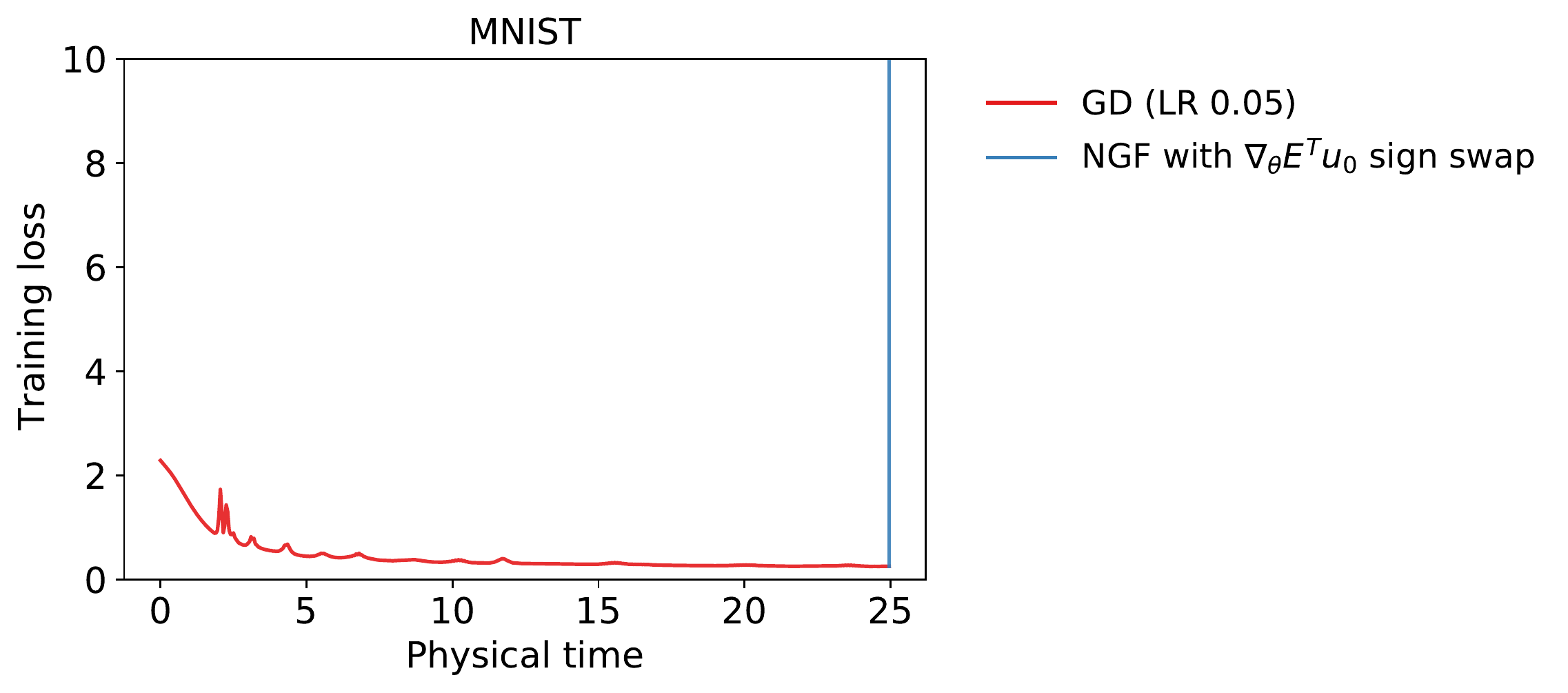}
  \includegraphics[width=0.45\columnwidth]{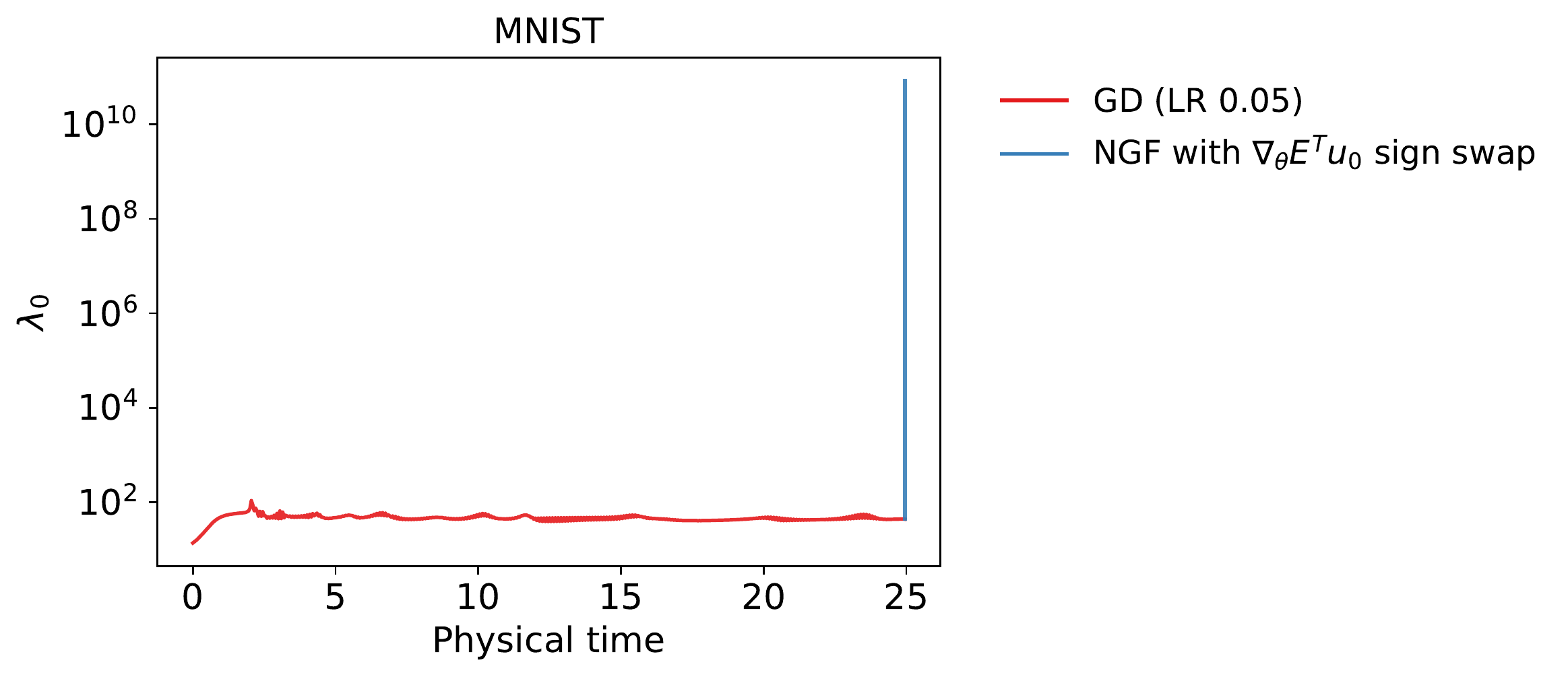}
\caption[MNIST results when changing one direction in continuous time. The model is a 4 layer MLP with 100 hidden units per layer. The model was trained with learning rate $0.05$ until the edge of stability area is reached after which the flow $\dot{\vtheta} = \nabla_{\vtheta}E^T \vu_0 \vu_0 + \sum_{i=1}^{D-1} -\nabla_{\vtheta}E^T \vu_i \vu_i$ is used.]{MNIST results when changing one direction in continuous time: one eigendirection is sufficient to cause instability. The model was trained with learning rate $0.05$ until the edge of stability area is reached after which the flow $\dot{\vtheta} = \nabla_{\vtheta}E^T \vu_0 \vu_0 + \sum_{i=1}^{D-1} -\nabla_{\vtheta}E^T \vu_i \vu_i$ is used.}
\label{fig:mnist_continuous_time_swap}
\end{figure}

\begin{figure}[tbh]
  \includegraphics[width=0.45\columnwidth]{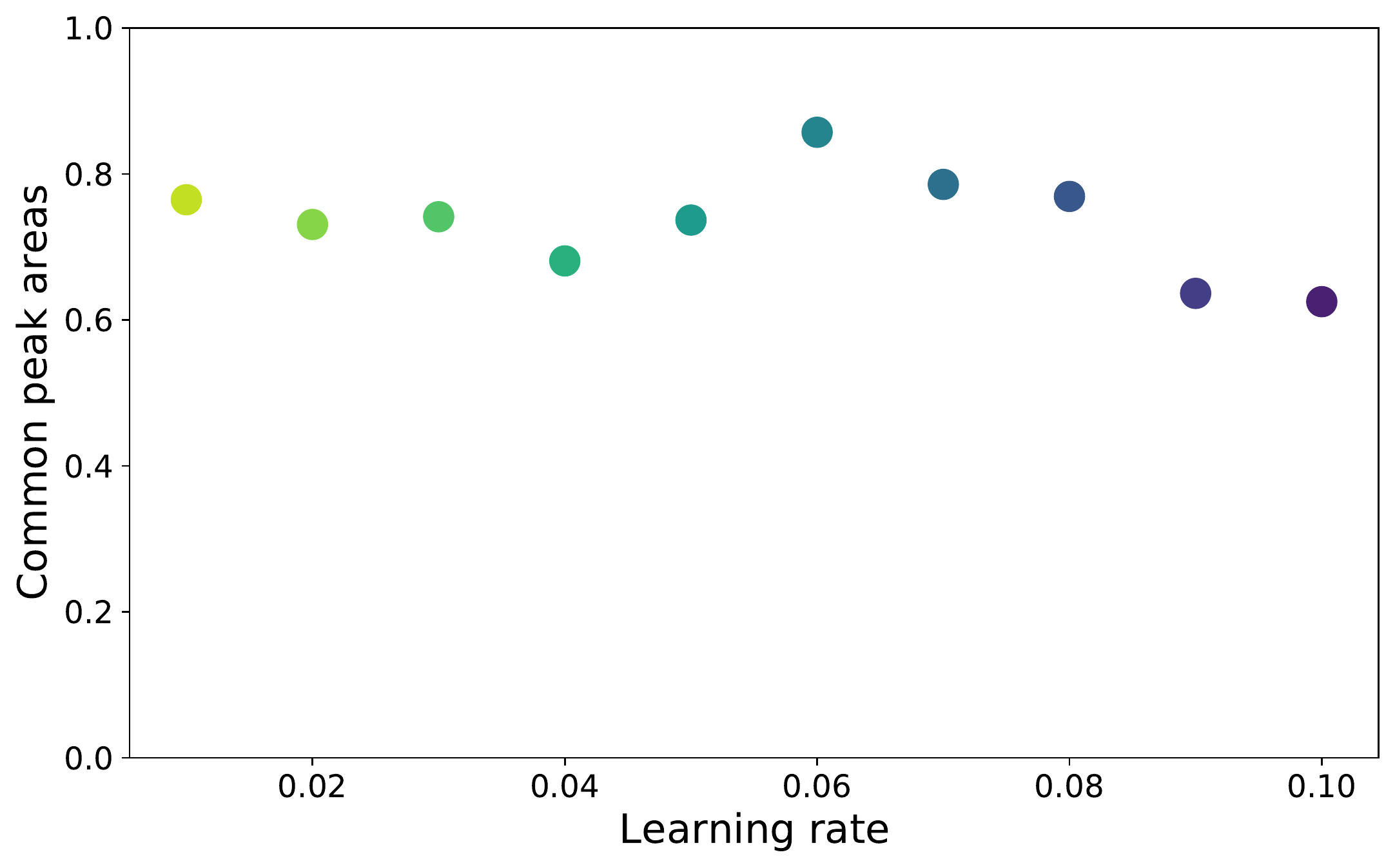}
  \includegraphics[width=0.45\columnwidth]{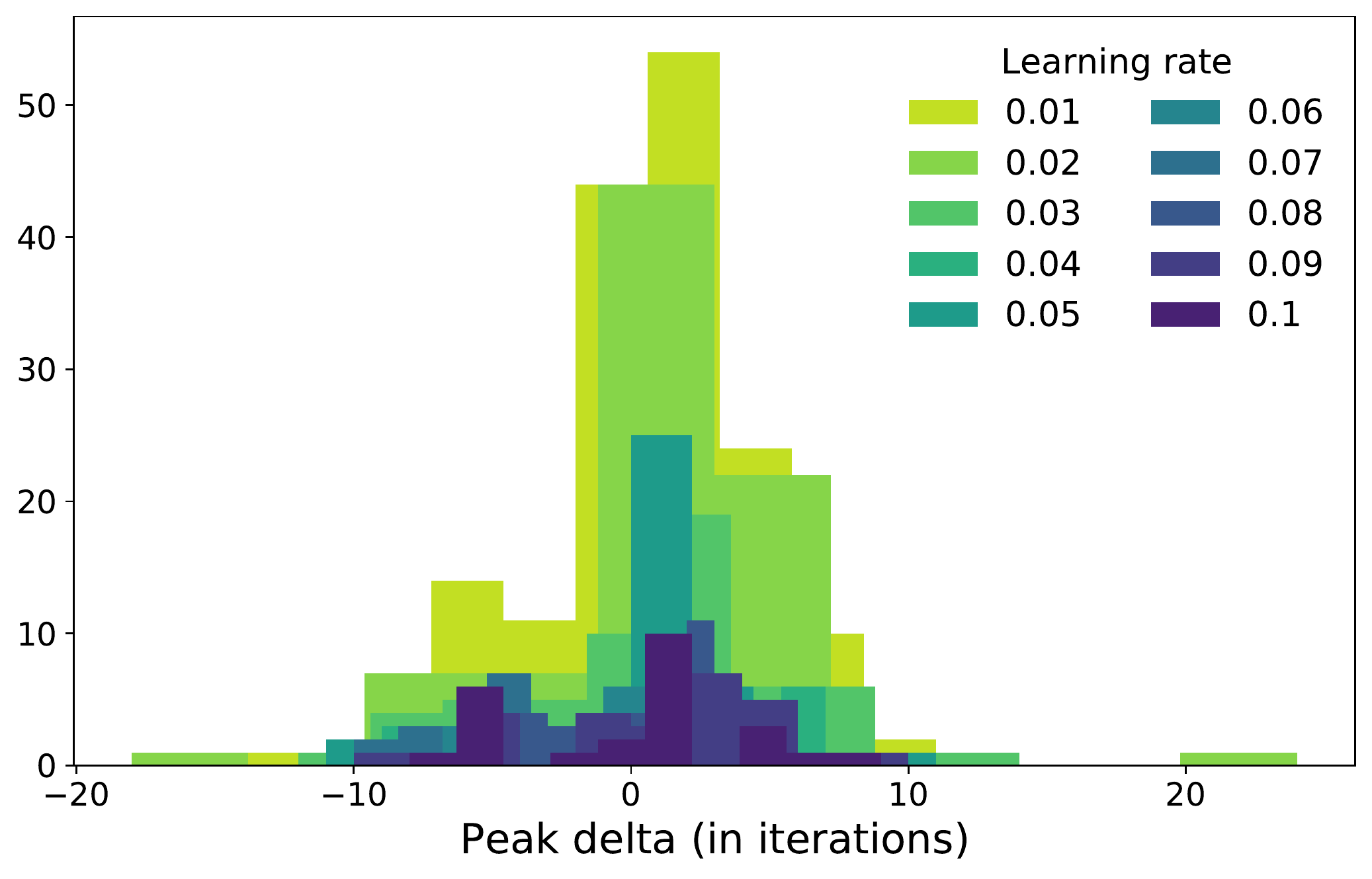}
\caption[Peak areas of loss increase and the stability coefficient for the largest eigendirection of the PF. Results using a VGG model trained with full batch gradient descent on CIFAR-10 across a learning rate sweep.]{Peak areas of loss increase and the coefficient for the largest eigendirection of the principal flow. Here we normalise the coefficient by the gradient norm to remove the gradient norm as a potential confounder. Results on CIFAR-10.}
\label{fig:cifar_peak_commons}
\end{figure}

\begin{figure}[tbh]
\begin{subfigure}[MNIST MLP]{
 \includegraphics[width=0.48\columnwidth]{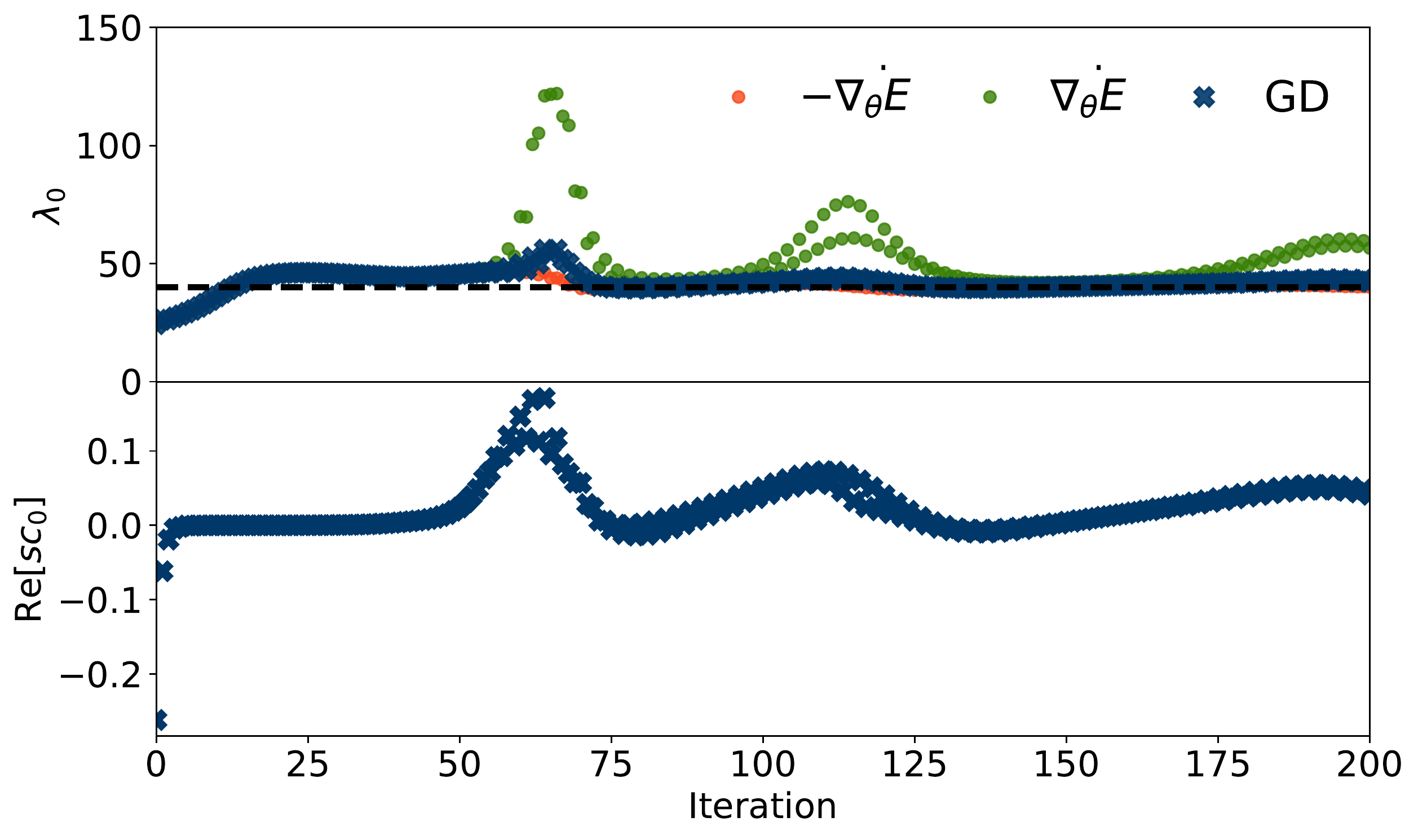}
 }
 \end{subfigure}
\begin{subfigure}[CIFAR-10 VGG]{
 \includegraphics[width=0.48\columnwidth]{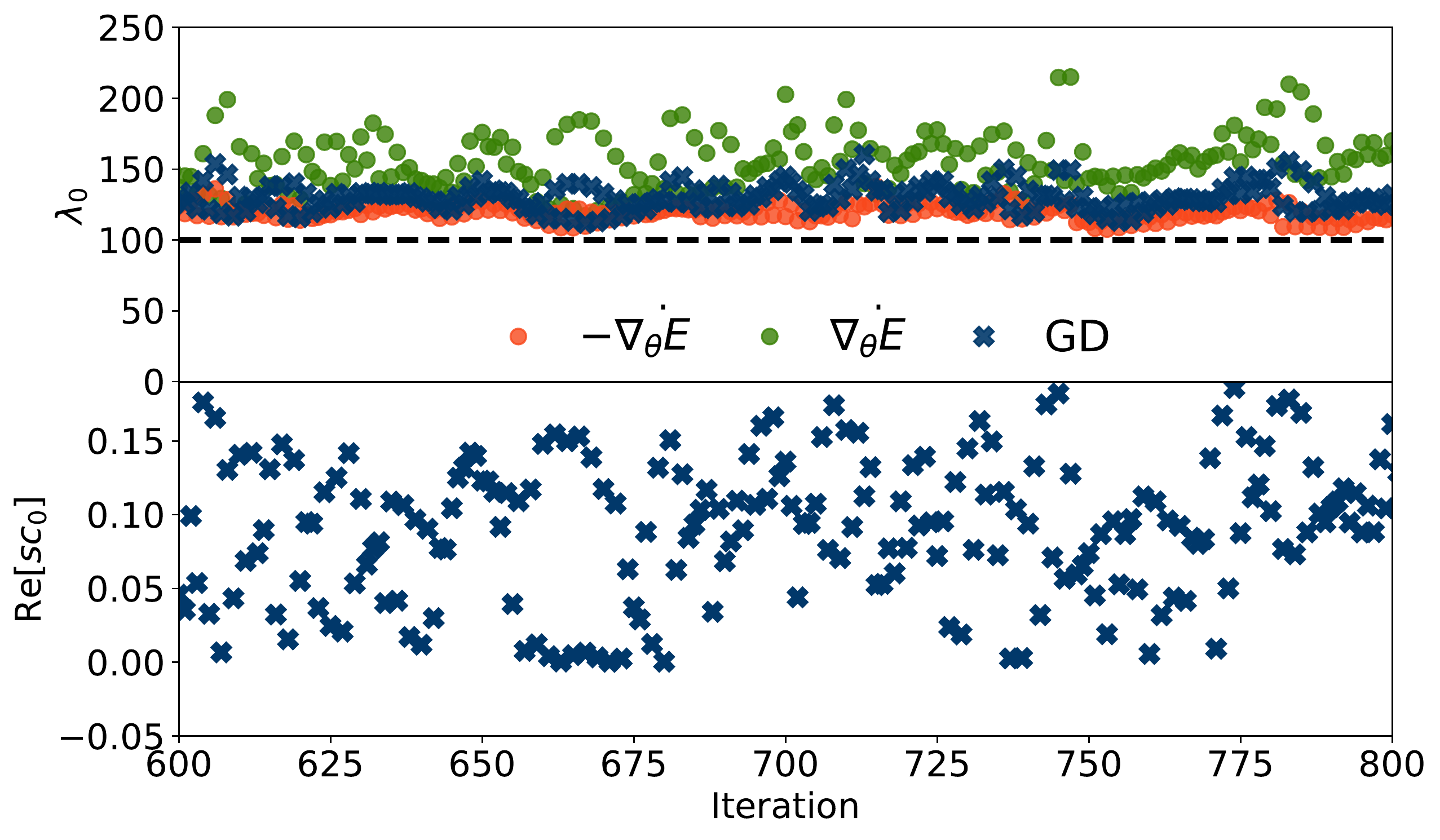}
 }
 \end{subfigure}
\caption[Understanding changes to $\lambda_0$ using the principal flow. MLP with 4 layers with 100 units each trained on MNIST, and VGG model trained on CIFAR-10. Learning rate $0.05$ and $0.01$ respectively.]{Understanding changes to $\lambda_0$ using the principal flow: increases in $\lambda_0$ above $2/h$ correspond to increases in the flow $\dot{\vtheta} = \nabla_{\vtheta} E$. The stability coefficient reflects the changes in $\lambda_0$.}
\label{fig:edge_of_stability_results_lambda}
\end{figure}

\begin{figure}[tbh]
  \includegraphics[width=0.9\columnwidth]{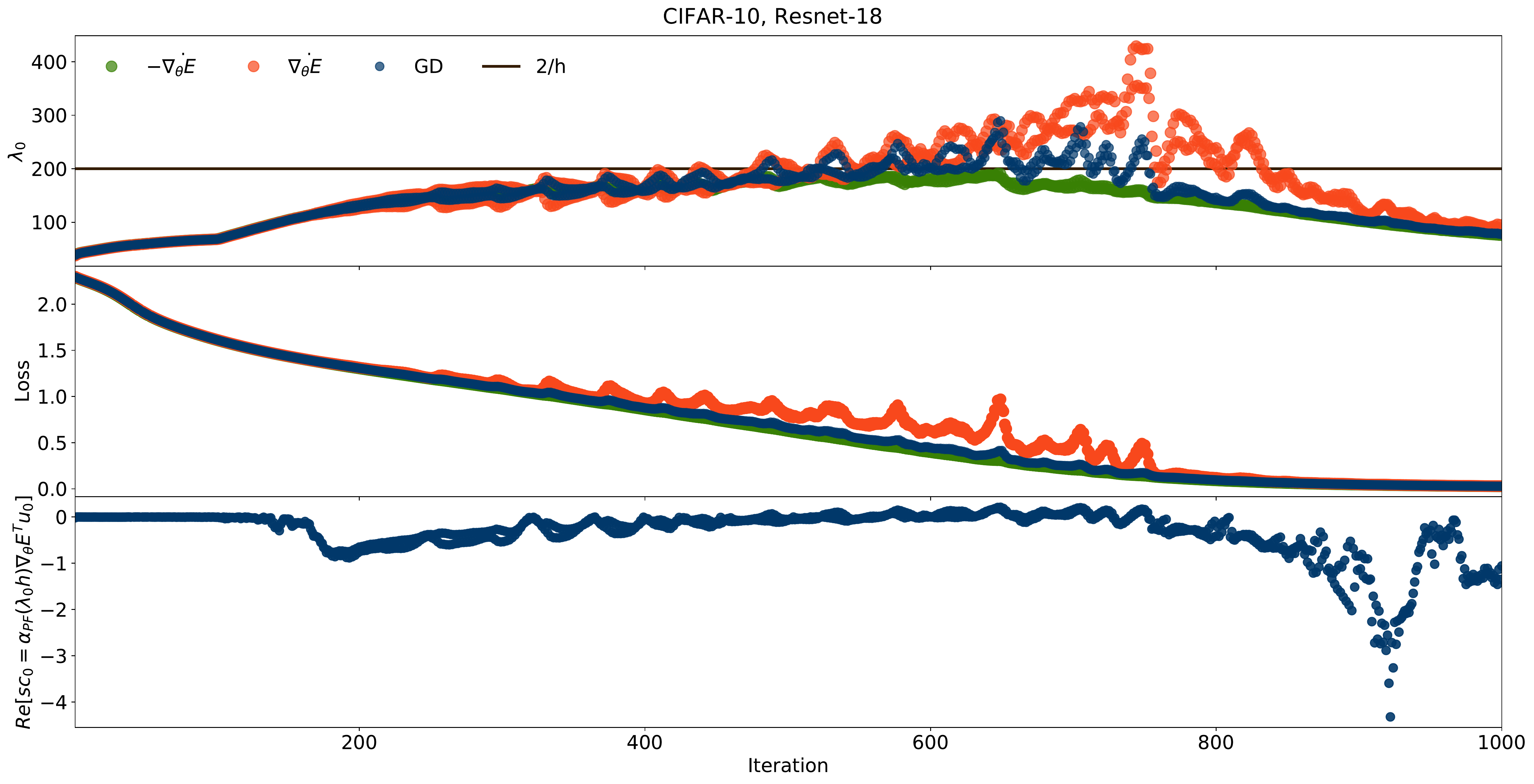}
\caption[The connection between the principal flow, $\lambda_0$ and loss instabilities with Resnet-18 trained on CIFAR-10 with a learning rate $0.01$.]{Learning rate: $0.01$. Edge of stability results: the connection between the principal flow, stability coefficients and loss instabilities with Resnet-18.}
\label{fig:edge_of_stability_results_cifar_resnet} 
\end{figure}

\textbf{DAL}.
We use gradient descent training with a fixed learning rate to measure the quantities we would like to use as a learning rate in DAL, to see if they have a reasonable range and show results in Figure~\ref{fig:connections_quantities}. We show sweeps across batch sizes Figures \ref{fig:power_sweeps_batch_sizes_vgg}, \ref{fig:imagenet_lr_scaling_across_batch_sizes}, \ref{fig:imagenet_lr_sqrt_scaling_across_batch_sizes}. We show DAL-$p$ sweeps with effective learning rates, training losses and test accuracies in FIgure~\ref{fig:power_sweeps_all}. We show results with DAL on a mean square loss in Figure \ref{fig:least_square_loss_dal}.

\begin{figure}[tbh]
  \includegraphics[width=0.32\columnwidth]{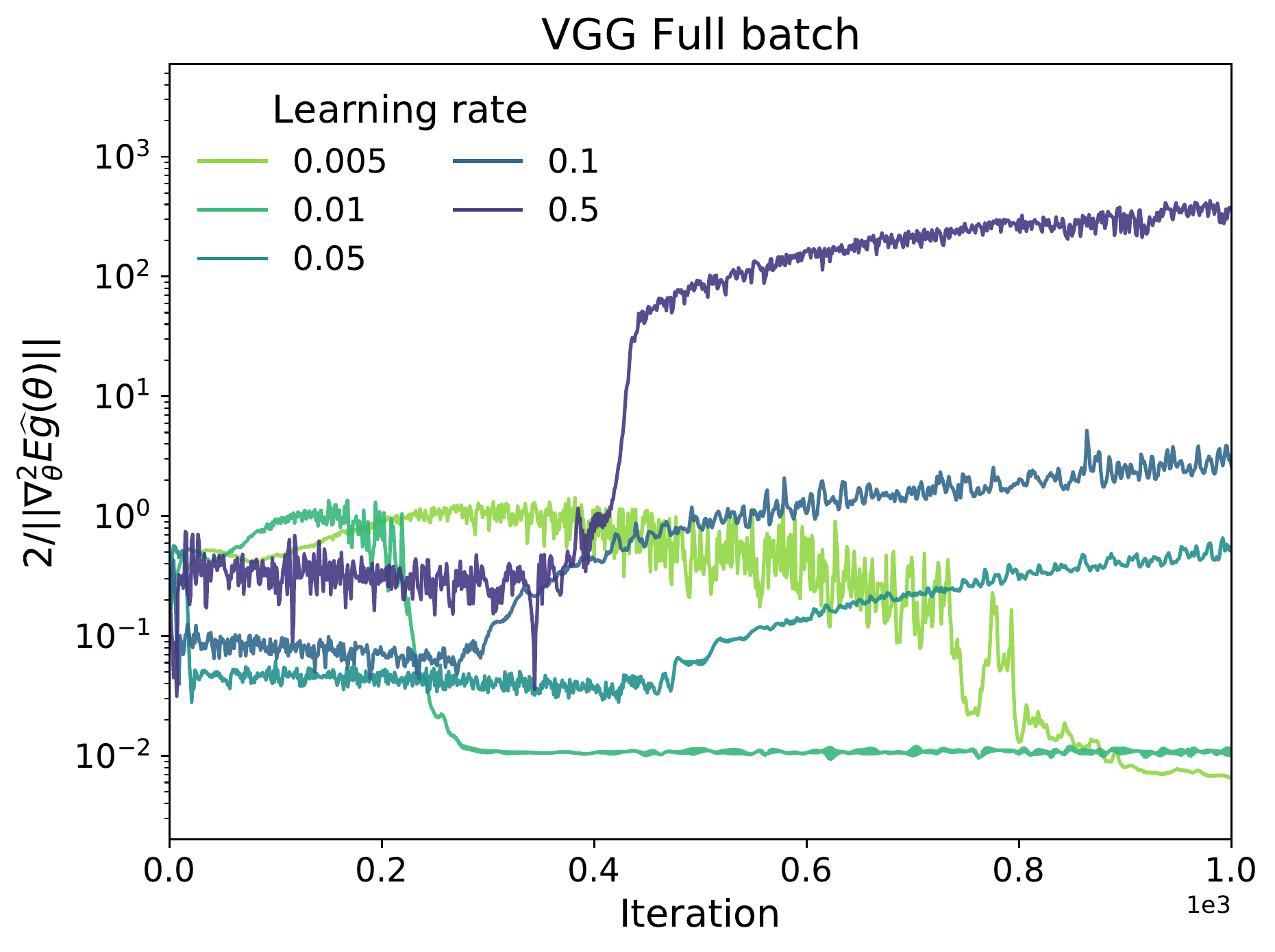}
 \includegraphics[width=0.32\columnwidth]{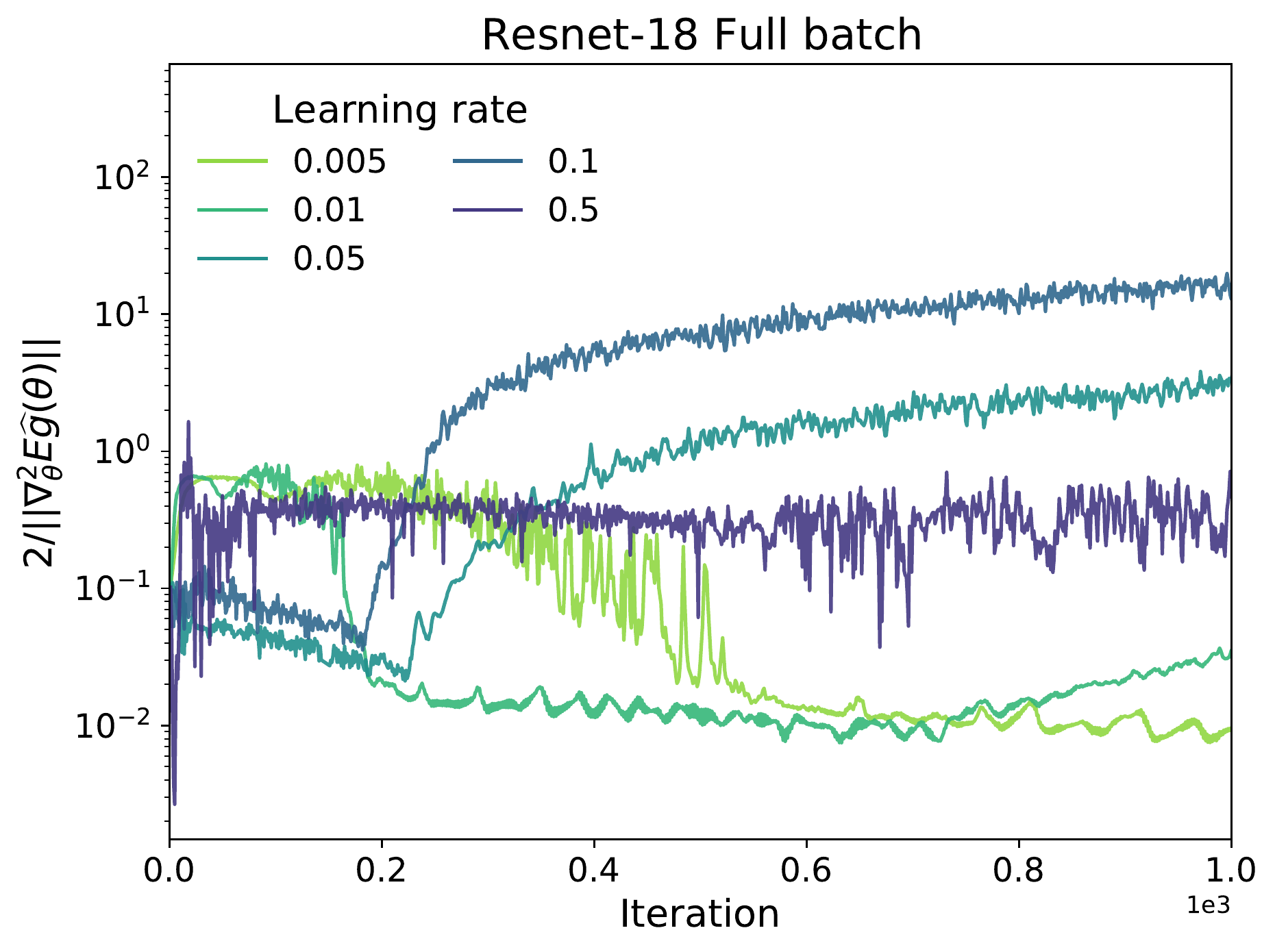}
  \includegraphics[width=0.32\columnwidth]{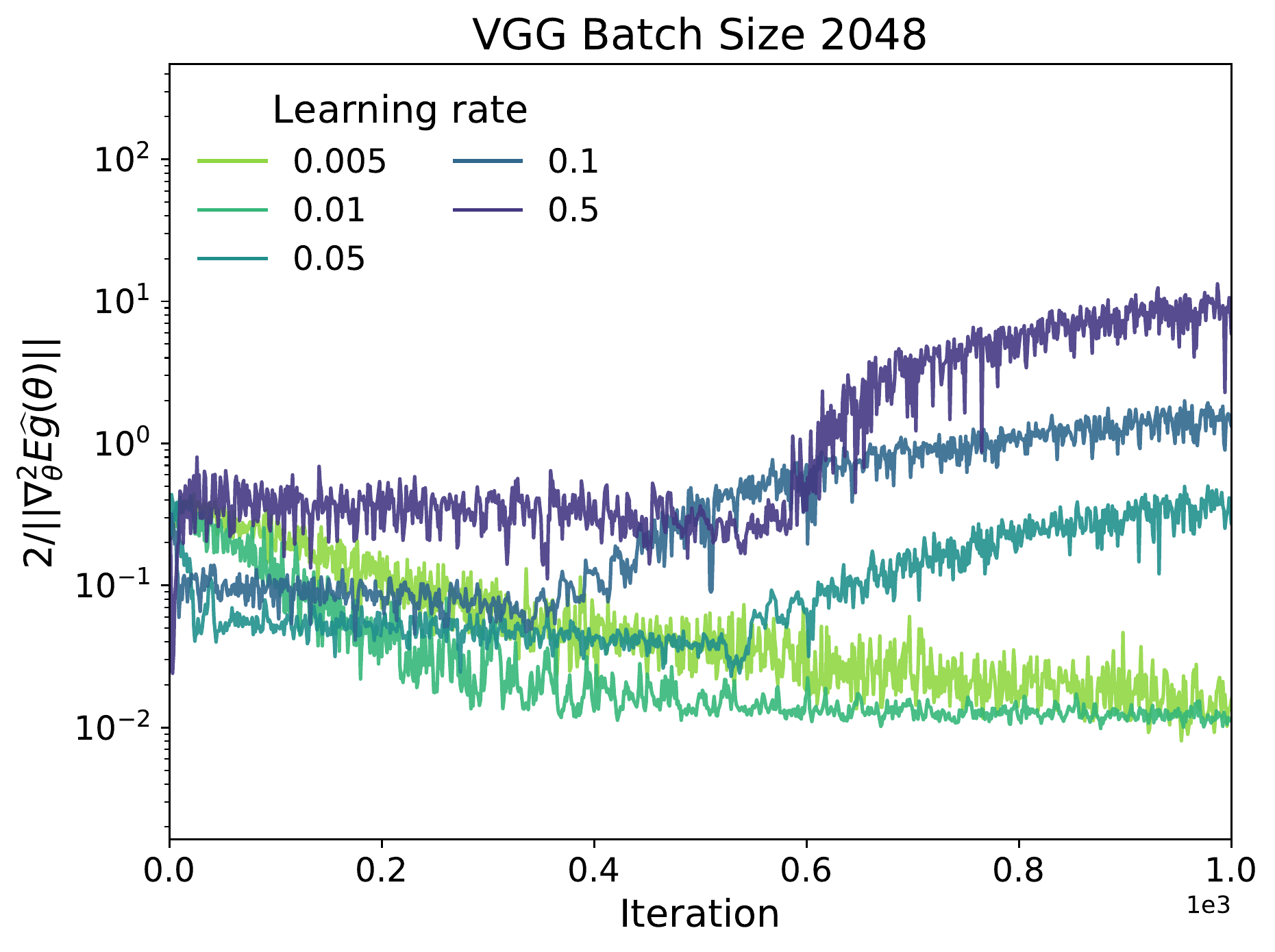}
 \caption[Evaluating the value of the DAL learning rate for a vanilla (fixed) learning rate sweep to assess whether it would have reasonable magnitudes. Models trained on CIFAR-10 using a learning rate sweep.]{Evaluating the value of the DAL learning rate for a vanilla (fixed) learning rate sweep to assess whether it would have reasonable magnitudes.}
\label{fig:connections_quantities}
\end{figure}

\begin{figure}
  \includegraphics[width=0.33\columnwidth]{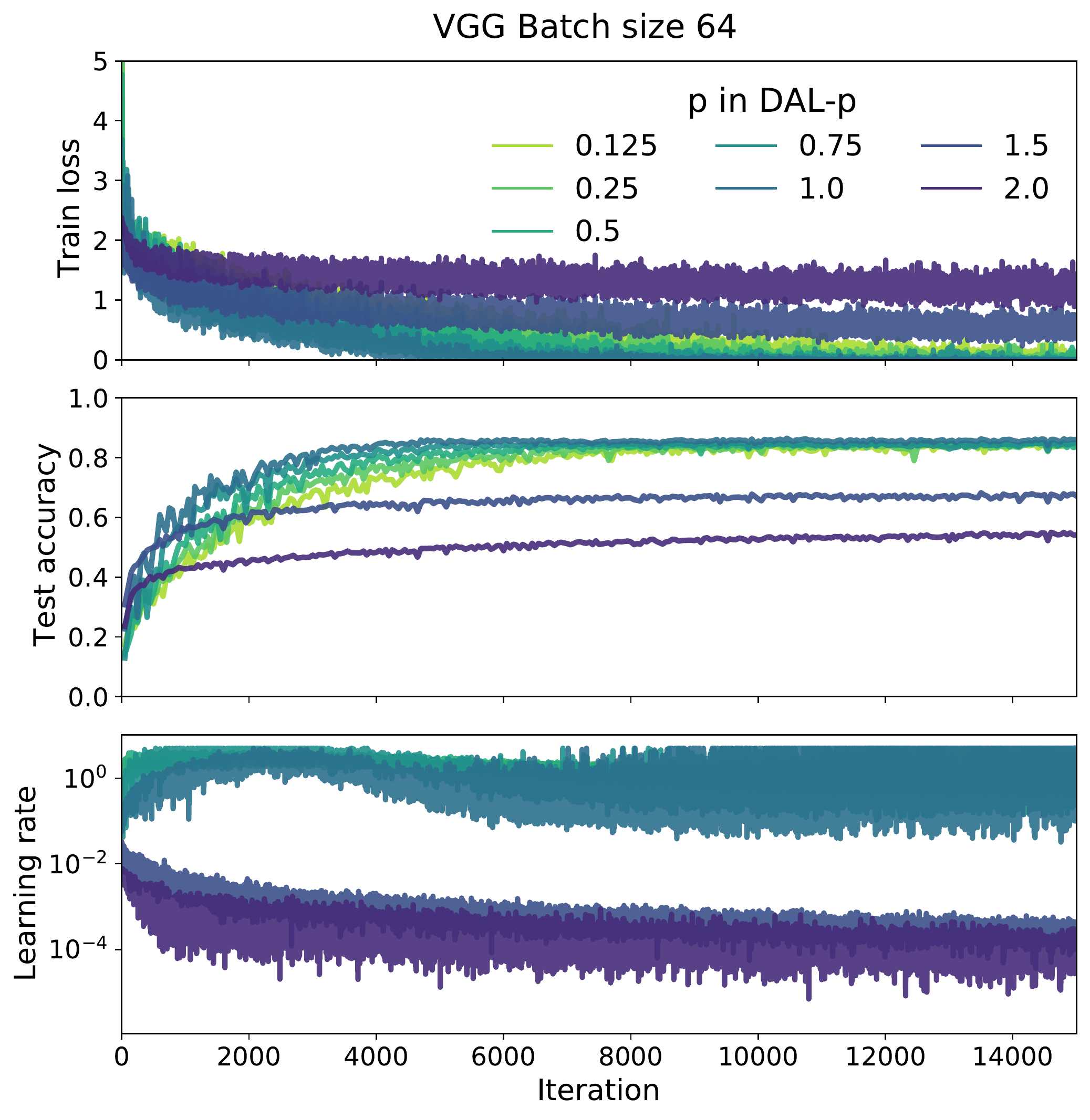}
 \includegraphics[width=0.33\columnwidth]{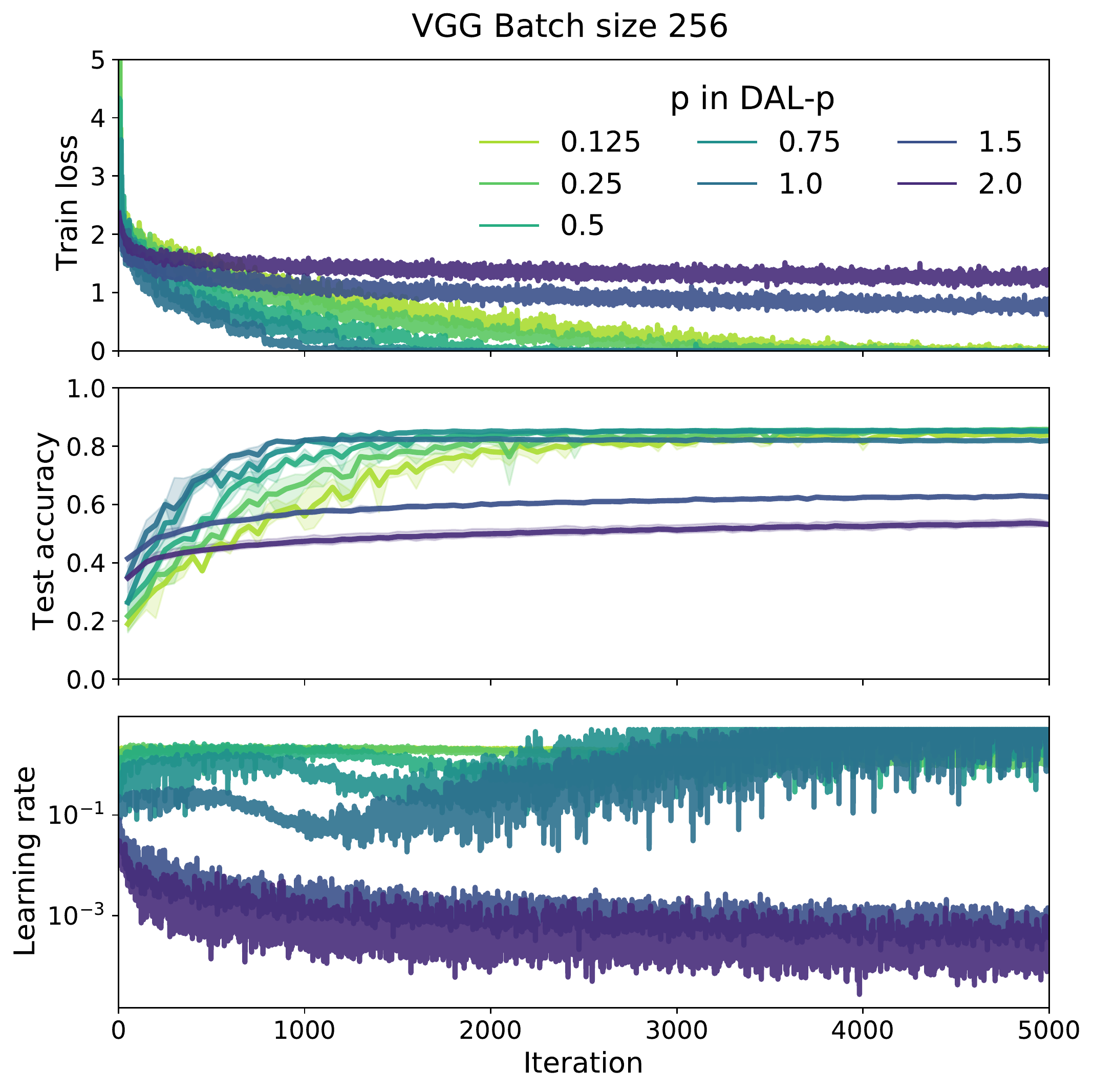}
 \includegraphics[width=0.33\columnwidth]{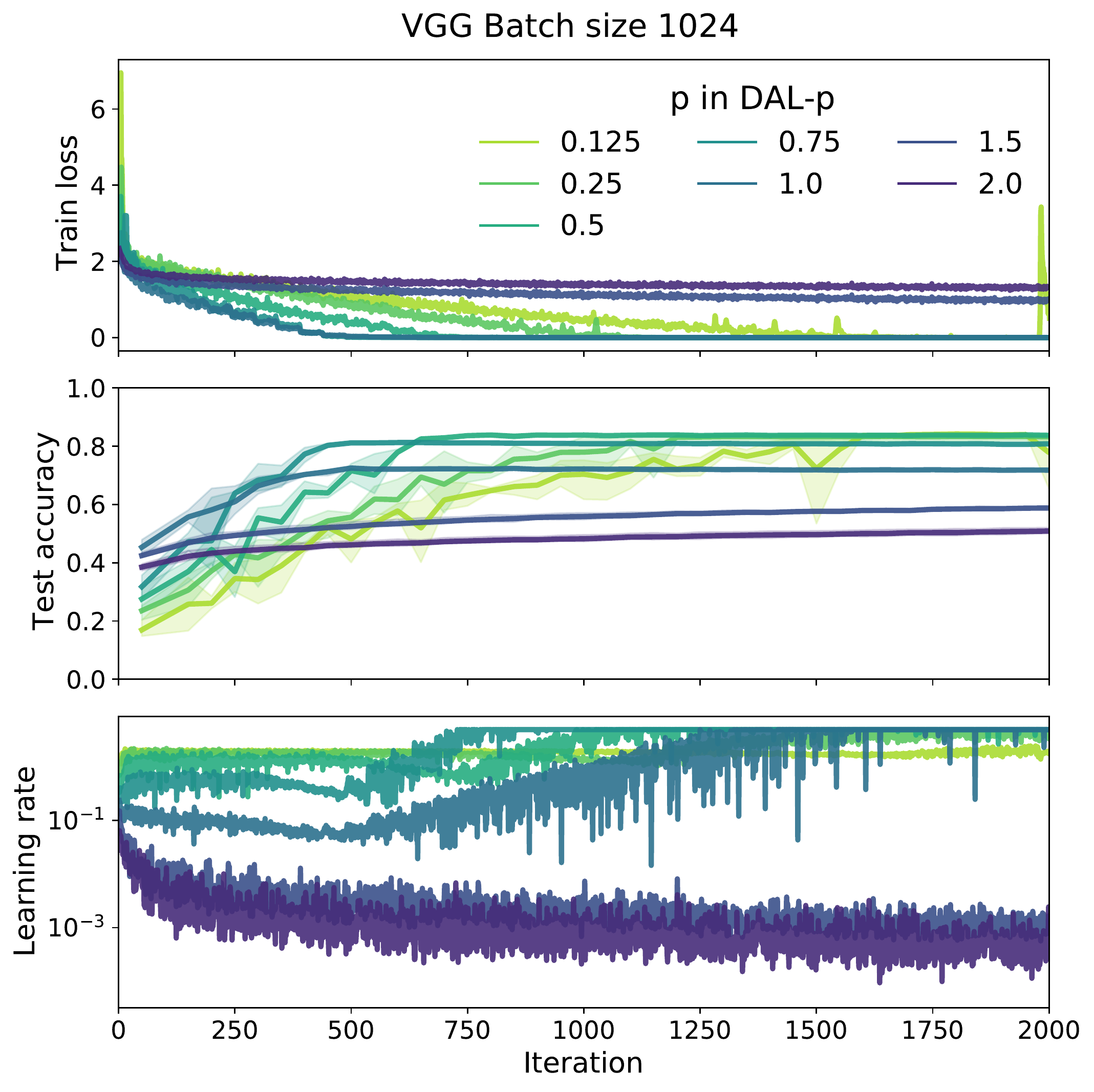}
\caption[DAL-$p$ sweep with a VGG model trained on CIFAR-10 for small and large batch sizes.]{DAL-$p$ training a VGG model on CIFAR-10. Sweep across batch sizes: discretization drift helps test performance, but at the cost of stability. We also show the effective learning rate and train losses and test accuracies.}
\label{fig:power_sweeps_batch_sizes_vgg}
\end{figure}

\begin{figure}[tbh]
 \includegraphics[width=0.24\columnwidth]{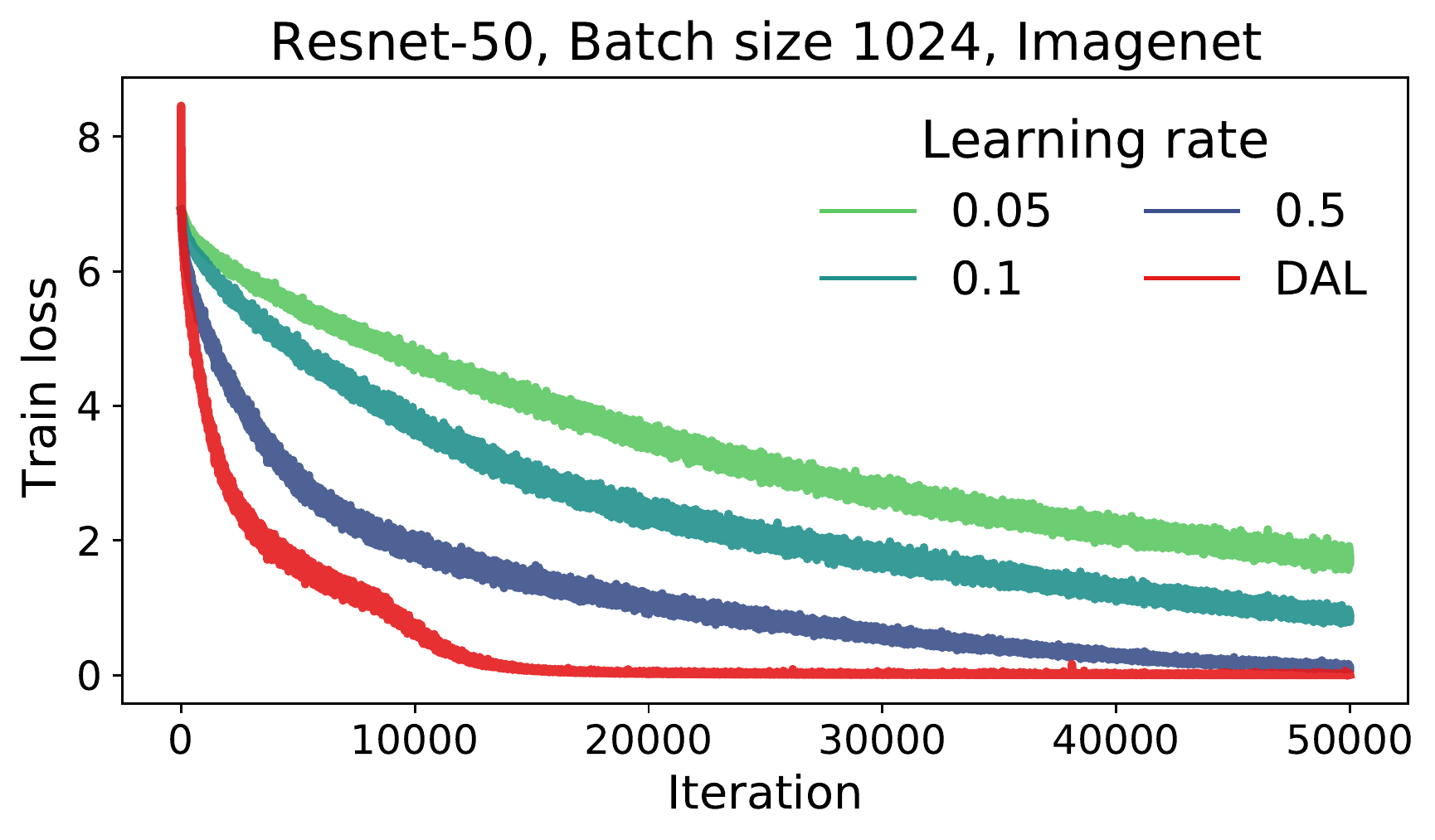}
  \includegraphics[width=0.24\columnwidth]{imagenet_train_loss_h_g_scaling_2048}
  \includegraphics[width=0.24\columnwidth]{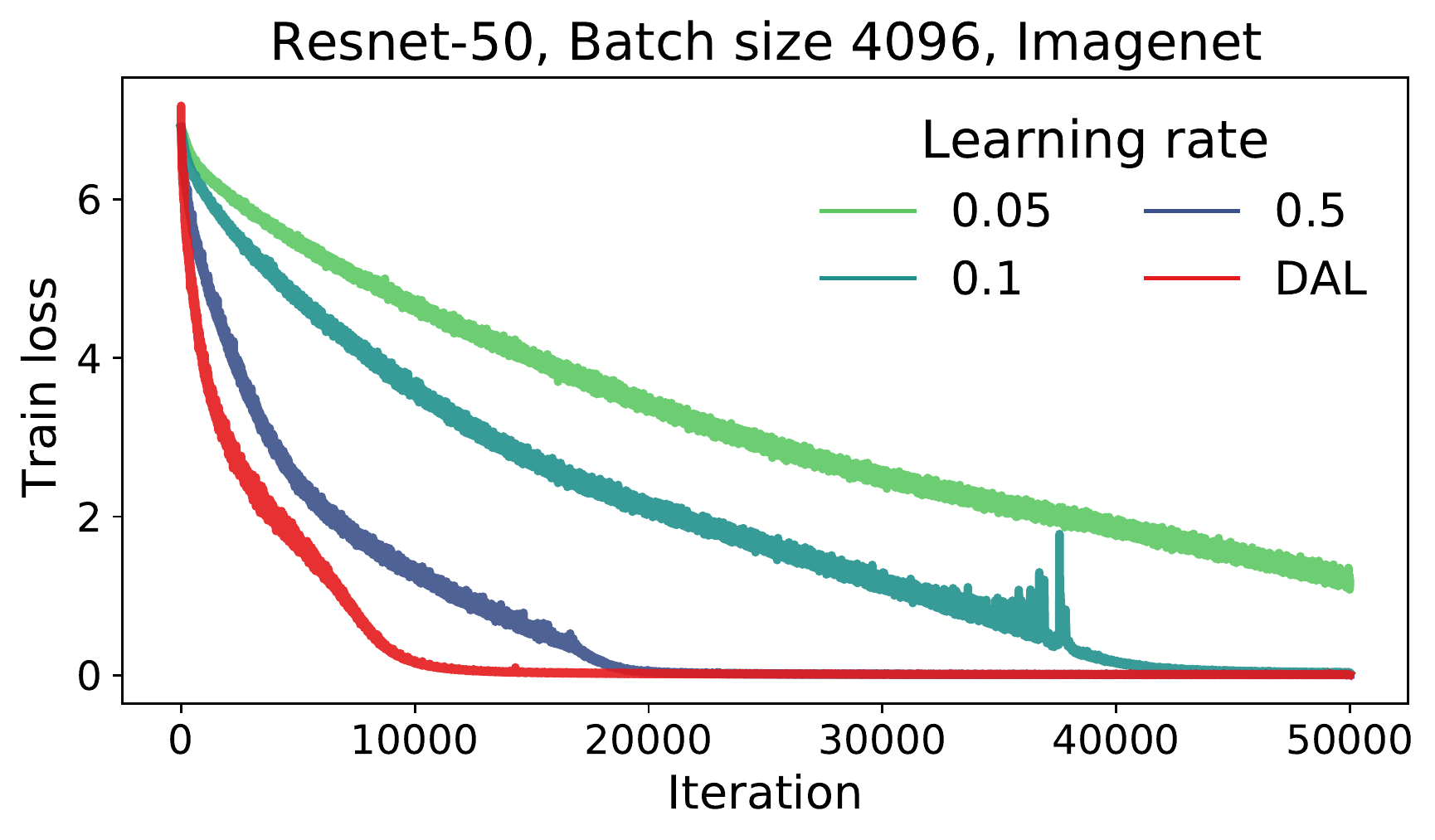}
  \includegraphics[width=0.24\columnwidth]{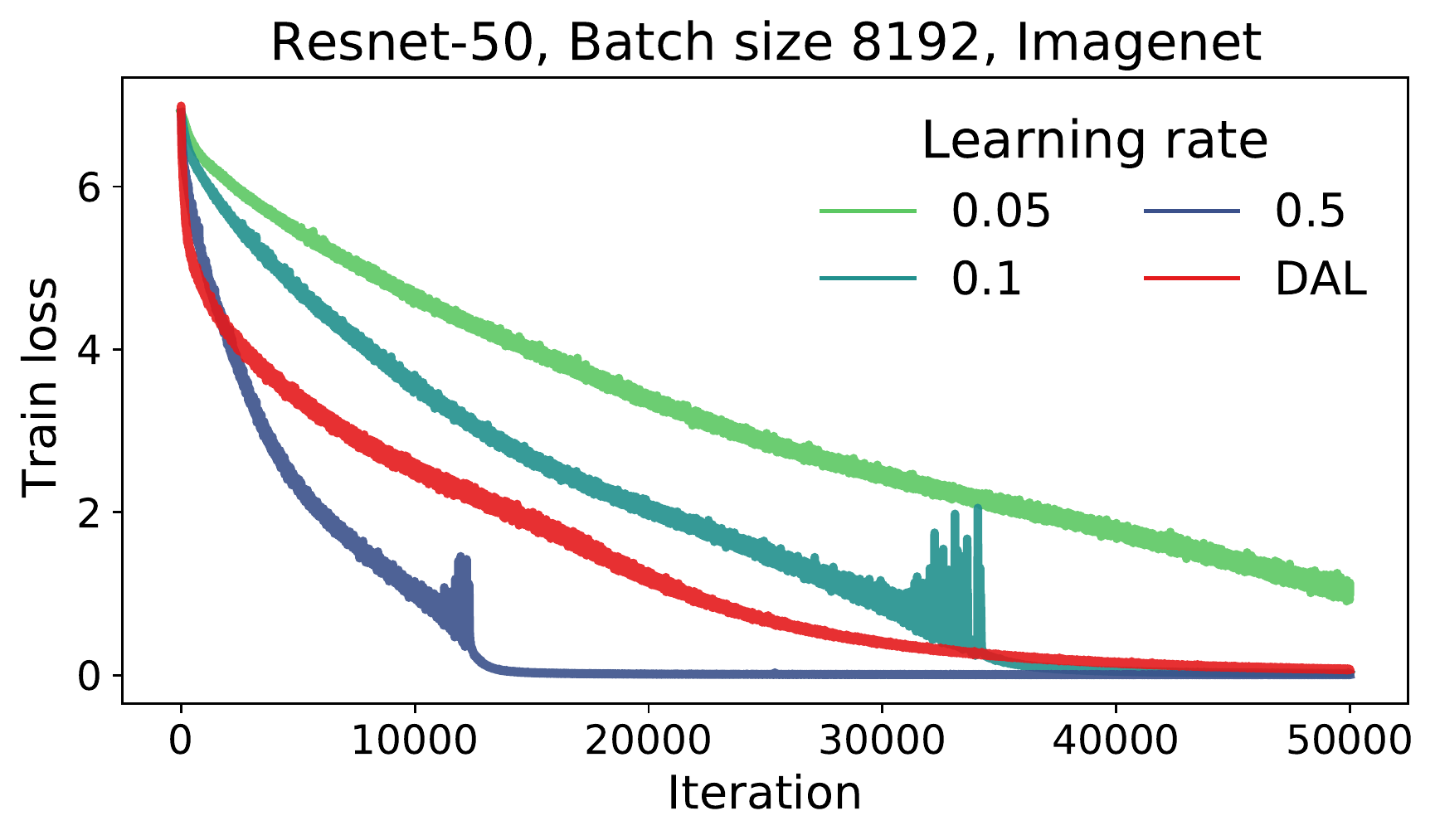}
\caption[DAL on Imagenet, a learning rate sweep across batch sizes.]{DAL: Imagenet results across batch sizes.}
\label{fig:imagenet_lr_scaling_across_batch_sizes}
\end{figure}

\begin{figure}[tbh]
 \includegraphics[width=0.24\columnwidth]{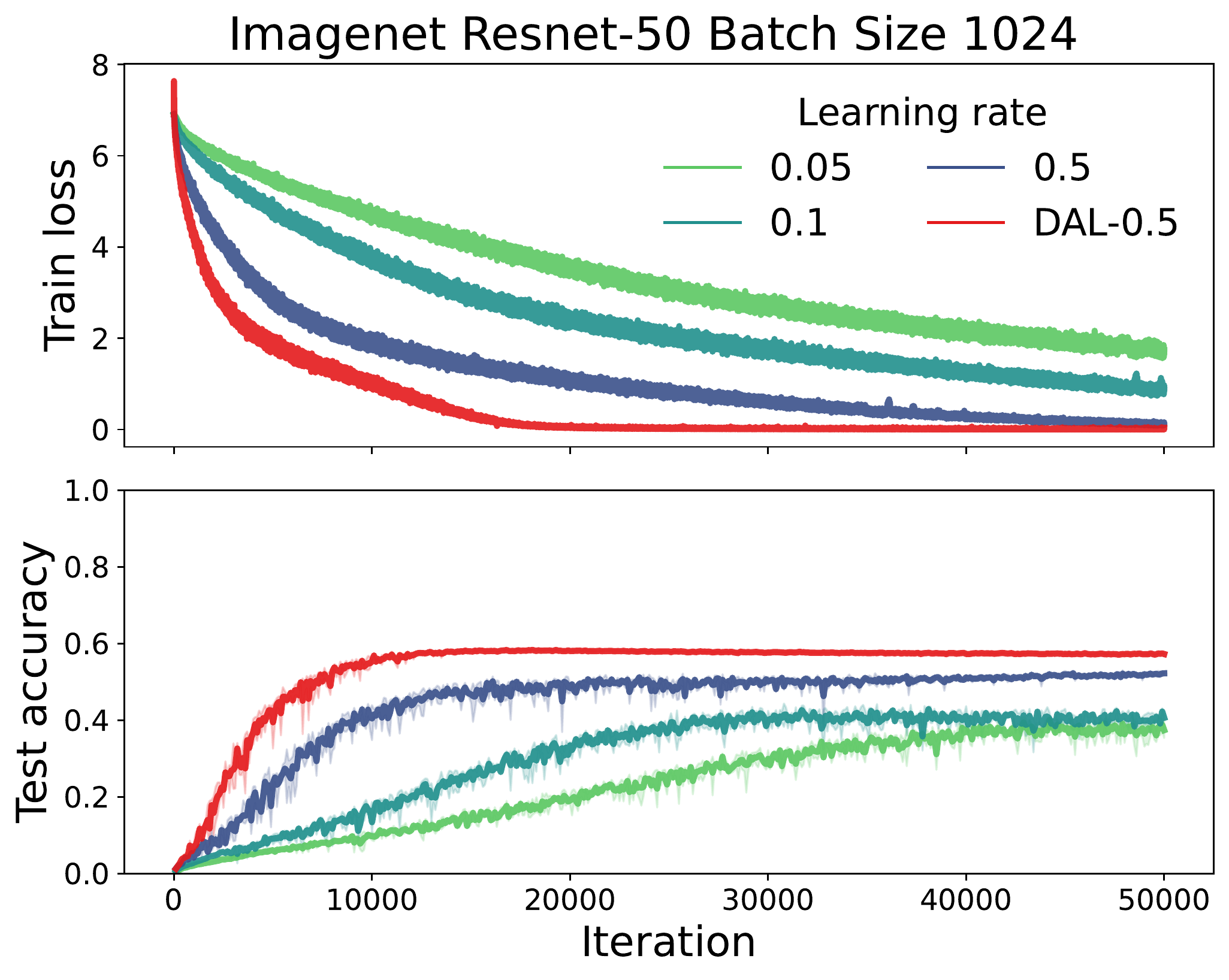}
  \includegraphics[width=0.24\columnwidth]{imagenet_2048_train_test_h_g_scaling_sqrt}
  \includegraphics[width=0.24\columnwidth]{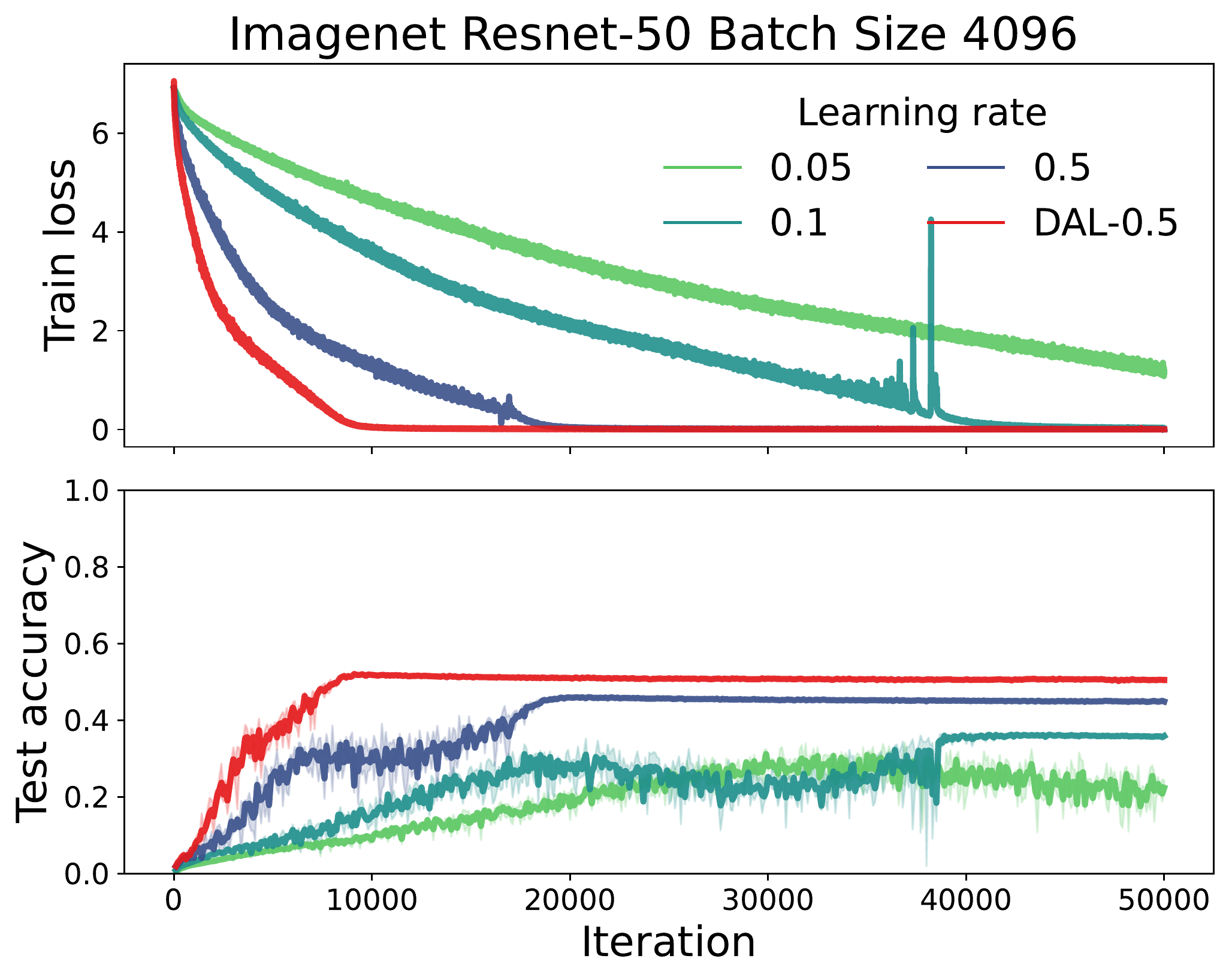}
  \includegraphics[width=0.24\columnwidth]{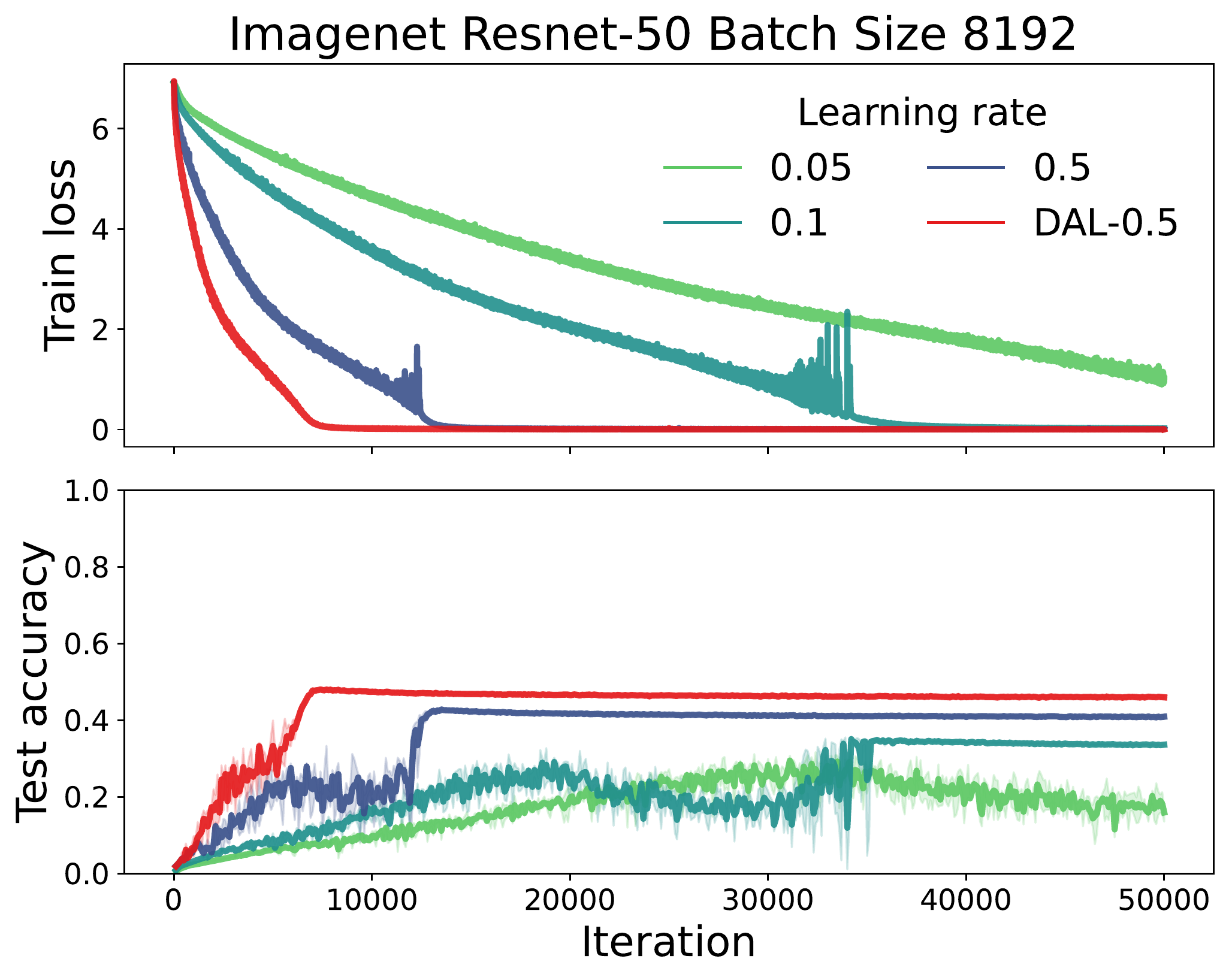}
\caption[DAL$-0.5$: on Imagenet, a learning rate sweep across batch sizes.]{DAL$-0.5$: Imagenet results across batch sizes.}
\label{fig:imagenet_lr_sqrt_scaling_across_batch_sizes}
\end{figure}

\begin{figure}
  \includegraphics[width=0.33\columnwidth]{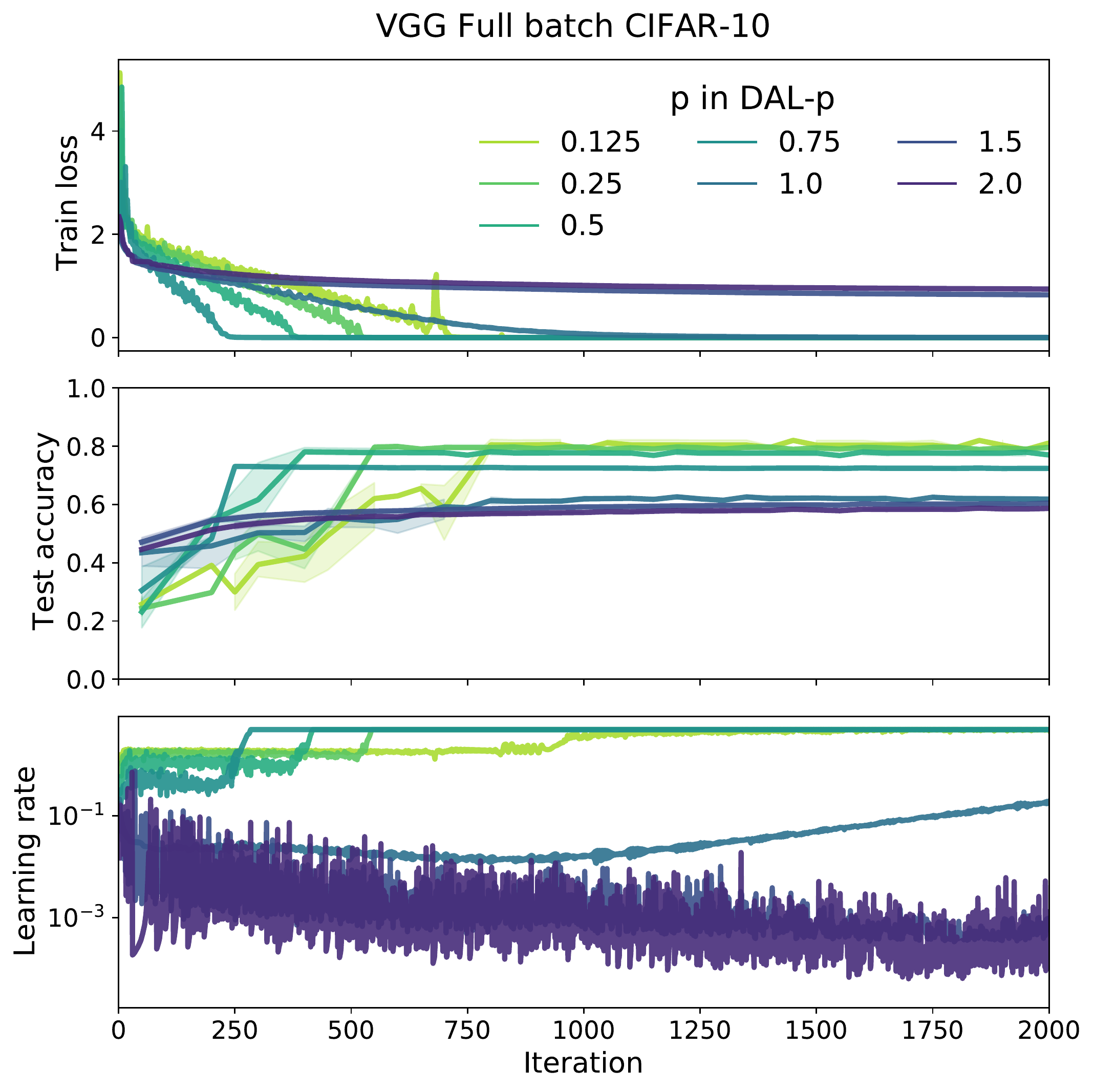}
 \includegraphics[width=0.33\columnwidth]{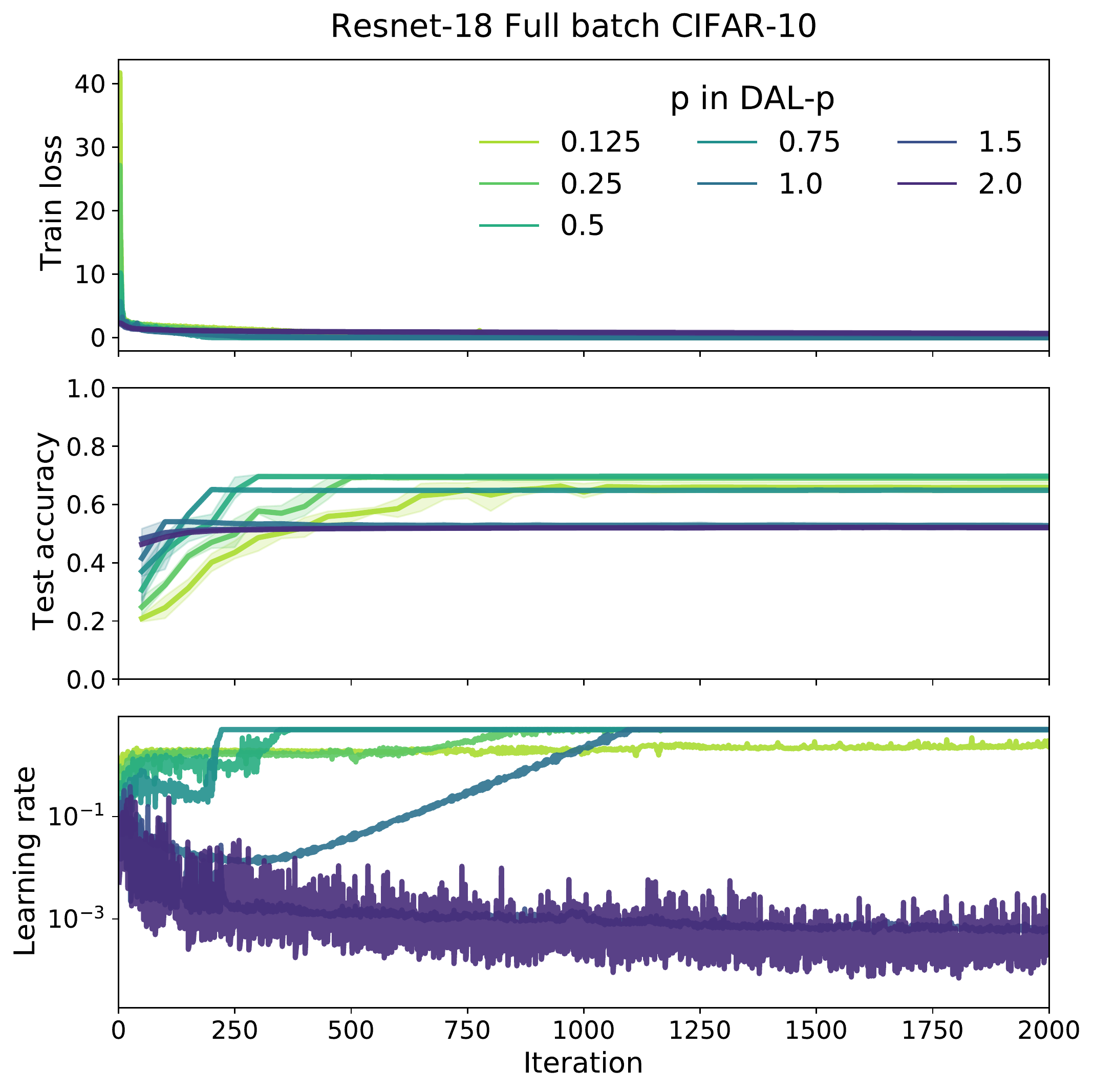}
 \includegraphics[width=0.33\columnwidth]{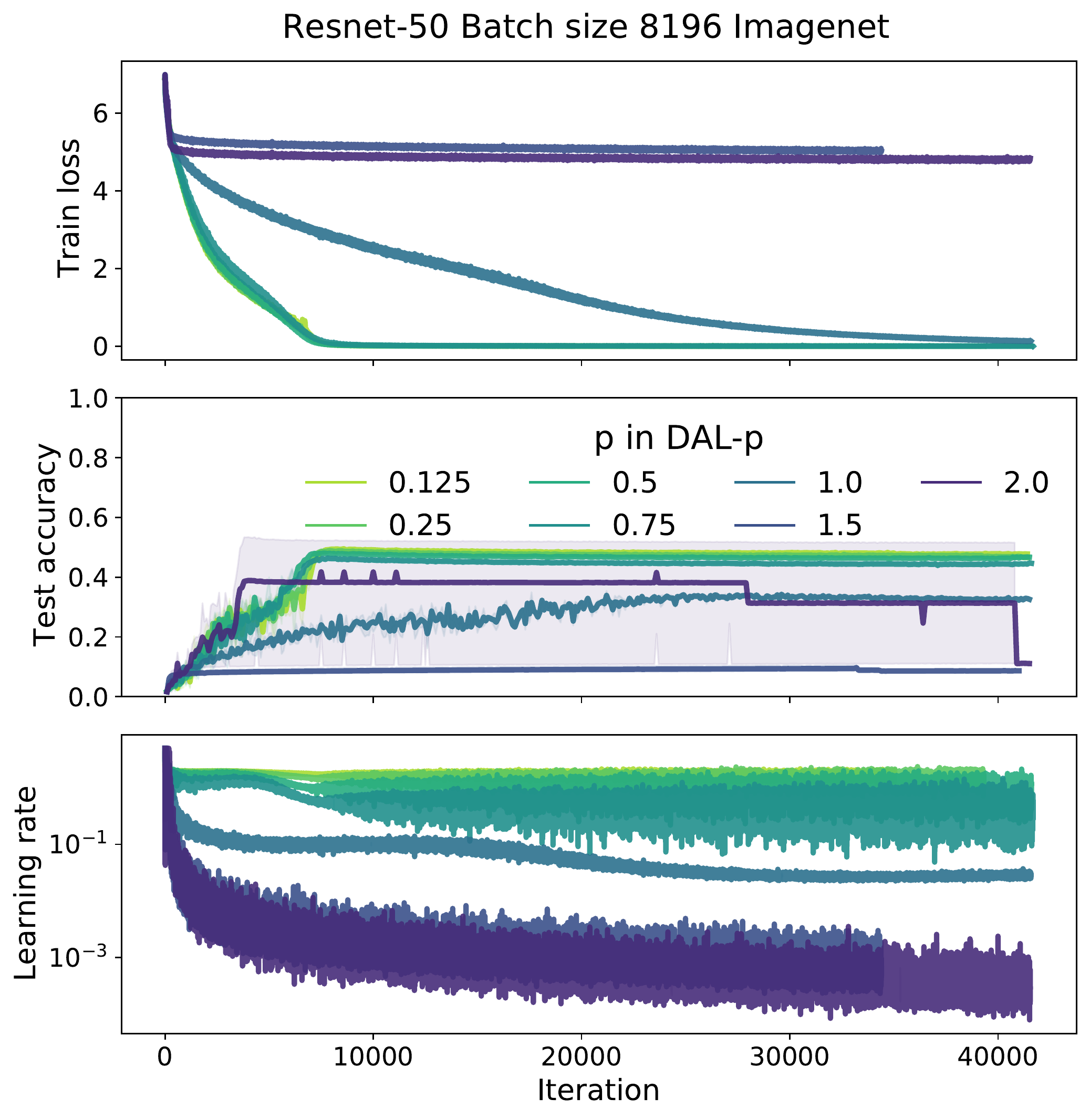}
\caption[DAL-$p$ sweep on VGG, Resnet-18 and Imagenet.]{DAL-$p$ sweep: discretization drift helps test performance, but at the cost of stability. We also show the effective learning rate and train losses and test accuracies.}
\label{fig:power_sweeps_all}
\end{figure}

\begin{figure}
 \includegraphics[width=0.48\columnwidth]{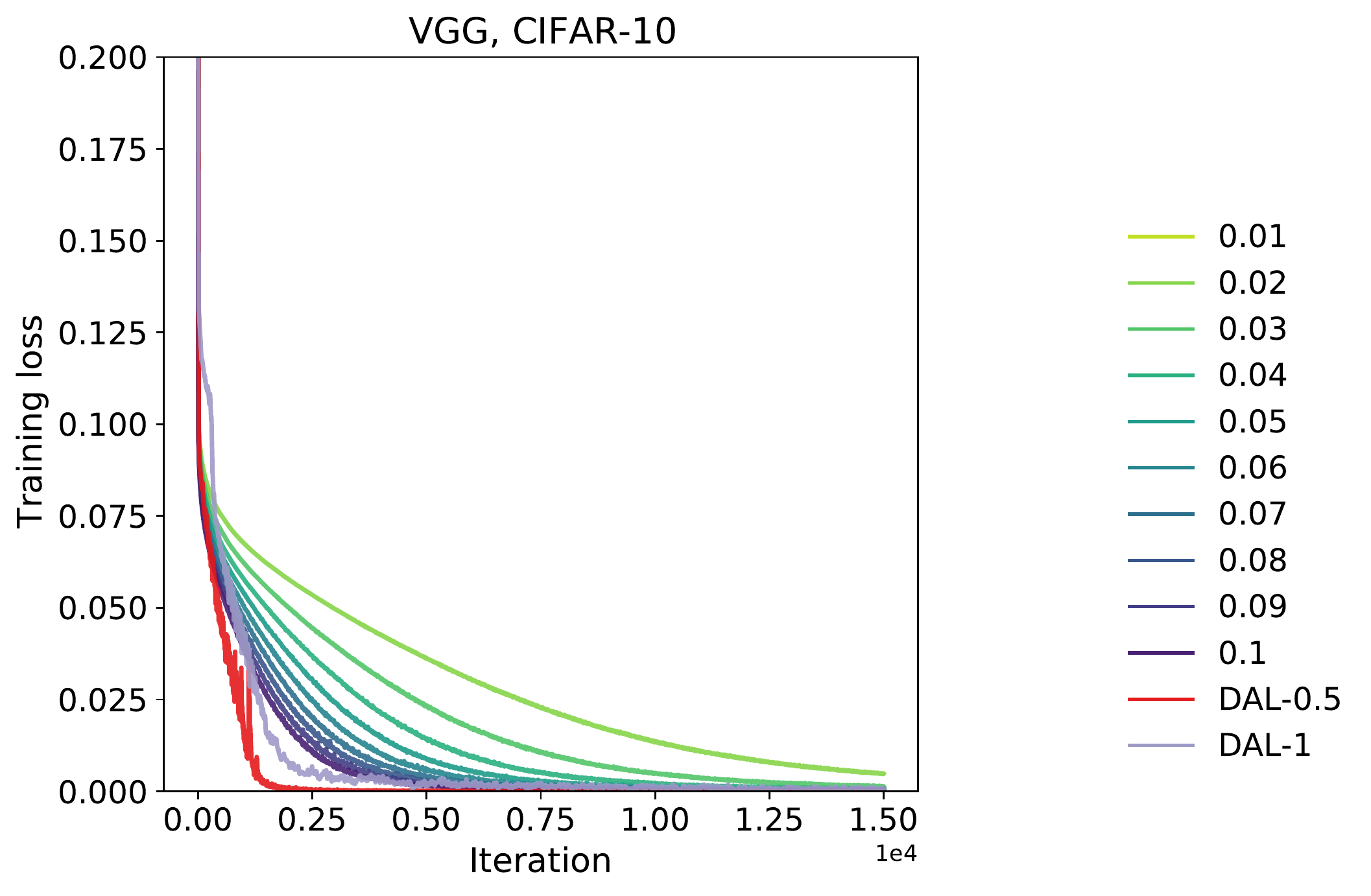}
 \includegraphics[width=0.48\columnwidth]{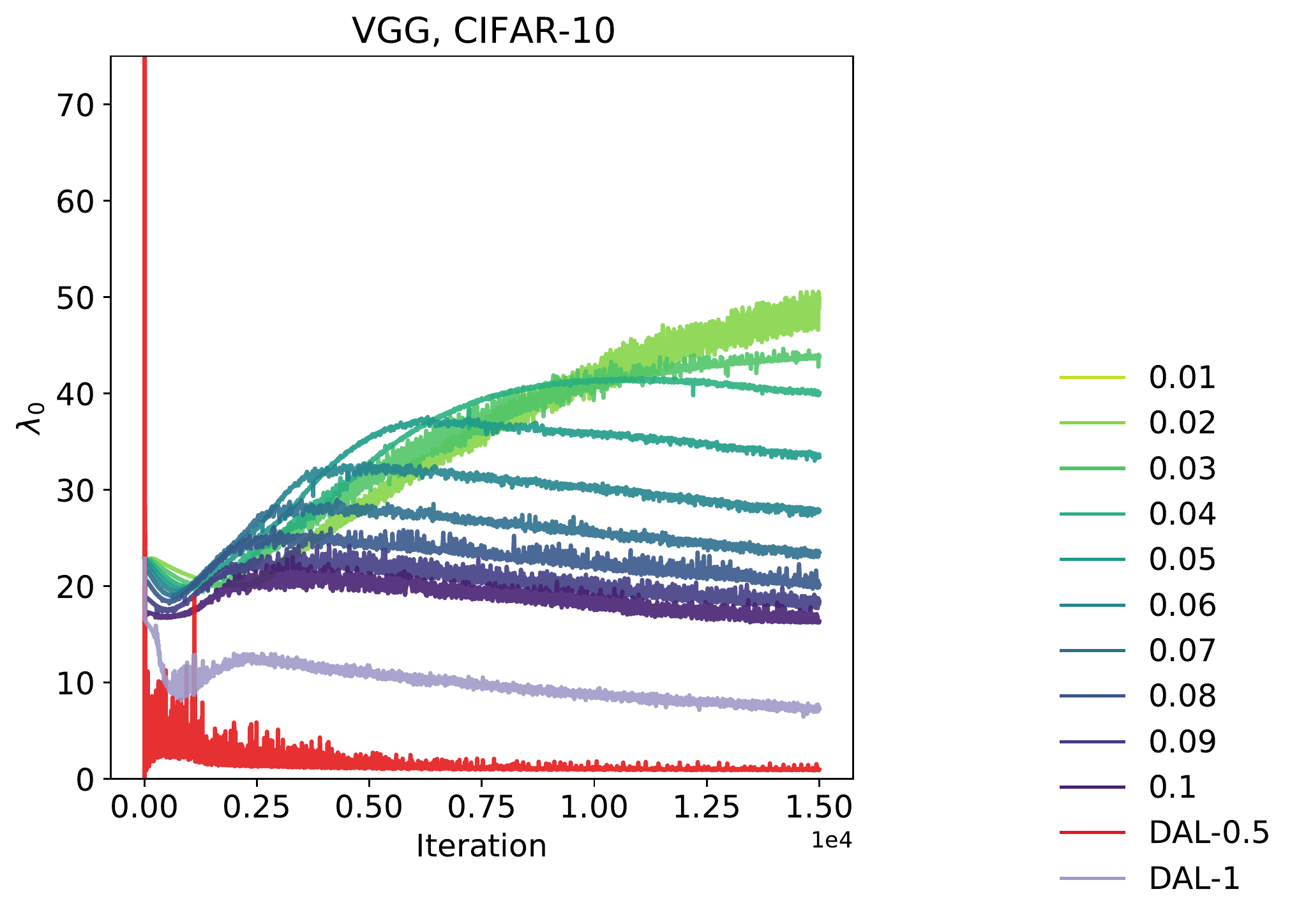}
\caption[DAL results with a least square loss. Full batch training of a VGG model on CIFAR-10.]{Results with a least square loss. DAL and DAL-$0.5$ lead to quicker training in this setting as well.}
\label{fig:least_square_loss_dal}
\end{figure}

\textbf{DAL learned landscapes.}
To investigate the landscapes learned by DAL-$p$, we use the method of~\citet{li2018visualizing} and compare against the landscapes learned using SGD in Figures~\ref{fig:DAL_p_64_landscape},~\ref{fig:DAL_p_full_batch_landscape},~\ref{fig:DAL_p_imagenet_landscape}. Across datasets and batch sizes, we consistenly observe that DAL-$p$ learns flatter landscapes. We also consistenly observe that during training, $\lambda_0$ is smaller with DAL-$p$ than with SGD, as shown in Figures~\ref{fig:DAL_p_64_eigen} and~\ref{fig:DAL_p_full_batch_eigen}. Crucially, we observe these results \textit{even when the accuracy of the two models is similar, as is the case with a batch size of 64 on CIFAR-10.}

\begin{figure}[t]
\begin{subfigure}[DAL]{
\includegraphics[width=0.45\columnwidth]{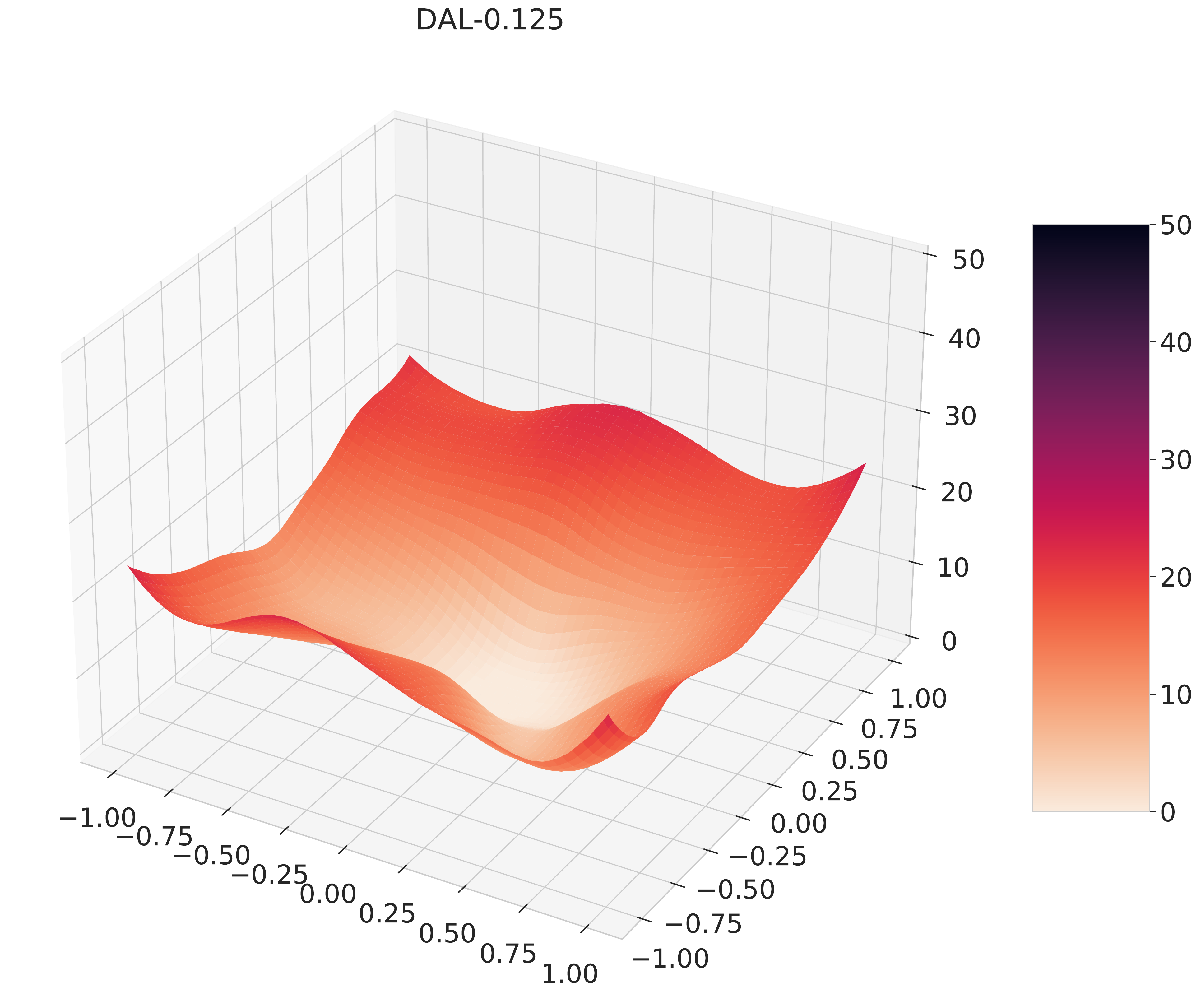}
}
\end{subfigure}
\begin{subfigure}[SGD]{
\includegraphics[width=0.45\columnwidth]{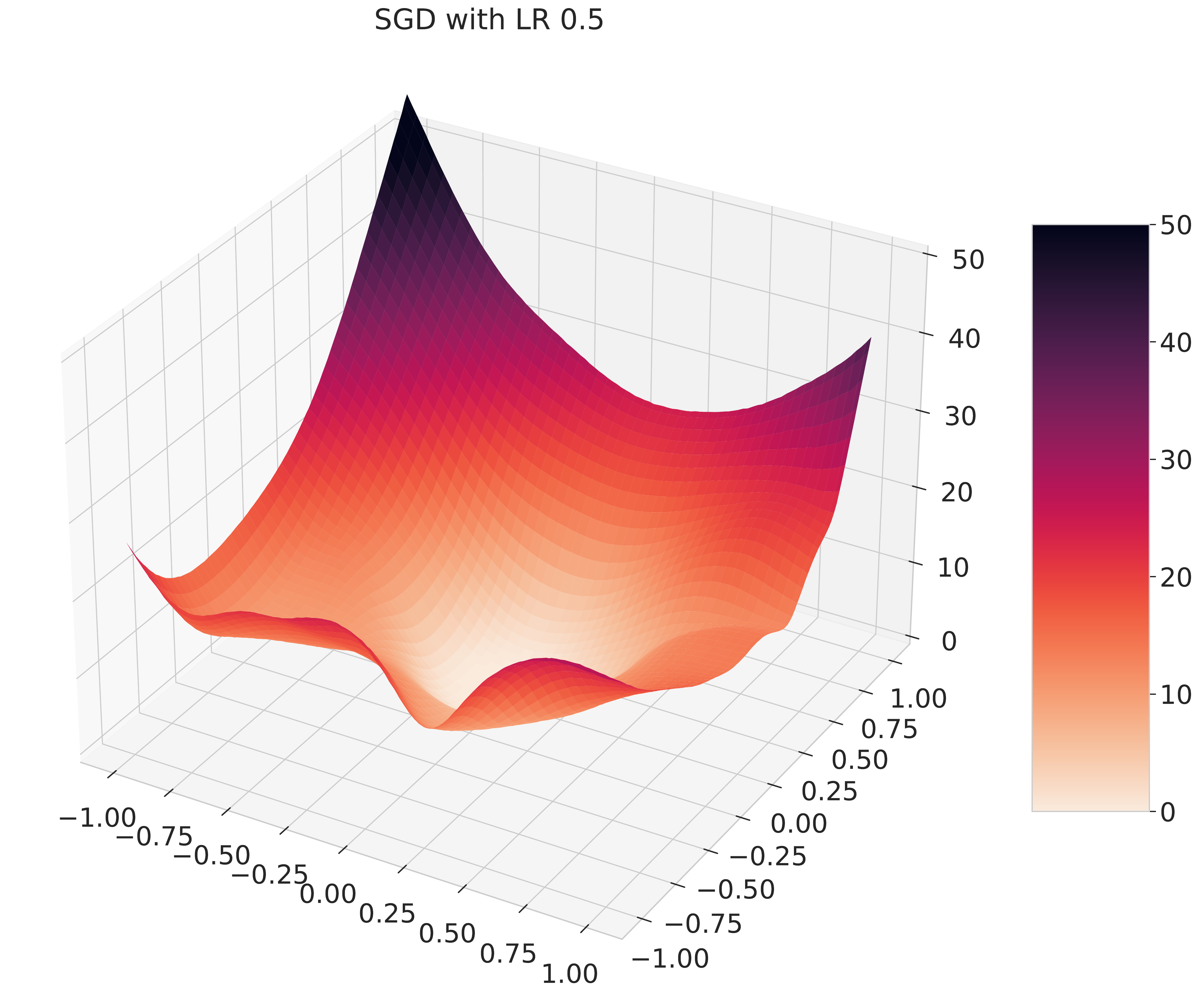}
}
\end{subfigure}
\caption{\textbf{CIFAR-10, batch size 64.} The 2D projection of the DAL-$p$ and SGD learned landscapes on CIFAR-10 using a VGG model. The visualisation is made using the method of~\citet{li2018visualizing}. \textit{Though both models achieve a 86\% accuracy on the test dataset, DAL-$p$ has learned a flatter model.} We also show the trajectory of $\lambda_0$ for both models in Figure~\ref{fig:DAL_p_64_eigen}.}
\label{fig:DAL_p_64_landscape}
\end{figure}

\begin{figure}[t]
\begin{subfigure}[DAL]{
\includegraphics[width=0.4\columnwidth]{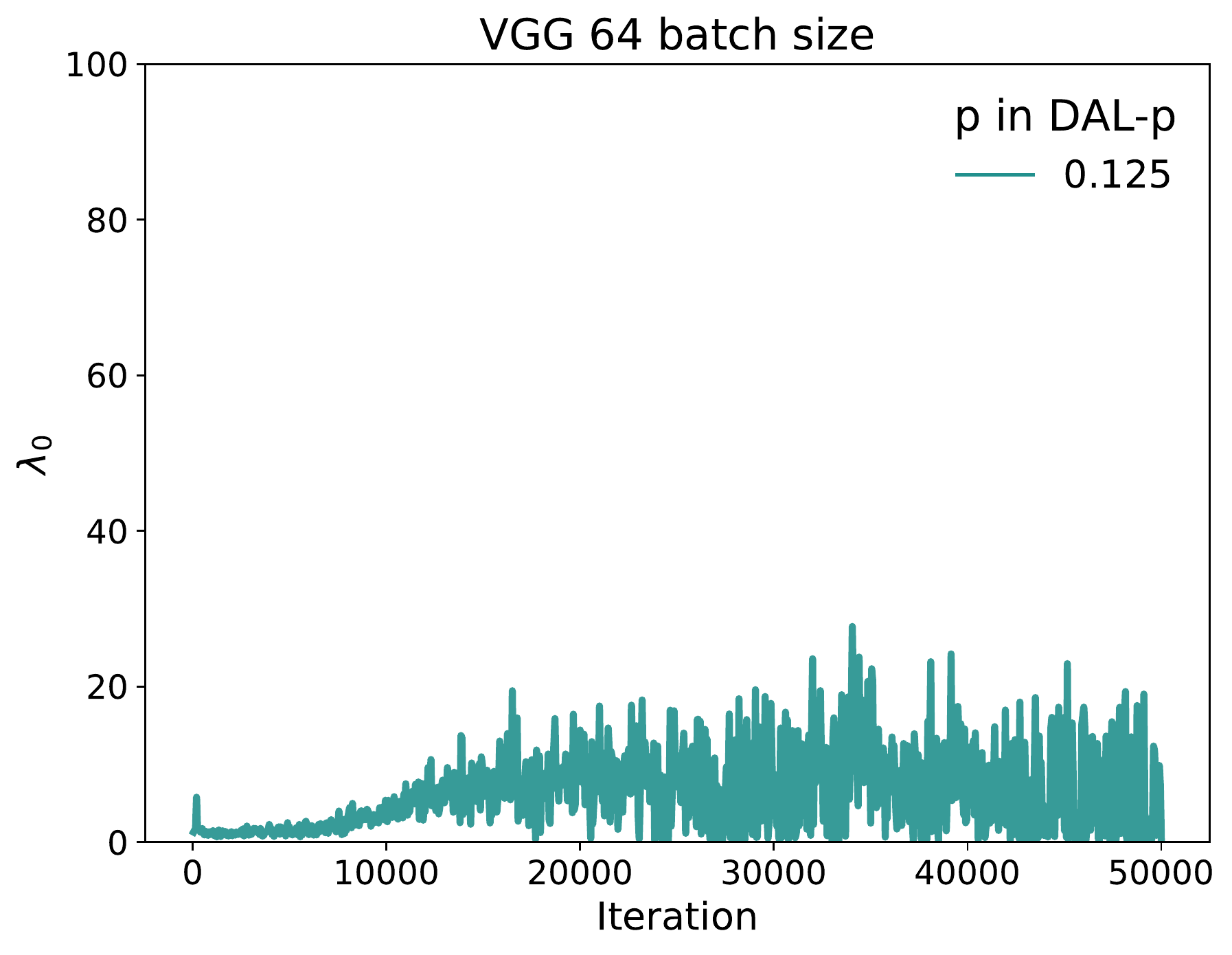}
}
\end{subfigure}
\begin{subfigure}[SGD]{
\includegraphics[width=0.4\columnwidth]{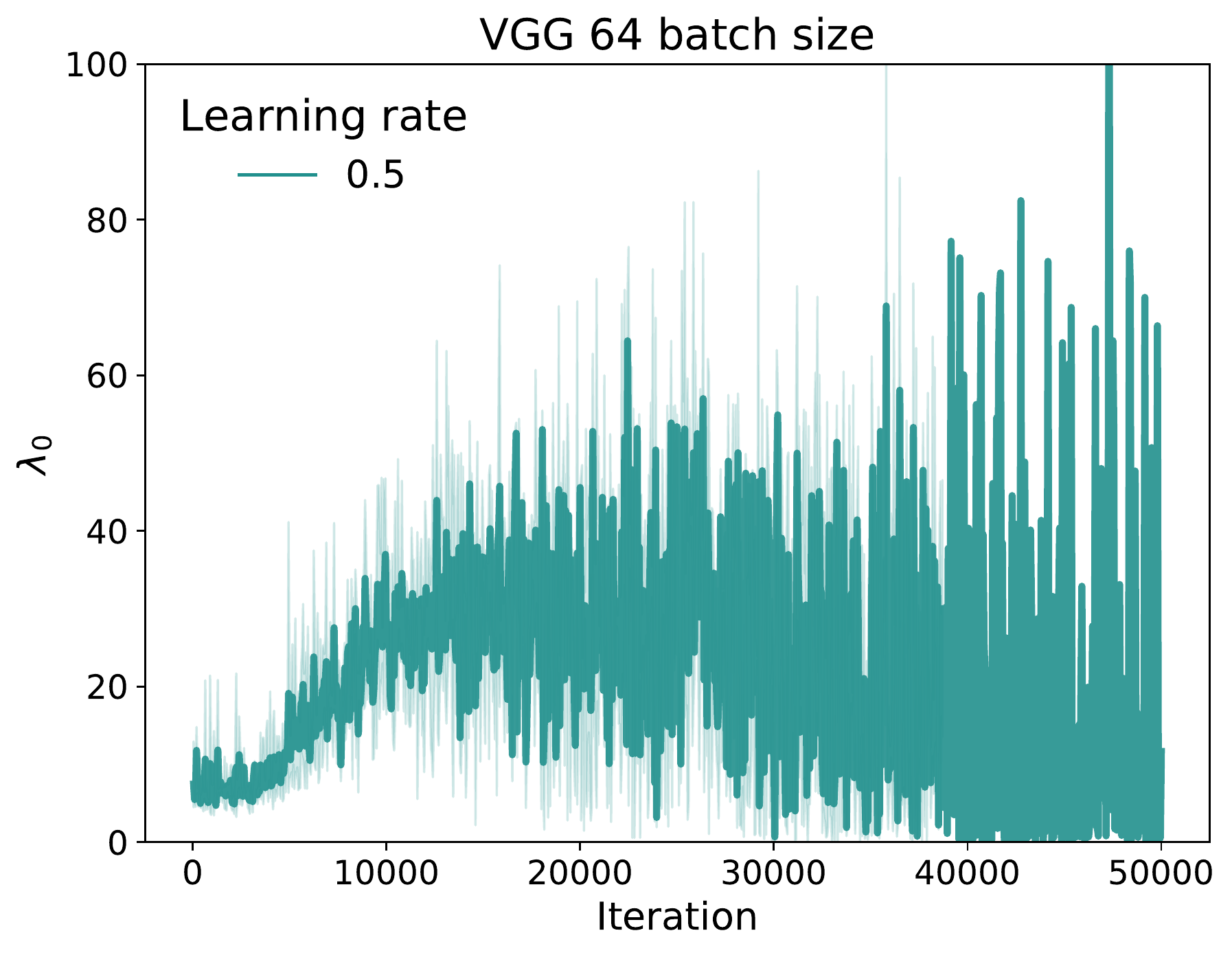}
}
\end{subfigure}
\caption{\textbf{CIFAR-10, batch size 64.} $\lambda_0$ for learned models with using SGD and DAL-$p$ on CIFAR-10.}
\label{fig:DAL_p_64_eigen}
\end{figure}

\begin{figure}[t]
\centering
\begin{subfigure}[DAL]{
\includegraphics[width=0.43\columnwidth]{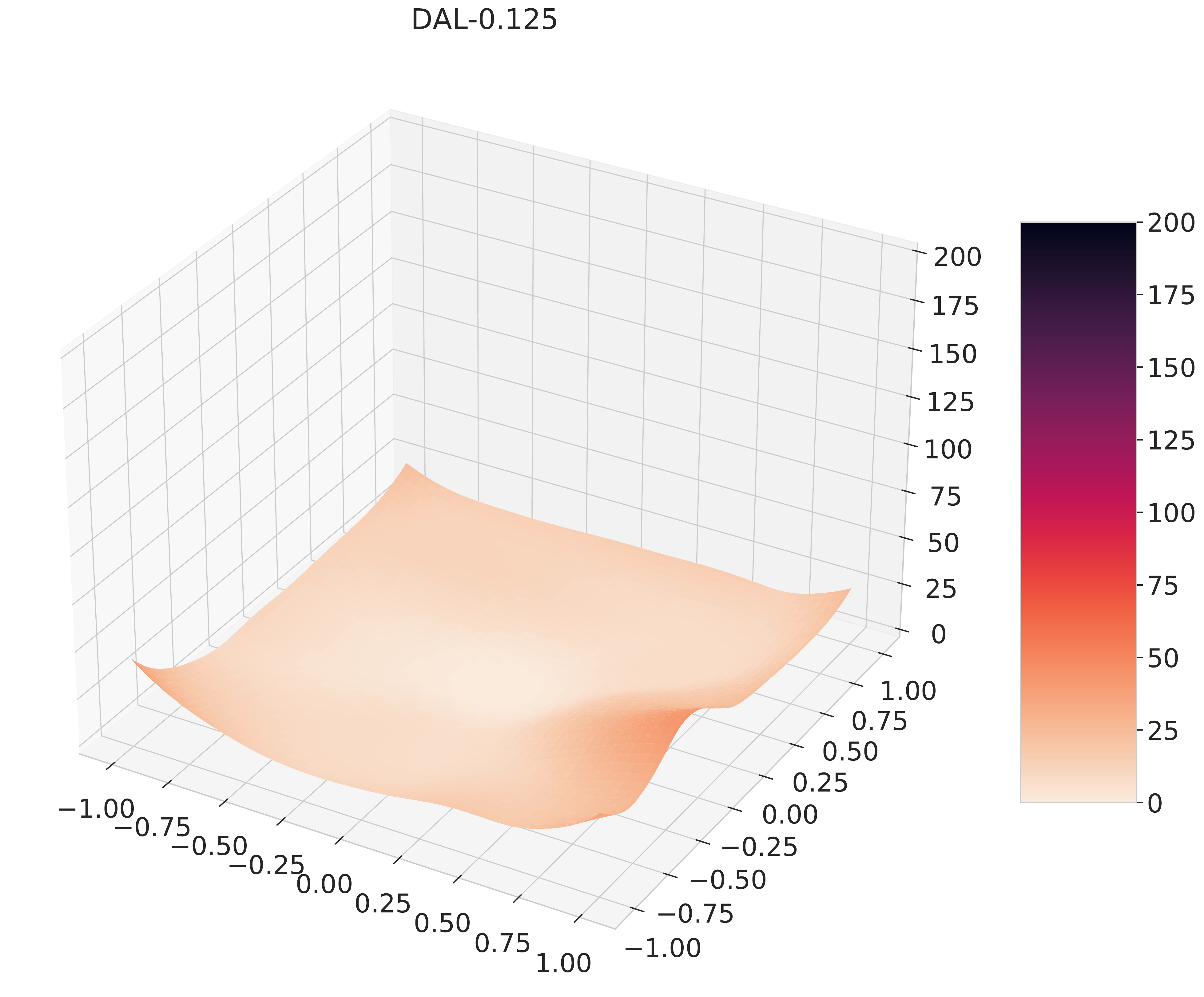}
}
\end{subfigure}
\begin{subfigure}[SGD]{
\includegraphics[width=0.43\columnwidth]{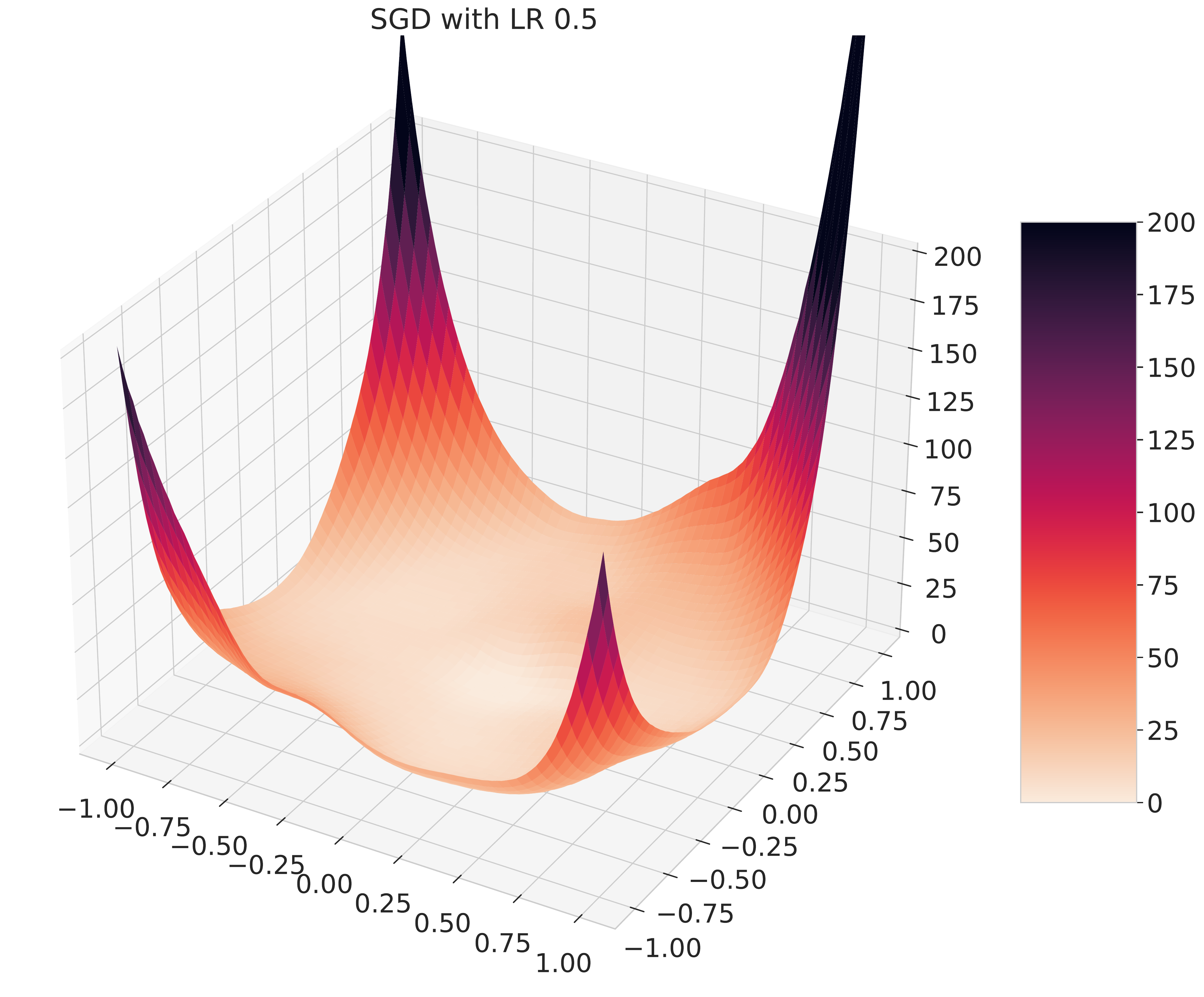}
}
\end{subfigure}

\begin{subfigure}[DAL]{
\includegraphics[width=0.43\columnwidth]{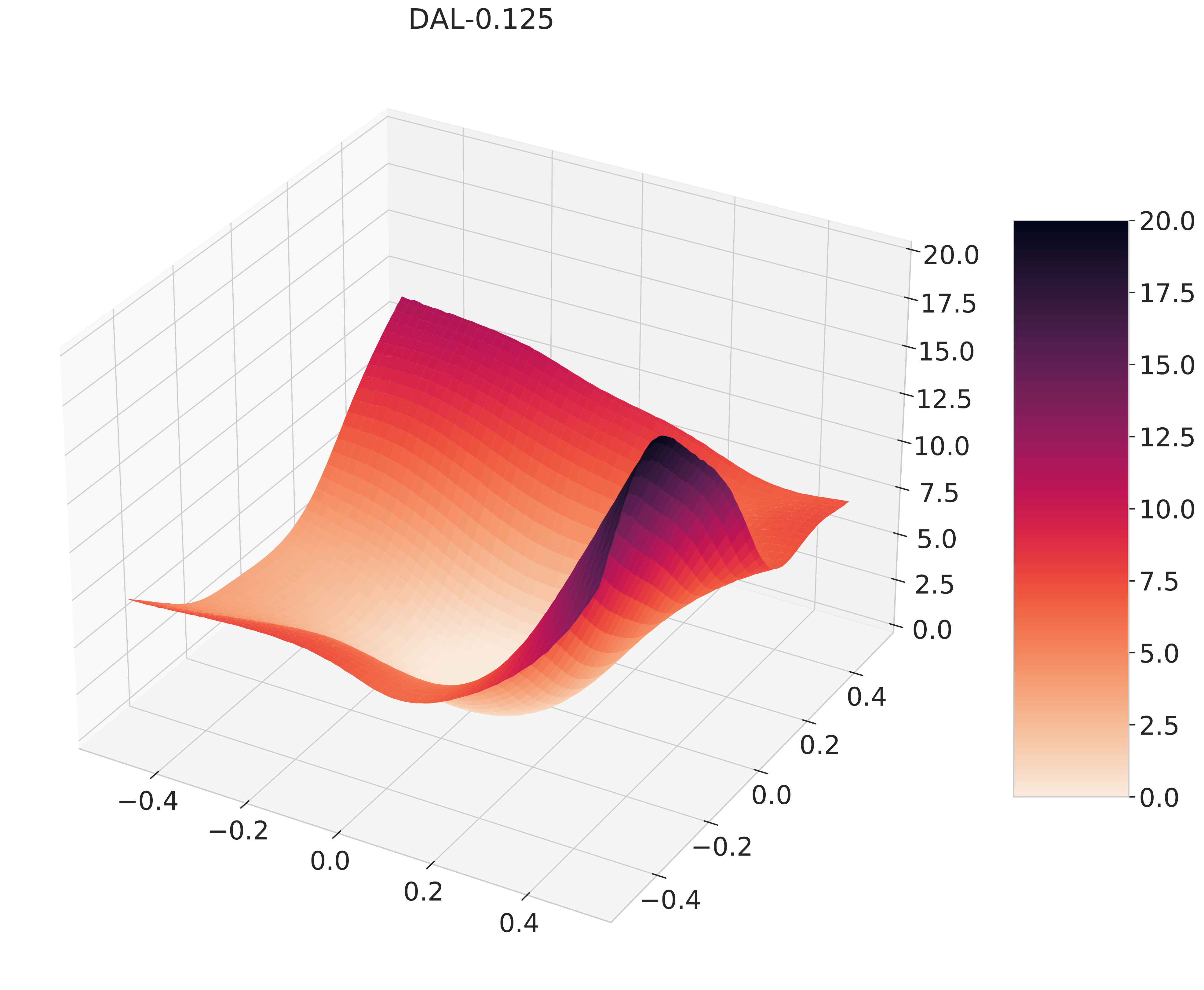}
}
\end{subfigure}
\begin{subfigure}[SGD]{
\includegraphics[width=0.43\columnwidth]{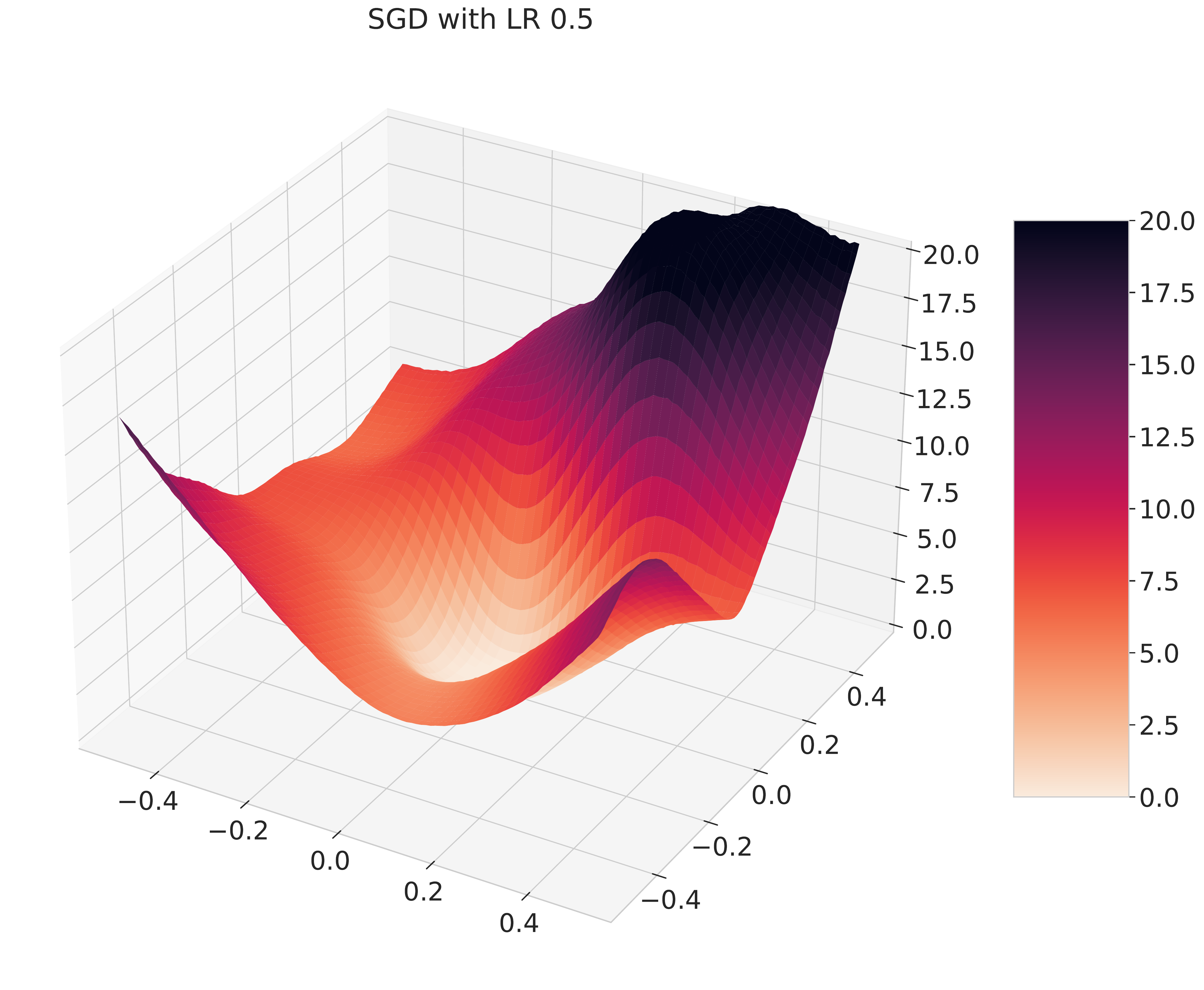}
}
\end{subfigure}

\caption{\textbf{CIFAR-10, full batch, at various levels of zoom into around the convergence points.}  The 2D projection of the DAL-$p$ and SGD learned landscapes on CIFAR-10, using a VGG model. The visualisation is made using the method of~\citet{li2018visualizing}. The DAL-$p$ model achieves an accuracy of 82\% which the SGD model achieves 77\% accuracy. We also show the trajectory of $\lambda_0$ for both models in Figure~\ref{fig:DAL_p_full_batch_eigen}.}
\label{fig:DAL_p_full_batch_landscape}
\end{figure}

\begin{figure}[t]
\begin{subfigure}[]{
\includegraphics[width=0.4\columnwidth]{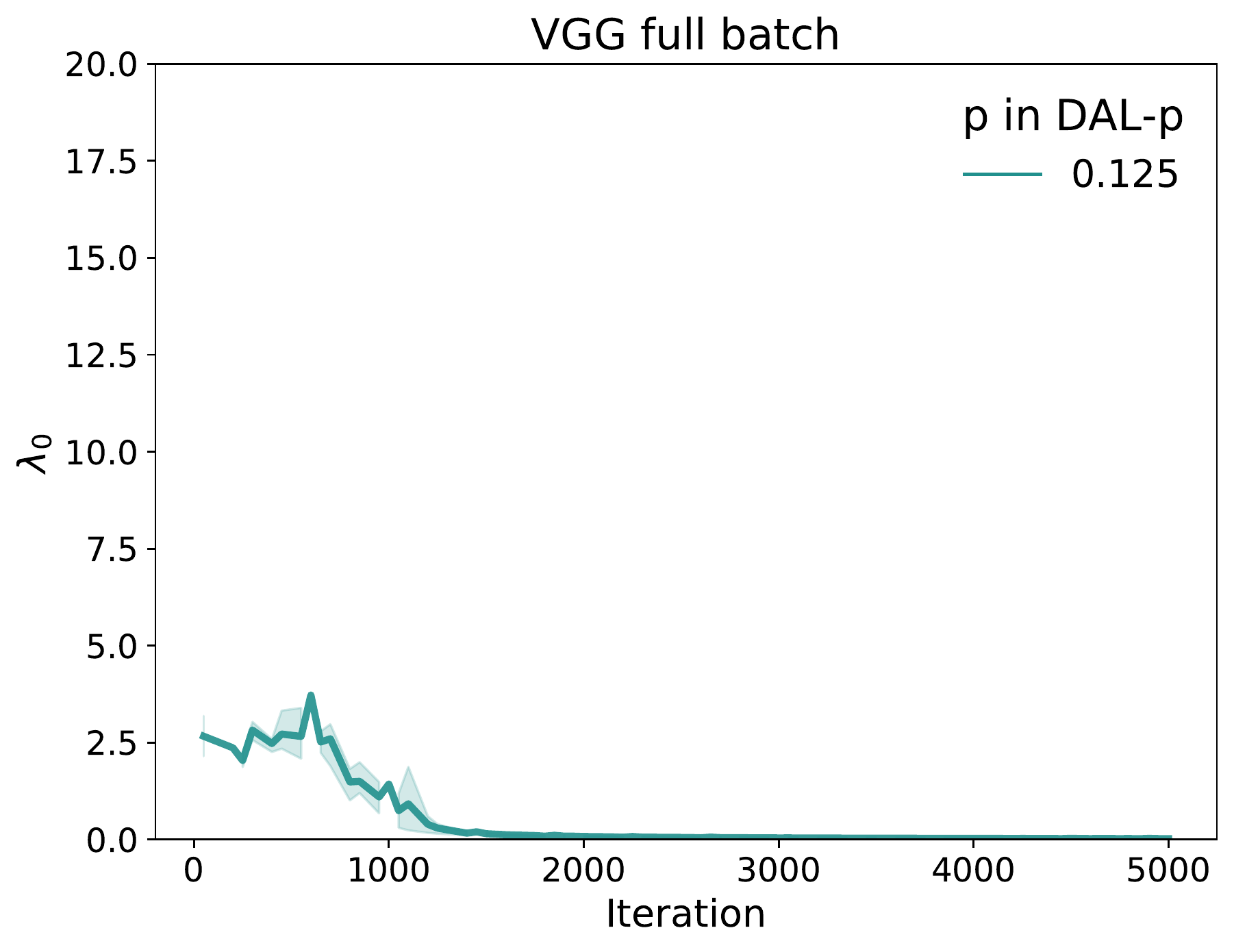}
}
\end{subfigure}
\begin{subfigure}[]{
\includegraphics[width=0.4\columnwidth]{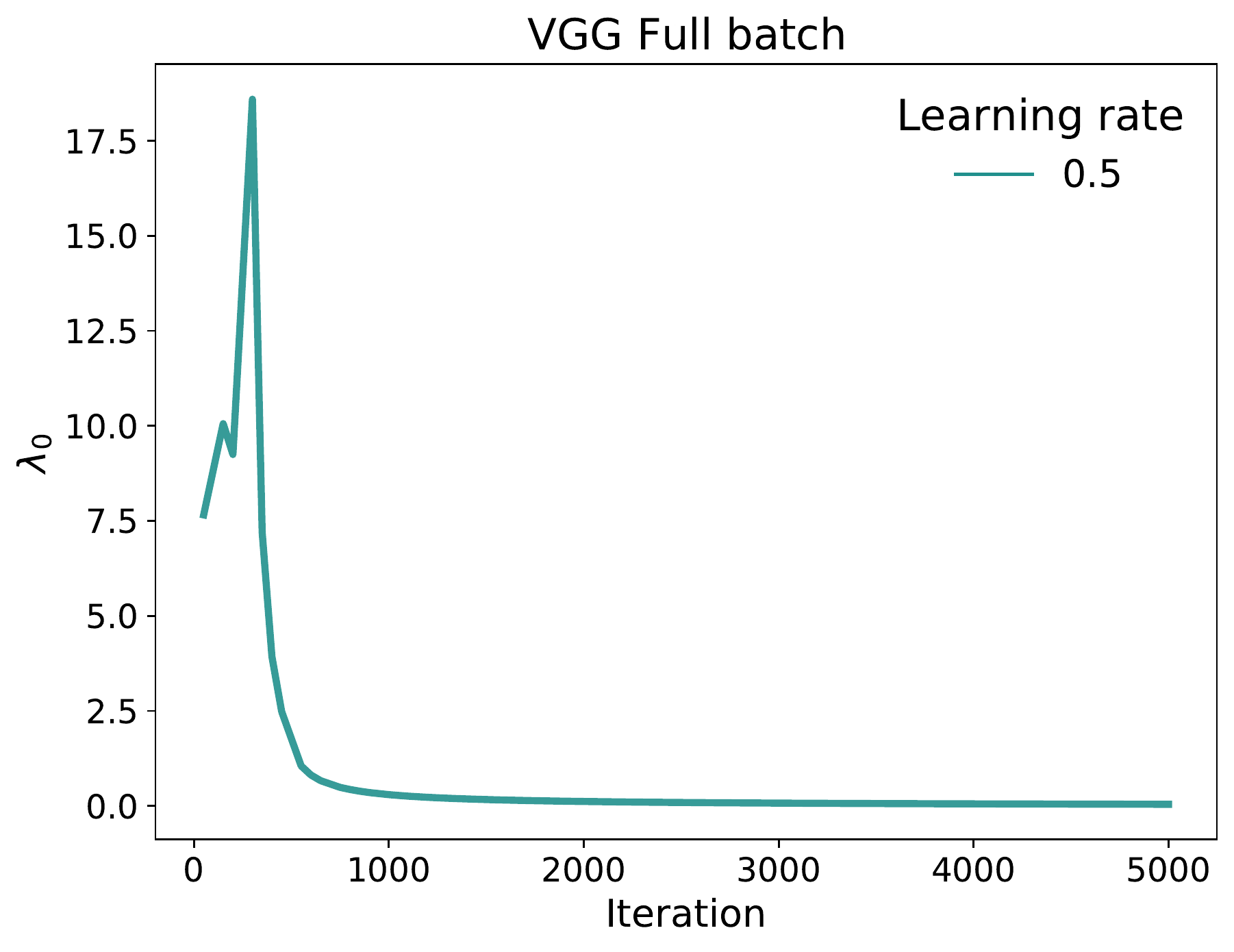}
}
\end{subfigure}
\caption{$\lambda_0$ for learned models with using full batch gradient descent and DAL-$p$ on CIFAR-10.}
\label{fig:DAL_p_full_batch_eigen}
\end{figure}

\begin{figure}[t]
\centering
\begin{subfigure}[DAL]{
\includegraphics[width=0.45\columnwidth]{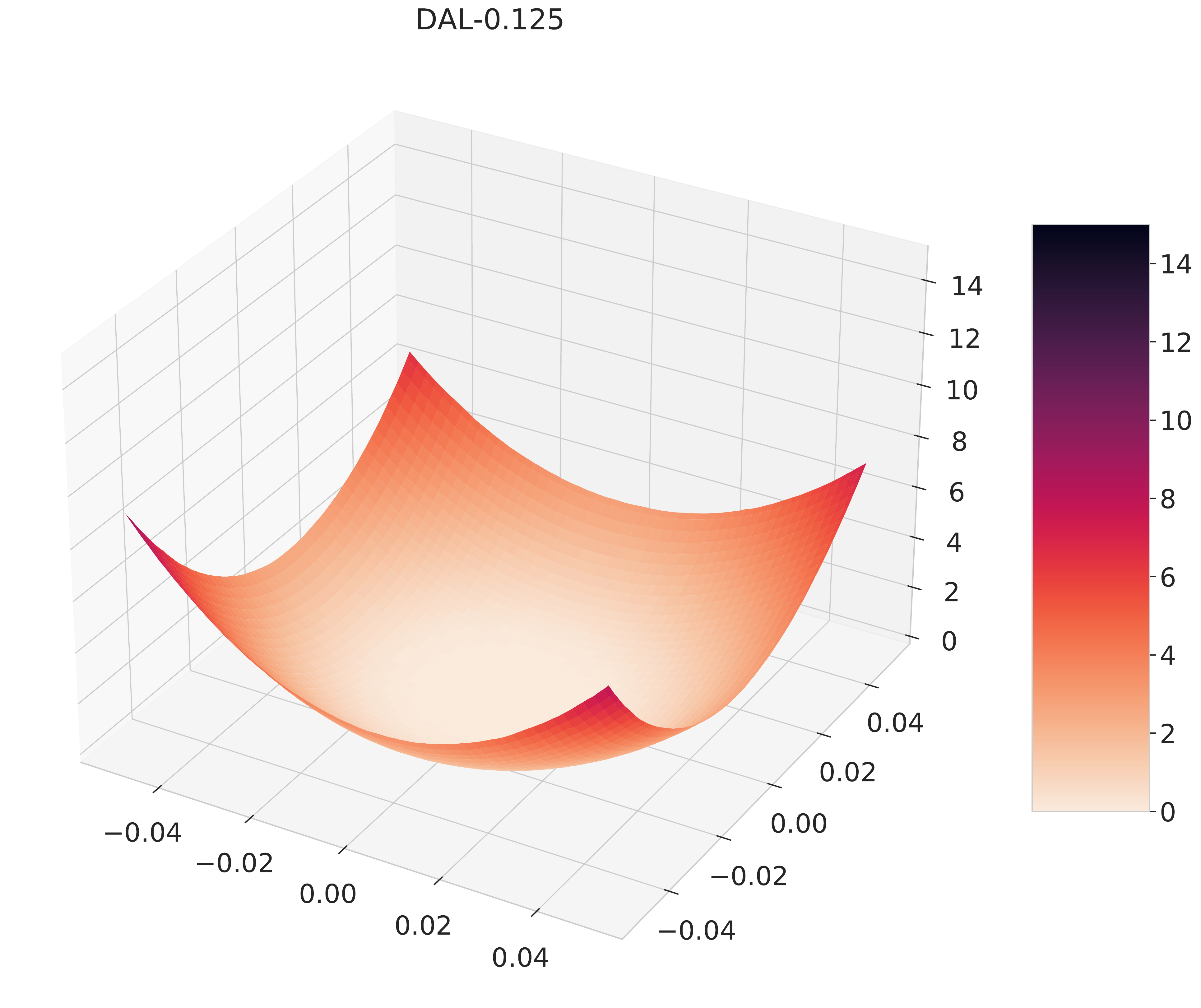}
}
\end{subfigure}
\begin{subfigure}[SGD]{
\includegraphics[width=0.45\columnwidth]{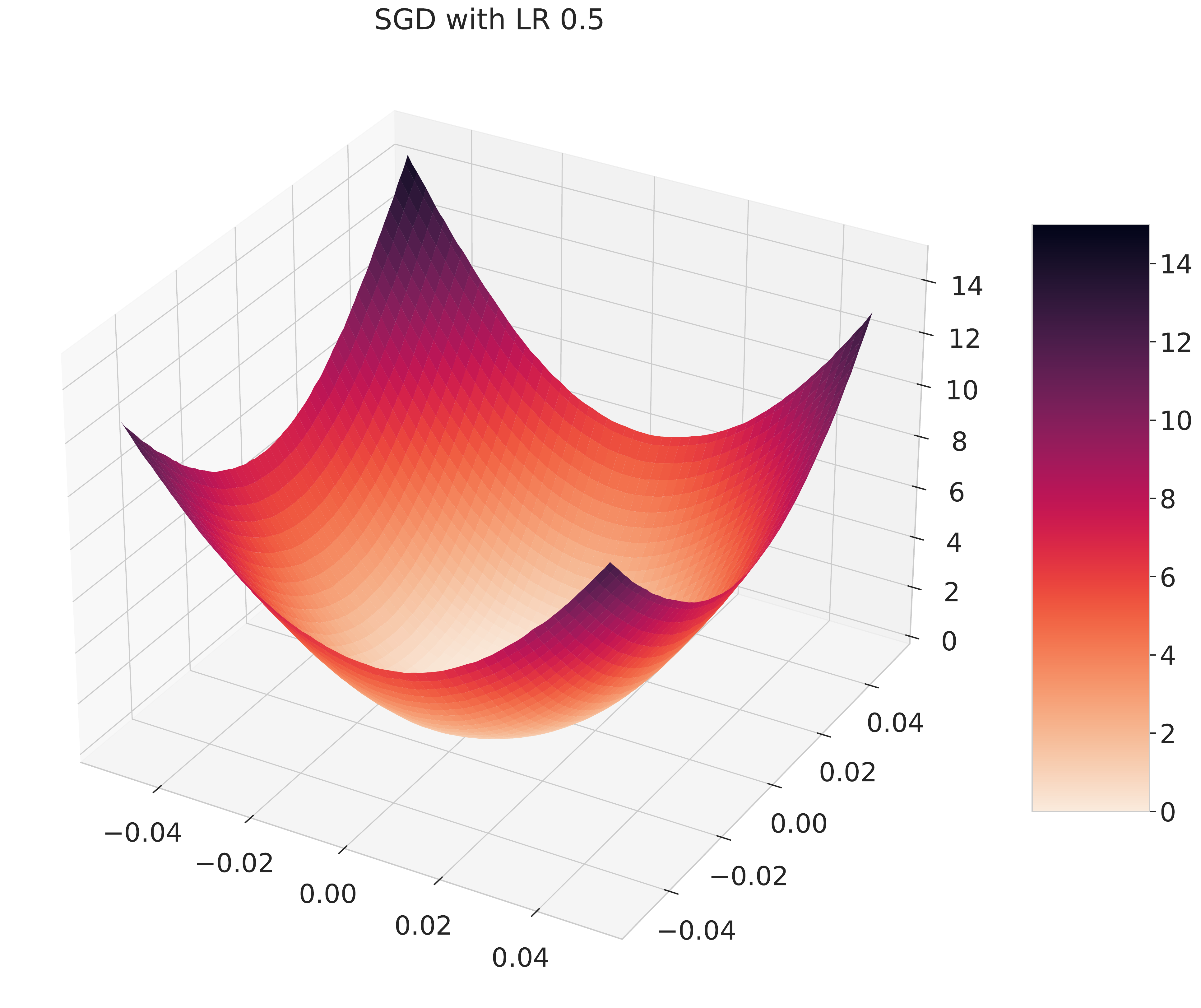}
}
\end{subfigure}
\caption{\textbf{Imagenet, batch size 8192}. The 2D projection of the DAL-$p$ and SGD learned landscapes, using the method of~\citet{li2018visualizing}. The DAL-$p$ model achieves an accuracy of 50\% which the SGD model achieves 42\% accuracy.}
\label{fig:DAL_p_imagenet_landscape}
\end{figure}

\textbf{Learning rate per parameter}. We show preliminary results of how to use a DAL like learning rate to use a per-parameter learning rate in Figure~\ref{fig:vgg_lr_scaling_per_parameter} and~\ref{fig:imagenet_lr_scaling_per_parameter}. We note that here we found that we can directly use 
$\frac{1}{\left(\nabla_{\vtheta}^2 E \nabla_{\vtheta} E\right)_i}$ as the learning rate for parameter $i$. When using the per parameter learning rate update, the argument regarding the norm of the gradient used for DAL is no longer necessary.

\begin{figure}
  \includegraphics[width=0.24\columnwidth]{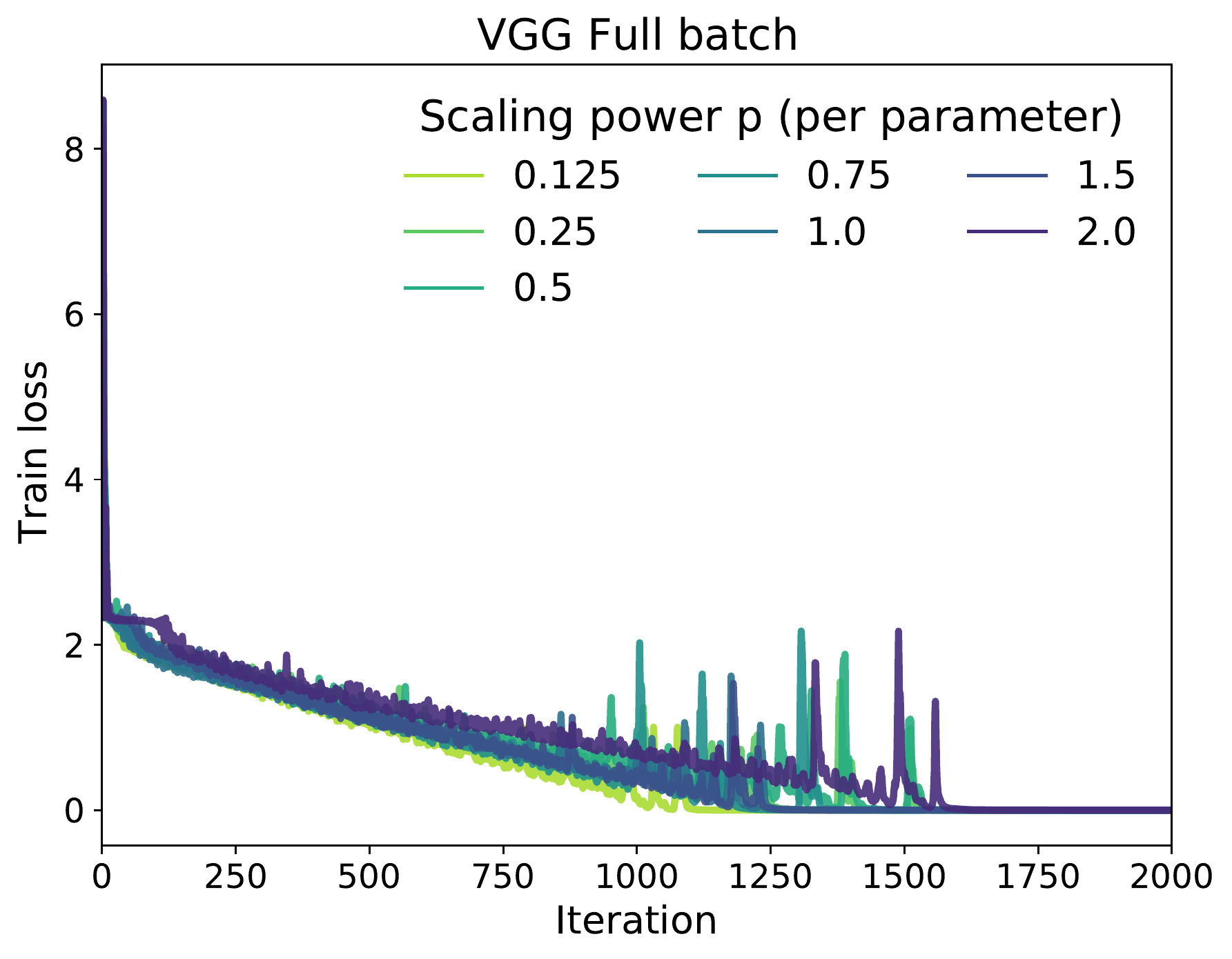}
  \includegraphics[width=0.24\columnwidth]{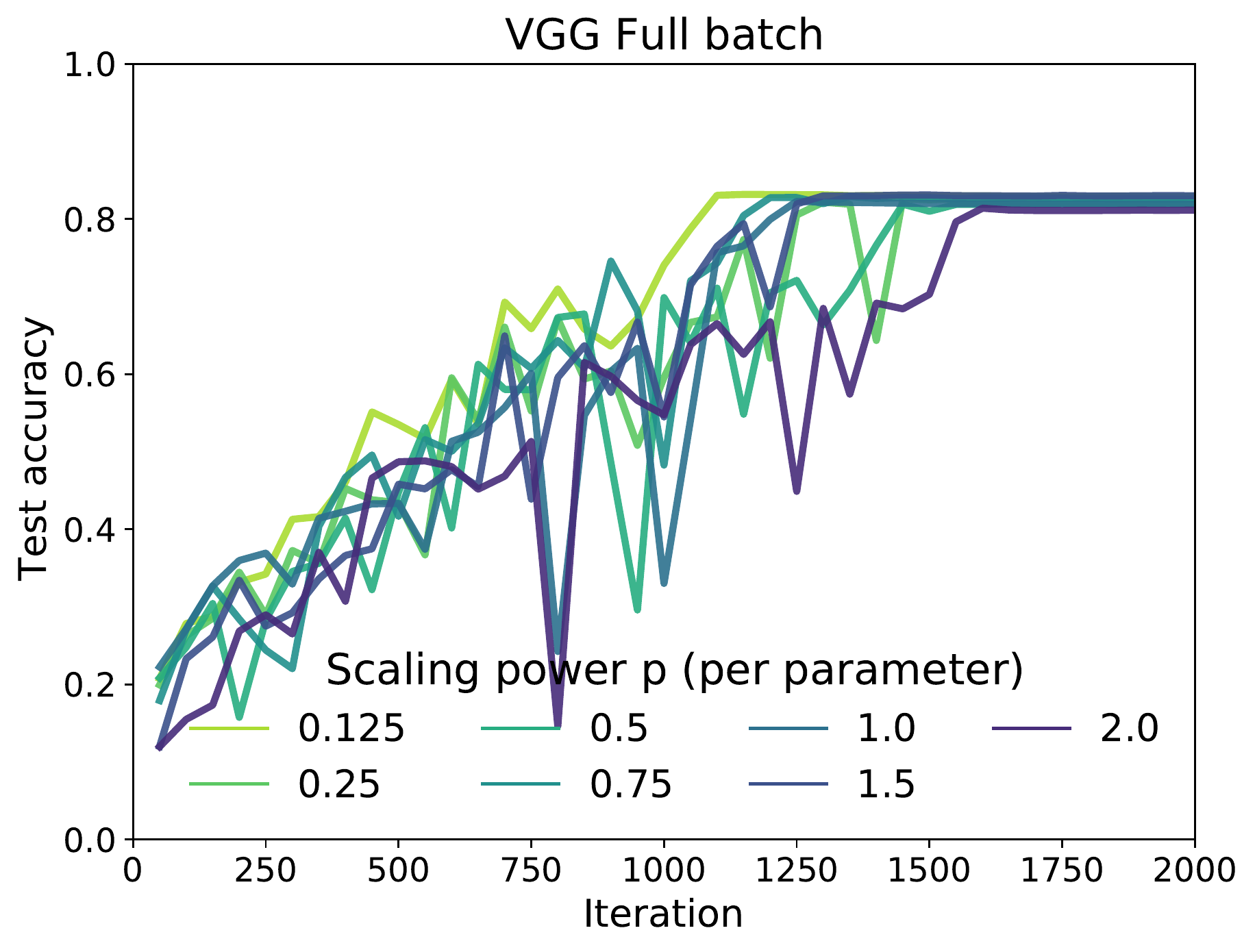}
  \includegraphics[width=0.24\columnwidth]{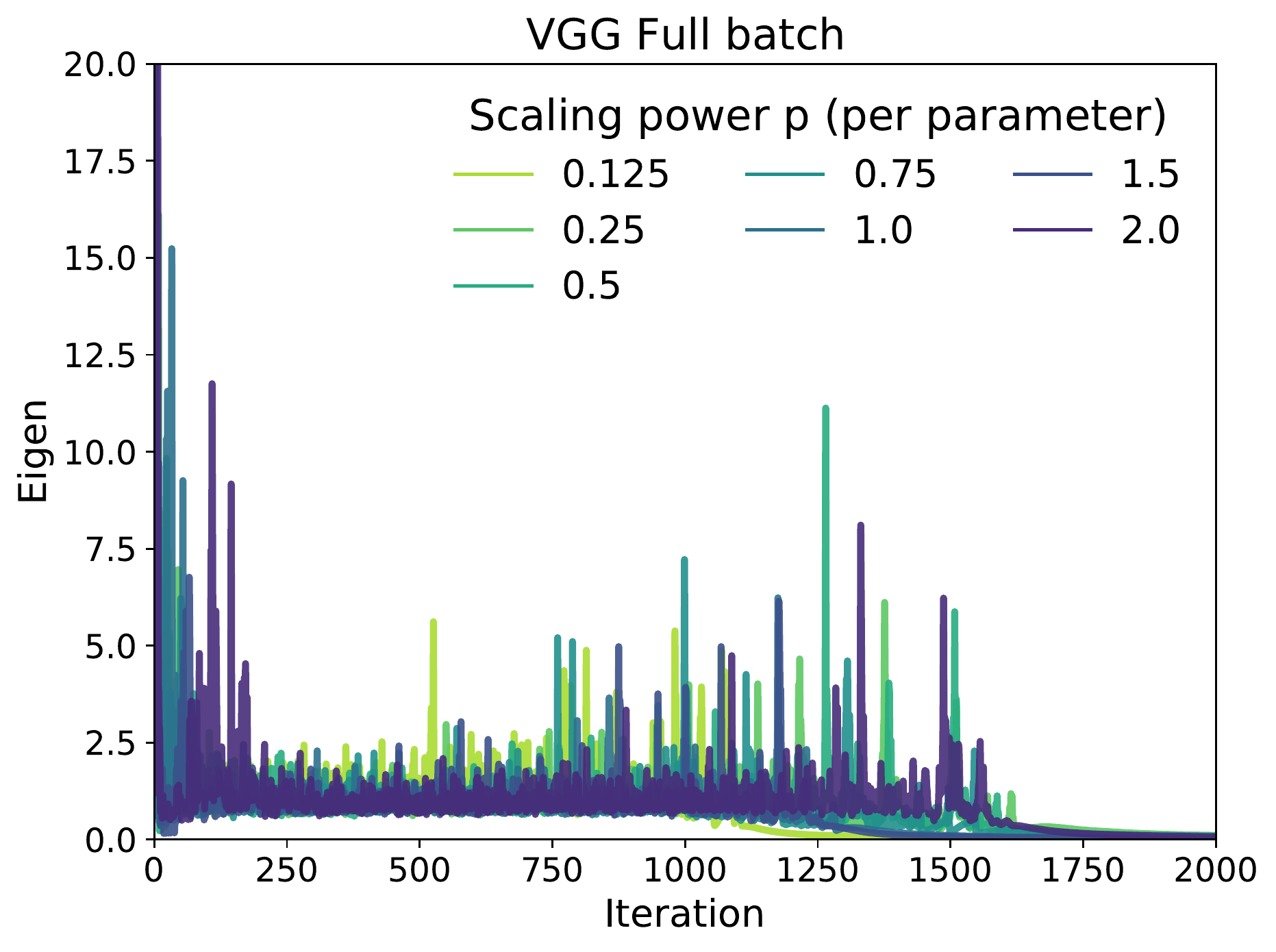}
  \includegraphics[width=0.24\columnwidth]{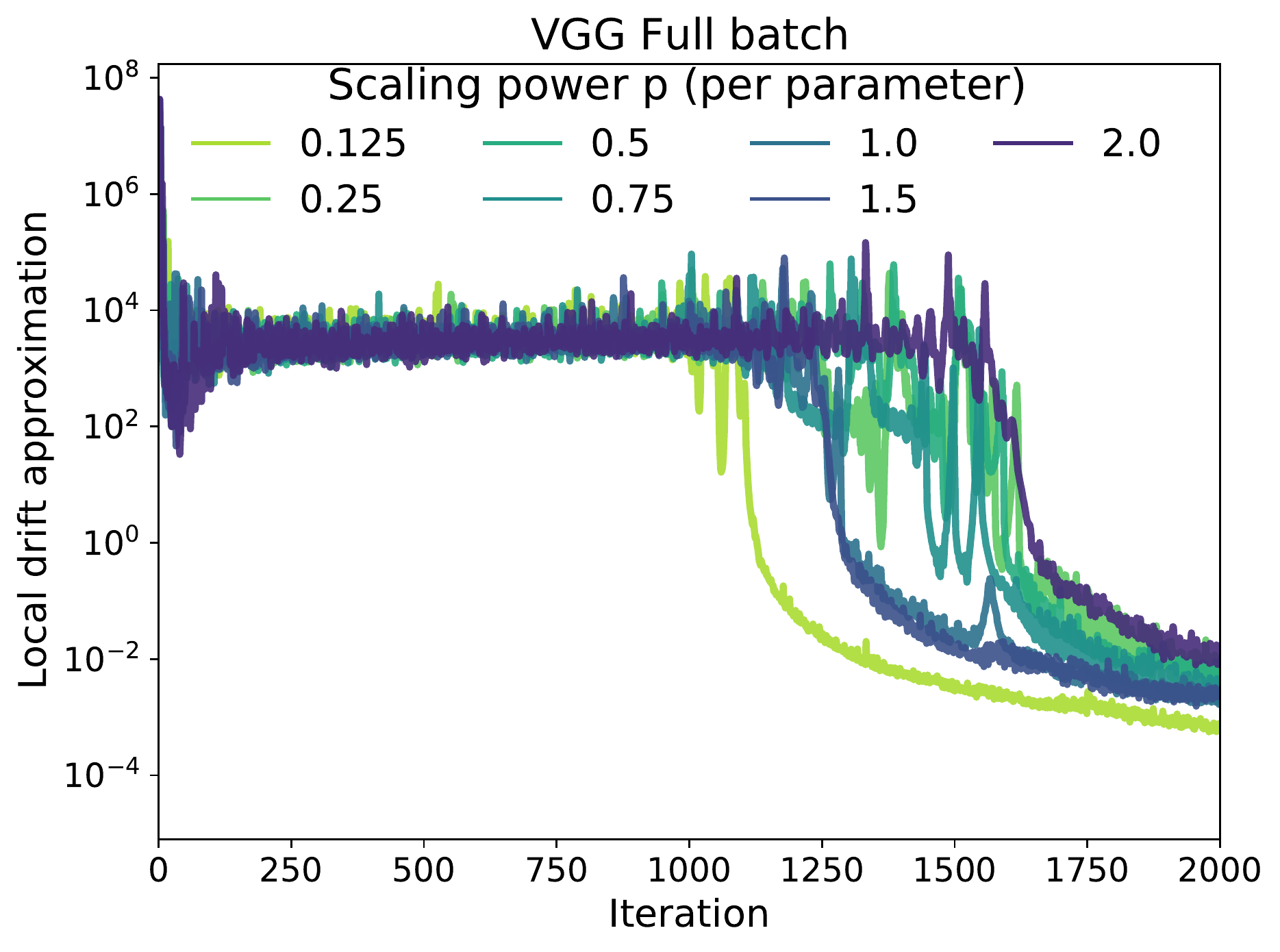}
\caption[VGG CIFAR-10 full batch training with a learning rate per parameter.]{VGG results with learning rate scaling per parameter, but instead of using the global learning rate $2/||\nabla_{\vtheta} E^2 \widehat{g}(\vtheta)||^p$, using the per parameter $\vtheta_i$ learning rate $2/(\nabla_{\vtheta} E^2 \nabla_{\vtheta} E)_i^p$.}
\label{fig:vgg_lr_scaling_per_parameter}
\end{figure}

\begin{figure}
  \includegraphics[width=0.33\columnwidth]{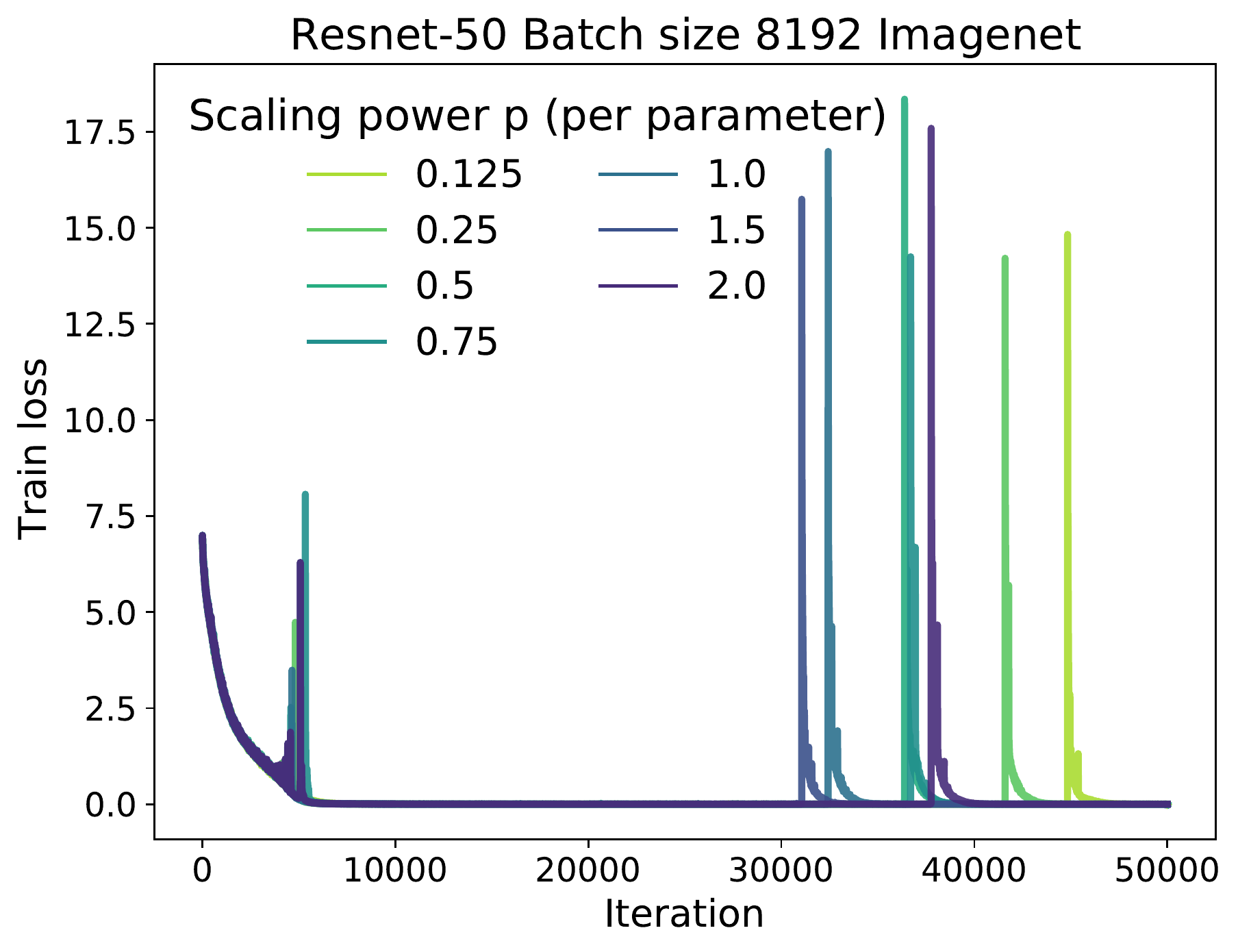}
  \includegraphics[width=0.33\columnwidth]{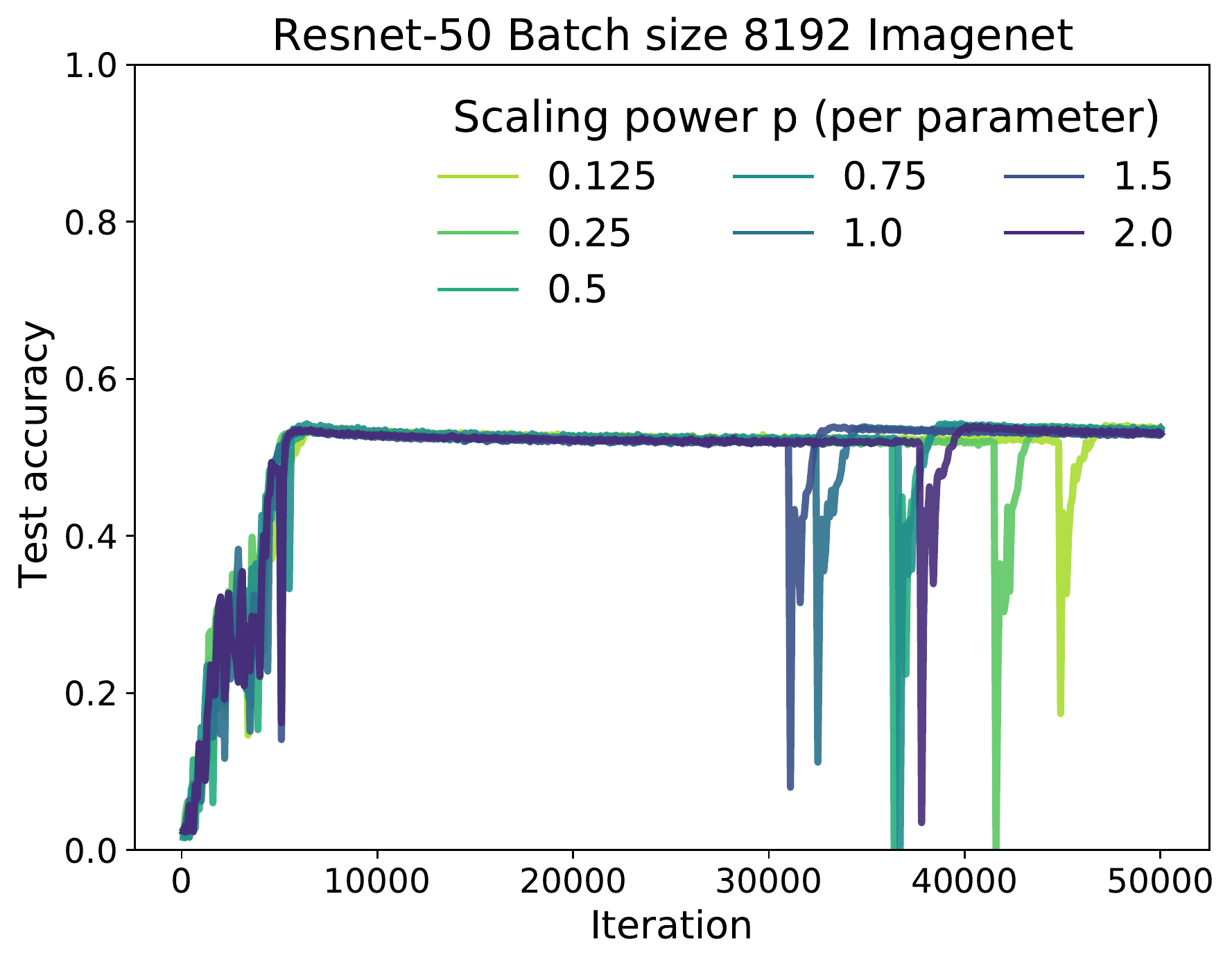}
  \includegraphics[width=0.33\columnwidth]{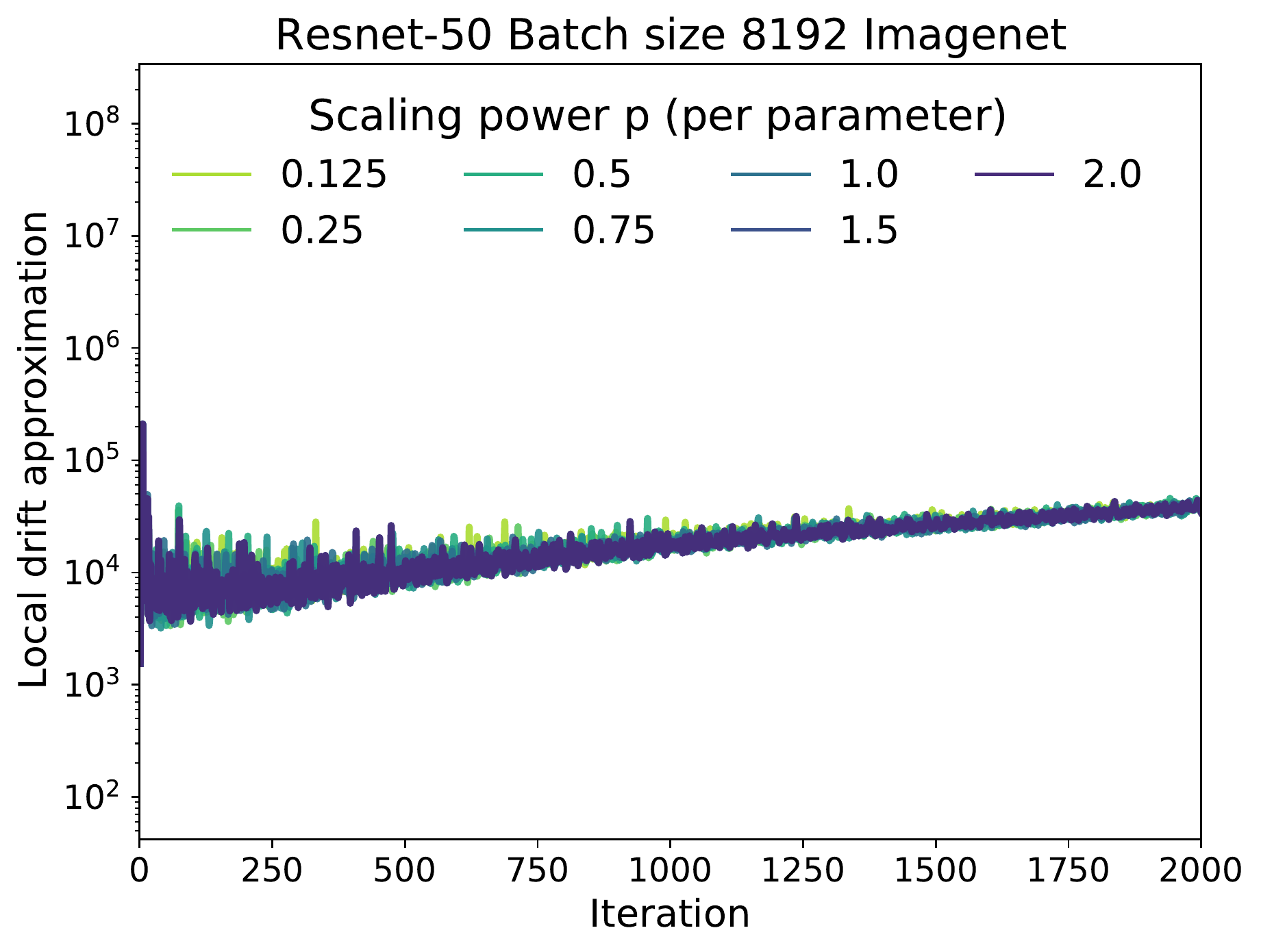}
\caption[Resnet-50 Imagenet training with a learning rate per parameter. Batch size 8192.]{Imagenet results when using a DAL like method of scaling per parameter, but instead of using the global learning rate $2/||\nabla_{\vtheta} E^2 \widehat{g}(\vtheta)||$, using the per parameter $\vtheta_i$ learning rate $2/(\nabla_{\vtheta} E^2 \nabla_{\vtheta} E)_i^p$.}
\label{fig:imagenet_lr_scaling_per_parameter}
\end{figure}

\begin{figure}
\begin{subfigure}[Fixed learning rate sweep.]{
  \includegraphics[width=0.23\columnwidth]{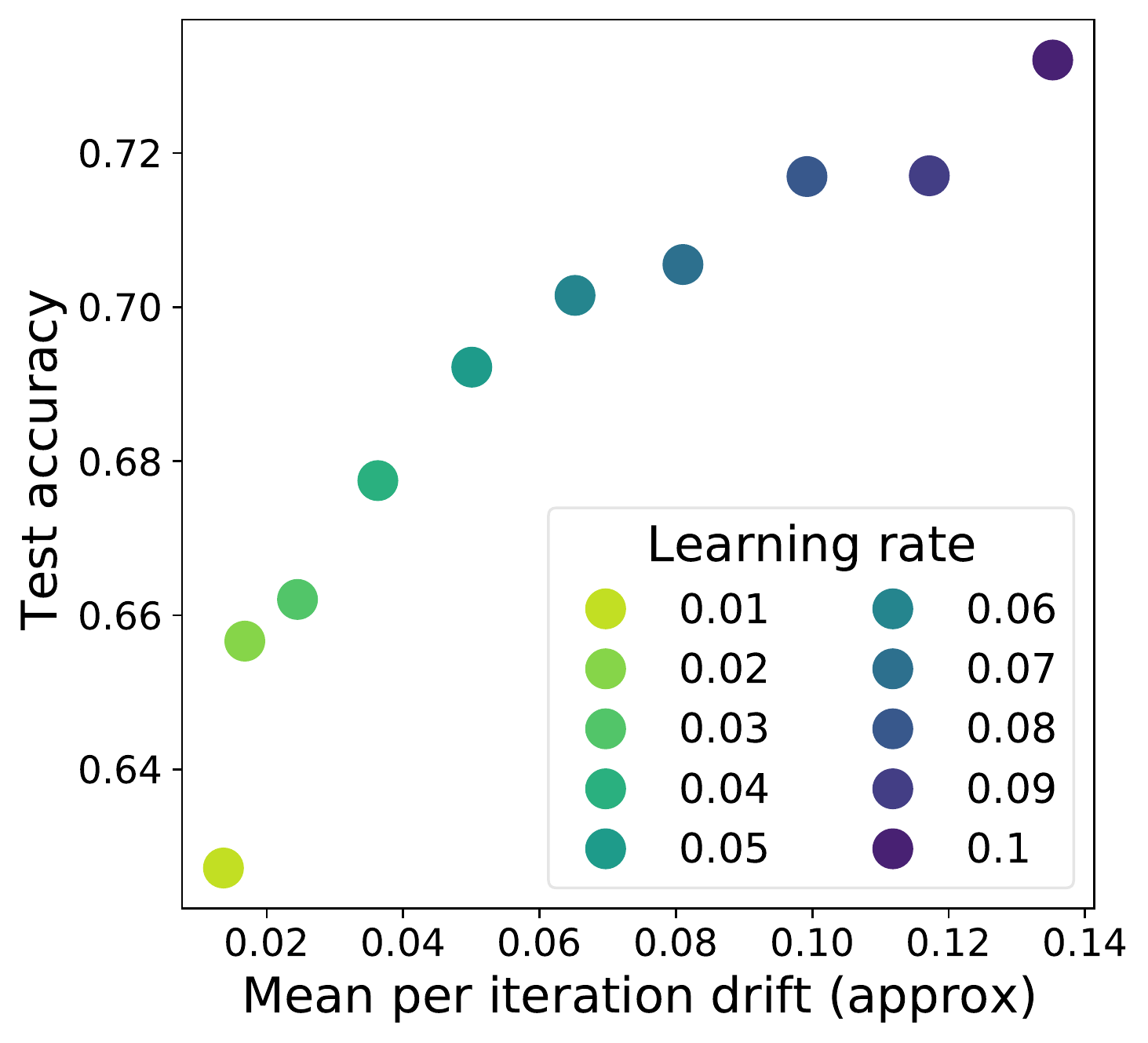}
  \includegraphics[width=0.23\columnwidth]{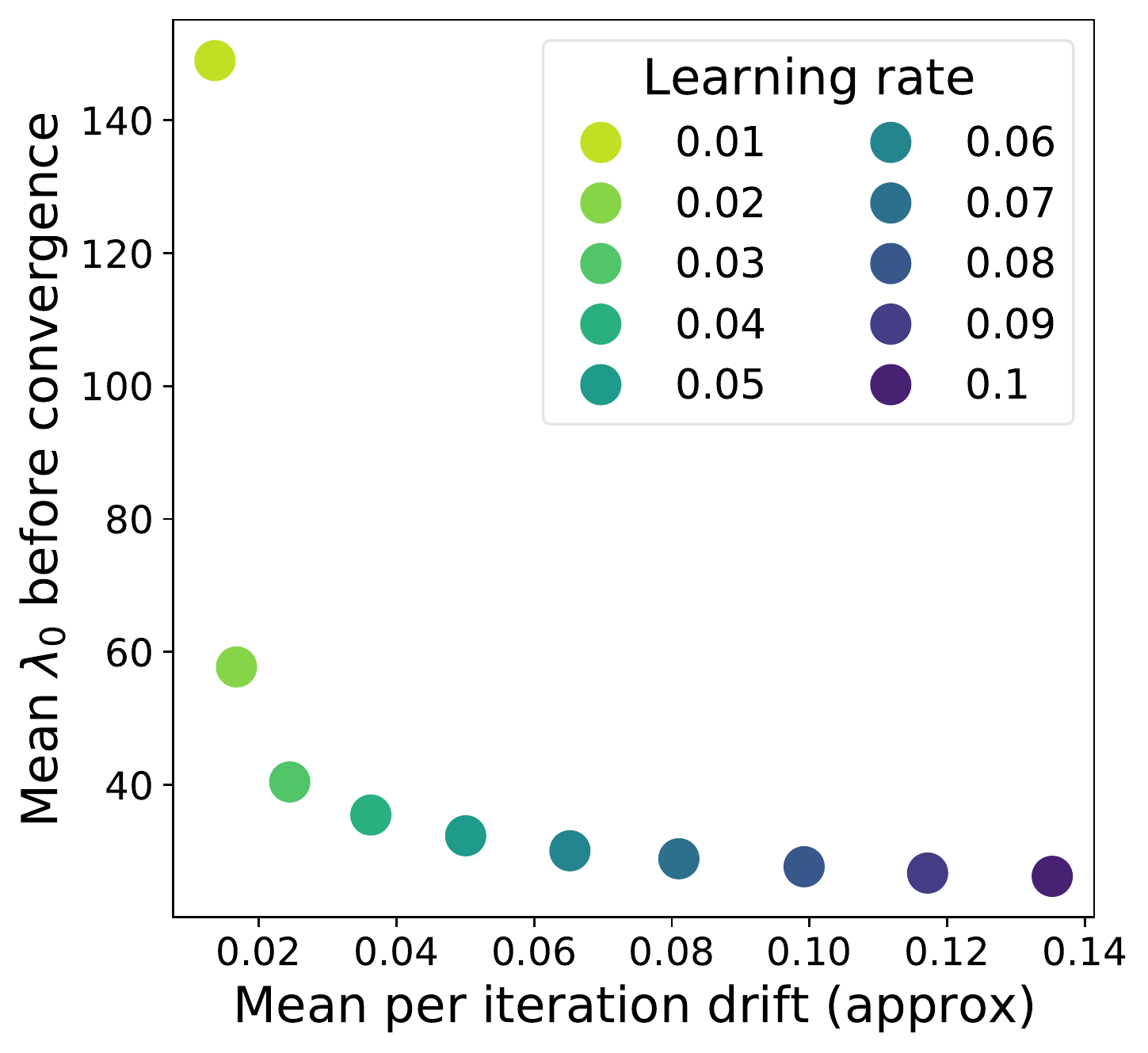}
}\end{subfigure}
\begin{subfigure}[DAL-$p$ sweep.]{
  \includegraphics[width=0.23\columnwidth]{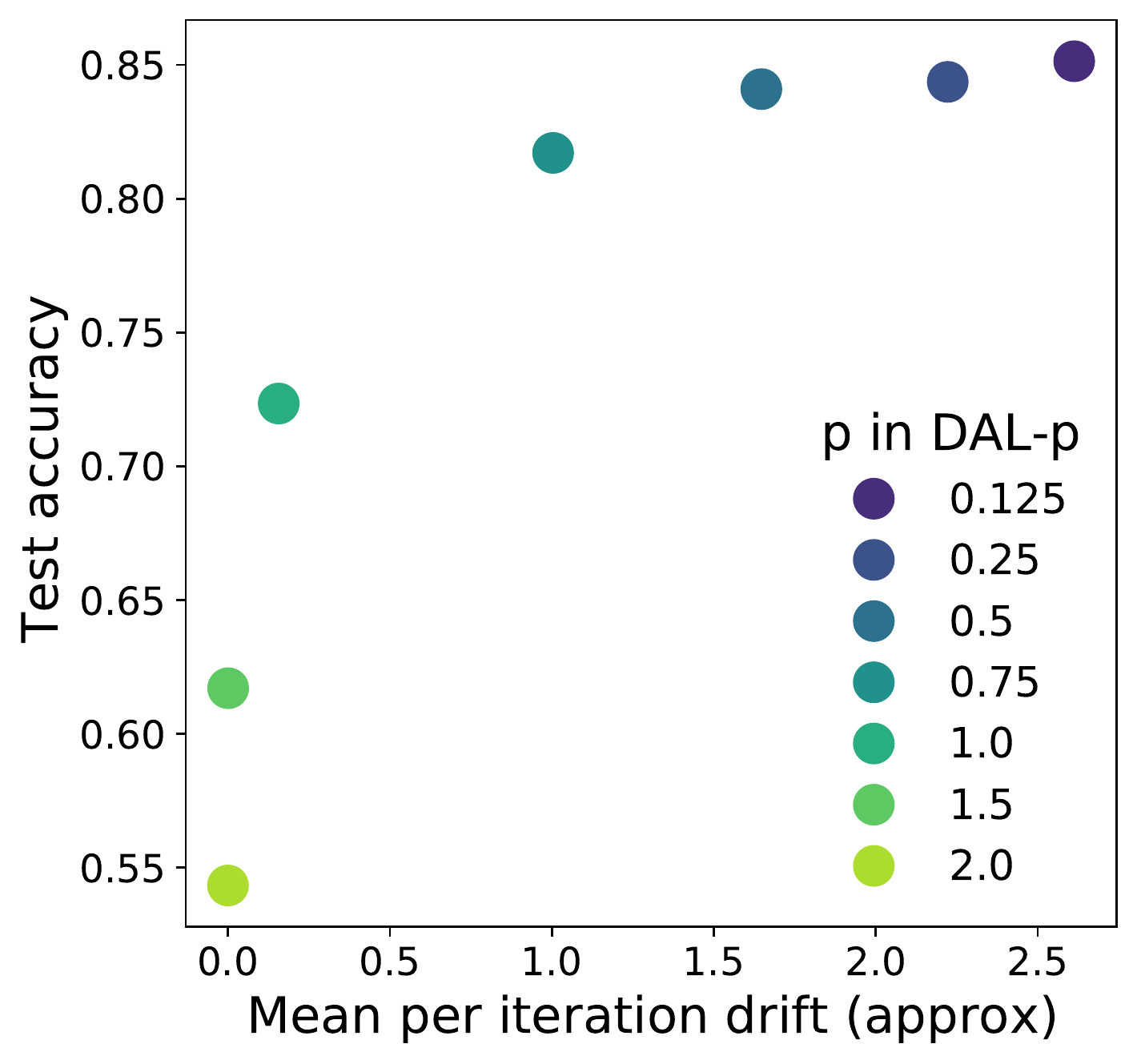}
  \includegraphics[width=0.23\columnwidth]{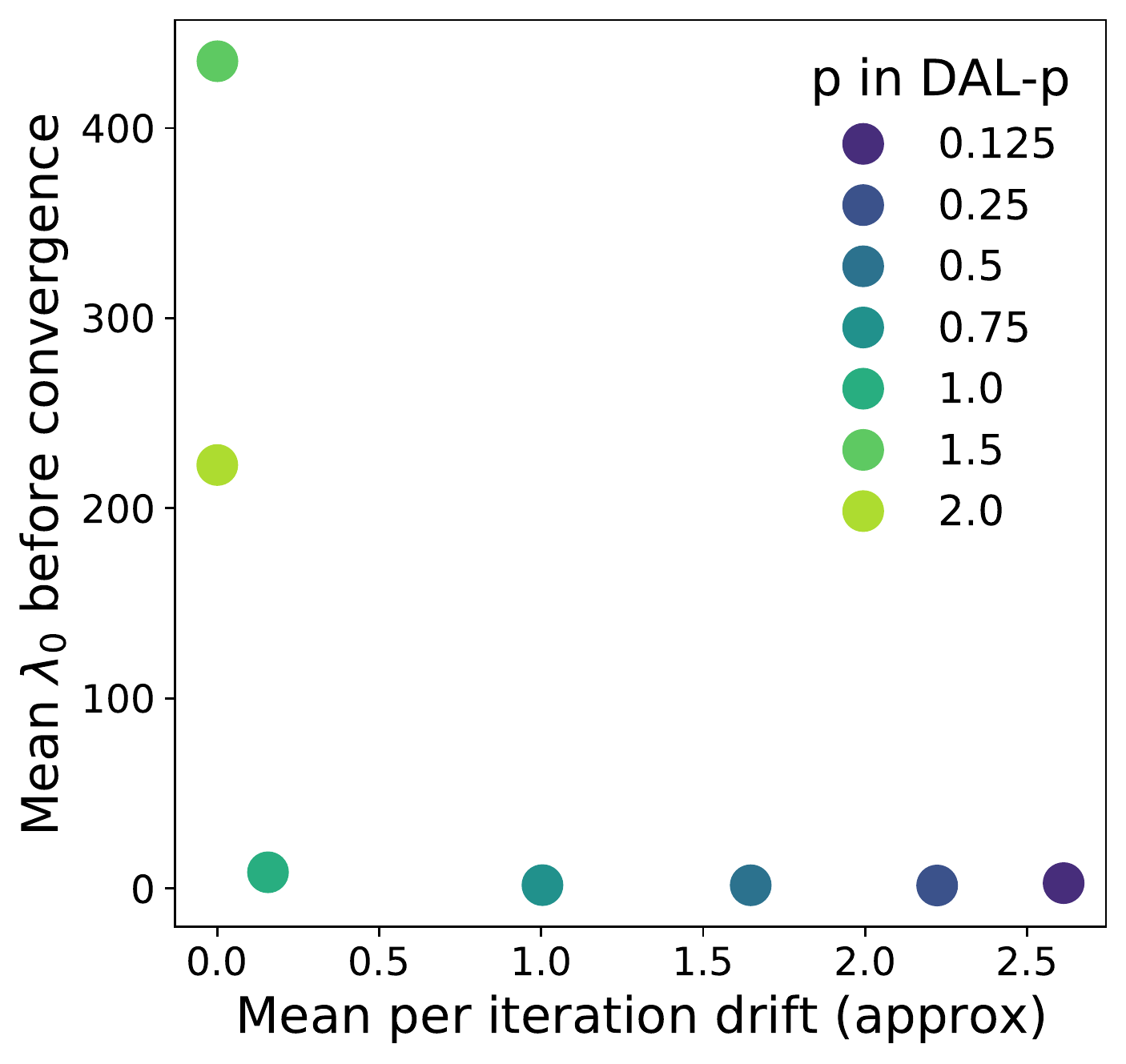}
}\end{subfigure}
\caption[VGG trained on CIFAR-10 with batch size 1024: connection between drift, test error and $\lambda_0$.]{VGG CIFAR-10 with batch size 1024: connection between drift, test error and $\lambda_0$. We observe the same patterns for other batch sizes.}
\label{fig:drift_lambda_sgd_connection}
\end{figure}

\textbf{Momentum}. We show additional results with learning rate adaptation and momentum in Figure \ref{fig:momentum_imagenet_1024}.

\begin{figure}
 \includegraphics[width=0.45\columnwidth]{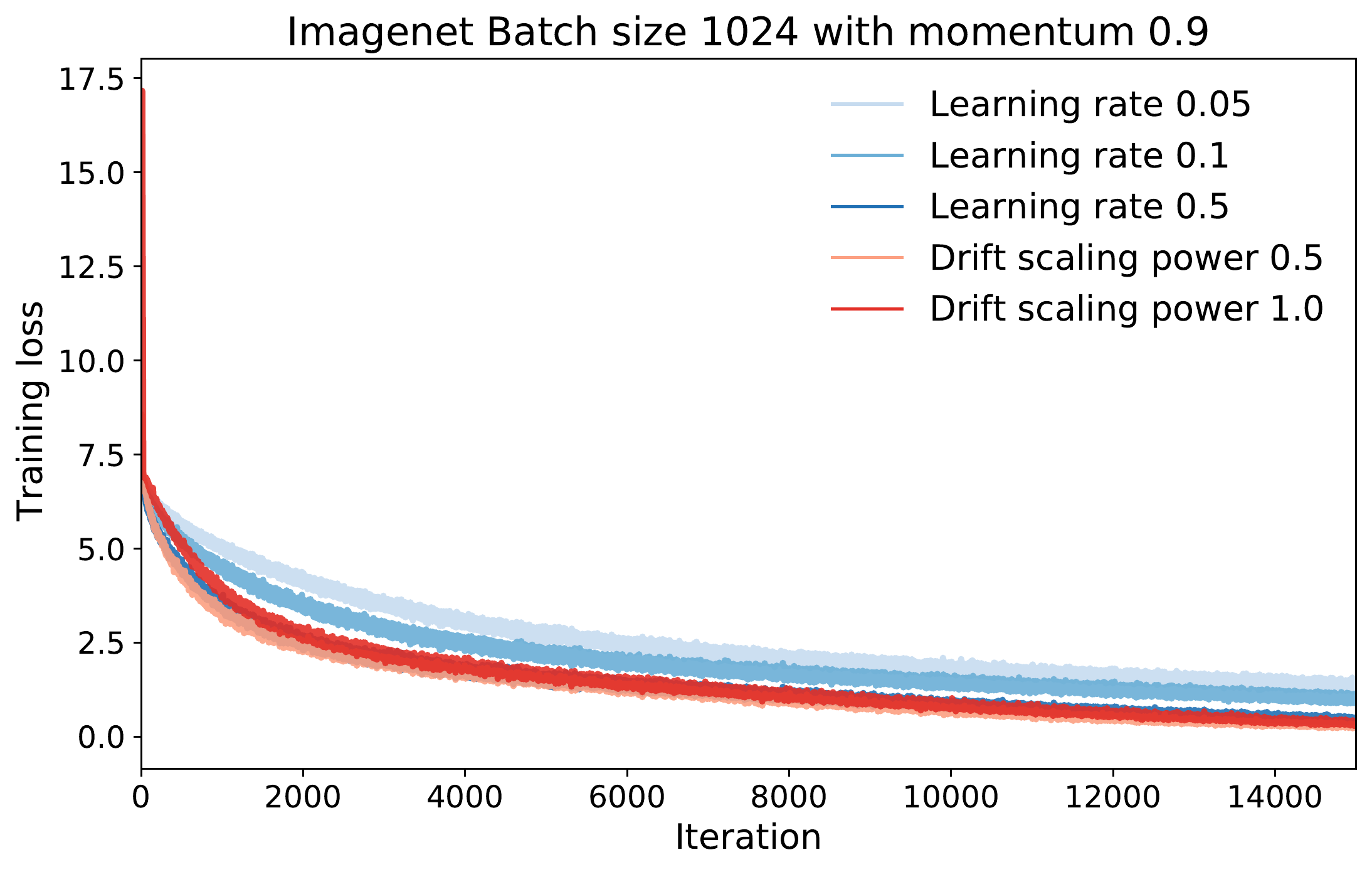}
  \includegraphics[width=0.45\columnwidth]{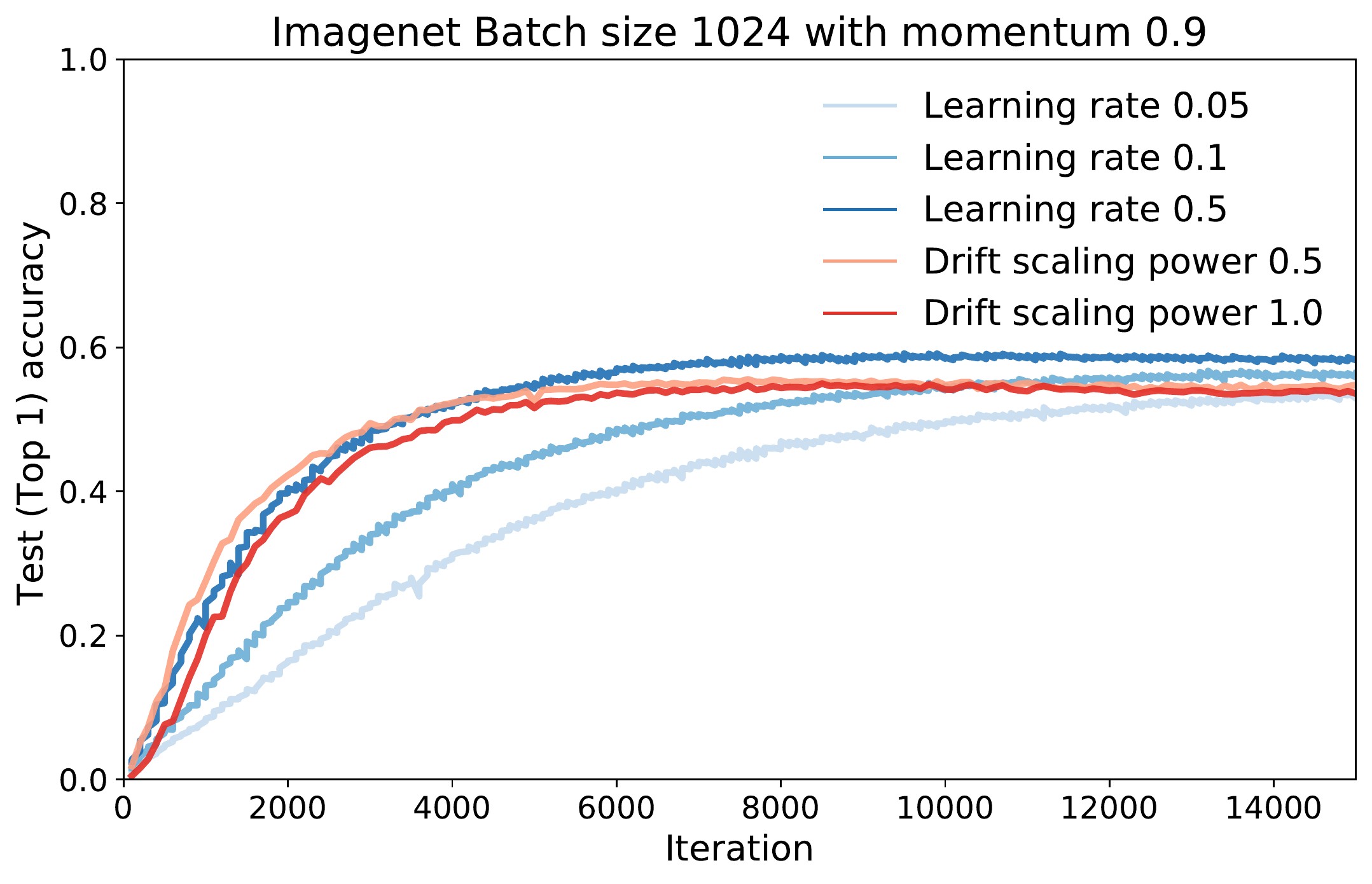}
\caption[DAL-$p$ with momentum $0.9$ on Imagenet. The model is Resnet-50 trained with batch size 1024.]{Integrating information on drift with momentum: with smaller batch sizes, we observe a smaller empirical gain. Since these are preliminary results, more investigation is needed into why that is. These results are consistent with the observation obtained with vanilla gradient descent, where we say larger gains with larger batch sizes on Imagenet.}
\label{fig:momentum_imagenet_1024}
\end{figure}

\textbf{Global discretization error}. We show the global error in trajectory between the NGF and gradient descent in Figure \ref{fig:mnist_gradient_flow}. As previously observed \citep{cohen2021gradient}, gradient descent follows the NGF early in training, and the eigenvalues of the trajectory following the NGF keeps growing.

\begin{figure}
 \includegraphics[width=0.45\columnwidth]{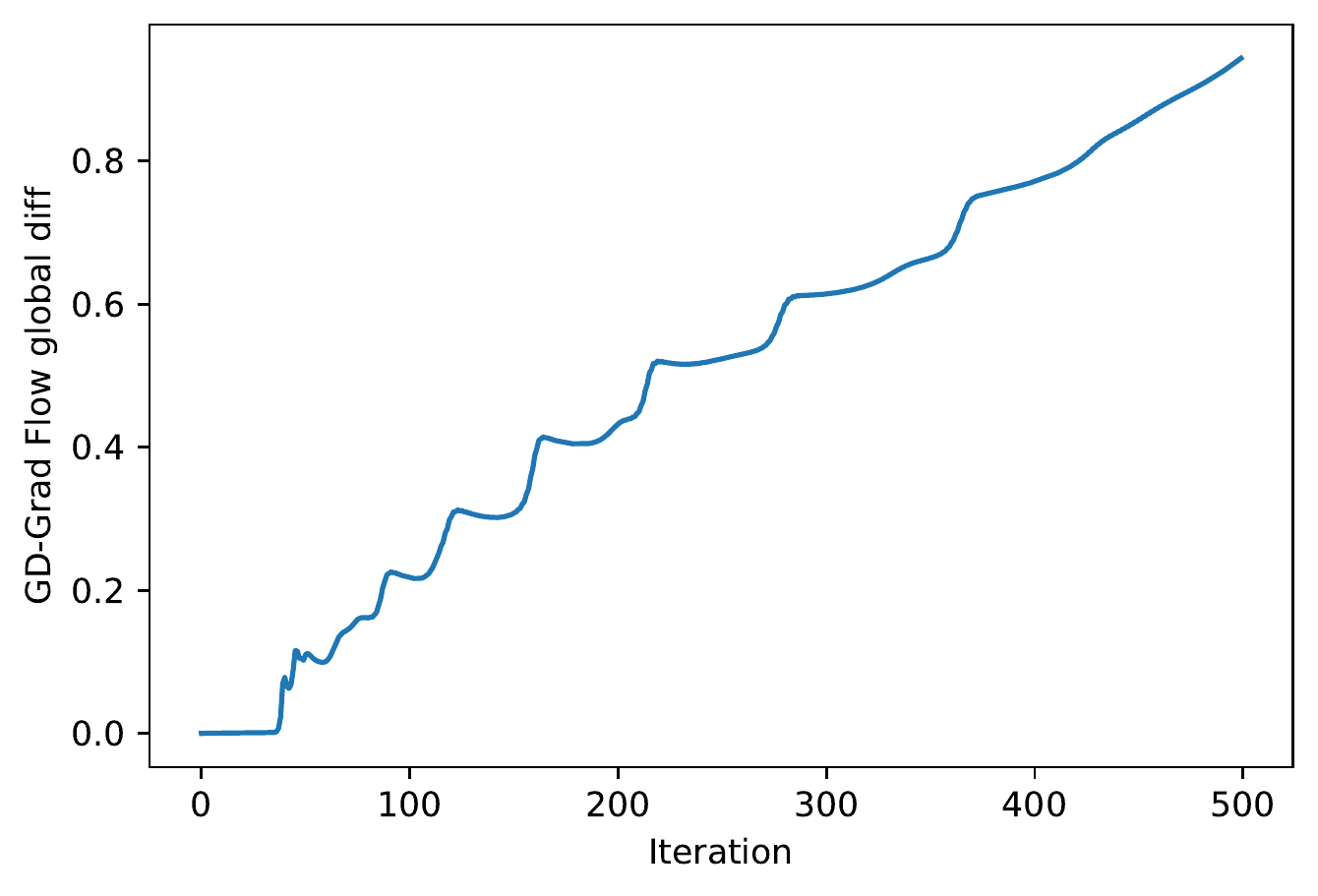}
 \includegraphics[width=0.45\columnwidth]{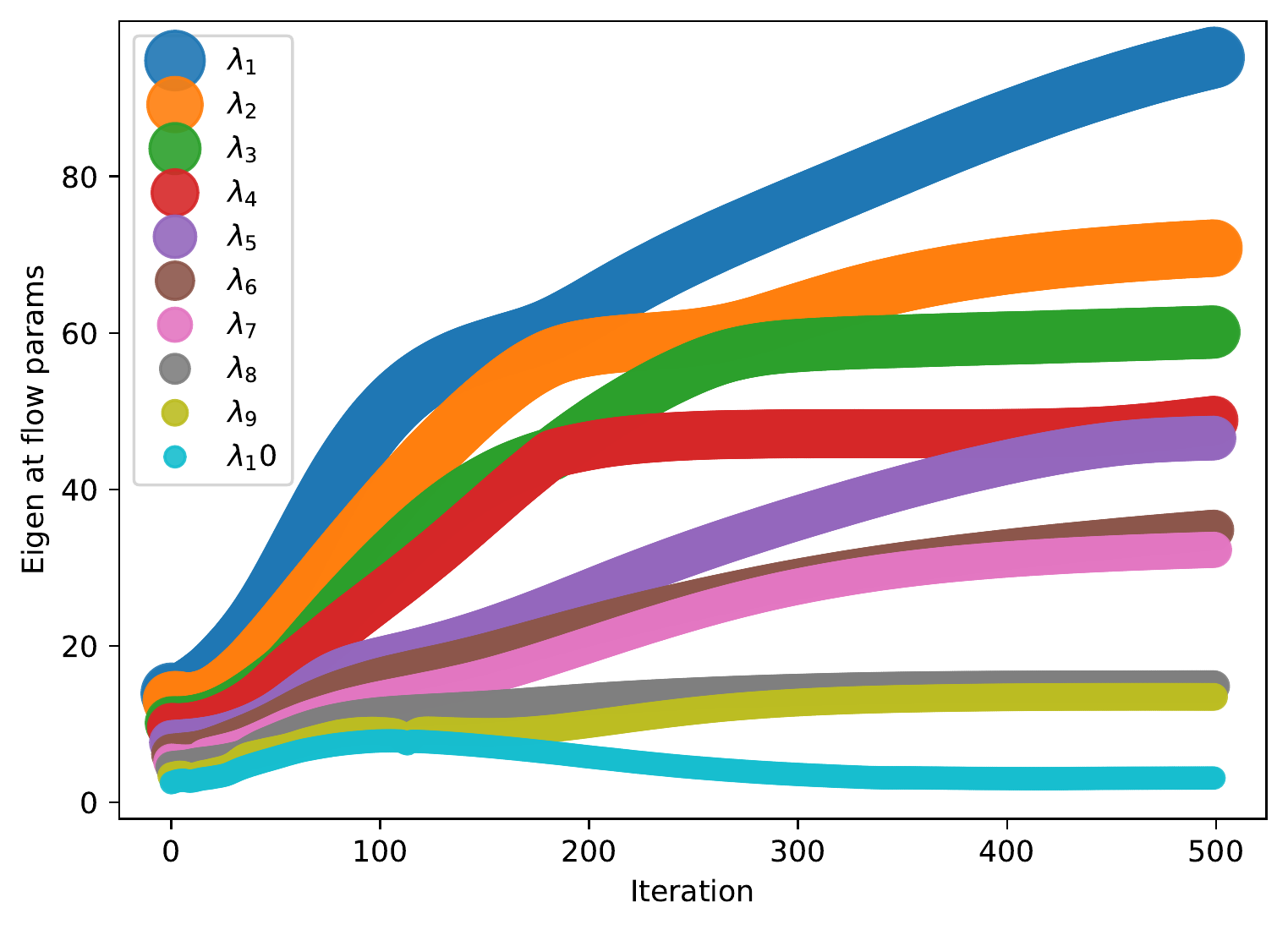}
\caption[MNIST results with a 4 layer MLP with 100 hidden units. The learning rate used is $0.05$.]{The global error between the NGF and gradient descent (left), as well as the behavior of $\lambda_0$ of the NGF.}
\label{fig:mnist_gradient_flow}
\end{figure}

\clearpage
\listoffigures

\end{document}